\documentclass{article}

\usepackage{hyperref}

\usepackage[utf8]{inputenc} 
\usepackage[T1]{fontenc}    
\usepackage{url}            
\usepackage{booktabs}       
\usepackage{amsfonts}       
\usepackage{nicefrac}       
\usepackage{microtype}      

\usepackage{graphicx} 
\usepackage{caption}
\usepackage{amsmath, amsfonts, amssymb, amsthm, dsfont}
\usepackage{xcolor}
\usepackage{algorithm}
\usepackage{algorithmic}
\usepackage{xspace}
\usepackage{wrapfig}
\usepackage{thmtools}
\usepackage{thm-restate}
\usepackage[inline]{enumitem}
\allowdisplaybreaks[4]
\usepackage{cleveref}
\crefname{lemma}{Lem.}{lemmas}
\crefname{proposition}{Prop.}{propositions}
\crefname{theorem}{Thm.}{theorems}
\crefname{corollary}{Cor.}{corollaries}
\crefname{definition}{Def.}{definitions}
\crefname{equation}{}{Eqs.}
\crefname{assumption}{Asm.}{assumptions}
\crefname{section}{Sect.}{sections}
\crefname{appendix}{App.}{appendices}
\crefname{condition}{Cond.}{conditions}

\newtheorem{innercustomthm}{Theorem}
\newenvironment{customthm}[1]
  {\renewcommand\theinnercustomthm{#1}\innercustomthm}
  {\endinnercustomthm}

\newlength{\minipagewidth}
\newlength{\minipagewidthx}
\newcommand{\bookboxx}[1]{\small
\par\medskip\noindent
\framebox[\textwidth]{
\begin{minipage}{0.98\dimexpr\textwidth-\parindent\relax} {#1} \end{minipage} } \par\medskip }

\newcommand{\tpi}{\wt{\pi}}
\newcommand{\bpi}{\wb{\pi}}

\newcommand{\transp}{\mathsf{T}}
\newcommand{\Tr}{\text{Tr}}

\newcommand{\rls}{{\small\textsc{RLS}}\xspace}
\newcommand{\ts}{{\small\textsc{TS}}\xspace}
\newcommand{\ofu}{{\small\textsc{OFU}}\xspace}

\newcommand{\ofulq}{{\small\textsc{OFU-LQ}}\xspace}
\newcommand{\ofulqplus}{{\small\textsc{OFU-LQ++}}\xspace}

\newcommand{\evi}{{\small\textsc{EVI}}\xspace}
\newcommand{\dsllq}{{\small\textsc{DS-OFU}}\xspace}
\newcommand{\laglq}{{\small\textsc{LagLQ}}\xspace}
\newcommand{\oslo}{{\small\textsc{OSLO}}\xspace}
\newcommand{\cecce}{{\small\textsc{CECCE}}\xspace}

\newcommand{\ce}{{\small\textsc{CE}}\xspace}
\newcommand{\backproc}{{\small\textsc{BackupProc}}\xspace}

\newcommand{\wt}[1]{\widetilde{#1}}
\newcommand{\wh}[1]{\widehat{#1}}

\newcommand{\wb}[1]{\overline{#1}}
\def\:#1{\protect \ifmmode {\mathbf{#1}} \else {\textbf{#1}} \fi}

\newcommand{\Ac}{A^{\mathsf{c}}}
\newcommand{\Aceta}{A^{\eta,\mathsf{c}}}
\newcommand{\sys}{\mathfrak{s}}
\newcommand{\Sys}{\mathfrak{S}}

\newcommand{\A}{\mathcal A}
\newcommand{\J}{\mathcal J}
\newcommand{\calL}{\mathcal L}

\newcommand{\C}{\mathcal C}
\newcommand{\D}{\mathcal D}
\newcommand{\F}{\mathcal F}

\newcommand{\K}{\mathcal K}
\newcommand{\M}{\mathcal M}

\newcommand{\I}{\mathds{1}}
\newcommand{\N}{\mathcal N}

\newcommand{\R}{\mathcal{R}}

\renewcommand{\Re}{\mathbb{R}}

\newtheorem{lemma}{Lemma}
\newtheorem{assumption}{Assumption}
\newtheorem{corollary}{Corollary}
\newtheorem{proposition}{Proposition}
\newtheorem{property}{Property}
\newtheorem{definition}{Definition}
\newtheorem{theorem}{Theorem}

\DeclareMathOperator*{\argmin}{arg\,min}

\usepackage[accepted]{icml2020}

\icmltitlerunning{Efficient Optimistic Exploration in Linear-Quadratic Regulators via Lagrangian Relaxation}

\begin{document}

\twocolumn[
\icmltitle{Efficient Optimistic Exploration in Linear-Quadratic Regulators\\ via Lagrangian Relaxation}
 
 \icmlsetsymbol{equal}{*}

\begin{icmlauthorlist}
\icmlauthor{Marc Abeille}{cail}
\icmlauthor{Alessandro Lazaric}{fair}
\end{icmlauthorlist}

\icmlaffiliation{cail}{Criteo AI Lab}
\icmlaffiliation{fair}{Facebook AI Research}

\icmlcorrespondingauthor{Marc Abeille}{m.abeille@criteo.com}

\icmlkeywords{LQR, Optimism, exploration, reinforcement learning}

\vskip 0.3in
]
\printAffiliationsAndNotice{}  

\begin{abstract}
	We study the exploration-exploitation dilemma in the linear quadratic regulator (LQR) setting.
	Inspired by the extended value iteration algorithm used in optimistic algorithms for finite MDPs, we propose to relax the optimistic optimization of \ofulq and cast it into a constrained \textit{extended} LQR problem, where an additional control variable implicitly selects the system dynamics within a confidence interval. We then move to the corresponding Lagrangian formulation for which we prove strong duality.  As a result, we show that an $\epsilon$-optimistic controller can be computed efficiently by solving at most $O\big(\log(1/\epsilon)\big)$ Riccati equations. Finally, we prove that relaxing the original \ofu problem does not impact the learning performance, thus recovering the $\wt O(\sqrt{T})$ regret of \ofulq. To the best of our knowledge, this is the first computationally efficient confidence-based algorithm for LQR with worst-case optimal regret guarantees.
\end{abstract}



\section{Introduction}

Exploration-exploitation in Markov decision processes (MDPs) with continuous state-action spaces is a challenging problem: estimating the parameters of a generic MDP may require many samples, and computing the corresponding optimal policy may be computationally prohibitive. The linear quadratic regulator (LQR) model formalizes continuous state-action problems, where the dynamics is linear and the cost is quadratic in state and action variables. Thanks to its specific structure, it is possible to efficiently estimate the parameters of the LQR by least-squares regression and the optimal policy can be computed by solving a Riccati equation. As a result, several exploration strategies have been adapted to the LQR to obtain effective learning algorithms. 

\textbf{Confidence-based exploration.}
\citet{bittanti2006adaptive} introduced an adaptive control system based on the ``bet on best'' principle and proved asymptotic performance guarantees showing that their method would eventually converge to the optimal control. \citet{abbasi2011regret} later proved a finite-time $\wt O(\sqrt{T})$ regret bound for \ofulq, later generalized to less restrictive stabilization and noise assumptions by~\citet{faradonbeh2017finite}. Unfortunately, neither exploration strategy comes with a computationally efficient algorithm to solve the optimistic LQR, and thus they cannot be directly implemented. On the \ts side, \citet{ouyang2017learning-based} proved a $\wt O(\sqrt{T})$ regret for the Bayesian regret, while~\citet{abeille2018improved} showed that a similar bound holds in the frequentist case but restricted to 1-dimensional problems. While \ts-based approaches require solving a single (random) LQR, the theoretical analysis of~\citet{abeille2018improved} suggests that a new LQR instance should be solved at each time step, thus leading to a computational complexity growing linearly with the total number of steps. On the other hand, \ofu-based methods allow for ``lazy'' updates, which require solving an optimistic LQR only a \textit{logarithmic} number of times w.r.t.\ the total number of steps. A similar lazy-update scheme is used by~\citet{dean2018regret}, who leveraged robust control theory to devise the first learning algorithm with polynomial complexity and sublinear regret. Nonetheless, the resulting adaptive algorithm suffers from a $\wt O(T^{2/3})$ regret, which is significantly worse than the $\wt O(\sqrt{T})$ achieved by \ofulq. 

To the best of our knowledge, the only efficient algorithm for confidence-based exploration with $\wt O(\sqrt{T})$ regret has been recently proposed by~\citet{cohen2019learning}. Their method, called \oslo, leverages an SDP formulation of the LQ problem, where an optimistic version of the constraints is used. As such, it translates the original non-convex \ofulq optimization problem into a \textit{convex} SDP. While solving an SDP is known to have \textit{polynomial} complexity, no explicit analysis is provided and it is said that the runtime may scale polynomially with LQ-specific parameters and the time horizon $T$ (Cor.~5), suggesting that \oslo may become impractical for moderately large $T$. Furthermore, \oslo requires an initial system identification phase of length $\wt O(\sqrt{T})$ to properly initialize the method. This strong requirement effectively reduces \oslo to an explore-then-commit strategy, whose regret is dominated by the length of the initial phase.


\textbf{Perturbed certainty equivalence exploration.} A recent stream of research~\citep{faradonbeh2018input,mania2019certainty,simchowitz2020naive} studies variants of the perturbed certainty equivalence (\ce) controller (i.e., the optimal controller for the estimated LQR) and showed that this simple exploration strategy is sufficient to reach worst-case optimal regret $\wt O(\sqrt{T})$. Since the \ce controller is not recomputed at each step (i.e., lazy updates) and the perturbation is obtained by sampling from a Gaussian distribution, the resulting methods are computationally efficient. Nonetheless, these methods rely on an isotropic perturbation (i.e., all control dimensions are equally perturbed) and they require the variance to be \textit{large enough} so as to \textit{eventually} reduce the uncertainty on the system estimate along the dimensions that are not naturally ``excited'' by the \ce controller and the environment noise. Being agnostic to the uncertainty of the model estimate and its impact on the average cost, may lead this type of approaches to have longer (unnecessary) exploration and larger regret. On the other hand, confidence-based methods relies on exploration controllers that are explicitly designed to excite more the dimensions with higher uncertainty and impact on the performance. As a result, they are able to perform more effective exploration. We further discuss this difference in Sect.~\ref{sec:conclusion}.

In this paper, we introduce a novel instance of \ofu, for which we derive a computationally efficient algorithm to solve the optimistic LQR with explicit computational complexity and $\wt O(\sqrt{T})$ regret guarantees. Our approach is inspired by the extended value iteration (EVI) used to solve a similar optimistic optimization problem in finite state-action MDPs~\citep[e.g.][]{jaksch2010near-optimal}. Relying on an initial estimate of the system obtained after a \textit{finite} number of system identification steps, we first relax the confidence ellipsoid constraints and we cast the \ofu optimization problem into a constrained LQR with extended control. We show that the relaxation of the confidence ellipsoid constraint does not impact the regret and we recover a $\wt O(\sqrt{T})$ bound. We then turn the constrained LQR into a regularized optimization problem via Lagrangian relaxation. We prove strong duality and show that we can compute an $\epsilon$-optimistic and $\epsilon$-feasible solution for the constrained LQR by solving only $O(\log(1/\epsilon))$ algebraic Riccati equations. As a result, we obtain the \textit{first efficient worst-case optimal confidence-based algorithm for LQR}. In deriving these results, we introduce a novel derivation of explicit conditions on the accuracy of the system identification phase leveraging tools from Lyapunov stability theory that may be of independent interest. 


\vspace{-0.1in}
\section{Preliminaries}\label{sec:preliminaries}

We consider the discrete-time linear quadratic regulator (LQR) problem. At any time $t$, given state $x_t\in\Re^n$ and control $u_t\in\Re^d$, the next state and cost are obtained as
\begin{equation}
\begin{aligned}
x_{t+1} &= A_* x_t + B_* u_t + \epsilon_{t+1}; \\
c(x_t,u_t) &= x_t^\transp Q x_t + u_t^\transp R u_t,
\end{aligned}
\label{eq:lq-linear_dynamic_quadratic_cost}
\end{equation}
where $A_*$, $B_*$, $Q$, $R$ are matrices of appropriate dimension and $\{ \epsilon_{t+1}\}_{t}$ is the process noise. Let $\F_t = \sigma (x_0,u_0, \dots, x_t,u_t)$ be the filtration up to time $t$, we rely on the following assumption on the noise.\footnote{As shown by~\citet{faradonbeh2017finite}, this can be relaxed to Weibull distributions with known covariance.}

\begin{assumption}\label{asm:noise}
The noise $\{ \epsilon_{t}\}_{t}$ is a martingale difference sequence w.r.t.\ the filtration $\F_t$ and it is componentwise conditionally sub-Gaussian, i.e., there exists $\sigma >0$ such that
$\mathbb{E}(\exp(\gamma \epsilon_{t+1,i} ) | \mathcal{F}_t) \leq \exp(\gamma^2 \sigma^2 /2)$ for all $\gamma \in \mathbb{R}$. Furthermore, we assume that the covariance of $\epsilon_t$ is the identity matrix.
\end{assumption}

The dynamics parameters are summarized in $\theta_*^\transp = (A_*, B_*)$ and the cost function can be written as $c(x_t, u_t) = z_t^\transp C z_t$ with $z_t = (x_t, u_t)^\transp$ and the cost matrix
\begin{equation}
\label{eq:H.definition}
C = \begin{pmatrix} Q & 0 \\ 0 & R \end{pmatrix}.
\end{equation}
The solution to an LQ is a stationary deterministic policy $\pi: \Re^n \rightarrow \Re^d$ mapping states to controls minimizing the infinite-horizon average expected cost 
\begin{equation}
J_\pi(\theta_*) = \limsup_{T \rightarrow \infty} \frac{1}{T} \mathbb{E}\bigg[\sum_{t=0}^T c(x_t,u_t)  \bigg],
\label{eq:lq-average_cost_criterion}
\end{equation}
with $x_0=0$ and $u_t = \pi(x_t)$. We assume that the LQR problem is ``well-posed''.

\begin{assumption}\label{asm:good.lqr}
The cost matrices $Q$ and $R$ are symmetric p.d.\ and known, and $(A_*, B_*)$ is stabilizable, i.e., there exists a controller $K$, such that $\rho(A_* + B_*K) < 1$.\footnote{$\rho(A)$ is the spectral radius of the matrix $A$, i.e., the largest absolute value of the eigenvalues of $A$.}
\end{assumption}

In this case, Thm.16.6.4 in~\citep{lancaster1995algebraic} guarantees the existence and uniqueness of an optimal policy $\pi_* = \arg\min_\pi J_\pi(\theta_*)$, which is linear in the state, i.e., $\pi_*(x) = K(\theta_*) x$, where,
\begin{equation}\label{eq:lqr.solution}
\begin{aligned}
K(\theta_*) &= -\big(R + B_*^\transp P(\theta_*) B_*\big)^{-1} B_*^\transp P(\theta_*) A_*, \\
P(\theta_*) &= Q + A_*^\transp P(\theta_*) A_* +  A_*^\transp P(\theta_*) B_* K(\theta_*).
\end{aligned}
\end{equation}

For convenience, we will denote $P_*  = P(\theta_*)$. The optimal average cost is $J_* = J_{\pi_*}(\theta_*) = \Tr (P_*)$. Further, let $L(\theta_*)^\transp = \big(I \; K(\theta_*)^\transp\big)$, then the closed-loop matrix $\Ac(\theta_*) = A_* + B_* K(\theta_*) = \theta_*^\transp L(\theta_*)$ is asymptotically stable. 

While Asm.~\ref{asm:good.lqr} guarantees the existence of an optimal linear controller, its optimal cost $J_*$ may still grow unbounded when $\theta_*$ is nearly unstable. A popular solution is to introduce a ``strong'' stability assumption (i.e., $\rho(\Ac(\theta_*)) \leq \wb\rho < 1$). Nonetheless, this imposes stability uniformly over all state dimensions, whereas, depending on the cost matrices $Q$ and $R$, some dimensions may be less sensitive than others in terms of their impact on the cost. Here we prefer imposing an assumption directly on the optimal cost.\footnote{An alternative assumption may bound the operator norm of $P_*$, (see e.g., \citet{simchowitz2020naive}).} 

\begin{assumption}\label{asm:stability.margin}
There exists $D >0$ such that $J_* = \Tr(P_*) \leq D$ and $D$ is known.
\end{assumption}

Finally, we introduce $\kappa = D/\lambda_{\min}(C)$, a quantity that will characterize the complexity of many aspects of the learning problem in the following. Intuitively, $\kappa$ measures the cost of controlling w.r.t.\ the minimum cost incurred if the uncontrolled system was perfectly stable.

\textbf{The learning problem.} 
We assume that $Q$ and $R$ are known, while $\theta_*$ needs to be estimated from data. We consider the online learning problem where at each step $t$ the learner observes the current state $x_t$, it executes a control $u_t$ and it incurs the associated cost $c(x_t, u_t)$; the system then transitions to the next state $x_{t+1}$ according to Eq.~\ref{eq:lq-linear_dynamic_quadratic_cost}. The learning performance is measured by the cumulative regret over $T$ steps defined as 
%
$\R_T(\theta_*) = \sum_{t=0}^T \big(c_t - J_*(\theta_*)\big).$
%
Exploiting the linearity of the dynamics, the unknown parameter $\theta_*$ can be directly estimated from data by regularized least-squares (RLS). For any sequence of controls $(u_0,\ldots,u_t)$ and the induced states $(x_0,x_1,\ldots,x_{t+1})$, let $z_t = (x_t,  u_t)^\transp$, the RLS estimator with a regularization bias $\theta_0$ and regularization parameter $\lambda \in \mathbb{R}_+^* $ defined as\footnote{For $\theta_0 = 0$, this reduces to the standard estimator. The need for a ``centered'' regularization term is explained in the next section.}
\begin{equation}\label{eq:lq-least.square}
\begin{aligned}
\wh{\theta}_t &= \argmin_{\theta \in \mathbb{R}^{(n+d)\times n}} \sum_{s=0}^{t-1} \| x_{s+1} - \theta^\transp z_s \|^2 + \lambda \|\theta - \theta_0\|^2_F\\
&= V_{t}^{-1} \Big( \lambda \theta_0 + \sum_{s=0}^{t-1} z_s x_{s+1}^\transp \Big),
\end{aligned}
\end{equation}
where $V_{t} = \lambda I + \sum_{s=0}^{t-1} z_s z_s^\transp$ is the design matrix. The RLS estimator concentrates as follows (see~\cref{subsec:p:concentration}).

\begin{proposition}[Thm.~1 in~\citealt{abbasi2011regret}]\label{p:concentration}
For any $\delta \in (0,1)$ and any $\mathcal{F}_t$-adapted sequence $(z_0,\ldots,z_t)$, the RLS estimator $\hat{\theta}_t$ is such that
\begin{align}\label{eq:lqr.beta-definition}
&\| \theta_* - \wh{\theta}_{t} \|_{V_{t}} \leq \beta_t(\delta)\quad \quad \text{where} \\
&\beta_t(\delta) = \sigma \sqrt{2n \log \Big( \frac{\det(V_t)^{1/2} n }{\det(\lambda I)^{1/2} \delta} \Big)} + \lambda^{1/2} \|\theta_0 - \theta_*\|_F,\nonumber
\end{align}
w.p.\ $1-\delta$ (w.r.t.\ the noise $\{\epsilon_{t+1}\}_t$ and any randomization in the choice of the control).
\end{proposition}

Finally, we recall a standard result of RLS.
\begin{proposition}[Lem.~10 in~\citealt{abbasi2011regret}]\label{p:lq-self_normalized_determinant}
Let $\lambda \geq 1$, for any arbitrary $\mathcal{F}_t$-adapted sequence $(z_0, z_1, \ldots, z_{t})$, let $V_{t+1}$ be the corresponding design matrix, then
\begin{equation*}
\sum_{s=0}^{t} \min\big( \|z_s  \|_{V_s^{-1}}^2, 1 \big) \leq 2 \log \frac{\det(V_{t+1})}{\det(\lambda I)}.
\end{equation*}

Moreover when $\|z_t\| \leq Z$ for all $t \geq 0$, then 
\begin{equation}\label{eq:pred.with.bounded.z}
\sum_{s=0}^{t} \|z_s \|_{V_s^{-1}}^2 \leq \left( \!1 \!+ \!\frac{Z^2}{\lambda}\!\right)  (n\! +\! d) \log \Big(1 + \frac{(t\!+\! 1) Z^2}{\lambda (n\!+\!d)} \Big).
\end{equation}
\end{proposition}
%


\section{A Sequentially Stable Variant of \ofulq}

In this section we introduce a variant of the original \ofulq of~\citet{abbasi2011regret} that we refer to as \ofulqplus. Similar to~\citep{faradonbeh2017finite}, we use an initial system identification phase to initialize the system and we provide \textit{explicit} conditions on the accuracy required to guarantee sequential stability thereafter. This result is obtained leveraging tools from Lyapunov stability theory which may be of independent interest. 

\citet{faradonbeh2018finite} showed that it is possible to construct a set $\Theta_0 = \{\theta: \|\theta-\theta_0\| \leq \epsilon_0 \}$ containing the true parameters $\theta^*$ with high probability, through a system identification phase where a randomized sequence of linear controllers is used to accurately estimate the dynamics. In particular, they proved that a set $\Theta_0$ with accuracy $\epsilon_0$ can be obtained by running the system identification phase for as long as $T_0 = \Omega(\epsilon_0^{-2})$ steps.\footnote{An alternative scheme for system identification requires access to a stable controller $K_0$ and to perturb the corresponding controls to returned a set $\Theta_0$ of desired accuracy $\epsilon_0$ (see e.g.,~\citealt{simchowitz2020naive}).}

After the initial phase, \ofulqplus uses the estimate $\theta_0$ to regularize the RLS as in~\eqref{eq:lq-least.square} and it proceeds through episodes. At the beginning of episode $k$ it computes a parameter
\begin{align}\label{eq:ofu.real}
\theta_k = \arg\min_{\theta \in \mathcal{C}_k} J(\theta),
\end{align}
where $t_k$ is the step at which episode $k$ begins and the constrained set $\mathcal{C}_k$ is defined as									
\begin{align}\label{eq:conf.interval}
\mathcal{C}_k = \mathcal{C}(\beta_{t_k}, V_{t_k}) := \{\theta: \| \theta - \wh{\theta}_{t_k} \|_{V_{t_k}} \leq \beta_{t_k}\},
\end{align}
where $\beta_t = \beta_t(\delta/4)$ is defined in~\eqref{eq:lqr.beta-definition}. Then the corresponding optimal control $K(\theta_k)$ is computed~\eqref{eq:lqr.solution} and the associated policy is executed until $\det(V_t) \geq 2\det(V_{t_k})$.

\begin{lemma}\label{lem:ofu.dynamic.stability2}
Let $\Theta_0 = \{\theta: \|\theta-\theta_0\| \leq \epsilon_0 \}$ be the output of the initial system identification phase of~\citet{faradonbeh2018finite}. For all $t\geq 0$, consider the confidence ellipsoid $\mathcal{C}_t := \{\theta: \| \theta - \wh{\theta}_{t} \|_{V_{t}} \leq \beta_{t}\}$, where $\wh{\theta}_t$ and $V_t$ are defined in~\cref{eq:lq-least.square} with regularization bias $\theta_0$ (the center of $\Theta_0$), regularization parameter
\begin{equation}
\label{eq:regularization.condition}
\begin{aligned}
&\lambda\! = \!\frac{2 n \sigma^2}{\epsilon_0^2}\! \Big( \!\log(4n / \delta) \! + \!(n+d) \! \log\Big(\! 1 +\!  \kappa X^2 T \Big)\! \Big),
\end{aligned}
\end{equation}
and $\beta_t$ is defined in~\eqref{eq:lqr.beta-definition} where $\|\theta-\theta_*\|$ is replaced by its upper-bound $\epsilon_0$. Let $\{ K(\theta_t)\}_{t\geq 1}$ be the sequence of optimistic controllers generated by \ofulqplus and let $\{x_t\}_{t\geq 0}$ be the induced state process (Eq.~\ref{eq:lq-linear_dynamic_quadratic_cost}). If $\epsilon_0 \leq O(1/\kappa^2)$, then with probability at least $1-\delta/2$, for all $t\leq T$,
\begin{equation}\label{eq:state.bound.ofulqplus}
\left\{
\begin{aligned}
&\theta_* \in \mathcal{C}_t \\
&\|x_t\| \leq X\! :=\! 20 \sigma \sqrt{\kappa \|P_*\|_2 \log(4 T/\delta)/\lambda_{\min}(C)}.
\end{aligned}
\right.
\end{equation}
%
%
\end{lemma}

\ofulqplus has some crucial differences w.r.t. the original algorithm. \ofulq receives as input a $\Theta_0$ such that for any $\theta\in\Theta_0$ the condition $\| \theta^\transp L(\theta)\| < 1$ holds. While this condition ensures that all the LQ systems in $\Theta_0$ are indeed stable, it does not immediately imply that the optimal controllers $K(\theta)$ stabilize the \textit{true} system $\theta_*$. Nonetheless, \citet{abbasi2011regret} proved that the sequence of controllers generated by \ofulq naturally defines a state process $\{x_t\}_t$ which remains bounded at any step $t$ with high probability. Unfortunately, their analysis suffers from several drawbacks: \textbf{1)} the state bound scales exponentially with the dimensionality, \textbf{2)} as $\theta_*$ is required to belong to $\Theta_0$, it should satisfy itself the condition $\| \theta_*^\transp L(\theta_*)\| < 1$, which significantly restricts the type of LQR systems that can be solved by \ofulq, \textbf{3)} the existence of $\Theta_0$ is stated \textit{by assumption} and no concrete algorithm to construct it is provided. 

Furthermore, \ofulq requires solving~\eqref{eq:ofu.real} under the constraint that $\theta$ belongs to the intersection $\mathcal{C}_k \cap \Theta_0$, while \ofulqplus only uses the confidence set $\mathcal{C}_k$ to guarantee that the controllers $K(\wt\theta_k)$ generated through the episodes induces a sequentially stable state process. Although the resulting optimization problem is still non-convex and difficult to solve directly, removing the constraint of $\Theta_0$ enables the relaxation that we introduce in the next section. Finally, we notice that our novel analysis of the sequential stability of \ofulqplus leads to a tighter bound on the state, more explicit conditions on $\epsilon_0$, and lighter assumptions than~\citet{faradonbeh2018finite}.

As a result, we can refine the analysis of \ofulq and obtain a much sharper regret bound for \ofulqplus.

\begin{lemma}\label{p:regret.ofulq.plus}
For any LQR $(A_*, B_*, Q, R)$ satisfying Asm.~\ref{asm:noise},~\ref{asm:good.lqr}, and~\ref{asm:stability.margin}, after $T$ steps \ofulqplus, if properly initialized and tuned as in~\cref{lem:ofu.dynamic.stability2}, suffers a regret
\begin{align}\label{eq:regret.ofulq.plus}
\R(T)\! =\!\wt O\Big( \!\big( \kappa \|P_*\|_2^{2}\!+\!\sqrt{\kappa} \|P_*\|^{3/2}_2 (n\! +\! d) \!\sqrt{n}\big)  \sqrt{T}\Big).
\end{align}
\end{lemma}


\section{An Extended Formulation of \ofulqplus}
\label{sec:perturbation}

The optimization in~\eqref{eq:ofu.real} is non-convex and it cannot be solved directly. 
In this section we introduce a relaxed constrained formulation of~\eqref{eq:ofu.real} and show that its solution is an optimistic controller with similar regret as \ofulqplus at the cost of requiring a slightly more accurate initial exploration phase (i.e., smaller $\epsilon_0$). 

\subsection{The Extended Optimistic LQR with Relaxed Constraints}

 Our approach is directly inspired by the extended value iteration (\evi) used to solve a similar optimistic optimization problem in finite state-action MDPs~\citep[e.g.][]{jaksch2010near-optimal}. In \evi, the choice of dynamics $\theta$ from $\C$ is added as an additional control variable, thus obtaining an extended policy $\tpi$. Exploiting the specific structure of finite MDPs, it is shown that optimizing over policies $\tpi$ through value iteration is equivalent to solving a (finite) MDP with the same state space and an extended (compact) control space and the resulting optimal policy, which prescribes both actions and a choice of the model $\theta$, is indeed optimistic w.r.t.\ the original MDP.
Leveraging a similar idea, we ``extend'' the LQR with estimated parameter $\wh\theta_t$ by introducing an additional control variable $w$ corresponding to a specific choice of $\theta\in\mathcal{C}_k$. In the following we remove the dependency of $\wh\theta$, $\beta$, $V$, and $\mathcal{C}$ on the learning step $t_k$ and episode $k$, while we use a generic time $s$ to formulate the extended LQR. 

Let $\theta\in\mathcal{C}$ such that $\theta = \wh\theta+\delta_\theta = (A, B) = (\wh A + \delta_A, \wh B + \delta_B)$, then the dynamics of the corresponding LQR is
\begin{equation}\label{eq:perturbed.lqr}
\begin{aligned}
x_{s+1} &= A x_s + B u_s + \epsilon_{s+1} \\
&= \wh A x_s + \wh B u_s + \delta_A x_s + \delta_B u_s + \epsilon_{s+1},\\
&= \wh A x_s + \wh B u_s + \delta_\theta z_s + \epsilon_{s+1},
\end{aligned}
\end{equation}
where we isolate the ``perturbations'' $\delta_A$ and $\delta_B$ applied to the current estimates. We replace the perturbation associated to $\theta$ with a novel control variable $w_s$, the \textit{perturbation control variable}, which effectively plays the role of ``choosing'' the parameters of the perturbed LQR, thus obtaining
\begin{equation}\label{eq:perturbed.lqr2}
\begin{aligned}
x_{s+1} &= \wh A x_s + \wh B u_s + w_s + \epsilon_{s+1},\\
&= \wh A x_s + \wt B \wt u_s + \epsilon_{s+1},
\end{aligned}
\end{equation}
where we conveniently introduced $\wt B = [\wh B, I]$ and $\wt u_s = [u_s, w_s]$.\footnote{In the following we use tilde-notation such as $\tpi$ and $wt B$ to denote quantities in the extended LQR.} This \textit{extended} system has the same state variables as the original LQ, while the number of control variables moves from $d$ to $n+d$. Since perturbations $\delta_\theta$ are such that $\theta = \wh\theta+\delta_\theta\in\mathcal{C}$, we introduce a constraint on the perturbation control such that $\|w_s \| = \|\delta_\theta^\transp z_s\| \leq \beta \|z_s\|_{V^{-1}}$ (see~\cref{p:concentration}). We refer to the resulting system as the \textit{extended LQR with hard constraints}. Unfortunately, this constrained system is no longer a ``standard'' LQR structure, as the constraint should be verified \textit{at each step}. To overcome this difficulty, we relax the previous constraint and define
\begin{equation}\label{eq:constraint}
g_{\tpi}(\wh\theta, \beta, V) \!=\!  \lim_{S \rightarrow \infty} \frac{1}{S}  \mathbb{E} \Big( \sum_{s=0}^S\|w_s \|^2 - \beta^2 \| z_s \|^2_{V^{-1}} \Big),
\end{equation}
where $\tpi = (\pi^u, \pi^w)$ is an \textit{extended} policy defining both standard and perturbation controls, so that $u_s = \pi^u(x_s)$ and $w_s = \pi^w(x_s)$, the expectation is w.r.t.\ the noise $\epsilon_{s+1}$, and the dynamics of $x_s$ follows~\eqref{eq:perturbed.lqr2}. As a result, we translate the original constraint $\theta \in \mathcal{C}$, which imposed a per-step condition on $w_s$ to $g_{\tpi}(\wh\theta, \beta, V) \leq 0$, which considers the asymptotic average behavior of $w_s$.  We are now ready to define the \textit{extended LQR with relaxed constraints} as
\begin{align}\label{eq:optimal.avg.cost.extended}
&\min_{\tpi} \mathcal{J}_{\tpi}(\wh\theta, \beta, V)&\!\!\!\!\!:= &\limsup_{S \rightarrow \infty} \frac{1}{S} \mathbb{E}\bigg[\sum_{s=0}^S c(x_s,\pi^u(x_s))  \bigg] \nonumber\\
&  \text{subject to} & & x_{s+1} = \wh A x_s + \wt B \wt u_s+ \epsilon_{s+1} \\
& & &  g_{\tpi}(\wh\theta, \beta, V) \leq 0,\nonumber
\end{align}
where $\mathcal{J}_{\tpi}$ is the average cost of~\eqref{eq:perturbed.lqr2} when controlled with $\tpi$ and $c$ is the cost of the original LQ. We also denote by $\mathcal{J}_*(\wh\theta, \beta, V)$ the minimum of~\eqref{eq:optimal.avg.cost.extended}.  Once the constrained LQR is solved, the component $\wt \pi^u$ relative to the variable $u$ is used to control the real system for the whole episode until the termination condition is met. When $\tpi$ is linear, we denote by $\wt K$ the associated gain (i.e., $\tpi(x) = \wt K x$), and we use $K_u$ (resp. $K_w$) to refer to the block of $\wt K$ corresponding to the control $u$ (resp. the perturbation control $w$).



\subsection{Optimism and Regret}
\label{sec:optimism}

We show that the optimization in~\eqref{eq:optimal.avg.cost.extended} preserves the learning guarantees of the original \ofulq algorithm at the cost of a slightly stronger requirement on $\epsilon_0$. This is not obvious as~\eqref{eq:optimal.avg.cost.extended} is relaxing the constraints imposed by the confidence set used in~\eqref{eq:ofu.real} and solving the extended LQR might lead to a perturbation control $\wt\pi^w$ that does not actually correspond to any feasible model $\theta$ in $\mathcal{C}$. Intuitively, we need the constraint $g_{\tpi}$ to be loose enough so as to guarantee optimism and tight enough to preserve good regret performance. 
We start by showing that optimizing~\eqref{eq:optimal.avg.cost.extended} gives an optimistic solution.

\begin{lemma}\label{lem:optimism}
Under Asm.~\ref{asm:good.lqr}, and~\ref{asm:stability.margin}, whenever $\theta_* \in \mathcal{C}$, the optimal solution to~\eqref{eq:optimal.avg.cost.extended} is optimistic, i.e., 
\begin{equation}
\mathcal{J}_{*}(\wh\theta, \beta, V) \leq J_*,
\label{eq:optimism}
\end{equation}
\end{lemma}

The lemma above shows that the optimal controller in the extended LQR has an average cost (in the extended LQR) that is smaller than the true optimal average cost, thus certifying its  optimistic nature of~\eqref{eq:optimal.avg.cost.extended}. This is expected, since~\eqref{eq:optimal.avg.cost.extended} is a relaxed version of the original \ofulqplus problem, which returns optimistic solutions by definition. Then we show that applying the optimistic extended controllers induce a sequentially stable state process.

\begin{lemma}\label{lem:bounded.state}
Given the same system identification phase and RLS estimator of \ofulqplus (see~\cref{lem:ofu.dynamic.stability2}), let $\{ \wt K_t\}_{t\geq 1}$ be the sequence of extended optimistic controllers generated by solving~\cref{eq:optimal.avg.cost.extended} and $\{x_t\}_{t\geq 0}$ be the state process (Eq.~\ref{eq:lq-linear_dynamic_quadratic_cost}) induced when by the sequence of controllers $\{K_{u,t}\}_{t\geq 0}$. If $\epsilon_0 \leq O(1 / \kappa^{3/2})$, then with probability at least $1-\delta/2$, for all $t\leq T$,
\begin{equation}\label{eq:state.bound.ofulqplus.extended}
\left\{
\begin{aligned}
&\theta_* \in \mathcal{C}_t \\
&\|x_t\| \leq X\! :=\! 20 \sigma \sqrt{\kappa \|P_*\|_2 \log(4 T/\delta)/\lambda_{\min}(C)}
\end{aligned}
\right.
\end{equation}

%
%
\end{lemma}

This lemma is the counterpart of~\cref{lem:ofu.dynamic.stability2} for the extended LQR and it illustrates that, due to the relaxed constraint, the condition on $\epsilon_0$ is tighter by a factor $1/\sqrt{\kappa}$, while the bound on the state remains the same and this, in turn, leads to the same regret as \ofulqplus (\cref{p:regret.ofulq.plus}) but for problem dependent constants.
\begin{theorem}\label{thm:regret}
Let $(A_*, B_*, Q, R)$ be any LQR satisfying Asm.~\ref{asm:noise},~\ref{asm:good.lqr}, and~\ref{asm:stability.margin}. If the conditions in~\cref{lem:bounded.state} are satisfied and the extended LQR with relaxed constrained~\eqref{eq:optimal.avg.cost.extended} is solved exactly at each episode, then w.p. at least $1 - \delta$,
\begin{equation}\label{eq:regret.extended}
\R(T)= \wt{O} \big( (n+d) \sqrt{n} \kappa^{3/2} \|P_*\|_2^2 \sqrt{T} \big).
\end{equation}
\end{theorem}
%

\section{Efficient Solution to the Constrained Extended LQR via Lagrangian Relaxation}
\label{sec:lagrange.algorithm}

We introduce the Lagrangian formulation of~\eqref{eq:optimal.avg.cost.extended}. Let $\mu\in\Re$ be the Lagrangian parameter, we define 
\begin{equation}\label{eq:lagrange.objective}
\mathcal{L}_{\tpi}(\wh\theta, \beta, V; \mu) := \J_{\tpi}(\wh\theta, \beta, V) + \mu g_{\tpi}(\wh\theta,\beta, V).
\end{equation}
Since both average cost $\J_{\tpi}$ and constraint $g_{\tpi}$ measure asymptotic average quantities, we can conveniently define the matrices ($C_\dagger$ being the bordering of matrix $C$ in~\eqref{eq:H.definition})
\begin{align*}
& C_\dagger = \begin{pmatrix} Q & 0 & 0 \\ 0 & R & 0 \\ 0 & 0 & 0 \end{pmatrix}; \quad 
C_g = \begin{pmatrix} -\beta^2 V^{-1} & 0 \\ 0 & I \end{pmatrix},
\end{align*}
and write the Lagrangian as 
\begin{align*}
\mathcal L_{\tpi}(\wh\theta, \beta, V; \mu) = \lim_{S \rightarrow \infty} \frac{1}{S} \mathbb{E} \bigg[ \sum_{s=0}^{S-1} \begin{pmatrix} x_s^\transp & \tilde{u}_s^\transp \end{pmatrix} C_{\mu} \begin{pmatrix} x_s \\ \tilde{u}_s \end{pmatrix} \bigg],
\end{align*}
with $C_{\mu} = C_\dagger + \mu C_g $. This formulation shows that $\mathcal L_{\tpi}(\mu)$ can be seen as the average cost of an extended LQR problem with state $x_s$, control $\wt u_s$, linear dynamics $(\wh A,\wt B)$ and quadratic cost with matrix $C_{\mu}$. As a result, we introduce the \textit{Lagrangian extended LQR} problem associated to the extended LQR with relaxed constraints of~\eqref{eq:optimal.avg.cost.extended} as
\begin{equation}\label{eq:lagrange.lqr}
\begin{aligned}
&\mathcal{L}_{*}(\wh\theta, \beta, V)  =  \sup_{\mu \in \mathcal{M}}\min_{\tpi} \mathcal L_{\tpi}(\wh\theta, \beta, V; \mu)\\
&  \text{subject to} \quad x_{s+1} = \wh A x_s + \wt B \wt u_s + \epsilon_{s+1}
\end{aligned}\;,
\end{equation}
where $\M = [0, \wt \mu)$ is the domain of the Lagrangian parameter (more details on $\wt\mu$ are reported in App.~\ref{ssec:app.dual.domain}). We prove the following fundamental result.

\begin{theorem}\label{thm:dual.gap.specific}
	For any extended LQR parametrized by $\wh\theta$, $V$, $\beta$, and psd cost matrices $Q,R$, there exists a domain $\M = [0, \wt\mu)$ with $\wt\mu \in \Re_+$, such that strong duality between the relaxed optimistic optimization in~\eqref{eq:optimal.avg.cost.extended} and its Lagrangian formulation~\eqref{eq:lagrange.lqr} holds:
	\begin{equation*}
	\mathcal{J}_{*}(\wh\theta, \beta, V) = \mathcal{L}_{*}(\wh\theta, \beta, V).
	\end{equation*}
\end{theorem}

Supported by the strong duality above, we provide a more detailed characterization of $\mathcal L_{\tpi}(\wh\theta, \beta, V; \mu)$, which motivates the design of an efficient algorithm to optimize over $\tpi$, $\mu$. For ease of notation, in the following we consider $\wh \theta$, $\beta$, and $V$ as fixed and we drop them from the definition of $\mathcal L_{\tpi}(\mu)$, which we study as a function of $\tpi$ and $\mu$.


\subsection{The Lagrangian Dual Function}
\label{ssec:solving.lagrange}



We introduce the Lagrangian dual function, for any $\mu\in\Re_+$, 
\begin{equation}\label{eq:dual.lagrange.lqr}
\begin{aligned}
 \mathcal{D}(\mu)  = & \min_{\tpi} \mathcal L_{\tpi}(\mu) \\
 &  \text{s.t.} \;\; x_{s+1} = \wh A x_s + \wt B \wt u_s + \epsilon_{s+1}
\end{aligned},
\end{equation}
and we denote by $\wt\pi_\mu$ the corresponding extended optimal policy. For \textit{small enough} $\mu$, the cost matrix $C_{\mu}$ is p.d., which allows solving~\eqref{eq:dual.lagrange.lqr} using standard Riccati theory. The main technical challenge arises for larger values of $\mu$ when the solution of the dual Lagrangian function may not be computable by Riccati equations or it may not even be defined. Fortunately, the following lemma shows that within the domain $\M$ where Thm.~\ref{thm:dual.gap.specific} holds, there always exists a Riccati solution for~\eqref{eq:dual.lagrange.lqr}.

\begin{lemma}\label{lem:dual.function.specific}
For any extended LQR parametrized by $\wh\theta$, $V$, $\beta$, and psd cost matrices $Q,R$, consider the domain $\M = [0, \wt \mu)$ where Thm.~\ref{thm:dual.gap.specific} holds, then for any $\mu\in\M$
\begin{enumerate}
	\item The extended LQ in~\eqref{eq:dual.lagrange.lqr} is controllable and it admits a unique solution
	\begin{equation}\label{eq:extended.optimal.pi.mu}
	\wt\pi_{\mu} = \arg\min_{\tpi}\mathcal L_{\tpi}(\mu),
	\end{equation} 
	obtained by solving the generalized discrete algebraic Riccati equation (DARE)~\citep{molinari1975}\footnote{The need for generalized DARE is due to the fact that for some $\mu\in\M$, the associated cost $C_\mu$ may not be p.s.d.} associated with the Lagrange LQR $(\wh{A},\wt{B}, C_\mu)$. Let $C_\mu = (R_\mu\; N_\mu; N_\mu\; Q_\mu)$ be the canonical formulation for the cost matrix, then 
	\begin{align}
		\label{eq:defD}
	D_\mu &= R_{\mu}+ \tilde{B}^\transp P_{\mu} \tilde{B} \\
	P_{\mu} &= Q_{\mu} + A^\transp P_{\mu} A  \nonumber\\
	&\;\;- [ A^\transp P_{\mu} \tilde{B} + N_{\mu}^\transp ]  D_\mu^{-1} [ \tilde{B}^\transp P_{\mu} A + N_{\mu} ], \nonumber
	\end{align} 
	and the optimal control is $\wt K_{\mu} = - D_{\mu}^{-1} [ \tilde{B}^\transp P_{\mu} A + N_{\mu} ]$,
	while the dual function is $\mathcal{D}(\mu) = \Tr(P_\mu)$.
	\item $\D(\mu)$ is concave and continuously differentiable.
	\item The derivative $\D'(\mu) = g_{\tpi_\mu}$, i.e., it is equal to the constraint evaluated at the optimal extended policy for $\mu$. As a result, the Lagrangian dual can be written as
	\begin{align}\label{eq:lagrange.dual.bis}
	\calL_{\tpi_\mu}(\mu) = \D(\mu) = \J_{\tpi_\mu} + \mu\D'(\mu).
	\end{align}
\end{enumerate}
\end{lemma}

The previous lemma implies that in order to solve~\eqref{eq:lagrange.lqr} we may need to evaluate the dual function $\D(\mu)$ only where the the optimal control can be computed by solving a DARE. 

Since $\D(\mu)$ is concave and smooth, we can envision using a simple dichotomy approach to optimize $\D(\mu)$ over $\M$ and solve~\eqref{eq:lagrange.lqr}. Nonetheless, we notice that~\cref{thm:dual.gap.specific} only provides strong duality in a sup/min sense, which means that the optimum may not be attainable within $\M$. Furthermore, even when there exists a maximum, computing an $\epsilon$-optimal solution in terms of the Lagrangian formulation, i.e., finding a pair $\mu,\tpi$ such that $|\calL_{\tpi}(\mu) - \calL_*| \leq \epsilon$, may not directly translate in a policy with desirable performance in terms of its average cost $\J_{\tpi}$ and feasibility w.r.t.\ the constraint $g_{\tpi}$.

We illustrate this issue in the example in Fig.~\ref{fig:plot}. We display a qualitative plot of the Lagrangian dual $\D(\mu)$ and its derivative $\D'(\mu)$ when~\eqref{eq:lagrange.lqr} admits a maximum at $\mu^*$ and the dichotomy search returned values $\mu_l$ and $\mu_r$ that are $\epsilon$-close and $\mu^* \in [\mu_l, \mu_r]$. We consider the case where the algorithm returns $\mu_l$ as the candidate solution. By concavity and the fact that $\D'(0) > 0$, the function $\D(\mu)$ is Lipschitz in the interval $[0,\mu^*)$ with constant bounded by $\D'(0)$. Thus the accuracy $\mu^*-\mu_l \leq \epsilon$ translates into an equivalent $\epsilon$-optimality in $\D$ (i.e., $\D(\mu^*) - \D(\mu_l) = \calL^* - \calL_{\tpi_{\mu_l}}(\mu_l) \leq \D'(0) \epsilon$). Nonetheless, this does not imply a similar guarantee for $D'(\mu)$. If the second derivative of $\D(\mu)$ (i.e., the curvature of the function) is large close to $\mu^*$, the original error $\epsilon$ can be greatly amplified when evaluating $\D'(\mu_l)$. For instance, if $\D''(\mu) \gg 1/\epsilon$, then $\D'(\mu_l) = \Omega(1)$. Given the last point of~\cref{lem:dual.function.specific}, this means that despite returning an $\epsilon$-optimal solution in the sense of~\eqref{eq:lagrange.lqr}, $\tpi_{\mu_l}$ may significantly violate the constraint (as $\D'(\mu_l) = g_{\tpi_{\mu_l}} = \Omega(1)$). While Eq.~\eqref{eq:lagrange.dual.bis} implies that $\tpi_{\mu_l}$ is still optimistic (i.e., $\J_{\tpi_{\mu_l}} \leq J^\star$, as in~\cref{lem:optimism}), the regret accumulated by $\tpi_{\mu_l}$ cannot be controlled anymore, since the pertubration control $w_s$ may be arbitrarily outside the confidence interval. Interestingly, the curvature becomes larger and larger as the optimum shifts to the extremum of $\M$ and, in the limit,~\eqref{eq:lagrange.lqr} only admits a supremum. In this case, no matter how close $\mu_l$ is to $\mu^*$, the associated policy $\tpi_{\mu_l}$ may perform arbitrarily bad. 

More formally, we have the following lemma (the explicit value of $\alpha$ is reported in Lem.~\ref{le:smoothness.D}).

\begin{lemma}\label{le:smoothness.gradient}
For any LQR parametrized by $\wh\theta$, $V$, $\beta$, and psd cost matrices $Q,R$, consider the domain $\M = [0, \wt \mu)$ where Thm.~\ref{thm:dual.gap.specific} holds. Let $\mathcal{M}_+$ be a subset of $\mathcal{M}$ such that
$
	\mathcal{M}_+ = \{ \mu \in \mathcal{M} \text{ s.t. } \mathcal{D}^\prime(\mu) \geq 0 \}.
$
	Then, $\mathcal{D}$ has Lipschitz gradient, i.e., there exists a constant $\alpha$ depending on $\wh\theta$, $V$, $\beta$, and the cost matrices $Q,R$, such that for all $(\mu_1,\mu_2) \in \mathcal{M}_+^2$, 
	\begin{equation*}
	|\mathcal{D}^\prime(\mu_1) - \mathcal{D}^\prime(\mu_2)| \leq  |\mu_1 - \mu_2| \frac{\alpha}{\lambda_{\min}(D_{\mu_1})},
	\end{equation*}
	where $D_\mu$ is defined in~\eqref{eq:defD}.
\end{lemma}

This result shows that even when $|\mu_1\!-\!\mu_2| \leq \epsilon$, the difference in gradients may be arbitrarily large when $\lambda_{\min}(D_\mu) \ll \epsilon$ (i.e., large curvature). In the next section, we build on this lemma to craft an adaptive stopping condition for the dichotomy search and to detect that case of large curvature.

\subsection{An Efficient Dichotomy Search}
\label{ssec:solving.lagrange.alg}
\begin{figure*}[t]
		\begin{minipage}{0.48\textwidth}
		\centering
	\includegraphics[trim={0 0 0 0.05in},clip,width=0.7\linewidth]{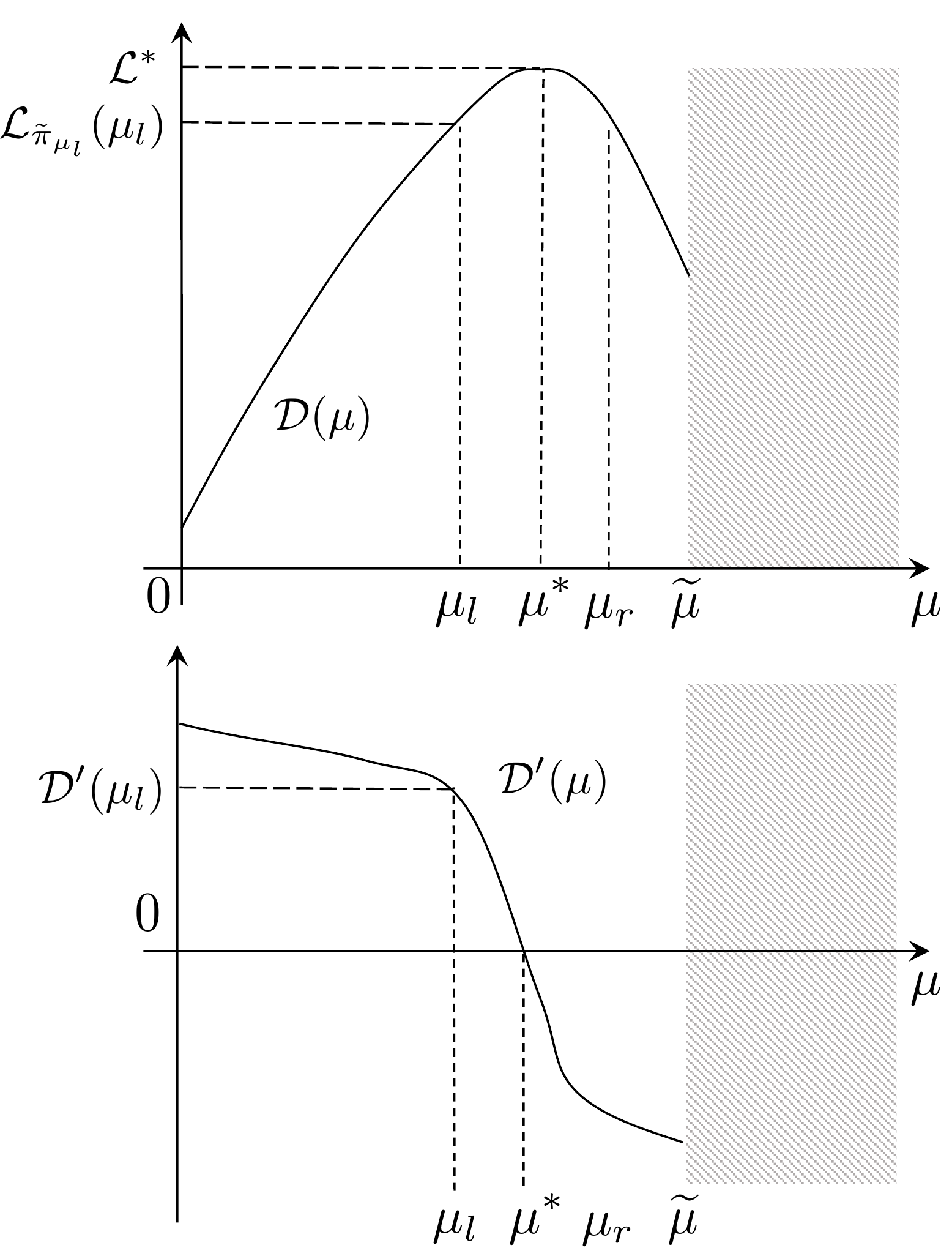}
	\caption{Qualitative plots of $\D(\mu)$ and its derivative $\D'(\mu)$.}
	\label{fig:plot}	
	\end{minipage}
	\hfill
	\begin{minipage}{0.46\textwidth}
	\begin{center}
	\bookboxx{
		\begin{small}
			\begin{algorithmic}[1]
				\renewcommand{\algorithmicrequire}{\textbf{Input:}}
				\renewcommand{\algorithmicensure}{\textbf{Output:}}
				\vspace{-0.02in}
				\REQUIRE $\wh\theta$, $\beta$, $V$, $\epsilon$, $\alpha$, $\lambda_0$
				\IF{$\D(0) \leq 0$}
				\STATE Set $\wb \mu = 0$ and $\tpi_\epsilon = \tpi_{\wb\mu}$
				\ELSE
					\STATE Set $\mu_l = 0$, $\mu_r = \mu_{\max}$ (\cref{lem:mu.max})
					\WHILE{$\alpha \frac{\mu_r - \mu_l}{\lambda_{\min}(D_{\mu_l})} \geq \epsilon$ \textbf{ or } $\lambda_{\min}(D_{\mu_l}) \geq \lambda_0 \epsilon^2$}
						\STATE Set $\wb\mu = (\mu_l+\mu_r)/2$
						\IF{$\D'(\wb\mu) > 0$}
							\STATE $\mu_l = \wb\mu$
						\ELSE
							\STATE $\mu_r = \wb\mu$
						\ENDIF
					\ENDWHILE				
				\ENDIF
				\IF{$\alpha \frac{\mu_r - \mu_l}{\lambda_{\min}(D_{\mu_l})} < \epsilon$}
					\STATE Set $\wb\mu = \mu_l$ and $\tpi_\epsilon = \tpi_{\wb\mu}$
				\ELSE
					\STATE Set $\tpi_\epsilon$ to the control return by the \textit{backup procedure}
				\ENDIF
				\RETURN Control policy $\tpi_\epsilon$
			\end{algorithmic}
			\vspace{-0.1in}
			\caption{\small The \dsllq algorithm to solve~\cref{eq:lagrange.lqr}.}
			\label{alg:dsllq}
		\end{small}
	}
	\end{center}
	\end{minipage}
\vspace{-0.1in}
\end{figure*}

The algorithm we propose, \dsllq, seeks to find a value of $\mu$ of zero gradient $\D'(\mu)$ by dichotomy search. While $\D(\mu)$ is a 1-dim function and Lem.~\ref{lem:dual.function.specific} guarantees that it is concave in $\M$, there are three major challenges to address: \textbf{1)} Thm.~\ref{thm:dual.gap.specific} does not provide any explicit value for $\wt\mu$; \textbf{2)} The algorithm needs to evaluate $\D'(\mu)$; \textbf{3)} For any $\epsilon$, \dsllq must return a policy $\tpi_\epsilon$ that is $\epsilon$-optimistic and $\epsilon$-feasible for the extended LQR with relaxed constraints~\eqref{eq:optimal.avg.cost.extended}.

\dsllq starts by checking the sign of the gradient $\mathcal{D}^\prime(0)$. If $\mathcal{D}^\prime(0) \leq 0$, the algorithm ends and outputs the optimal policy $\tpi_0$ since by concavity $0$ is the arg-max of $\mathcal{D}$ and $\tpi_0$ is the exact solution to~\eqref{eq:lagrange.lqr}. If $\mathcal{D}^\prime(0) > 0$, the dichotomy starts with accuracy $\epsilon$ and a \textit{valid}\footnote{We say that $[\mu_l, \mu_r]$ is valid if $\D'(\mu_l) \geq 0$ and $\D'(\mu_r) \leq 0$.} search interval $[0,\mu_{\max}]$, where $\mu_{\max}$ is defined as follows.

\begin{lemma}\label{lem:mu.max}
Let $\mu_{\max} := \beta^{-2}\lambda_{\max}\big(C\big) \lambda_{\max}(V)$, 
then $\mathcal{D}^\prime(\mu_{\max}) < 0$.
\end{lemma}
The previous lemma does not imply that $[0, \mu_{\max}] \supseteq \M$, but it provides an explicit value of $\mu$ with negative gradient, thus defining a bounded and valid search interval for the dichotomy process. At each iteration, \dsllq updates either $\mu_l$ or $\mu_r$ so that the interval $[\mu_l,\mu_r]$ is always valid. 

The second challenge is addressed in the following proposition, which illustrates how the derivative $\D'(\mu)$ (equivalently the constraint $g_{\tpi_\mu}$) can be efficiently computed.

\begin{proposition}\label{prop:lyapunov.characterization.constraint}
For any $\mu \in \M$, let $\tpi_\mu$ (Eq.~\ref{eq:extended.optimal.pi.mu}) have an associated controller $\wt K_\mu$ that induces a closed-loop dynamics $\Ac(\wt K_\mu) = \wh A + \wt B \wt K_\mu$ then $\D'(\mu) = g_{\tpi_\mu} = \Tr\big(G_\mu\big)$,
%
where $G_\mu$ is the unique solutions of the Lyapunov equation
	\begin{equation*}
	G_\mu = \big( \Ac(\wt{K}_\mu)\big)^\transp G_\mu  \Ac(\wt{K}_\mu) + \begin{pmatrix} I \\ \wt K_\mu \end{pmatrix}^\transp C_{g} \begin{pmatrix} I \\ \wt K_\mu \end{pmatrix},
	\end{equation*}
\end{proposition}

This directly from the fact that $g_{\tpi}$ is an asymptotic average quadratic quantity (as much as the average cost $J$), and it is thus the solution of a Lyapunov equation of dimension $n$.

The remaining key challenge is to design an adaptive stopping condition that is able to keep refining the interval $[\mu_l, \mu_r]$ until either an accurate enough solution is returned, or, the curvature is too large (or even infinite). In the latter case, the algorithm switches to a failure mode, for which we design an ad-hoc solution.

Since the objective is to achieve an $\epsilon$-feasible solution (i.e., $g_{\tpi_{\wb\mu}} \leq \epsilon$), we leverage \cref{le:smoothness.gradient} and we interrupt the dichotomy process whenever $(\mu_r - \mu_l) \alpha/\lambda_{\min}(D_{\mu_l}) \leq \epsilon$.
%
%
Nonetheless, when the optimum of~\eqref{eq:lagrange.lqr} is not attainable in $\M$, the previous stopping condition may never be verified and the algorithm would never stop. As a result, we interrupt the dichotomy process when $\lambda_{\min}(D_{\mu_l}) \leq \lambda_0 \epsilon^2$ for a given constant $\lambda_0$. In this case, the dichotomy fails to return a viable solution and we need to revert to a \textit{backup} strategy, which consists in either modifying the controller found at $\mu_l$ or applying a suitable perturbation to the original cost matrix $C_\dagger$. In the latter case, we design a perturbation such that \textbf{1)} the optimization problem~\eqref{eq:lagrange.lqr} associated to the system with the perturbed cost $C_\dagger'$ admits a maximum and can be efficiently solved by the same dichotomy process illustrated before and \textbf{2)} the corresponding solution $\tpi'$ is $\epsilon$-optimistic and $\epsilon$-feasible in the \textit{original} system. The explicit backup strategy is reported in App.~\ref{sec:app.algo}.

\begin{theorem}
	\label{thm:optimal.complexity.algo.main}
	For any LQR parametrized by $\wh\theta$, $V$, $\beta$, and psd cost matrices $Q,R$, and any accuracy $\epsilon \in (0,1/2)$, there exists values of $\alpha$ and $\lambda_0$ and a backup strategy such that
	\vspace{-0.1in}
	\begin{enumerate}
		\item \dsllq outputs an $\epsilon$-optimistic and $\epsilon$-feasible policy $\tpi_\epsilon$ given by the linear controller $\wt{K}_\epsilon$ such that
		\begin{equation*}
		\mathcal{J}_{\tpi_\epsilon} \leq  \mathcal{J}_*+  \epsilon \quad \text{ and }\quad  g_{\tpi_\epsilon} \leq \epsilon.
		\end{equation*}
		\item \dsllq terminates within at most $N = O\big(\log(\mu_{\max} /\epsilon) \big)$ iterations, each solving one Riccati and one Lyapunov equation for the extended Lagrangian LQR, both with complexity $O\big(n^3\big)$.
	\end{enumerate}
	\vspace{-0.1in}
\end{theorem}

This result shows that \dsllq returns a solution to~\eqref{eq:optimal.avg.cost.extended} at any level of accuracy $\epsilon$ in a total runtime $O(n^3 \log(1/\epsilon))$. We refer to the algorithm resulting by plugging \dsllq into the \ofulqplus learning scheme as \laglq (Lagrangian-LQ). By running \dsllq with $\epsilon = 1/\sqrt{t}$ provides the regret guarantee of~\cref{thm:regret}.



\vspace{-0.1in}
\section{Discussion}
\label{sec:conclusion}

We investigate the difference between confidence-based and isotropic exploration in term of complexity, bounds and empirical performance.  While not conclusive, we believe this discussion sheds light on how confidence-based methods may be better at adapting to the structure of the problem.  As a representative for isotropic exploration, we refer to \cecce~\cite{simchowitz2020naive}, which offers the tightest regret guarantee among the \ce strategies. For confidence-based exploration, we discuss the guarantee of both \ofulqplus and \laglq, but limit the computational and experiment comparisons with \laglq only.

%

\textbf{Computational complexity.} Both \laglq and \cecce proceeds through episodes of increasing length. \laglq relies on the standard determinant-based rule ($\det(V_t) \geq 2\det(V_{t_k})$) to decide when to stop, while in \cecce the length of each episode is twice longer than the previous one. 
In both cases, the length of the episodes is increasing exponentially over time, thus leading to $O(\log T)$ updates. Given the complexity analysis in~\cref{thm:optimal.complexity.algo.main}, we notice that \dsllq and computing the \ce controller have the same order of complexity, where \cecce solves \textit{one} Riccati equation, while \dsllq solves as many as $\log(1/\epsilon)$ Riccati and Lyapunov equations. On systems of moderate side and given the small number of recomputations, the difference between the two approaches is relatively narrow.

\textbf{Regret.} We limit the comparison to the main order term $\wt O(\sqrt{T})$ and the dependencies on dimensions $n$ and $d$, and problem-dependent constants such as $\kappa$ and $\|P^*\|_2$, 

\begin{align*}
\R_{\cecce} &= \wt O\big( \|P^*\|_2^{11/2} d\sqrt{n T}\big), \\
\R_{\ofulqplus} &= \wt O\big(\kappa^{1/2} \|P_*\|^{3/2}_2 (n +d) \!\sqrt{n}   \sqrt{T}\big), \\
\R_{\laglq} &= \wt{O} \big( \kappa^{3/2} \|P_*\|_2^2 (n+d) \sqrt{n} \sqrt{T} \big).
\end{align*}

The first difference is that \cecce has worst-case optimal dependency $d \sqrt{n}$ on the dimension of the problem, while optimistic algorithms \ofulqplus and \laglq are slightly worse, scaling with $(n+d) \sqrt{n}$. While this shows that \ofulqplus and \laglq are worst-case optimal when $n\approx d$, it is an open question whether $\epsilon-$greedy is by nature superior to confidence-based method when $d \ll n$ or whether it is due to a loose analysis. In fact, those dependencies are mostly inherited from the confidence intervals in~\eqref{p:concentration} which is treated differently in~\cite{simchowitz2020naive}, thanks to a \textit{refined} bound and a \textit{different} regret decomposition. This suggests that a finer analysis for \ofulqplus and \laglq may close this gap.

The main difference lies in the dependency on complexity-related quantities $\kappa$ and $\|P^*\|_2$. While there is no strict ordering between them,\footnote{The definition of $\kappa = D/\lambda_{\min}{Q}$, with $D \geq \Tr(P^*)$ may suggest $\kappa > \|P^*\|_2$, but the smallest eigenvalue of $Q$ may be large enough so that $\kappa \leq \|P^*\|_2$.} they both measure the cost of controlling the system. In this respect, \cecce suffers from a significantly larger dependency than optimistic algorithms: \ofulqplus offers the best performance while \laglq is slightly worse than \ofulqplus, due to the use of a relaxed constraint to obtain tractability. We believe this difference in performance may be intrinsic in the fact that methods based on isotropic perturbations of the \ce are less effective in adapting to the actual structure of the problem. As the isotropic perturbation is tuned to guarantee a sufficient estimation in all directions of the state-control space, it leads to over-exploration w.r.t. some directions as soon as there is an asymmetry in the cost sensitivity in the estimation error. On the other hand, \ofulqplus and \laglq further leverage this asymmetry from the confidence set, and by optimism, do not waste exploration to learn accurately directions which have little to no impact on the performance.

%

\begin{figure}[h]
		\centering
		\includegraphics[width=0.8\linewidth]{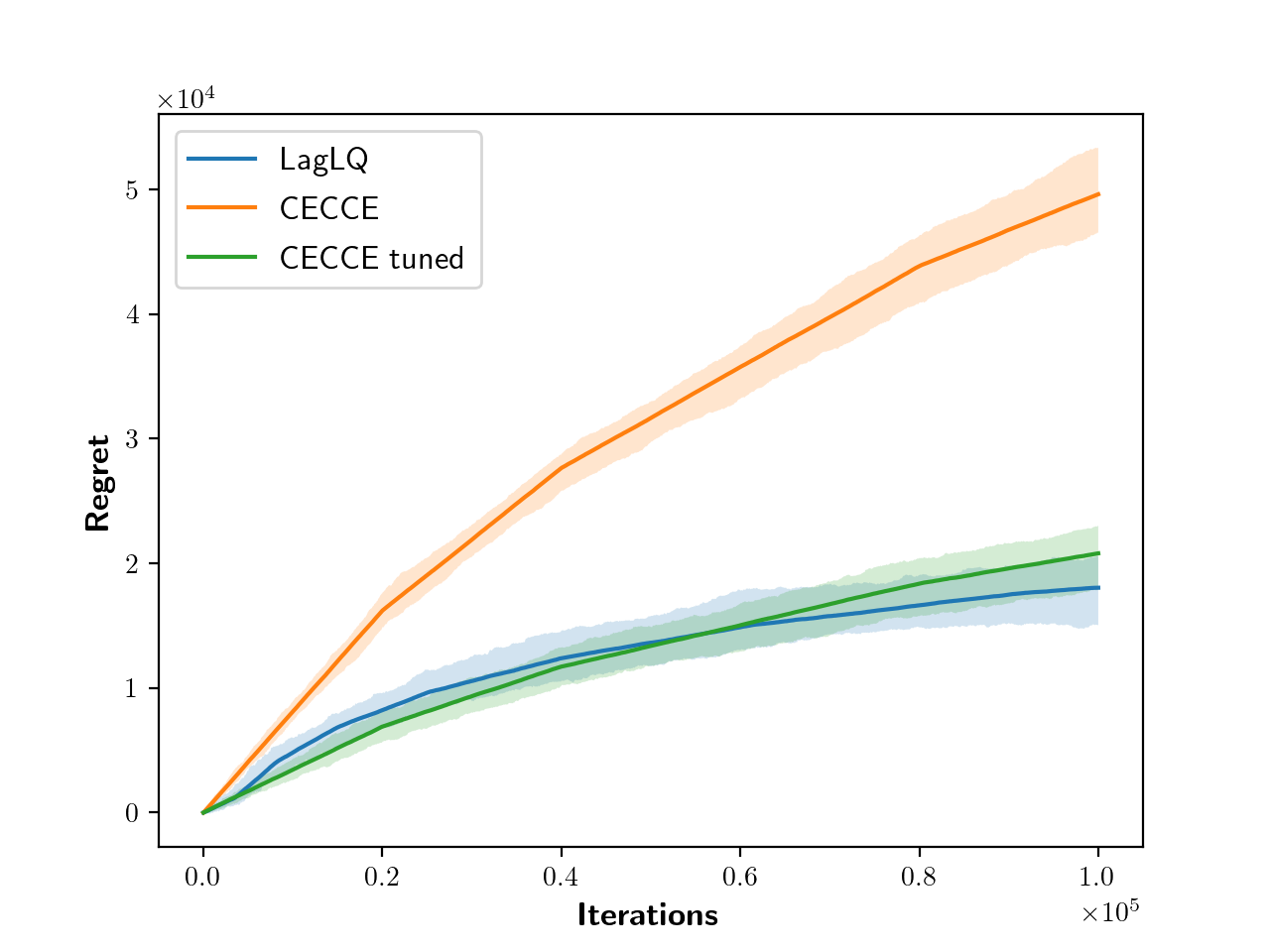}
		\caption{Regret curves for \cecce and \laglq.}
		\label{fig:exp}
\end{figure}

\textbf{Empirical comparison.} We conclude with a simple numerical simulation (details in App.~\ref{app:experiments}). We compare \cecce with the variance parameter ($\sigma_{in}^2$) set as suggested in the original paper and a tuned version where we shrink it by a factor $\sqrt{\|P_*\|_2}$, and \laglq where the confidence interval is set according to~\eqref{p:concentration}. Both algorithms receive the same set $\Theta_0$ obtained from an initial system identification phase. In Fig.~\ref{fig:exp} we see that \laglq performs better than both the original and tuned versions of \cecce. More interestingly, while \cecce is ``constrained'' to have a $O(\sqrt{T})$ regret by the definition of the perturbation itself, which scales as $1/\sqrt{t}$, it seems \laglq's regret is $o(\sqrt{T})$, suggesting that despite the worst-case lower bound $\Omega(\sqrt{T})$, \laglq may be adapt to the structure of the problem and achieve better regret.
%


\begin{small}
\bibliography{biblio_thesis}

\begin{thebibliography}{18}
\providecommand{\natexlab}[1]{#1}
\providecommand{\url}[1]{\texttt{#1}}
\expandafter\ifx\csname urlstyle\endcsname\relax
  \providecommand{\doi}[1]{doi: #1}\else
  \providecommand{\doi}{doi: \begingroup \urlstyle{rm}\Url}\fi

\bibitem[Abbasi-Yadkori \& Szepesv{\'a}ri(2011)Abbasi-Yadkori and
  Szepesv{\'a}ri]{abbasi2011regret}
Abbasi-Yadkori, Y. and Szepesv{\'a}ri, C.
\newblock Regret bounds for the adaptive control of linear quadratic systems.
\newblock In \emph{COLT}, pp.\  1--26, 2011.

\bibitem[Abbasi-Yadkori et~al.(2011)Abbasi-Yadkori, P{\'a}l, and
  Szepesv{\'a}ri]{abbasi2011online}
Abbasi-Yadkori, Y., P{\'a}l, D., and Szepesv{\'a}ri, C.
\newblock Online least squares estimation with self-normalized processes: An
  application to bandit problems.
\newblock \emph{arXiv preprint arXiv:1102.2670}, 2011.

\bibitem[Abeille \& Lazaric(2018)Abeille and Lazaric]{abeille2018improved}
Abeille, M. and Lazaric, A.
\newblock Improved regret bounds for thompson sampling in linear quadratic
  control problems.
\newblock In \emph{Proceedings of the 35th International Conference on Machine
  Learning, {ICML} 2018, Stockholmsm{\"{a}}ssan, Stockholm, Sweden, July 10-15,
  2018}, pp.\  1--9, 2018.

\bibitem[Bittanti et~al.(2006)Bittanti, Campi, et~al.]{bittanti2006adaptive}
Bittanti, S., Campi, M., et~al.
\newblock Adaptive control of linear time invariant systems: the ``bet on the
  best'' principle.
\newblock \emph{Communications in Information \& Systems}, 6\penalty0
  (4):\penalty0 299--320, 2006.

\bibitem[Cohen et~al.(2019)Cohen, Koren, and Mansour]{cohen2019learning}
Cohen, A., Koren, T., and Mansour, Y.
\newblock Learning linear-quadratic regulators efficiently with only
  {\textdollar}{\textbackslash}sqrt\{T\}{\textdollar} regret.
\newblock \emph{CoRR}, abs/1902.06223, 2019.
\newblock URL \url{http://arxiv.org/abs/1902.06223}.

\bibitem[Dean et~al.(2018)Dean, Mania, Matni, Recht, and Tu]{dean2018regret}
Dean, S., Mania, H., Matni, N., Recht, B., and Tu, S.
\newblock Regret bounds for robust adaptive control of the linear quadratic
  regulator.
\newblock \emph{CoRR}, abs/1805.09388, 2018.
\newblock URL \url{http://arxiv.org/abs/1805.09388}.

\bibitem[Faradonbeh et~al.(2017)Faradonbeh, Tewari, and
  Michailidis]{faradonbeh2017finite}
Faradonbeh, M. K.~S., Tewari, A., and Michailidis, G.
\newblock Finite time analysis of optimal adaptive policies for
  linear-quadratic systems.
\newblock \emph{CoRR}, abs/1711.07230, 2017.
\newblock URL \url{http://arxiv.org/abs/1711.07230}.

\bibitem[Faradonbeh et~al.(2018{\natexlab{a}})Faradonbeh, Tewari, and
  Michailidis]{faradonbeh2018finite}
Faradonbeh, M. K.~S., Tewari, A., and Michailidis, G.
\newblock Finite time adaptive stabilization of {LQ} systems.
\newblock \emph{CoRR}, abs/1807.09120, 2018{\natexlab{a}}.
\newblock URL \url{http://arxiv.org/abs/1807.09120}.

\bibitem[Faradonbeh et~al.(2018{\natexlab{b}})Faradonbeh, Tewari, and
  Michailidis]{faradonbeh2018input}
Faradonbeh, M. K.~S., Tewari, A., and Michailidis, G.
\newblock Input perturbations for adaptive regulation and learning.
\newblock \emph{CoRR}, abs/1811.04258, 2018{\natexlab{b}}.
\newblock URL \url{http://arxiv.org/abs/1811.04258}.

\bibitem[Jaksch et~al.(2010)Jaksch, Ortner, and Auer]{jaksch2010near-optimal}
Jaksch, T., Ortner, R., and Auer, P.
\newblock Near-optimal regret bounds for reinforcement learning.
\newblock \emph{J. Mach. Learn. Res.}, 11:\penalty0 1563--1600, August 2010.

\bibitem[Lancaster \& Rodman(1995)Lancaster and Rodman]{lancaster1995algebraic}
Lancaster, P. and Rodman, L.
\newblock \emph{Algebraic riccati equations}.
\newblock Oxford University Press, 1995.

\bibitem[{Mania} et~al.(2019){Mania}, {Tu}, and {Recht}]{mania2019certainty}
{Mania}, H., {Tu}, S., and {Recht}, B.
\newblock {Certainty Equivalent Control of LQR is Efficient}.
\newblock \emph{arXiv e-prints}, art. arXiv:1902.07826, Feb 2019.

\bibitem[Molinari(1975)]{molinari1975}
Molinari, B.~P.
\newblock The stabilizing solution of the discrete algebraic riccati equation.
\newblock \emph{Automatic Control, IEEE Transactions on}, 20\penalty0
  (3):\penalty0 396--399, Jun 1975.

\bibitem[Ouyang et~al.(2017)Ouyang, Gagrani, and
  Jain]{ouyang2017learning-based}
Ouyang, Y., Gagrani, M., and Jain, R.
\newblock Learning-based control of unknown linear systems with thompson
  sampling.
\newblock \emph{CoRR}, abs/1709.04047, 2017.
\newblock URL \url{http://arxiv.org/abs/1709.04047}.

\bibitem[Rockafellar(1970)]{rockafellar1970convex}
Rockafellar, R.~T.
\newblock \emph{Convex analysis}.
\newblock Number~28. Princeton university press, 1970.

\bibitem[Rugh(1996)]{rugh1996linear}
Rugh, W.~J.
\newblock \emph{Linear system theory}, volume~2.
\newblock prentice hall Upper Saddle River, NJ, 1996.

\bibitem[{Simchowitz} \& {Foster}(2020){Simchowitz} and
  {Foster}]{simchowitz2020naive}
{Simchowitz}, M. and {Foster}, D.~J.
\newblock {Naive Exploration is Optimal for Online LQR}.
\newblock \emph{arXiv e-prints}, art. arXiv:2001.09576, Jan 2020.

\bibitem[{Van Dooren}(1981)]{van1981generalized}
{Van Dooren}, P.
\newblock A generalized eigenvalue approach for solving riccati equations.
\newblock \emph{SIAM Journal on Scientific and Statistical Computing},
  2\penalty0 (2):\penalty0 121--135, 1981.

\end{thebibliography}
\bibliographystyle{icml2020}
\end{small}

\onecolumn
\newpage
\appendix

\section{Notation}\label{sec:notation}

We recall the basic notation defined in the paper and introduce additional convenient notation.\\

\textbf{Linear algebra.}
\begin{itemize}
\item $\lambda(A)$ : the spectrum of the matrix $A$
\item $\rho(A)$ : the spectral radius of a matrix $A$, $\rho(A) = \max |\lambda(A)|$
\item $\lambda_{\min}(M)$, $\lambda_{\max}(M)$ : the minimum and maximum eigenvalue of a symmetric matrix $M$
\item $Im(A)$, $\ker(A)$ : the row and null space of a matrix $A$
\item $|x|$ : the absolute value (resp. the modulus) of $x\in\mathbb{R}$ (resp. $x \in \mathbb{C}$)
\item $\|x\|$ : the euclidian norm of a vector $x$
\item $\|A\|_2$ : the $2-$ norm of a matrix $A$, $\|A\|_2 = \max_{\|x\|=1} \|A x\|$
\item $\|A\|_F$ :  the Frobenius norm of a matrix $A$, $\|A\|_F = \sqrt{\Tr(A^\transp A)}$
\item $\|A\|_{M}$ : the weighted Frobenius norm of $A$ wrt to a psd matrix $M$, $\|A\|_M = \| M^{1/2} A\|_F$
\end{itemize}

\textbf{Original LQR problem.}
\begin{itemize}
\item $n,d$: the dimensions of the state and input variables
\item $x_t,u_t,\epsilon_t$ : the state/input/noise variables at time $t$
\item $z_t$ : aggregated state and input variable at time $t$, $z_t^\transp = (x_t^\transp, u_t^\transp)$
\item $A_*,B_*$ : ground truth for the state dynamics parameters
\item $Q,R$ : cost matrices of the LQR
\item $C$ : aggregated cost matrix, $C = (Q,0; 0,R)$
\item $\pi$, $K$ : an arbitrary linear policy mapping states to inputs, and its matrix representation.
\item $\theta$ : aggregated state dynamics parametrization, $\theta^\transp = (A,B)$
\item $P(\theta)$ : solution of the Riccati equation associated with the LQR system parametrized by $\theta$ and $C$
\item $K(\theta)$ : optimal controller for the LQR system parametrized by $\theta$ and $C$
\item $\Ac(\theta,K)$ : the closed-loop matrix of a dynamical system parametrized by $\theta$ controlled by $K$, $\Ac(\theta,K) = \theta^\transp \begin{pmatrix}I\\K\end{pmatrix}$
\item $\Ac(\theta)$ : the closed-loop matrix of a dynamical system parametrized by $\theta$ controlled by $K(\theta)$, $\Ac(\theta) = \Ac(\theta,K(\theta))$
\item $\Sigma(\theta,K)$ : The steady-state covariance of the state process driven by $\Ac(\theta,K)$
\item $J_{\pi}(\theta)$ : the infinite horizon cost of an LQR parametrized by $\theta$ and controlled by $\pi$
\item $J(\theta)$ : the infinite horizon cost of an LQR parametrized by $\theta$ and optimally controlled by $K(\theta)$
\item $J_*$ : the optimal infinite horizon cost for the true LQR system, $J_* = J(\theta_*)$
\end{itemize}

\textbf{Extended/Lagrangian LQR problem.}
\begin{itemize}
\item $w_t$ : the input perturbation at time $t$
\item $\wt{u}_t$ : the extended input at time $t$, $\wt{u}_t^\transp = (u_t^\transp, w_t^\transp)$
\item $C_\dagger$, $C_g$ : the cost matrices of the extended LQR objective and constraint
\item $\wt{B}$ : the extended input matrix, $\wt{B} = (B,I)$
\item $\tpi$, $\wt{K}$ : an arbitrary linear extended policy mapping states to extended inputs, and its matrix representation
\item $\mathcal{J}_{\tpi}(\theta,\beta,V)$ :  the infinite horizon cost of the constrained extended LQR parametrized by $(\theta,\beta,V)$ and controlled by $\tpi$
\item $\mathcal{J}_*$ :  the optimal infinite horizon cost of the constrained extended LQR.
\item $g_{\tpi}(\theta,\beta,V)$ : the infinite horizon constraint value of the constrained extended LQR parametrized by $(\theta,\beta,V)$ and controlled by $\tpi$
\item $\mathcal{L}_{\tpi}(\theta,\beta,V; \cdot)$ : the Lagrangian function associated with the extended constrained LQR
\item $\mathcal{D}(\cdot)$ : the dual function associated with the Lagrangian relaxation
\end{itemize}

\section{Structure of the Appendix}

The appendix is organized over different sections:
\begin{itemize}
\item \cref{sec:app_preliminaries} recalls elements of Lyapunov stability theory and concentration inequality (high-probability events) (\cref{p:concentration}).
\item \cref{sec:proof_ofulqplus} contains the proofs of \ofulqplus (sequential stability~\cref{lem:ofu.dynamic.stability2} and regret~\cref{p:regret.ofulq.plus}).
\item \cref{sec:proof_laglq} contains the proofs of \laglq (optimism~\cref{lem:optimism},  sequential stability~\cref{lem:bounded.state}, and regret~\cref{thm:regret}) .
\item Appendices~\ref{sec:app.proof.duality},\ref{sec:app.proof.lem:dual.function},~\ref{sec:app.proof.duality.final} reports the strong-duality proofs (\cref{thm:dual.gap.specific}, \cref{lem:dual.function.specific}, \cref{le:smoothness.gradient}, \cref{lem:mu.max}, \cref{prop:lyapunov.characterization.constraint}).
\item \cref{sec:app.algo} details the algorithms and the proof of~\cref{thm:optimal.complexity.algo.main}.
\item \cref{app:experiments} details the experimental protocol used to obtain Fig.~\ref{fig:exp}.
\end{itemize}

At the beginning of each part we recall and complement the notation used in that specific part.

\newpage

\section{Preliminaries}\label{sec:app_preliminaries}

\subsection{Lyapunov Stability}\label{sec:app_lyapunov_stability}

We recall basic results from Lyapunov stability theory. Most of the definitions below are adapted from~\citet{rugh1996linear}. We consider a general discrete-time deterministic linear dynamical process $\{\wb{x}_t\}_{t\geq0}$ with dynamics $A$ and initial state $\xi$ defined as
\begin{align}\label{eq:deterministic.process}
\wb{x}_{t+1} = A \wb{x}_{t}, \quad \wb{x}_0 = \xi.
\end{align}
We recall the following definition of stability.

\begin{definition}
A deterministic process~\eqref{eq:deterministic.process} is \textit{uniformly exponentially stable} (UES) if there exists a constant $\gamma > 0$ and a constant $0 \leq \lambda < 1$ such that for any $t \geq s \geq 0$, 
\begin{align}\label{eq:unif.exp.stability}
\|\wb{x}_{t}\| \leq \gamma \lambda^{t-s} \| \wb{x}_s\|.
\end{align}
\end{definition}

The stability of the deterministic process~\eqref{eq:deterministic.process} can be translated into stability properties of the transition matrix $A$. The following propositions provide spectral characterizations.
\begin{proposition}[Thm.~22.1~\citep{rugh1996linear}]
A time-invariant deterministic process~\eqref{eq:deterministic.process} is uniformly exponentially stable if and only if all eigenvalues of $A$ have magnitude strictly smaller than 1, i.e., $\rho(A) < 1$. 
\end{proposition}

\begin{proposition}[Thm.~22.7~\citep{rugh1996linear}]
The process~\eqref{eq:deterministic.process} is uniformly exponentially stable (UES) if and only if there exists a constant $\gamma>0$ and a constant $0 \leq \lambda < 1$ such that for all $t\geq s$, $\|A^{t-s}\| \leq \gamma\lambda^{t-s}$. 
\end{proposition}

An alternative characterization is provided by Lyapunov stability theory.

\begin{proposition}[Thm.~23.3~\citep{rugh1996linear}]\label{thm:ues.rugh}
A deterministic process~\eqref{eq:deterministic.process} is uniformly exponentially stable if and only if there exists a symmetric p.s.d. matrix $\Sigma$ such that 
\begin{equation}\label{eq:seq.stability}
\Sigma = A \Sigma A^\transp+ I.
\end{equation}
Further, let $\rho$ be finite positive constants such that $ \Sigma \preccurlyeq \rho I$, then for any $t \geq s \geq 0$, we have 
$\|A^{t-s}\| \leq \gamma\lambda^{t-s}$ with
\begin{equation}\label{eq:connection.constants}
\gamma = \sqrt{\rho}, \quad \lambda = 1 - 1 / \rho.
\end{equation}
\end{proposition}

Uniform exponential stability for a deterministic process directly implies boundedness for a perturbed version of the process with bounded perturbations.

\begin{proposition}\label{p:ues.perturbed}
Consider the perturbed instance $\{\wt{x}_t\}_{t\geq0}$ of the deterministic process~\eqref{eq:deterministic.process}, defined as
\begin{equation}\label{eq:perturbed.deterministic.process}
\wt{x}_{t+1} = A \wt{x}_{t} + \epsilon_t, \quad \wt{x}_0 = \xi,  \quad \text{ for all } t\geq 0.
\end{equation}
where $\{\epsilon_t\}_{t\geq 0}$ is a martingale difference sequence with zero-mean, $\sigma-$subGaussian increments. 
If the deterministic process is UES associated with a psd matrix $\Sigma$ given in~\cref{eq:seq.stability}, the perturbed process $\{\wt{x}_t\}_{t\geq0}$ is uniformly bounded, that is, with probability at least $1-\delta$,
\begin{equation}\label{eq:perturbed.process.bound}
\text{for all } t \geq 0,\quad  \| \wt{x}_t - A^{t}\xi \|^2 \leq 8 \sigma^2 \Tr(\Sigma) \log(1/\delta).
\end{equation}
\end{proposition}

\begin{proof}[Proof of Prop.~\ref{p:ues.perturbed}]
First we use~\cref{eq:perturbed.deterministic.process} to reformulate the perturbed process as
\begin{equation}\label{eq:pert.process}
\wt{x}_t = A^{t} \xi + \sum_{s=1}^t A^{t-s} \epsilon_{s-1}.
\end{equation}
Then, we follow similar steps as in~\citet{abbasi2011online}: we begin by constructing an exponential super-martingale. Formally, for all $\alpha \in \mathbb{R}^n$, 
\begin{equation*}
M_t^{\alpha} = \exp\left( \alpha^\transp (\wt{x}_t - A^{t} \xi) - \frac{\sigma^2}{2} \|\alpha\|^2_{\sum_{s=1}^{t} A^{t-s}(A^{t-s})^\transp }\right),
\end{equation*}
is a super-martingale, and $\mathbb{E}( M^{\alpha}_t) \leq 1$. By~\cref{eq:pert.process}, we have for any $s\geq 0$, 
\begin{equation*}
\mathbb{E}( M^{\alpha}_s | \mathcal{F}_{s-1}) = M_{s-1}^{\alpha} \mathbb{E}\left( \alpha^\transp A^{t-s} \epsilon_{s-1} - \frac{\sigma^2}{2} \alpha^\transp A^{t-s} \big(A^{t-s}\big)^\transp \alpha\Big| \mathcal{F}_{s-1} \right) \leq M_{s-1}^{\alpha},
\end{equation*}
which ensures that $M^{\alpha}_t$ is a super-martingale while $M^{\alpha}_{0} =1$ ensures that $\mathbb{E}( M^{\alpha}_t) \leq 1$. Further, recursively applying~\cref{eq:seq.stability} ensures that:
\begin{equation*}
\text{for all } t \geq 0, \quad \sum_{s=1}^{t} A^{t-s} \big(A^{t-s}\big)^\transp \preccurlyeq \Sigma.
\end{equation*}
Following~\citet{abbasi2011online}, we consider the random variable $\A \sim \mathcal{N}(0, \sigma^{-2} \lambda_{\max}(\Sigma)^{-1} I)$, independent of the noise sequence $\{\epsilon_s\}_{s\geq 0}$ and obtain, by the tower rule, that for all $t\geq 0$, $M_t = \mathbb{E}_{\A} (M^{\A}_t)$ is such that $\mathbb{E}(M_t) \leq 1$. Further, 
\begin{equation*}
\begin{aligned}
M_t &\geq \sqrt{\frac{\det\left(\sigma^{2} \lambda_{\max}(\Sigma)  I \right)}{(2 \pi)^{n}}} \int_{\alpha \in \mathbb{R}^{n}} 
\exp\left( \alpha^\transp (\wt{x}_t -A^{t} \xi) - \frac{\sigma^2}{2} \|\alpha\|^2_{\Sigma + \lambda_{\max}(\Sigma) I}\right) \\
&= \exp \left( \frac{1}{2\sigma^2} \| \wt{x}_t - A^{t} \xi \|^2_{(\Sigma + \lambda_{\max}(\Sigma) I)^{-1}} \right) 
\sqrt{\frac{\det\left(\sigma^2 \lambda_{\max}(\Sigma)  I \right)}{\det\left(\sigma^{2} (\Sigma + \lambda_{\max}(\Sigma)  I) \right)}}\\
 &= \exp \left( \frac{1}{2\sigma^2} \| \wt{x}_t - A^{t} \xi \|^2_{(\Sigma + \lambda_{\max}(\Sigma) I)^{-1}} \right) 
 \sqrt{\frac{\det\left(\lambda_{\max}(\Sigma)  I \right)}{\det\left( \Sigma + \lambda_{\max}(\Sigma)I \right)}}\\
\end{aligned}
\end{equation*}
We conclude as in Cor.~1 in~\citep{abbasi2011online} using Markov inequality and a stopping time construction to obtain that with probability at least $1-\delta$,
\begin{equation*}
\text{for all } t \geq 0,\;\;\;\;   \| \wt{x}_t - A^{t} \xi \|^2 \leq 2 \sigma^2 \lambda_{\max}( \Sigma ) \log\left(\frac{\det\left( \Sigma + \lambda_{\max}(\Sigma)I \right)}{\det\left(\lambda_{\max}(\Sigma)  I \right)}\right) + 4 \sigma^2  \lambda_{\max}( \Sigma ) \log(1/\delta).
\end{equation*}
Finally, from $\log(1+x) \leq x$, we obtain
\begin{equation*}
 \lambda_{\max}( \Sigma ) \log\left(\frac{\det\left( \Sigma + \lambda_{\max}(\Sigma)I \right)}{\det\left(\lambda_{\max}(\Sigma)  I \right)}\right) \leq \Tr(\Sigma), \quad \lambda_{\max}(\Sigma) \leq \Tr(\Sigma),
\end{equation*}
which leads to, with probability at least $1 - \delta$, 
\begin{equation*}
\text{for all } t \geq 0,\;\;\;\;   \| \wt{x}_t -A^{t} \xi \|^2 \leq 8 \sigma^2 \Tr(\Sigma) \log(1/\delta).
\end{equation*}

\end{proof}

\subsection{Technical results on Lyapunov stability}

We recall useful technical results on Lyapunov manipulations. 
\begin{proposition}[Thm.~23.7~\citep{rugh1996linear}]\label{thm.lyapunov.stability.time.invariant.0}
Given a $n \times n$ matrix $A$, if there exists symmetric positive definite $n \times n$ matrices $M$ and $P$ satisfying the discrete time Lyapunov equation 
\begin{equation}
\label{eq:lyap.stability.margin}
A^\transp P A - P = - M
\end{equation}
then all eigenvalues of $A$ have magnitude strictly smaller than 1 (i.e., $\rho(A) < 1$) and $A$ is stable. Conversely, if $\rho(A) < 1$, then for every p.s.d. matrix $M$, the solution of Eq.~\ref{eq:lyap.stability.margin} exists and is unique.
\end{proposition}

While the above theorem is stated for some matrix $A$, in the following we rely extensively on a similar characterization but for $A^\transp$. As $A^\transp$ and $A$ have the same spectrum, Prop.~\ref{thm.lyapunov.stability.time.invariant.0} can be rephrased in term of $A^\transp$.

\begin{corollary}
\label{thm.lyapunov.stability.time.invariant}
Given a $n \times n$ matrix $A$, if there exists a symmetric, positive definite $n \times n$ matrices $M$ and $\Sigma$ satisfying the discrete time Lyapunov equation 
\begin{equation}
\label{eq:lyap.stability.margin.transp}
A \Sigma A^\transp - \Sigma = - M
\end{equation}
then all eigenvalues of $A$ have magnitude strictly less than one (i.e., $\rho(A) < 1$) and $A$ is stable. Conversely, if $\rho(A) < 1$, then for every p.s.d. matrix $M$, the solution of Eq.~\ref{eq:lyap.stability.margin} exists and is unique.
\end{corollary}

\begin{proposition}\label{p:lyap.comparison}
Let $A$ be stable and $(\Sigma_1, M_1)$, and $(\Sigma_2, M_2)$ be pairs of p.s.d. matrices satisfying the Lyapunov equations
\begin{equation*}
A \Sigma_1 A^\transp - \Sigma_1 = - M_1; \quad\quad A \Sigma_2 A^\transp - \Sigma_2 = - M_2.
\end{equation*}
If $M_1 \preccurlyeq M_2$, then $\Sigma_1 \preccurlyeq \Sigma_2$.
\end{proposition}

\begin{proposition}\label{p:lyap.linearity}
Let $A$ be stable and $(\Sigma_1, M)$, and $(\Sigma_2, \alpha M)$ be pairs of p.s.d. matrices satisfying the Lyapunov equations
\begin{equation*}
A \Sigma_1 A^\transp - \Sigma_1 = - M; \quad\quad A \Sigma_2 A^\transp - \Sigma_2 = - \alpha M.
\end{equation*}
Then, $\Sigma_2 = \alpha \Sigma_1$.
\end{proposition}

\begin{proposition}\label{p:lyap.riccati.trace}
Let $A$ be stable and $M$ be a psd matrix. Let $P$ and $\Sigma$ be solution the Lyapunov equations
\begin{equation*}
A^\transp P A - P = - M; \quad \quad A \Sigma A^\transp - \Sigma = - I.
\end{equation*}
Then, $\Tr(P) = \Tr(\Sigma M)$.
\end{proposition}

We finally provide a technical result that we will intensively use in~\cref{ssec:app_proof.stability.ofulqplus,ssec:app_sequential.stability.laglq}. This proposition shows that the stability property of some matrix $A_0$ will be preserved for all matrix $A_1$ sufficiently close to $A_0$, where the distance between $A_0$ and $A_1$ is measured according to a Lyapunov metric associated with $A_0$.

\begin{proposition}
\label{prop:distance.to.stability.lyapunov.metric}
Let $A_0 \in \mathbb{R}^{n \times n}$ be a stable matrix, and let $\Sigma_0$ be the steady-state covariance associated with $A_0$ satisfying the Lyapunov:
\begin{equation*}
\Sigma_0 = A_0 \Sigma_0 A_0^\transp + I.
\end{equation*}
For any matrix $A_1$ such that $\|(A_0 - A_1)^\transp \|_{\Sigma_0} \leq \eta  c_0$ where $ c_0^2 \eta (\eta + 2) < 1$ and $c_0 \geq  \|A_0^\transp\|_{\Sigma_0}$, then,
\begin{enumerate}
\item $A_1$ is stable, and we denote by $\Sigma_1$ its steady-state covariance matrix defined as the unique solution of $$\Sigma_1 = A_1 \Sigma_1 A_1^\transp + I,$$
\item $\Sigma_1 \preccurlyeq \frac{1}{1 - c_0^2 \eta (\eta + 2)} \Sigma_0$.
\end{enumerate}
\end{proposition}
\begin{proof}
\begin{enumerate}
\item We first prove that $A_1$ is stable by~\cref{thm.lyapunov.stability.time.invariant}. Algebraic manipulations lead to
\begin{equation*}
\begin{aligned}
A_1 \Sigma_0 A_1^\transp &= A_0 \Sigma_0 A_0^\transp + (A_1 - A_0) \Sigma_0 (A_1 - A_0)^\transp + A_0 \Sigma_0 (A_1 - A_0)^\transp + (A_1 - A_0) \Sigma_0 A_0^\transp \\
&= \Sigma_0 - I + M,
\end{aligned}
\end{equation*}
where $M = (A_1 - A_0) \Sigma_0 (A_1 - A_0)^\transp + A_0 \Sigma_0 (A_1 - A_0)^\transp + (A_1 - A_0) \Sigma_0 A_0^\transp$. As a result,~\cref{thm.lyapunov.stability.time.invariant} guarantees the stability of $A_1$ whenever $\|M\|_2 < 1$. Further, this is asserted since 
\begin{equation*}
\begin{aligned}
\|M\|_2 &\leq \|\Sigma_0^{1/2}(A_1 - A_0)^\transp\|^2_2 + 2 \|\Sigma_0^{1/2}(A_1 - A_0)^\transp\|_2 \|\Sigma_0^{1/2} A_0^\transp\|_2 \leq \|(A_1 - A_0)^\transp\|^2_{\Sigma_0} + 2 \|(A_1 - A_0)^\transp\|_{\Sigma_0} \|A_0^\transp\|_{\Sigma_0} \\
&\leq  c_0^2 \big( \eta^2 + 2 \eta \big) < 1.
\end{aligned}
\end{equation*}
\item We now characterize the distance in term of stability between $A_0$ and $A_1$, showing how $\Sigma_1$ relates to $\Sigma_0$ (in a psd sense). Let $\Delta = \Sigma_1 - \Sigma _0$. Then, algebraic manipulations leads to
\begin{equation*}
\begin{aligned}
\Delta &= A_1 \Delta A_1^\transp + (A_1 - A_0) \Sigma_0 (A_1 - A_0)^\transp + (A_1 - A_0) \Sigma_0 A_0^\transp + A_0 \Sigma_0 (A_1 - A_0)^\transp \\
&= A_1 \Delta A_1^\transp + M.
\end{aligned}
\end{equation*}
From $M \preccurlyeq \|M\|_2 I$, we obtain that $\Delta \preccurlyeq Z$ where $Z = A_1 Z A_1^\transp + \|M\|_2 I$ by~\cref{p:lyap.comparison} which can be expressed as $Z = \|M\|_2 \Sigma_1$ by~\cref{p:lyap.linearity}. As a result, we have
\begin{equation*}
\Delta \preccurlyeq \|M\|_2 \Sigma_1 \quad \quad \Rightarrow \quad \quad \Sigma_1 \preccurlyeq \frac{1}{1 - \|M\|_2} \Sigma_0,
\end{equation*}
which concludes the proof since $\|M\|_2 \leq c_0^2 \big( \eta^2 + 2 \eta \big)$.
\end{enumerate}
\end{proof}

\subsection{High-probability events}
\label{subsec:p:concentration}
All the derivations provided in~\cref{sec:proof_ofulqplus,sec:proof_laglq} will be conducted under the high-probability events that $\theta_*$ belongs to the confidence ellipsoid associated with the RLS procedure. Formally, we consider the event
\begin{equation*}
E_t = \Big\{ \forall \hspace{1mm} s \leq t, \hspace{1mm} \theta_* \in C_s \Big\},
\end{equation*}
where $\mathcal{C}_t$ is defined in~\cref{eq:conf.interval}. By~\cref{p:concentration} and from $\beta_t = \beta_t(\delta/4)$, $E_T$ holds with probability at least $1-\delta/4$. This result is standard and we report the proofs here for sake of completeness.

\begin{proof}[Proof of~\cref{p:concentration}]
The proof of~\cref{p:concentration} is essentially the same as the one of~\citet{abbasi2011regret} but for the way the \rls is regularized. We provide it for sake of completeness.\\
Using the expression of $\wh{\theta}_t$, one gets:
\begin{equation*}
\begin{aligned}
\wh{\theta}_t &= V_t^{-1} \left( \theta_0 + \sum_{s=0}^{t-1} z_s x_{s+1}^\transp \right)  = V_{t}^{-1} \left( \lambda \theta_0 + \sum_{s=0}^{t-1} z_s z_s^\transp  \theta_* +  \sum_{s=0}^{t-1} z_s \epsilon_{s+1}^\transp \right) = V_{t}^{-1} \left( \lambda \theta_0 + (V_t - \lambda I)  \theta_* +  \sum_{s=0}^{t-1} z_s \epsilon_{s+1}^\transp \right)
\end{aligned}
\end{equation*}
which leads to 
\begin{equation*}
\| \wh{\theta}_t  - \theta_*\|_{V_t} \leq  \lambda \| \theta_0 - \theta_*\|_{V_t^{-1}} + \|S_t\|_{V_{t}^{-1}},
\end{equation*}
where $S_t = \sum_{s=0}^{t-1} z_s \epsilon_{s+1}^\transp$. We decompose $S_t = (S_t^{1},\dots,S_t^{n})$ as
\begin{equation*}
S_t^{i} =  \sum_{s=0}^{t-1} z_s \epsilon_{s+1}^{i}, \hspace{3mm} \forall i \in \{1,\dots, n\}
\end{equation*}
where $\{\epsilon_s^{i}\}_{s\geq 0}$ is a martingale difference sequence, with zero-mean $\sigma$-subGaussian increment. From Cor.~1 in~\citep{abbasi2011online}, for any $0< \delta < 1$, we have with probability at least $1-\delta$:
\begin{equation*}
\forall t\geq 0, \hspace{3mm} \|S_t^{i}\|_{V_{t}^{-1}} \leq  \sigma \sqrt{2 \log \Big( \frac{\det(V_t)^{1/2}}{\det(\lambda I)^{1/2} \delta} \Big)}.
\end{equation*}
Finally, a union bound argument leads to:
\begin{equation*}
\forall t\geq 0, \hspace{3mm} \|S_t\|_{V_{t}^{-1}} \leq  \sigma \sqrt{2n \log \Big( \frac{\det(V_t)^{1/2} n }{\det(\lambda I)^{1/2} \delta} \Big)}.
\end{equation*} which concludes the proof.
\end{proof}

\newpage
\section{Sequential Stability and Regret of \ofulqplus}
\label{sec:proof_ofulqplus}

This section is dedicated to the proofs of~\cref{lem:ofu.dynamic.stability2} and~\cref{p:regret.ofulq.plus}. We first address the sequential stability of \ofulqplus (\cref{ssec:app_proof.stability.ofulqplus}) and then show how it translates in a regret bound that scales polynomially with the dimension and the problem dependent constant $\kappa$ (~\cref{ssec:app_proof.regret.ofulqplus}). All the results in the section are derived under~\cref{asm:noise,asm:good.lqr,asm:stability.margin}.

\textbf{Notations.} We recall the LQR notation used in this section. We consider LQR systems of the form~\cref{eq:lq-linear_dynamic_quadratic_cost}, whose dynamics is characterized by matrices $A\in\Re^{n\times n}$ and $B\in\Re^{n\times d}$ and cost matrix characterized by $C = \begin{pmatrix} Q & 0 \\ 0 & R \end{pmatrix}$. We summarize these parameters as $\theta^\transp = (A\;\; B)\in\Re^{n\times (n+d)}$.\\
 Given an arbitrary control matrix $K\in\Re^{d \times n}$, we denote by $\Ac(\theta,K) = A + BK \in\Re^{n\times n}$ the associated closed-loop matrix. We often use $L^\transp = \big(I \;\; K^\transp\big) \in \Re^{n\times (n+d)}$, so that the closed-loop matrix can be written as $\Ac(\theta,K) = \theta^\transp L$.

Whenever $A^c(\theta,K)$ is stable, we denote as $P(\theta,K)$ the (unique) p.s.d. matrix satisfying the Lyapunov equation (see~\eqref{eq:lyap.stability.margin})
\begin{equation}\label{eq:cost.matrix.def}
P(\theta,K) = \Ac(\theta,K)^\transp P(\theta,K) A^c(\theta,K) + Q + K^\transp R K.
\end{equation}
Notice that $\Tr(P(\theta,K))$ represents the average cost of the LQ problem parametrized with $\theta$ under control $K$. In a similar fashion, we denote as $\Sigma(\theta,K)$ the steady-state covariance satisfying (see~\eqref{eq:lyap.stability.margin.transp})
\begin{equation}\label{eq:ss.matrix.def}
\Sigma(\theta,K) = \Ac(\theta,K) \Sigma(\theta,K) \Ac(\theta,K)^\transp +I.
\end{equation}
Furthermore, whenever $\theta$ is stabilizable, we denote by $K(\theta)$ the associated optimal control and $L(\theta)^\transp = \big(I \; K(\theta)^\transp\big)$. When $\theta$ is controlled by $K(\theta)$ we often use the shorthand notation $\Ac(\theta) = \Ac(\theta, K(\theta))$, $\Sigma(\theta) = \Sigma(\theta, K(\theta))$ and $P(\theta) = P(\theta, K(\theta))$.\\
The parameters of the true system are summarized in $\theta_*^\transp = (A_*\;\; B_*)\in\Re^{n\times (n+d)}$  and we denote as $K_*$, $L_*$, $P_*$ and $\Sigma_*$ the optimal quantities $K(\theta_*)$, $L(\theta_*)$, $P(\theta_*,K_*)$ and $\Sigma(\theta_*,K_*)$ respectively.

\textbf{\ofulqplus.} At each time step $t\geq 0$,  \ofulqplus maintains an estimates $\wh{\theta}_t$ of $\theta_*$ together with a confidence set $\mathcal{C}_t$ as given~\cref{eq:lq-least.square,eq:conf.interval}. \ofulqplus proceeds in episodes, triggered when $\det(V_t) \geq 2 \det(V_{t_k})$, and at the beginning of each episode $k$, \ofulqplus selects an \textit{optimistic} parameters $\theta_{t_k} \in \mathcal{C}_{t_k}$ and computes the optimal controller $K(\theta_{t_k})$ associated with $\theta_{t_k}$. Then, the \textit{optimistic} controller $K(\theta_{t_k})$ is used until the end of the episode. \\
The optimistic nature of $\theta_{t_k}$ and $K(\theta_{t_k})$ ensures that for all $k\geq 0$, $\Tr\big(P(\theta_{t_k},K(\theta_{t_k})) \big) \leq \Tr(P_*) \leq D$ on the high-probability event $\theta_* \in \mathcal{C}_{t_k}$. 

\subsection{Sequential stability of \ofulqplus}
\label{ssec:app_proof.stability.ofulqplus}

The proof of the sequential stability of \ofulqplus is derived on the high-probability event $E_T$ and is conducted in two steps.\\ 
\textbf{Step 1)} We show that \ofulqplus outputs a sequence of stable controller $\{ K(\theta_{t_k})\}_{k\geq 1}$ (i.e., such that $\Ac(\theta_*,K(\theta_{t_k}))$ is UES) as long as $\|\theta_{t_k} - \theta_*\| \leq  2 \epsilon_0$ for $\epsilon_0$ satisfying
\begin{equation*}
\kappa \eta(\epsilon_0) \leq 1;\quad\quad \eta(\epsilon_0) := 4 \kappa \epsilon_0 (1 + \epsilon_0) < 1,
\end{equation*} 
where $\kappa = D / \lambda_{\min}(C)$.\\
\textbf{Step 2)} We prove~\cref{lem:ofu.dynamic.stability2} showing that $\|\theta_{t_k} - \theta_*\| \leq  2 \epsilon_0$ does hold for a suitable initialization phase and choice of regularization parameter $\lambda$, by induction. This allows us to construct recursively a high-probability event over which the state process remains bounded.

 \subsubsection{Step 1}
 
The following proposition summarizes how optimism translates into stability properties.
\begin{proposition}
\label{prop.optimistic.control.properties.stability.ofulqplus}
Assume that $\theta_* \in \mathcal{C}_t$ for all $t\geq 0$. Let$\{\theta_{t_k}\}_{k\geq 1}$ and $\{ K(\theta_{t_k})\}_{k\geq 1}$ be the sequence of optimistic parameters/controllers generated by \ofulqplus over episodes $k\geq0$. Then, under~\cref{asm:good.lqr,asm:stability.margin}, for all $k\geq 0$,
\begin{enumerate}
\item $\theta_{t_k}$ is a stabilizable system and $\Ac(\theta_{t_k}, K(\theta_{t_k}))$ is stable,
\item $\|L(\theta_{t_k})^\transp \|_{\Sigma(\theta_{t_k}, K(\theta_{t_k}))} \leq \sqrt{\kappa}$ and $\|\Sigma(\theta_{t_k}, K(\theta_{t_k}))^{1/2}\|_F \leq \sqrt{\kappa}$, where $\kappa = D / \lambda_{\min}(C)$ and
\begin{equation*}
\Sigma(\theta_{t_k}, K(\theta_{t_k})) = \Ac(\theta_{t_k}, K(\theta_{t_k})) \Sigma(\theta_{t_k}, K(\theta_{t_k}))  \Ac(\theta_{t_k}, K(\theta_{t_k}))^\transp +I.
\end{equation*}
 \item $\|L(\theta_{t_k}) \|_F \leq \sqrt{\kappa}$ and $\| \Ac(\theta_{t_k}, K(\theta_{t_k}))^\transp \|_{\Sigma(\theta_{t_k}, K(\theta_{t_k}))} \leq \sqrt{\kappa}$.
\end{enumerate}
\end{proposition}

\begin{proof}
 \hspace{1mm}
\begin{enumerate}
\item By construction of $\theta_{t_k}$ and $K(\theta_{t_k})$, optimism guarantees stability of the closed-loop matrix $\Ac(\theta_{t_k}, K(\theta_{t_k}))$. Indeed, if $\Ac(\theta_{t_k}, K(\theta_{t_k}))$ was unstable, we would have $\Tr(P(\theta_{t_k}), K(\theta_{t_k}))) = + \infty$, which contradicts the optimistic property $\Tr(P(\theta_{t_k} K(\theta_{t_k}))) \leq \Tr(P_*) \leq D < +\infty$. As a result, $\Ac(\theta_{t_k}, K(\theta_{t_k}))$ is stable and $\theta_{t_k}$ is a stabilizable pair.
\item By definition, we have
\begin{equation*}
\begin{aligned}
P(\theta_{t_k}, K(\theta_{t_k}) &= \Ac(\theta_{t_k}, K(\theta_{t_k}))^\transp P(\theta_{t_k}, K(\theta_{t_k})) \Ac(\theta_{t_k}, K(\theta_{t_k})) + L(\theta_{t_k})^\transp C L(\theta_{t_k}),\\
\Sigma(\theta_{t_k}, K(\theta_{t_k})) &= \Ac(\theta_{t_k}, K(\theta_{t_k})) \Sigma(\theta_{t_k}, K(\theta_{t_k}))  \Ac(\theta_{t_k}, K(\theta_{t_k}))^\transp +I.
\end{aligned}
\end{equation*}
Then, applying~\cref{p:lyap.riccati.trace} with $M =  L(\theta_{t_k})^\transp C L(\theta_{t_k})$ leads to 
\begin{equation*}
D \geq \Tr\big(P(\theta_{t_k}, K(\theta_{t_k}))\big)  = \Tr\big( \Sigma(\theta_{t_k}, K(\theta_{t_k})) M).
\end{equation*}
Noticing that $M \succcurlyeq \lambda_{\min}(C) L(\theta_{t_k})^\transp L(\theta_{t_k})$ we obtain that 
\begin{equation*}
D \geq \Tr\big(P(\theta_{t_k}, K(\theta_{t_k}))\big) \geq \Tr\big( \Sigma(\theta_{t_k}, K(\theta_{t_k})) L(\theta_{t_k})^\transp L(\theta_{t_k}) \big)\lambda_{\min}(C) = \lambda_{\min}(C) \|  L(\theta_{t_k})^\transp \|^2_{\Sigma(\theta_{t_k}, K(\theta_{t_k}))}.
\end{equation*}
Similarly, from $M  \succcurlyeq Q + K(\theta_{t_k})^\transp R K(\theta_{t_k}) \succcurlyeq  Q \succcurlyeq  \lambda_{\min}(C) I$ we obtain
\begin{equation*}
D \geq \Tr\big(P(\theta_{t_k}, K(\theta_{t_k}))\big) \geq \Tr\big( \Sigma(\theta_{t_k}, K(\theta_{t_k})) \big)\lambda_{\min}(C) = \lambda_{\min}(C) \|\Sigma(\theta_{t_k}, K(\theta_{t_k}))^{1/2}\|_F^{2}.
\end{equation*}
We conclude using that $\kappa = D / \lambda_{\min}(C)$.
\item The last assertions come directly from the previous one, noticing that $ \Sigma(\theta_{t_k}, K(\theta_{t_k})) \succcurlyeq I$ and $ \Ac(\theta_{t_k}, K(\theta_{t_k})) \Sigma(\theta_{t_k}, K(\theta_{t_k}))  \Ac(\theta_{t_k}, K(\theta_{t_k}))^\transp \preccurlyeq \Sigma(\theta_{t_k}, K(\theta_{t_k}))$.
\end{enumerate}
\end{proof}

Equipped with those properties, we now show that for any $k\geq 0$, the optimistic controller $K(\theta_{t_k})$ stabilizes the true system parameterized by $\theta_*$, as long as $\theta_{t_k}$ is \textit{sufficiently close} to $\theta_*$ and characterize the stability margin in term of steady-state covariance and cost.
\begin{proposition}
\label{prop.optimistic.control.stabilizing.theta.star}
Let $k\geq 0$ and assume that $\theta_* \in \mathcal{C}_{t_k}$. Let $\theta_{t_k}$ and $K(\theta_{t_k})$ be an optimistic parameter/control pair. Then, if $\|\theta_{t_k} - \theta_*\|_F\leq 2 \epsilon_0$ where $\epsilon_0$ is such that $\eta(\epsilon_0) := 4 \kappa \epsilon_0 (1 + \epsilon_0) < 1$, under~\cref{asm:good.lqr,asm:stability.margin},
\begin{enumerate}
\item $\Ac(\theta_*, K(\theta_{t_k}))$ is stable, 
\item $\Sigma(\theta_*,K(\theta_{t_k})) \preccurlyeq \frac{1}{1 - \eta(\epsilon_0)}  \Sigma(\theta_{t_k},K(\theta_{t_k}))$,
\item $\big\|P\big(\theta_*,K(\theta_{t_k})\big)\big\|_2 \preccurlyeq \big\|P\big(\theta_*,K(\theta_*))\big\|_2  + \frac{D \eta(\epsilon_0)}{1 - \eta(\epsilon_0)}$,
\item $\big\|P\big(\theta_{t_k},K(\theta_{t_k})\big)\|_2 \leq \big\|P\big(\theta_*,K(\theta_*))\big\|_2(1 +  \kappa \eta(\epsilon_0)/2)   + \frac{D \eta(\epsilon_0)}{1 - \eta(\epsilon_0)}(1 +  \kappa \eta(\epsilon_0)/2) $.
\end{enumerate}
\end{proposition}

\begin{proof}
\begin{enumerate}
\item The first two assertions directly follows from~\cref{prop.optimistic.control.properties.stability.ofulqplus,prop:distance.to.stability.lyapunov.metric}. From 
\begin{equation*}
\left\| \left(\Ac(\theta_*, K(\theta_{t_k})) - \Ac(\theta_{t_k},K(\theta_{t_k}))\right)^\transp \right\|_{\Sigma(\theta_{t_k},K(\theta_{t_k}))} \leq \|\theta_* - \theta_{t_k}\|_F \| L(\theta_{t_k})^\transp\|_{\Sigma(\theta_{t_k},K(\theta_{t_k}))} \leq 2 \epsilon_0 \sqrt{\kappa}
\end{equation*}
and 
\begin{equation*}
\left\|  \Ac(\theta_{t_k},K(\theta_{t_k}))^\transp \right\|_{\Sigma(\theta_{t_k},K(\theta_{t_k}))}  \leq \sqrt{\kappa},
\end{equation*}
we can apply~\cref{prop:distance.to.stability.lyapunov.metric} with $A_0 = \Ac(\theta_{t_k},K(\theta_{t_k}))$, $A_1 =\Ac(\theta_*, K(\theta_{t_k}))$, $\eta = 2 \epsilon_0$ and $c_0 = \sqrt{\kappa}$, which leads to the desired result.
\item Let  $\Delta = P\big(\theta_*,K(\theta_{t_k})\big) - P\big(\theta_*,K(\theta_*)\big)$ and $\eta(\epsilon_0) = 4 \kappa \epsilon_0(1 + \epsilon_0)$. We have 
\begin{equation*}
\begin{aligned}
\|\Delta \|_2 = \lambda_{\max}(\Delta) &\leq \Tr\Big( P\big(\theta_*,K(\theta_{t_k})\big)\Big) - \Tr\Big( P\big(\theta_*,K(\theta_*)\big)\Big)\\
&= \Tr\Big( P\big(\theta_*,K(\theta_{t_k})\big)\Big)  - \Tr\Big( P\big(\theta_{t_k},K(\theta_{t_k})\big)\Big)  + \Tr\Big( P\big(\theta_{t_k},K(\theta_{t_k})\big)\Big)- \Tr\Big( P\big(\theta_*,K(\theta_*)\big)\Big)\\
&\leq \Tr\Big( P\big(\theta_*,K(\theta_{t_k})\big)\Big)  - \Tr\Big( P\big(\theta_{t_k},K(\theta_{t_k})\big)\Big),
\end{aligned}
\end{equation*}
where we used that by optimism, $\Tr\big( P(\theta_{t_k},K(\theta_{t_k}))\big)- \Tr\big( P(\theta_*,K(\theta_*))\big) \leq 0$. Using that 
\begin{align*}
\Tr \left( P\big(\theta_*,K(\theta_{t_k})\big)\right) &= \Tr\left( L(\theta_{t_k})^\transp C L(\theta_{t_k}) \Sigma(\theta_*,K(\theta_{t_k})) \right) \\
\Tr \left( P\big(\theta_{t_k},K(\theta_{t_k})\big)\right) &= \Tr\left( L(\theta_{t_k})^\transp C L(\theta_{t_k}) \Sigma(\theta_{t_k},K(\theta_{t_k})) \right)
\end{align*}
we obtain
\begin{equation*}
\lambda_{\max}(\Delta)  \leq \Tr \left( L(\theta_{t_k})^\transp C L(\theta_{t_k}) \big(\Sigma(\theta_*,K(\theta_{t_k})) - \Sigma(\theta_{t_k},K(\theta_{t_k}))\big) \right).
\end{equation*}

Finally, from $\Sigma(\theta_*,K(\theta_{t_k})) - \Sigma(\theta_{t_k},K(\theta_{t_k})) \preccurlyeq \frac{\eta(\epsilon_0)}{1 - \eta(\epsilon_0)} \Sigma(\theta_{t_k},K(\theta_{t_k}))$ we have
\begin{equation*}
\lambda_{\max}(\Delta)  \leq \frac{\eta(\epsilon_0)}{1 - \eta(\epsilon_0)}   \Tr \left( L(\theta_{t_k})^\transp C L(\theta_{t_k}) \Sigma(\theta_{t_k},K(\theta_{t_k})) \right) = \frac{\eta(\epsilon_0)}{1 - \eta(\epsilon_0)}  \Tr\big( P(\theta_{t_k},K(\theta_{t_k})) \big) \leq 
\frac{D \eta(\epsilon_0)}{1 - \eta(\epsilon_0)}.
\end{equation*}

\item Let  $\Delta = P\big(\theta_{t_k},K(\theta_{t_k})\big) - P\big(\theta_*,K(\theta_{t_k})\big)$ and $\eta(\epsilon_0) = 4 \kappa \epsilon_0(1 + \epsilon_0)$. We have 
\begin{equation*}
\begin{aligned}
\Delta  &= \Ac(\theta_{t_k},K_{t_k})^\transp \Delta  \Ac(\theta_{t_k},K_{t_k}) + \Ac(\theta_{t_k},K_{t_k})^\transp P\big(\theta_*,K(\theta_{t_k})\big) \Ac(\theta_{t_k},K_{t_k}) - \Ac(\theta_*,K_{t_k})^\transp P\big(\theta_*,K(\theta_{t_k})\big) \Ac(\theta_*,K_{t_k})  \\
&\leq  \Ac(\theta_{t_k},K_{t_k})^\transp \Delta  \Ac(\theta_{t_k},K_{t_k}) + \epsilon_0^2 \|L_t\|_2^2 \| P\big(\theta_*,K(\theta_{t_k})\big)\|_2 + 2 \epsilon_0 \|L_t\|_2 \| P\big(\theta_*,K(\theta_{t_k})\big)\|_2 \\
&\leq \Ac(\theta_{t_k},K_{t_k})^\transp \Delta  \Ac(\theta_{t_k},K_{t_k}) + (\epsilon_0^2 \kappa + 2 \epsilon_0 \sqrt{\kappa} )  \| P\big(\theta_*,K(\theta_{t_k})\big)\|_2,\\
&\leq \Ac(\theta_{t_k},K_{t_k})^\transp \Delta  \Ac(\theta_{t_k},K_{t_k})  + \frac{1}{2} \eta(\epsilon_0)  \| P\big(\theta_*,K(\theta_{t_k})\big)\|_2.
\end{aligned}
\end{equation*}
which implies that $\|\Delta\|_2 \leq \kappa \eta(\epsilon_0)/2  \| P\big(\theta_*,K(\theta_{t_k})\big)\|_2$. Combined with the previous point, we obtain,
\begin{equation*}
\big\|P\big(\theta_{t_k},K(\theta_{t_k})\big)\big\|_2 \leq \big\|P\big(\theta_*,K(\theta_{t_k})\big)\big\|_2(1 +  \kappa \eta(\epsilon_0)/2) \leq \big\|P\big(\theta_*,K(\theta_*)\big\|_2(1 +  \kappa \eta(\epsilon_0)/2)   + \frac{D \eta(\epsilon_0)}{1 - \eta(\epsilon_0)}(1 +  \kappa \eta(\epsilon_0)/2) 
\end{equation*}
\end{enumerate}
\end{proof}

We conclude \textbf{Step 1)} showing that whenever $\theta_{t_k}$ is \textit{sufficiently close} to $\theta_*$ for all $t\geq 1$, the sequence of optimistic controller $\{K(\theta_{t_k})\}_{k\geq1}$ induces a state process $\{x_t\}_{t\geq 1}$ for each episode $k$ given by
\begin{equation}
\label{eq:deterministic.process.lqr}
x_{t+1} =  \Ac(\theta_*, K(\theta_{t_k}))x_{t} + \epsilon_{t}, \quad \quad \forall t \in [t_k,t_{k+1}-1],
\end{equation}
that is \textit{uniformly exponentially stable}.
\begin{proposition}
\label{prop.sequential.stability.ofulqplus.deterministic}
Assume that $\theta_* \in \mathcal{C}_t$ for all $t\geq 0$. Let $\{\theta_{t_k}\}_{k\geq 1}$ and $\{K(\theta_{t_k})\}_{k\geq 1}$ be the sequence of optimistic parameters/controls pair. Then, if for all $k\geq 0$, $\|\theta_{t_k} - \theta_*\|\leq 2 \epsilon_0$ where $\epsilon_0$ is such that 
\begin{equation}
\label{eq:epsilon_0.cond.deterministic.ofulqplus}
\kappa \eta(\epsilon_0) \leq 1;\quad\quad \eta(\epsilon_0) := 4 \kappa \epsilon_0 (1 + \epsilon_0) \leq 1/2, 
\end{equation} 
under~\cref{asm:noise,asm:good.lqr,asm:stability.margin},
\begin{enumerate}
\item  for all $k$, the state process in~\cref{eq:deterministic.process.lqr} is uniformly exponentially stable, 
\item for all $k$, for any $0 < \delta < 1$, with probability at least $1-\delta$,
\begin{equation}\label{eq:perturbed.process.bound.ofulqplus.episode}
\text{for all } t \in [t_{k}, t_{k+1}-1],\quad 
\begin{aligned}
\|x_t\| &\leq 4 \sigma \sqrt{\kappa \log(1/\delta)} + \sqrt{\kappa}(1 - 1/\kappa)^{t} \|x_{t_k}\| \\
\|x_t\| &\leq 4 \sigma \sqrt{\kappa \log(1/\delta)} +  2 \sqrt{\big\| P(\theta_*,K(\theta_*)) \big\|_2 /\lambda_{\min}(C)}\|x_{t_k}\|.
\end{aligned}
\end{equation}
\end{enumerate}
\end{proposition}
\begin{proof}\hspace{\fill}
\begin{enumerate}
\item By~\cref{prop.optimistic.control.properties.stability.ofulqplus,prop.optimistic.control.stabilizing.theta.star}, it exists $\Sigma_k =\Sigma(\theta_*,K(\theta_{t_k})) $ such that 
\begin{equation*}
\Sigma_k = \Ac(\theta_*, K(\theta_{t_k})) \Sigma_k \Ac(\theta_*, K(\theta_{t_k}))^\transp + I, \quad\quad \Tr(\Sigma_k) \leq \kappa.
\end{equation*}
As a result,~\cref{thm:ues.rugh} ensures that the process in~\cref{eq:deterministic.process.lqr} is \textit{uniformly exponentially stable}. 

\item Noticing that $t_{k+1} = \inf \{t \geq t_k \text{ s.t. } \det(V_t)> 2 \det(V_{t_k}) \}$ is a stopping time, we obtain the first bound in the second assertion directly by applying~\cref{thm:ues.rugh,p:ues.perturbed}. To obtain the second bound, we use~\cref{prop.optimistic.control.stabilizing.theta.star}
\begin{equation*}
\begin{aligned}
\|\Ac(\theta_*,K(\theta_{t_k}))^t\xi\| &\leq \|\Ac(\theta_*,K(\theta_{t_k}))^t \|_2 \|\xi\| \leq \big\| \sqrt{P(\theta_*,K(\theta_{t_k})) \big\|_2 / \lambda_{\min}(C)} \|\xi\| \\
&\leq \sqrt{ \big\| P(\theta_*,K(\theta_*)) \big\|_2/\lambda_{\min}(C) + 1}\|\xi\| \\
&\leq 2 \sqrt{\big\| P(\theta_*,K(\theta_*)) \big\|_2 /\lambda_{\min}(C)} \|\xi\|.
\end{aligned}
\end{equation*}
\end{enumerate}
\end{proof}

\subsubsection{Step 2. Proof of~\cref{lem:ofu.dynamic.stability2}}
Let $X := 20 \sigma \sqrt{\kappa \log(4 T/\delta)} \sqrt{\big\| P(\theta_*,K(\theta_*)) \big\|_2 /\lambda_{\min}(C)}$, we are now ready to construct recursively a sequence of high-probability events
 $$F_t = \Big\{ \forall \hspace{1mm} s \leq t, \hspace{1mm} \|x_{s}\| \leq X\Big\}$$ over which the state process remains bounded. We work on $\{E_t\}_{t\leq T}$ and we proceed by episode on the increasing sequence of probability events $\{ E_{t_k} \cap F_{t_k} \}_{k\geq 1}$.

\begin{proof}[Proof of~\cref{lem:ofu.dynamic.stability2}]
We assume that \ofulqplus is initialized with a stabilizing set $\Theta_0 = \{\theta: \|\theta-\theta_0\| \leq \epsilon_0 \}$ such that $\epsilon_0$ satisfies
\begin{equation*}
\kappa \eta(\epsilon_0^\prime) \leq 1; \quad\quad \eta(\epsilon_0^\prime) := 4 \kappa \epsilon_0^\prime (1 + \epsilon_0^\prime) \leq 1/2; \quad\quad \epsilon_0^\prime = 2 \epsilon_0.
\end{equation*} 
Further, we assume that the \rls estimation in~\cref{eq:lqr.beta-definition} is instantiated with 
\begin{equation*}
\lambda =\frac{2 n \sigma^2}{\epsilon_0^2} \Big( \log(4n / \delta)  + (n+d)  \log\Big( 1 +  \kappa X^2 T \Big) \Big).
\end{equation*}
 Setting $\epsilon_0 \leq \frac{1}{32 \kappa^2}$ ensures that $\eta(\epsilon_0^\prime) \leq 1/2$, $\kappa \eta(\epsilon_0^\prime) \leq 1$ and that $\lambda \geq 4^3 \sigma^2 \kappa^4 \log(4T/\delta) / \log(2)$.\\
 
We now proceed by induction, for sake of simplicity, we assume that $x_0 = 0$. We show that w.h.p. 
\begin{itemize}
\item [\textbf{1)}] $\|x_t\| \leq X := 20 \sigma \sqrt{\kappa \log(4 T/\delta)} \sqrt{\big\| P(\theta_*,K(\theta_*)) \big\|_2 /\lambda_{\min}(C)}$ for all $t \in [t_k, t_{k+1}]$
\item[\textbf{2)}] $\|x_{t_{k+1}}\| \leq 8 \sigma \sqrt{\kappa \log(4T/\delta)}$.
\end{itemize}
\begin{itemize}
\item At $k=0$, $\theta_{t_{0}} \in \mathcal{C}_0$ implies that 
\begin{equation*}
\|\theta_{t_{0}}  - \theta_*\| \leq 2 \beta_{t_{0}}  \big/ \sqrt{\lambda_{\min}(V_{t_{0}} )} = 2 \big( \epsilon_0 + \sigma \sqrt{2n \log(4 n / \delta)} / \sqrt{\lambda} \big) \leq 4 \epsilon_0 :=  2 \epsilon_0^\prime.
\end{equation*}
As a result,~\cref{prop.sequential.stability.ofulqplus.deterministic} guarantees that with probability at least $1 - \delta/4 T$, 
\begin{equation*}
\text{for all } s \leq t_1, \quad \|x_s\| \leq 4 \sigma \sqrt{\kappa \log(4 T/\delta)} + 2 \sqrt{\big\| P(\theta_*,K(\theta_*)) \big\|_2 /\lambda_{\min}(C)}\| x_0\| \leq X \quad \text{ and } \quad \|x_{t_1}\| \leq 8 \sigma \sqrt{\kappa \log(4 T/\delta)}.
\end{equation*}
Hence, $\mathbb{P}( F_{t_1} | E_{t_0}) \geq 1 - \delta/4 T$. 

\item At $k+1$, assume that $E_{t_k} \cap F_{t_k}$ holds. From $\|x_s \| \leq X$ for all $s\leq t_{k+1}$, we obtain that 
\begin{equation*}
\begin{aligned}
\beta_{t_{k+1}} & = \sigma \sqrt{2n \log \Big( \frac{\det(V_{t_{k+1}})^{1/2} 4 n }{\det(\lambda I)^{1/2} \delta} \Big)} + \lambda^{1/2} \|\theta_0 - \theta_*\|_F\\
&\leq  \sigma \sqrt{2n \log(4 n / \delta)  +  2 n (n+d)  \log\Big( 1 + \sup_{s\leq t_{k+1} -1 } \|z_s\|^2 T \Big)} + \lambda^{1/2} \epsilon_0\\
&\leq  \sigma \sqrt{2n \log(4 n / \delta)  +  2 n (n+d)  \log\Big( 1 +  \kappa X^2 T \Big)} + \lambda^{1/2} \epsilon_0\\
&\leq 2 \epsilon_0 \sqrt{\lambda},
\end{aligned}
\end{equation*}
where we used that $z_s = L(\theta_s) x_s$ and $\|L(\theta_s)\|_2 \leq \sqrt{\kappa}$ by~\cref{prop.optimistic.control.properties.stability.ofulqplus}, and the definition of $\lambda$. As a result, $\|\theta_{t_{k+1}}  - \theta_*\| \leq 2 \epsilon_0^\prime$ and we can apply~\cref{prop.sequential.stability.ofulqplus.deterministic}. Using that $\|x_{t_k}\| \leq 8 \sigma \sqrt{\kappa \log(4 T/\delta)}$ we obtain with probability at least $1 - \delta/4 T$,
\begin{equation*}
\begin{aligned}
\text{for all } t_{k} \leq s \leq t_{k+1}, \quad \|x_s\| &\leq 4 \sigma \sqrt{\kappa \log(4 T/\delta)} +  2 \sqrt{\big\| P(\theta_*,K(\theta_*)) \big\|_2 /\lambda_{\min}(C)}  \|x_{t_k}\| \\
&\leq 4 \sigma \sqrt{\kappa \log(4 T/\delta)}(1 + 4 \sqrt{\big\| P(\theta_*,K(\theta_*)) \big\|_2 /\lambda_{\min}(C)} \\
&\leq X.
\end{aligned}
\end{equation*}
Hence, $\mathbb{P}( F_{t_{k+1}} | E_{t_{k}}, F_{t_k}) \geq 1 - \delta/4 T$. Further, on this event,
\begin{equation*}
\|x_{t_{k+1}}\| \leq 4 \sigma \sqrt{\kappa \log(4T/\delta)}(1 +  \sqrt{\kappa} (1 - 1 / \kappa)^{t_{k+1} - t_{k}+1}) \leq 8 \sigma \sqrt{\kappa \log(4T/\delta)}
\end{equation*}
which concludes the induction. To obtain the last inequality, we used that $t_{k+1} - t_{k} + 1 \geq \kappa \log(\kappa)/2$
since otherwise it enters in contradiction with the update rule: suppose that $t_{k+1} - t_{k} + 1 \leq \kappa \log(\kappa) /2$,
\begin{equation*}
\begin{aligned}
\det(V_{t_{k+1}})& \leq \det(V_{t_k}) \prod_{t = t_{k}}^{t_{k+1} - 1} (1 + \| z_t\|^2_{V^{-1}_t}) \leq \det(V_{t_k}) \prod_{t = t_{k}}^{t_{k+1} - 1} \big( 1 + \|x_t\|^2 \kappa / \lambda \big) \\
& \leq \det(V_{t_k}) \exp \Big( \frac{ \kappa}{\lambda} \sum_{t = t_{k}}^{t_{k+1} - 1} \|x_t\|^2 \Big) \leq \det(V_{t_k}) \exp \Big(32 \sigma^2 \kappa^2 \log(4 T/\delta)  \big((t_{k+1} - t_k + 1) + \kappa^2\big) /\lambda \Big)\\
&\leq \det(V_{t_k}) \exp \Big(4^3 \sigma^2 \kappa^4 \log(4 T/\delta) / \lambda \Big) \leq 2\det(V_{t_k}),
\end{aligned}
\end{equation*}
where we used that $\lambda \geq 4^3 \sigma^2 \kappa^4 \log(4 T/\delta) / \log(2)$.
\end{itemize}

Overall, we showed that the sequence of event $F_{t}$ is such that $\mathbb{P}(F_{t_{k+1}}|F_{t_k}, E_{t_k}) \geq 1 - \delta / 4 T$ for all $k\geq 0$. As a result, we have 
\begin{equation*}
\mathbb{P}(F_T \cap E_T) \geq \mathbb{P}(E_T \bigcap \cap_{k\geq0} F_{t_{k}}|F_{t_{k-1}}, E_{t_{k-1}}) \geq  1 - \mathbb{P}(E_T^c) - \sum_{k \geq 0} \mathbb{P} \Big( (F_{t_{k}} | F_{t_{k-1}}, E_{t_{k-1}})^c \big) \geq 1 - \delta/2.
\end{equation*}
\end{proof}

\subsection{Regret bound for \ofulqplus}
\label{ssec:app_proof.regret.ofulqplus}

The proof of~\cref{p:regret.ofulq.plus} follows from carefully using~\cref{lem:ofu.dynamic.stability2} in the analysis of~\citet{abbasi2011regret}. We provide it for sake of completeness. In line with~\citet{abbasi2011regret}, we perform a regret decomposition on the sequence of events $\{ E_t \cap F_t \}_{t\geq0}$ and obtain
\begin{equation}
\label{eq:regret_decomposition.ofulqplus}
\begin{aligned}
R(T) =\sum_{t=0}^T x_t^\transp Q x_t + u_t^\transp R u_t - J_* = &\sum_{t=0}^T J(\theta_{t}) - J_*  \\
&+ \sum_{t=0}^T x_{t}^\transp P_t x_t - \mathbb{E} \big( x_{t+1}^\transp P_{t+1} x_{t+1} | \mathcal{F}_t \big) \\
&+ \sum_{t=0}^T \mathbb{E} \big( x_{t+1}^\transp (P_{t+1} - P_{t} ) x_{t+1} | \mathcal{F}_t \big) \\
&+ \sum_{t=0}^T  \big(\theta_{t}^\transp z_t \big)^\transp P_t \big(\theta_{t}^\transp z_t  \big) - \big( \theta_*^\transp z_t \big)^\transp P_t \big( \theta_*^\transp z_t \big).
\end{aligned}
\end{equation}
where $\theta_t = \theta_{t_k}$ and $P_t = P(\theta_{t_k})$ for all $[t_k, t_{k+1} - 1]$. $\theta_{t_k}$ is the optimistic parameter chosen by \ofulqplus at the beginning of episode $k$. Each regret terms are then bounded separately as in~\citet{abbasi2011regret} but for the last term.

\paragraph{Bounding $R^{opt}(T) = \sum_{t=0}^T J(\theta_{t}) - J_*$.}
Since $\theta_{t}$ is optimistic, on $E_t$, $ J(\theta_{t}) \leq J_*$. As a result, $R^{opt}(T) \leq 0$. Further, by~\cref{prop.optimistic.control.stabilizing.theta.star} and the tuning of $\epsilon_0$, $\|P_t\|_2 = O\big( \|P(\theta_*)\|_2\big)$.

\paragraph{Bounding $R^{mart}(T) =\sum_{t=0}^T x_{t}^\transp P_t x_t - \mathbb{E} \big( x_{t+1}^\transp P_{t+1} x_{t+1} | \mathcal{F}_t \big)$.}
On $E_t \cap F_t$, Lem.~\ref{lem:ofu.dynamic.stability2} guarantees that $\|x_t\| \leq X$ where $X = 20 \sigma \sqrt{\kappa \log(4 T/\delta)} \sqrt{\big\| P(\theta_*,K(\theta_*)) \big\|_2 /\lambda_{\min}(C)}$. Thus, $R^{mart}(T)$ is a bounded martingale sequence and applying Azuma's inequality leads to 
\begin{equation*}
R^{mart}(T) = \wt{O} \big( X^2 \|P(\theta_*)\|_2 \sqrt{T}\big) = \wt{O} \big(\kappa \|P(\theta_*)\|^2_2 \sqrt{T}\big).
\end{equation*}

\paragraph{Bounding $R^{lazy}(T) =\sum_{t=0}^T \mathbb{E} \big( x_{t+1}^\transp (P_{t+1} - P_{t} ) x_{t+1} | \mathcal{F}_t \big)$.}The lazy update scheme used to re-evaluate the controllers $K_t$s, ensures that the number of updates scales logarithmically. Hence, $R^{lazy}(T) = \wt{O}\big(X^2 \|P(\theta_*)\|_2 (n+d) \big)$.

\paragraph{Bounding $R^{pred}(T) =\sum_{t=0}^T  \big(\theta_{t}^\transp z_t \big)^\transp P_t \big( \theta_{t}^\transp z_t \big) - \big( \theta_*^\transp z_t \big)^\transp P_t \big( \theta_*^\transp z_t \big)$.} From the triangular inequality, we have:

\begin{equation*}
\begin{aligned}
R^{pred}(T) = \sum_{t=0}^T \| \theta_{t}^\transp z_t \|^2_{P_t} - \|\theta_*^\transp z_t\|_{P_t}^2 &= \sum_{t=0}^T \big(  \|\theta_{t}^\transp z_t \|_{P_t} - \|\theta_*^\transp z_t\|_{P_t} \big) \big( \|\theta_{t}^\transp z_t\|_{P_t} + \|\theta_*^\transp z_t\|_{P_t} \big) \\
&\leq \sum_{t=0}^T \| (\theta_{t} - \theta_*)^\transp z_t \|_{P_{t}}  \big( \|\theta_{t}^\transp z_t\|_{P_t} + \|\theta_*^\transp z_t\|_{P_t} \big). 
\end{aligned}
\end{equation*}

Further, $\|\theta_{t}^\transp z_t \|_{P_t}  = \|  \Ac(\theta_{t},K(\theta_{t})))x_t \|_{P_t} \leq \|x_t \|_{P_t} \leq \|P_t\|_2^{1/2} X$. Similarly, on $E_t \cap F_t$,
\begin{equation*}
\begin{aligned}
\|\theta_*^\transp z_t\|_{P_t} &\leq  \|\theta_{t}^\transp z_t\|_{P_t}  +  \|  \Ac(\theta_*,K(\theta_{t})) x_t - \Ac(\theta_{t},K(\theta_{t}))x_t  \|_{P_t}   \\
&\leq \|P_t\|_2^{1/2} X + 4 \|P_t\|_2^{1/2} X \sqrt{\kappa} \epsilon_0\\
&\leq 2 X\|P_t\|_2^{1/2}.
\end{aligned}
\end{equation*}
where we used $\|\theta_{t} - \theta_*\| \leq 2 \epsilon_0$ from line 1 to 2, and the condition on $\epsilon_0$ from line 2 to 3. Thus,
\begin{equation*}
R^{pred}(T) \leq 3 X  \|P_t\|_2 \sum_{t=0}^T \| (\theta_{t} - \theta_*)^\transp z_t \| \leq 6 X \|P_t\|_2 \beta_T \sum_{t=1}^T \|z_t\|_{V_t^{-1}}.
\end{equation*}
Finally, noticing that $\kappa X^2 / \lambda =O(1)$, we have by~\cref{eq:pred.with.bounded.z} that $\sum_{t=1}^T \|z_t\|_{V_t^{-1}} = \wt{O} \big( \sqrt{ (n+d) T}\big)$. Further, from $\|P_t\|_2 = O\big( \|P(\theta_*)\|_2\big)$, $X = \wt{O} (\sqrt{\kappa\|P(\theta_*)\|_2}) $ and $\beta_T = \wt{O}( \sqrt{n (n+d)})$, we obtain $R^{pred}(T) = \wt{O}\big( \sqrt{\kappa} \|P(\theta_*)\|^{3/2}_2 (n+d) \sqrt{n} \sqrt{T}\big)$.

\newpage
\section{Optimism, Sequential Stability and Regret of \laglq}
\label{sec:proof_laglq}

This section is dedicated to the proofs of \laglq and concerns the \textit{Extended Optimistic LQR with Relaxed Constraints} introduced in~\cref{sec:perturbation}. We first prove~\cref{lem:optimism} showing that \laglq does use an \textit{optimistic} controller (\cref{ssec:app_optimism.laglq}). We then prove~\cref{lem:bounded.state} showing that it preserves the sequential stability properties of  \ofulqplus (\cref{ssec:app_sequential.stability.laglq}). We finally prove~\cref{thm:regret} ensuring that \laglq suffersa similar regret compared to \ofulqplus (\cref{ssec:app_regret.laglq}). All the results in the section are derived under~\cref{asm:noise,asm:good.lqr,asm:stability.margin}.

\textbf{Notations.} We first recall the notation used in the definition of the extended LQR. The original LQR is associated with the dynamics and cost matrix parametrized respectively by $\theta = (A\;\;B)$ and $C = \begin{pmatrix} Q & 0 \\ 0 & R \end{pmatrix}$. The extended LQR is obtained by adding an additional perturbation control $w$, which induces the extended LQR associated with the dynamics and cost matrix parametrized respectively by
$\wt \theta^\transp = (A \;\; \wt B) \in \Re^{n\times (2n+d)}$, where $\wt B = ( B \;\; I) \in \Re^{n \times (n+d)}$ and $C_{\dagger} = \begin{pmatrix} C & 0 \\0& 0 \end{pmatrix}$. The extended control matrix $\wt K\in\Re^{(n+d)\times n}$ can be decomposed as $\wt K^\transp = (K^{u \transp} \;\; K^{w \transp})$ with the ``standard'' control $u$ and the perturbation control $w$. Following these definitions, it is possible to redefine the notation used for ``standard'' LQRs as
\begin{equation}
\begin{aligned}
\label{eq:extended.notation.app.proof.laglq}
\Ac(\wt\theta, \wt K) &= A + \wt B \wt K = A + B K^{u} + K^{w} \\
P(\wt\theta,\wt K) &= \Ac(\wt\theta,\wt K)^\transp P(\wt \theta,\wt K) A^c(\wt \theta,\wt K) + Q + K^{u \transp} R K^{u},\\
\Sigma(\wt\theta,\wt K) &= \Ac(\wt\theta,\wt K) \Sigma(\wt \theta,\wt K) \Ac(\wt \theta,\wt K)^\transp +I.
\end{aligned}
\end{equation}
From the last equation we also obtain that for any control policy $\wt \pi(x) = \wt K x$, the corresponding average expected reward is $\mathcal{J}_{\tpi}(\wh\theta) = \Tr(P(\wt\theta,\wt K))$. Finally, notice that the extended dynamics $\wt\theta$ is controllable even if the original one $\theta$ is not\footnote{This directly comes from the fact that the extended matrix $\wt B$ is full column rank.}.

\textbf{\laglq.} \laglq follows the same procedure as \ofulqplus and operates through episodes. A new episode starts as soon as $\det(V_t) \geq 2\det(V_{t_k})$, where $t_k$ is the initial starting time of episode $k$. At each time step $t\geq 0$, as for \ofulqplus,  \laglq maintains an estimates $\wh{\theta}_t$ of $\theta_*$ together with a confidence set $\mathcal{C}_t$ as given~\cref{eq:lq-least.square,eq:conf.interval}. At the beginning of each episode, \laglq computes an \textit{optimistic} policy solving the Extended LQR with Relaxed Constraint in~\cref{eq:optimal.avg.cost.extended} using the \dsllq procedure described in~\cref{ssec:solving.lagrange.alg}. Further,~\cref{lem:dual.function.specific} ensures that the \textit{extended optimistic} policy is linear and we denote by $\wt K_t$ the associated extended linear controller.  Finally, the effective controller $K^{u}_t$ is obtained by extracting the control  part of the extended controller $\wt K_t$ and is used until the end of the episode. 

\subsection{Optimism (Proof of Lemma~\ref{lem:optimism})}
\label{ssec:app_optimism.laglq}

In order to prove the lemma, it is sufficient to display a control policy $\wb\pi$ that satisfy the constraint, and such that $\mathcal{J}_{\bpi}(\wh\theta,\beta,V)  \leq J_*$. 

\textbf{Step 1 (optimism).}
Consider a controller $\wb\pi = (\wb\pi^u, \wb\pi^w)$ defined as
\begin{equation*}
\wb\pi^u(x) = K_* x \quad \text{ and } \quad \wb\pi^w(x) = (A_* + B_* K_*) x  - (\wh{A} + \wh{B} K_* )x = (A_* - \wh A)x + (B_*-\wh B) \wb\pi^u(x).
\end{equation*}
Then the dynamics of the corresponding extended LQR reduces to
\begin{equation}
\begin{aligned}
x_{s+1} &= \wh{A} x_s + \wh{B} u_s + w_s + \epsilon_{s+1} \\
&= \wh{A} x_s + \wh{B} \wb\pi^u(x_s) + (A_* - \wh A)x_s + (B_*-\wh B) \wb\pi^u(x_s) \\
&=  (A_* + B_* K_*) x_s + \epsilon_{s+1},
\end{aligned}
\end{equation}
which coincides with the dynamics induced by the optimal controller $K_*$ in the original LQR. Notice that Asm.~\ref{asm:good.lqr} guarantees $A_* + B_* K_*$ to be stable, and hence $\wb\pi$ is a stabilizing policy for the extended system parametrized with $\wh\theta$. As a result, the expected average cost of $\wb\pi$ is $J^* = \mathcal{J}_{\wb\pi}(\wh\theta)$. 

\textbf{Step 2 (constraint $g_{\tpi})$.}
We now need to verify that $\wb\pi$ satisfies the constraints $g_{\wb\pi}(\wh\theta,\beta,V)$. By definition of $\wb\pi$ we have that for any $s \geq 0$, 
\begin{equation}
\begin{aligned}
\| w_s \|= \| (A_* - \wh A)x_s + (B_*-\wh B) u_s\| = \bigg\| (\theta_* -\wh{\theta})^\transp \begin{pmatrix} x_s \\ u_s \end{pmatrix} \bigg\| \leq \beta \bigg\|\begin{pmatrix} x_s \\ u_s \end{pmatrix}\bigg\|_{V^{-1}} = \beta \| z_s \|_{V^{-1}},
\end{aligned}
\end{equation}
where the inequality is obtained from~\cref{p:concentration} under the event $\theta_* \in \mathcal{C}$. As a result, the constraint 
$g_{\wb\pi}(\wh\theta, \beta, V) =  \lim_{S \rightarrow \infty} \frac{1}{S}  \mathbb{E} \big( \sum_{s=0}^S\|w_s \|^2 - \beta \| z_s \|^2_{V^{-1}} \big) \leq 0$ is met.

\subsection{Sequential Stability (Proof of~\cref{lem:bounded.state})}
\label{ssec:app_sequential.stability.laglq}

The proof follows a similar path as the one of~\cref{lem:ofu.dynamic.stability2}. Notice that now, the effective closed-loop matrix which drives the system over the episode $k$ is given by 
\begin{equation}
\Ac(\theta_*,K^{u}_{t_k}) = A_* + B_* K^{u}_{t_k}= \theta_*^\transp L_{t_k}, \quad\text{where}\quad L_{t_k} = \begin{pmatrix} I \\ K^{u}_{t_k} \end{pmatrix},
\end{equation}
while the closed-loop matrix which drives the optimistic extended LQR is given by 
\begin{equation*}
\Ac(\wt \theta_{t_k},\wt K_{t_k}) = \wh{A}_{t_k} + \wh{B}_{t_k} K^{u}_{t_k} + K^{w}_{t_k} = \theta_{t_k}^\transp L_{t_k} = \Ac(\theta_{t_k}, K^{u}_{t_k}), \quad \text{ where } \quad \theta_{t_k}^\transp = \wh{\theta}_{t_k}^\transp + \begin{pmatrix} K^{w}_{t_k} & 0 \end{pmatrix}.
\end{equation*}

As for the proof of~\cref{lem:ofu.dynamic.stability2}, the proof of the sequential stability of \laglq is derived on the high-probability event $E_T$ and is conducted in two steps.\\ 
\textbf{Step 1)} we show that \laglq outputs a sequence of stable controller $\{ K^{u}_{t_k}\}_{k\geq 1}$ (i.e., such that $\Ac(\theta_*,K^{u}_{t_k})$ is UES) as long as $\|\theta_{t_k} - \theta_*\| \leq  2 \epsilon_0$ for $\epsilon_0$ satisfying
\begin{equation*}
 \eta(\epsilon_0) := 4 \kappa \epsilon_0 (1 + \epsilon_0) < 1,
\end{equation*} 
where $\kappa = D / \lambda_{\min}(C)$.\\
\textbf{Step 2)} we prove~\cref{lem:bounded.state} showing that $\|\theta_{t_k} - \theta_*\| \leq  2 \epsilon_0$ does hold for a suitable initialization phase and choice of regularization parameter $\lambda$ by induction. This allows us to construct recursively a high-probability event over which the state process remains bounded.\\

The following proposition summarizes how optimism (of the Extended LQR with Relaxed Constraint) translates into stability properties.
\begin{proposition}
\label{prop.optimistic.control.properties.stability.laglq}
Assume that $\theta_* \in \mathcal{C}_t$ for all $t\geq 0$. Let $\{\wt K_{t_k}\}_{k\geq 1}$ be the sequence of extended optimistic controllers generated by \laglq over episodes $k\geq0$. Let $\{K^{u}_{t_k}\}_{k\geq 1}$ and $\{K^{w}_{t_k}\}_{k\geq 0}$ be the sequence of control and perturbation controllers associated with $\{\wt K_{t_k}\}_{k\geq 1}$ and let $\theta^\transp_{t_k} = \wh{\theta}^\transp_{t_k} + \begin{pmatrix} K^{w}_{t_k}  &0 \end{pmatrix}$. Then, under~\cref{asm:good.lqr,asm:stability.margin}, for all $k\geq0$,
\begin{enumerate}
\item $\theta_{t_k}$ is a stabilizable pair and $\Ac(\theta_{t_k}, K^{u}_{t_k})$ is stable,
\item $\|L_{t_k}^\transp \|_{\Sigma(\theta_{t_k}, K^{u}_{t_k})} \leq \sqrt{\kappa}$ and $\|\Sigma(\theta_{t_k}, K^{u}_{t_k})^{1/2}\|_F \leq \sqrt{\kappa}$, where $\kappa = D / \lambda_{\min}(C)$, $L_{t_k} = \begin{pmatrix} I \\ K_{t_k}^{u} \end{pmatrix}$ and
\begin{equation*}
\Sigma(\theta_{t_k}, K^{u}_{t_k}) = \Ac(\theta_{t_k}, K^{u}_{t_k}) \Sigma(\theta_{t_k}, K^{u}_{t_k})  \Ac(\theta_{t_k}, K^{u}_{t_k})^\transp +I.
\end{equation*}
 \item $\|L_{t_k} \|_F \leq \sqrt{\kappa}$ and $\| \Ac(\theta_{t_k}, K^{u}_{t_k})^\transp \|_{\Sigma(\theta_{t_k}, K^{u}_{t_k})} \leq \sqrt{\kappa}$.
\end{enumerate}
\end{proposition}

\begin{proof}
 \hspace{1mm}
\begin{enumerate}
\item By construction, optimism guarantees that $\Tr\big(P( \wt \theta_{t_k}, \wt K_{t_k})\big)  \leq D < \infty$. As a result, $\Ac( \wt \theta_{t_k}, \wt K_{t_k})$ is stable. Noticing that 
$\Ac( \wt \theta_{t_k}, \wt K_{t_k}) = \wh{A}_{t_k} + \wh{B}_{t_k} K^{u}_{t_k} + K^{w}_{t_k} = \Ac(\theta_{t_k}, K^{u}_{t_k})$ we obtain that $\Ac(\theta_{t_k}, K^{u}_{t_k})$ is stable and hence that
$\theta_{t_k}$ is a stabilizable pair.
\item From $\Ac( \wt \theta_{t_k}, \wt K_{t_k})=  \Ac(\theta_{t_k}, K^{u}_{t_k})$ and $\begin{pmatrix} I \\ \wt{K}_{t_k} \end{pmatrix}^\transp C_{\dagger} \begin{pmatrix} I \\ \wt{K}_{t_k} \end{pmatrix} = L_{t_k}^\transp C L_{t_k}$, we have that
\begin{equation*}
\begin{aligned}
P(\wt\theta_{t_k}, \wt{K}_{t_k}) &=\Ac(\theta_{t_k}, K^{u}_{t_k})^\transp P(\wt\theta_{t_k}, \wt{K}_{t_k}) \Ac(\theta_{t_k}, K^{u}_{t_k}) + L(\theta_{t_k})^\transp C L(\theta_{t_k}),\\
\Sigma(\theta_{t_k}, K^{u}_{t_k}) &=\Ac(\theta_{t_k}, K^{u}_{t_k}) \Sigma(\theta_{t_k}, K^{u}_{t_k}) \Ac(\theta_{t_k}, K^{u}_{t_k})^\transp +I.
\end{aligned}
\end{equation*}
Hence,~\cref{p:lyap.riccati.trace} ensures that
\begin{equation*}
\begin{aligned}
D &\geq \Tr(P(\theta_{t_k}, K^{u}_{t_k}) \geq \lambda_{\min}(C) \|  L_{t_k}^\transp \|^2_{\Sigma(\theta_{t_k}, K^{u}_{t_k})}, \\
D &\geq \Tr(P(\theta_{t_k}, K^{u}_{t_k}) \geq \lambda_{\min}(C) \| \Sigma(\theta_{t_k}, K^{u}_{t_k})^{1/2}\|_F^{2},
\end{aligned}
\end{equation*}
and we conclude using that $\kappa = D / \lambda_{\min}(C)$.
\item The last assertions come directly from the previous one, noticing that $ \Sigma(\theta_{t_k}, K^{u}_{t_k}) \succcurlyeq I$ and $ \Ac(\theta_{t_k}, K^{u}_{t_k})  \Sigma(\theta_{t_k}, K^{u}_{t_k})  \Ac(\theta_{t_k}, K^{u}_{t_k})^\transp \preccurlyeq  \Sigma(\theta_{t_k}, K^{u}_{t_k})$.
\end{enumerate}
\end{proof}

In line with~\cref{lem:ofu.dynamic.stability2}, we now show that for any $k\geq 1$, the controller $K^{u}_{t_k}$ stabilizes the true system parameterized by $\theta_*$, as long as $\theta_{t_k}$ is \textit{sufficiently close} to $\theta_*$ and characterize the stability margin both in term of steady-state covariance and cost (\cref{prop.optimistic.control.stabilizing.theta.star.laglq}). Further, this translates in uniform exponential stability of the induced deterministic state process (\cref{prop.sequential.stability.laglq.deterministic}). The proofs are identical to~\cref{prop.optimistic.control.stabilizing.theta.star,prop.sequential.stability.ofulqplus.deterministic} so we only report the formal statements.

\begin{proposition}
\label{prop.optimistic.control.stabilizing.theta.star.laglq}
Let $k\geq 0$ and assume that $\theta_* \in \mathcal{C}_{t_k}$. Let $\theta_{t_k}$ and $K^{u}_{t_k}$ be an optimistic parameter/controller pair as given in~\cref{prop.optimistic.control.properties.stability.laglq}. Then, if $\|\theta_{t_k} - \theta_*\|_F\leq 2 \epsilon_0$ where $\epsilon_0$ is such that $\eta(\epsilon_0) := 4 \kappa \epsilon_0 (1 + \epsilon_0) < 1$, under~\cref{asm:good.lqr,asm:stability.margin},
\begin{enumerate}
\item $\Ac(\theta_*, K^{u}_{t_k})$ is stable, 
\item $\Sigma(\theta_*,K^{u}_{t_k}) \preccurlyeq \frac{1}{1 - \eta(\epsilon_0)}  \Sigma(\theta_{t_k},K^{u}_{t_k})$,
\item $\big\|P\big(\theta_*,K^{u}_{t_k}\big)\big\|_2 \preccurlyeq \big\|P\big(\theta_*,K(\theta_*)\big)\big\|_2  + \frac{D \eta(\epsilon_0)}{1 - \eta(\epsilon_0)}$,
\item $\big\|P\big(\theta_{t_k},K^{u}_{t_k}\big)\|_2 \leq \big\|P\big(\theta_*,K(\theta_*)\big)\big\|_2(1 +  \kappa \eta(\epsilon_0)/2)   + \frac{D \eta(\epsilon_0)}{1 - \eta(\epsilon_0)}(1 +  \kappa \eta(\epsilon_0)/2) $.
\end{enumerate}
\end{proposition}

\begin{proposition}
\label{prop.sequential.stability.laglq.deterministic}
Assume that $\theta_* \in \mathcal{C}_t$ for all $t\geq 0$.  Let $\{\theta_{t_k}\}_{k\geq1}$ and $\{K^{u}_{t_k}\}_{k\geq 1}$ be the sequences optimistic parameter/controller pair as given in~\cref{prop.optimistic.control.properties.stability.laglq}. Then, if for all $k\geq0$, $\|\theta_{t_k} - \theta_*\|_F\leq 2 \epsilon_0$ where $\epsilon_0$ is such that 
\begin{equation}
\label{eq:epsilon_0.cond.deterministic.laglq}
\kappa \eta(\epsilon_0) \leq 1; \quad\quad\eta(\epsilon_0) := 4 \kappa \epsilon_0 (1 + \epsilon_0) \leq 1/2,
\end{equation} 
under~\cref{asm:noise,asm:good.lqr,asm:stability.margin},
\begin{enumerate}
\item  for all $k$, the state process driven by $\Ac(\theta_*, K^{u}_{t_k})$ is uniformly exponentially stable, 
\item for all $k$, for any $0 < \delta < 1$, with probability at least $1-\delta$,
\begin{equation}\label{eq:perturbed.process.bound.laglq.episode}
\text{for all } t \in [t_{k}, t_{k+1}-1],\quad 
\begin{aligned}
\|x_t\| &\leq 4 \sigma \sqrt{\kappa \log(1/\delta)} + \sqrt{\kappa}(1 - 1/\kappa)^{t} \|x_{t_k}\| \\
\|x_t\| &\leq 4 \sigma \sqrt{\kappa \log(1/\delta)} + 2 \sqrt{\big\| P(\theta_*,K(\theta_*)) \big\|_2 /\lambda_{\min}(C)} \|x_{t_k}\|.
\end{aligned}
\end{equation}
\end{enumerate}
\end{proposition}

\subsubsection{Step 2. Proof of~\cref{lem:bounded.state}} 

Let $X :=  20 \sigma \sqrt{\kappa \log(4 T/\delta)} \sqrt{\big\| P(\theta_*,K(\theta_*)) \big\|_2 /\lambda_{\min}(C)}$, we are now ready to construct recursively a sequence of high-probability events
 $$F_t = \Big\{ \forall \hspace{1mm} s \leq t, \hspace{1mm} \|x_{s}\| \leq X\Big\}$$ over which the state process remains bounded. As we did for~\cref{lem:ofu.dynamic.stability2} in~\cref{ssec:app_proof.regret.ofulqplus}, we work on $\{E_t\}_{t\leq T}$ and we proceed by episode on the increasing sequence of probability events $\{ E_{t_k} \cap F_{t_k} \}_{k\geq 0}$.

\begin{proof}[Proof of~\cref{lem:bounded.state}]
We assume that \laglq is initialized with a stabilizing set $\Theta_0 = \{\theta: \|\theta-\theta_0\| \leq \epsilon_0 \}$ such that $\epsilon_0$ satisfies
\begin{equation*}
 \kappa \eta(\epsilon_0) \leq 1; \quad\quad \eta(\epsilon_0^\prime) := 4 \kappa \epsilon_0^\prime (1 + \epsilon_0^\prime) \leq 1/2; \quad\quad \epsilon_0^\prime = 2 \sqrt{\kappa} \epsilon_0.
\end{equation*} 
Further, we assume that the \rls estimation in~\cref{eq:lqr.beta-definition} is instantiated with 
\begin{equation*}
\lambda =\frac{2 n \sigma^2}{\epsilon_0^2} \Big( \log(4n / \delta)  + (n+d)  \log\Big( 1 +  \kappa X^2 T \Big) \Big).
\end{equation*}
 Setting $\epsilon_0 \leq \frac{1}{64 \kappa^{3/2}}$ ensures that $\eta(\epsilon_0^\prime) \leq 1/2$, $\kappa \eta(\epsilon_0) \leq 1$ and that $\lambda \geq 4^3 \sigma^2 \kappa^4 \log(4T/\delta) / \log(2)$.\\
 
As opposed to the proof of~\cref{lem:ofu.dynamic.stability2}, controlling $\theta_{t_k} - \theta_*$ where $\theta^\transp_{t_k} = \wh{\theta}^\transp_{t_k} + \begin{pmatrix} K^{w}_{t_k}  &0 \end{pmatrix}$ requires not only to guarantee that $\wh{\theta}_{t_k}$ is close to $\theta_*$ but also that $\|K^{w}_{t_k}\|$ is small. While addressing the former can be done in a similar fashion, addressing the later is specific to the Extended LQR with Relaxed Constraint resolution. Let $\tpi_{t_k}$ be the extended policy corresponding to the optimal solution of~\cref{eq:optimal.avg.cost.extended}. By construction, we have that $g_{\tpi_{t_k}}(\wh\theta_{t_k}, \beta_{t_k}, V_{t_k}) \leq 0$. Expanding $\tpi_{t_k} = (\pi^{u}_{t_k},\pi^{w}_{t_k})$ where $\pi^{u}_{t_k}$ and $\pi^{w}_{t_k}$ are the linear policy respectively associated with $K^{u}_{t_k}$ and $K^{w}_{t_k}$, the constraint  $g_{\tpi_{t_k}} \leq 0$ leads to:
\begin{equation*}
 \lim_{S \rightarrow \infty} \frac{1}{S}  \mathbb{E} \big( \sum_{s=0}^S\|w_s \|^2 - \beta_t \| z_s \|^2_{V_t^{-1}} \big) = \Tr\big( K^{w \transp}_{t_k} K^{w}_{t_k} \Sigma(\theta_{t_k},K^{u}_{t_k}) \big) - \beta^2_{t_k} \Tr\big( L_{t_k}^\transp V_{t_k}^{-1} L_{t_k} \Sigma(\theta_{t_k},K^{u}_{t_k})\big) \leq 0,
\end{equation*}
where $L_{t_k}^\transp = ( I, K^{u \transp}_{t_k})$. As a result,~\cref{prop.optimistic.control.properties.stability.laglq} ensures that
\begin{equation}
\label{eq:perturbation.radius.ellipsoid}
\|K^{w}_{t_k}\|_F^2 \leq \kappa \beta^2_{t_k} / \lambda_{\min}(V_{t_k}) \leq  \kappa \beta^2_{t_k} /  \lambda.
\end{equation}

We now proceed by induction, for sake of simplicity, we assume that $x_0 = 0$. We show that w.h.p.
\begin{itemize}
\item [\textbf{1)}] $\|x_t\| \leq X :=  20 \sigma \sqrt{\kappa \log(4 T/\delta)} \sqrt{\big\| P(\theta_*,K(\theta_*)) \big\|_2 /\lambda_{\min}(C)}$
\item [\textbf{2)}] $\|x_{t_{k+1}}\| \leq 8 \sigma \sqrt{\kappa \log(4T/\delta)}$.
\end{itemize}
\begin{itemize}
\item At $k=0$, we have
\begin{equation*}
\begin{aligned}
\|\theta_{t_0} - \theta_*\|  \leq \|\wh{\theta}_{t_0} - \theta_* + K^{w}_{t_0}\| &\leq (1 + \sqrt{\kappa}) \beta_{t_0} \big/ \sqrt{\lambda_{\min}(V_{t_0})}\\
&\leq  (1 + \sqrt{\kappa}) \big( \epsilon_0 + \sigma \sqrt{2n \log(4 n / \delta)} / \sqrt{\lambda} \big)\\
&\leq 4 \sqrt{\kappa} \epsilon_0 :=2 \epsilon_0^\prime.
\end{aligned}
\end{equation*}
As a result,~\cref{prop.sequential.stability.laglq.deterministic} guarantees that with probability at least $1 - \delta/4 T$, 
\begin{equation*}
\text{for all } s \leq t_1, \quad \|x_s\| \leq 4 \sigma \sqrt{\kappa \log(4 T/\delta)} \leq X \quad \text{ and } \quad \|x_{t_1}\| \leq 8 \sigma \sqrt{\kappa \log(4 T/\delta)}.
\end{equation*}
Hence, $\mathbb{P}( F_{t_1} | E_{t_0}) \geq 1 - \delta/4 T$. 

\item At $k+1$, assume that $E_{t_k} \cap F_{t_k}$ holds. From $\|x_s \| \leq X$ for all $s\leq t_{k+1}$, we obtain that 
\begin{equation*}
\begin{aligned}
\beta_{t_{k+1}} & = \sigma \sqrt{2n \log \Big( \frac{\det(V_{t_{k+1}})^{1/2} 4 n }{\det(\lambda I)^{1/2} \delta} \Big)} + \lambda^{1/2} \|\theta_0 - \theta_*\|_F\\
&\leq  \sigma \sqrt{2n \log(4 n / \delta)  +  2 n (n+d)  \log\Big( 1 + \sup_{s\leq t_{k+1} -1 } \|z_s\|^2 T \Big)} + \lambda^{1/2} \epsilon_0\\
&\leq  \sigma \sqrt{2n \log(4 n / \delta)  +  2 n (n+d)  \log\Big( 1 +  \kappa X^2 T \Big)} + \lambda^{1/2} \epsilon_0\\
&\leq 2 \epsilon_0 \sqrt{\lambda},
\end{aligned}
\end{equation*}
where we used that  $z_s = L_s x_s$ and $\|L_s\|_2 \leq \sqrt{\kappa}$ by~\cref{prop.optimistic.control.properties.stability.laglq}, and the definition of $\lambda$. As a result, 
\begin{equation*}
\begin{aligned}
\|\theta_{t_{k+1}} - \theta_*\|  \leq \|\wh{\theta}_{t_{k+1}} - \theta_* + K^{w}_{t_{k+1}}\| &\leq (1 + \sqrt{\kappa}) \beta_{t_{k+1}} \big/ \sqrt{\lambda_{\min}(V_{t_{k+1}})}\\
&\leq  2 (1 + \sqrt{\kappa})  \epsilon_0 \leq 4 \sqrt{\kappa} \epsilon_0 :=2 \epsilon_0^\prime
\end{aligned}
\end{equation*}
and we can apply~\cref{prop.sequential.stability.laglq.deterministic}. Using that $\|x_{t_k}\| \leq 8 \sigma \sqrt{\kappa \log(4 T/\delta)}$ we obtain with probability at least $1 - \delta/4 T$,
\begin{equation*}
\begin{aligned}
\text{for all } t_{k} \leq s \leq t_{k+1}, \quad \|x_s\| &\leq 4 \sigma \sqrt{\kappa \log(4 T/\delta)} +   2 \sqrt{\big\| P(\theta_*,K(\theta_*)) \big\|_2 /\lambda_{\min}(C)} \|x_{t_k}\| \\
&\leq 4 \sigma \sqrt{\kappa \log(4 T/\delta)}(1 + 4 \sqrt{\big\| P(\theta_*,K(\theta_*)) \big\|_2 /\lambda_{\min}(C)}) \\
&\leq X.
\end{aligned}
\end{equation*}
Hence, $\mathbb{P}( F_{t_{k+1}} | E_{t_{k}}, F_{t_k}) \geq 1 - \delta/4 T$. Further, on this event,
\begin{equation*}
\|x_{t_{k+1}}\| \leq 4 \sigma \sqrt{\kappa \log(4T/\delta)}(1 +  \sqrt{\kappa} (1 - 1 / \kappa)^{t_{k+1} - t_{k}+1}) \leq 8 \sigma \sqrt{\kappa \log(4T/\delta)}
\end{equation*}
since $t_{k+1} - t_{k} + 1 \geq \kappa \log(\kappa)/2$ otherwise it enters in contradiction with the update rule and concludes the induction.\\
We are now left to prove that $t_{k+1} - t_{k} + 1 \geq \kappa \log(\kappa)/2$. Suppose that $t_{k+1} - t_{k} + 1 \leq \kappa \log(\kappa) /2$,
\begin{equation*}
\begin{aligned}
\det(V_{t_{k+1}})& \leq \det(V_{t_k}) \prod_{t = t_{k}}^{t_{k+1} - 1} (1 + \| z_t\|^2_{V^{-1}_t}) \leq \det(V_{t_k}) \prod_{t = t_{k}}^{t_{k+1} - 1} \big( 1 + \|x_t\|^2 \kappa / \lambda \big) \\
& \leq \det(V_{t_k}) \exp \Big( \frac{ \kappa}{\lambda} \sum_{t = t_{k}}^{t_{k+1} - 1} \|x_t\|^2 \Big) \leq \det(V_{t_k}) \exp \Big(32 \sigma^2 \kappa^2 \log(4 T/\delta)  \big((t_{k+1} - t_k + 1) + \kappa^2\big) /\lambda \Big)\\
&\leq \det(V_{t_k}) \exp \Big(4^3 \sigma^2 \kappa^4 \log(4 T/\delta) / \lambda \Big) \leq 2\det(V_{t_k}),
\end{aligned}
\end{equation*}
where we used that $\lambda \geq 4^3 \sigma^2 \kappa^4 \log(4 T/\delta) / \log(2)$.
\end{itemize}
Overall, we showed that the sequence of event $F_{t}$ is such that $\mathbb{P}(F_{t_{k+1}}|F_{t_k}, E_{t_k}) \geq 1 - \delta / 4 T$ for all $k\geq 0$. As a result, we have 
\begin{equation*}
\mathbb{P}(F_T \cap E_T) \geq \mathbb{P}(E_T \bigcap \cap_{k\geq0} F_{t_{k}}|F_{t_{k-1}}, E_{t_{k-1}}) \geq  1 - \mathbb{P}(E_T^c) - \sum_{k \geq 0} \mathbb{P} \Big( (F_{t_{k}} | F_{t_{k-1}}, E_{t_{k-1}})^c \big) \geq 1 - \delta/2.
\end{equation*}
\end{proof}

\subsection{Regret bound (Proof of~\cref{thm:regret})}
\label{ssec:app_regret.laglq}

In line with~\citet{abbasi2011regret}, we conduct the analysis on the sequence of events $\{ E_t \cap F_t \}_{t\geq 1}$, where
\begin{equation*}
E_t = \big\{ \forall \hspace{1mm} s \leq t, \hspace{1mm} \theta_* \in \mathcal{C}_s \big\}, \quad\quad F_t = \big\{ \forall \hspace{1mm} s \leq t, \hspace{1mm} \|x_{s}\| \leq X\big\}, \quad \quad X :=  20 \sigma \sqrt{\kappa \log(4 T/\delta)}\sqrt{\| P(\theta_*)\|_2 /\lambda_{\min}(C)},
\end{equation*}
which holds with probability at least $1-\delta/2$ by~\cref{lem:bounded.state}. We further assume for sake of simplicity that at each update, the controller exactly solves the extended problem in~\cref{eq:optimal.avg.cost.extended} since, as discussed in~\cref{ssec:solving.lagrange.alg}, one can solve it with arbitrarily accuracy.

\subsubsection{Regret decomposition}

Let $\{t_k \}_{k=0}^m$ be the $m$ time steps at which the policy is updated. For any $t \in [t_k, t_{k+1}-1]$, let $\wt{K}_t = \wt{K}_{t_k}$ be the extended controller and let $\wt{K}_t = \begin{pmatrix} K^{u \transp}_{t} & K^{w \transp}_{t} \end{pmatrix}^\transp$ be the blocks corresponding to control $u$ and perturbation $w$ respectively. Then, 
\begin{equation}
\label{eq:cost.definition.regret}
\begin{aligned}
&\mathcal{J}_*(\wh\theta_{t_k}, \beta_{t_k}, V_{t_k}) = \Tr\big( P_t \big) \quad \text{ where }\\
&P_t = Q + K^{u \transp}_{t} R K^{u}_{t} + (\wh{A}_{t_k} + \wh{B}_{t_k} K^{u}_{t} + K^{w}_{t})^\transp P_t (\wh{A}_{t_k} + \wh{B}_{t_k}K^{u}_{t} + K^{w}_{t})
\end{aligned}
\end{equation}
which leads to the perturbed Bellman equation
\begin{equation}
\label{eq:bellman.1}
\mathcal{J}_*(\wh\theta_{t_k}, \beta_{t_k}, V_{t_k}) + x_t^\transp P_t x_t = x_t^\transp Q x_t + u_t^\transp R u_t + \mathbb{E}\big(\wt{x}_{t+1}^\transp P_t \wt{x}_{t+1} | \mathcal{F}_{t} \big), \quad \text{ where } \wt{x}_{t+1} = \wh\theta_{t_k}^\transp z_t + w_t + \epsilon_{t+1}.
\end{equation}
Introducing, $x_{t+1} = (A_* + B_* K^{u}_{t} ) x_t + \epsilon_{t+1} = (\theta_* - \wh\theta_{t_k})^\transp z_t - w_t +  \wt{x}_{t+1}$ leads to:
\begin{equation}
\label{eq:bellman.2}
\begin{aligned}
\forall t \in [t_k, t_{k+1}-1], \hspace{3mm} \mathcal{J}_*(\wh\theta_{t_k}, \beta_{t_k}, V_{t_k}) + x_t^\transp P_t x_t &= x_t^\transp Q x_t + u_t^\transp R u_t + \mathbb{E}\big(x_{t+1}^\transp P_t x_{t+1} | \mathcal{F}_{t} \big) \\
& + \big( \wh\theta_{t_k}^\transp z_t + w_t \big)^\transp P_t \big( \wh\theta_{t_k}^\transp z_t + w_t \big) - \big( \theta_*^\transp z_t \big)^\transp P_t \big( \theta_*^\transp z_t \big).
\end{aligned}
\end{equation}
Summing over $[0,T]$, and performing similar manipulation as in~\cite{abbasi2011regret} leads to the regret decomposition:
\begin{equation}
\label{eq:regret_decomposition}
\begin{aligned}
R(T) =\sum_{t=0}^T x_t^\transp Q x_t + u_t^\transp R u_t - J_* = &\sum_{k=0}^m \sum_{t=t_k}^{t_{k+1}-1} \mathcal{J}_*(\wh\theta_{t_k}, \beta_{t_k}, V_{t_k}) - J_*\\
&+ \sum_{t=0}^T x_{t}^\transp P_t x_t - \mathbb{E} \big( x_{t+1}^\transp P_{t+1} x_{t+1} | \mathcal{F}_t \big) \\
&+ \sum_{t=0}^T \mathbb{E} \big( x_{t+1}^\transp (P_{t+1} - P_{t} ) x_{t+1} | \mathcal{F}_t \big) \\
&+  \sum_{k=0}^m \sum_{t=t_k}^{t_{k+1}-1}[\big( \wh\theta_{t_k}^\transp z_t + w_t \big)^\transp P_t \big( \wh\theta_{t_k}^\transp z_t + w_t \big) - \big( \theta_*^\transp z_t \big)^\transp P_t \big( \theta_*^\transp z_t \big).
\end{aligned}
\end{equation}
We now can bound each term separately. While the analysis is similar to the one of~\citet{abbasi2011regret} but for the last term, we provide here a full analysis, for sake of completeness.

\paragraph{Bounding $R^{opt}(T) =  \sum_{k=0}^m \sum_{t=t_k}^{t_{k+1}-1} \mathcal{J}_*(\wh\theta_{t_k}, \beta_{t_k}, V_{t_k}) - J_* $.}
Lem.~\ref{lem:optimism} guarantees that on $E_t$, $ \mathcal{J}_*(\wh\theta_{t_k}, \beta_{t_k}, V_{t_k}) \leq  J_* $. As a result, $R^{opt}(T) \leq 0$. Further, by~\cref{prop.optimistic.control.stabilizing.theta.star.laglq} and the tuning of $\epsilon_0$, $\|P_t\|_2 = O\big( \|P(\theta_*)\|_2\big)$.

\paragraph{Bounding $R^{mart}(T) =\sum_{t=0}^T x_{t}^\transp P_t x_t - \mathbb{E} \big( x_{t+1}^\transp P_{t+1} x_{t+1} | \mathcal{F}_t \big)$.}
On $E_t \cap F_t$, Lem.~\ref{lem:bounded.state} guarantees that $\|x_t\| \leq X$ where $X = 20 \sigma \sqrt{\kappa \log(4 T/\delta)}\sqrt{\| P(\theta_*)\|_2 /\lambda_{\min}(C)}$. Thus, $R^{mart}(T)$ is a bounded martingale sequence and applying Azuma's inequality leads to, with probability at least $1-\delta/4$,
\begin{equation*}
R^{mart}(T) = \wt{O} \big( X^2 \|P(\theta_*)\|_2 \sqrt{T}\big) = \wt{O} \big(\kappa \|P(\theta_*)\|^2_2 \sqrt{T}\big).
\end{equation*}

\paragraph{Bounding $R^{lazy}(T) =\sum_{t=0}^T \mathbb{E} \big( x_{t+1}^\transp (P_{t+1} - P_{t} ) x_{t+1} | \mathcal{F}_t \big)$.} \laglq uses, as in~\cite{abbasi2011regret} a lazy update scheme to re-evaluate the controllers $K_t$s, which translates into $P_t = P_{t+1}$ for all $t \notin \{t_k\}_{k=1}^{m}$ where $t_k = \min \big\{ t \geq t_{k-1} \text{ s.t. } \det(V_t) \leq 2 \det(V_{t_{k-1}}) \big\}$. Thus, on $E_t \cap F_t$, one has $R^{lazy}(T) \leq 2 X^2 D m$. Finally, Lem.~8 in~\cite{abbasi2011regret} ensures that the number of updates scales logarithmically,
\begin{equation}
\label{eq:bounded.number.updates}
m \leq (n+d) \log_2\big( 1 + T X^2 \kappa/\lambda \big).
\end{equation}
which implies that $R^{lazy}(T) = \wt{O}\big(X^2 \|P(\theta_*)\|_2 (n+d) \big) = \wt{O}\big(\kappa \|P(\theta_*)\|^2_2 (n+d) \big)$.

\paragraph{Bounding $R^{pred}(T) =\sum_{k=0}^m \sum_{t=t_k}^{t_{k+1}-1} \big( \wh\theta_{t_k}^\transp z_t + w_t \big)^\transp P_t \big( \wh\theta_{t_k}^\transp z_t + w_t \big) - \big( \theta_*^\transp z_t \big)^\transp P_t \big( \theta_*^\transp z_t \big)$.} Using $\|P_t\| \leq \|P(\theta_*)\|_2$ together with the triangular inequality leads to:
\begin{equation*}
\begin{aligned}
R^{pred}(T) &= \sum_{k=0}^m \sum_{t=t_k}^{t_{k+1}-1}\| \wh\theta_{t_k}^\transp z_t + w_t\|^2_{P_t} - \|\theta_*^\transp z_t\|_{P_t}^2 = \sum_{k=0}^m \sum_{t=t_k}^{t_{k+1}-1} \big(  \| \wh\theta_{t_k}^\transp z_t + w_t \|_{P_t} - \|\theta_*^\transp z_t\|_{P_t} \big) \big( \| \wh\theta_{t_k}^\transp z_t  + w_t\|_{P_t} + \|\theta_*^\transp z_t\|_{P_t} \big)\\
&\leq  \|P(\theta_*)\|^{1/2}_2\sum_{k=0}^m \sum_{t=t_k}^{t_{k+1}-1} \| (\wh\theta_{t_k} - \theta_*)^\transp z_t + w_t\| \big( \| \wh\theta_{t_k}^\transp z_t  + w_t\|_{P_t} + \|\theta_*^\transp z_t\|_{P_t} \big).
\end{aligned}
\end{equation*}
Further, from~\cref{eq:cost.definition.regret}, for all $t \in [t_k, t_{k+1}-1]$, 
$$\| \wh\theta_{t_k}^\transp z_t  + w_t\|_{P_t}  = \| (\wh{A}_{t_k} + \wh{B}_{t_k} K_{u,t} + K_{w,t})x_t \|_{P_t} \leq \|x_t \|_{P_t} \leq \|P(\theta_*)\|^{1/2}_2 X.$$
Similarly, on $E_t \cap F_t$,
\begin{equation*}
\begin{aligned}
\|\theta_*^\transp z_t\|_{P_t} &\leq  \| \wh\theta_{t_k}^\transp z_t  + w_t\|_{P_t}  +  \| (A_* + B_* K_{u,t}) x_t - (\wh{A}_{t_k} + \wh{B}_{t_k} K_{u,t} + K_{w,t})x_t  \|_{P_t}   \\
&\leq \|P(\theta_*)\|^{1/2}_2 X + 4 \|P(\theta_*)\|^{1/2}_2 X \kappa\epsilon_0\\
&\leq 2 X \|P(\theta_*)\|^{1/2}_2.
\end{aligned}
\end{equation*}
where we used $\|\theta_{t_k} - \theta_*\| \leq 4 \sqrt{\kappa} \epsilon_0$ from line 1 to 2, and the condition on $\epsilon_0$ from line 2 to 3. Thus,
\begin{equation*}
R^{pred}(T)\leq 3 X \|P(\theta_*)\|_2 \bigg( \underbrace{\sum_{k=0}^m \sum_{t=t_k}^{t_{k+1}-1} \| (\wh\theta_{t_k} - \theta_*)^\transp z_t\|}_{:=R^{pred}_u(T)}  + \underbrace{\sum_{t=0}^T \|w_t\|}_{:=R^{pred}_w(T)} \bigg).
\end{equation*}

\subsubsection{Bounding the cumulative perturbations}
\label{app:sssec.cumulative.perturbations}

\paragraph{Bounding $R^{pred}_u(T)$.} 

In line with~\cite{abbasi2011regret} (Lem.~12), from Prop.~\ref{p:concentration}, one has for all $t \in [t_k , t_{k+1}-1]$, on $E_t$, 
\begin{equation*}
\| (\wh\theta_{t_k} - \theta_*)^\transp z_t\| = \| (\wh\theta_{t_k} - \theta_*)^\transp V_{t_k}^{1/2} V_{t_k}^{-1/2} z_t \| \leq \beta_{t_k} \|z_t\|_{V_{t_k}^{-1}}  \leq 2 \beta_{t_k} \|z_t\|_{V_t^{-1}}
\end{equation*}
which implies, using Cauchy-Schwarz and Prop.~\ref{p:lq-self_normalized_determinant}, $R^{pred}_u(T) \leq 2 \beta_T \sqrt{T}\left( \sum_{t=0}^T \|z_t\|^2_{V_t^{-1}}\right)^{1/2}$.
\begin{equation*}
\begin{aligned}
R^{pred}_u(T) &\leq 2 \beta_T \sqrt{T}\left( \sum_{t=0}^T \|z_t\|^2_{V_t^{-1}}\right)^{1/2} \leq 2 \beta_T   \gamma_u \sqrt{T}\\
 \text{ where } &\gamma_u =   \left( 2 (n+d) \log \left(1 + T \right)\right)^{1/2}.
\end{aligned}
\end{equation*}

\paragraph{Bounding $R^{pred}_w(T)$.}

So far the regret analysis, that relied mostly on optimism and Prop.~\ref{p:concentration}, was very similar than the one of~\citet{abbasi2011regret}. The following Lemma which ensures the perturbation $w_t$ to be cumulatively bounded is specific to the new perturbed approach presented in \laglq.

\begin{lemma}
\label{lem:bound.cumulative.perturbation}
On event $E_t \cap F_t$, with probability at least $1 - \delta/4$, $R_w^{pred}(T) = \wt{O} \big(\sqrt{\kappa}X \beta_T \gamma_u \sqrt{T} \big)$.
\end{lemma}
\begin{proof}
As discussed in Sec.~\ref{sec:perturbation}, if the perturbation $w_t$ were derived from an optimistic $\theta_{t_k} \in C_{t_k}$, it would translate into the constraint $\|w_t \| \leq \beta_{t_k} \|z_t\|_{V_{t_k}^{-1}}$ and~\cref{lem:bound.cumulative.perturbation} would trivially hold. Unfortunately, this constraint cannot be enforced as it does not lead to a \textit{feasible} extended LQR problem. To overcome this issue, we used in problem~\eqref{eq:optimal.avg.cost.extended} the relaxed constraint in Eq.~\ref{eq:constraint}. Formally, we only guarantee that 
\begin{equation*}
g_{\tpi}(\wh\theta_{t_k}, \beta_{t_k}, V_{t_k}) =  \lim_{S \rightarrow \infty} \frac{1}{S}  \mathbb{E} \Big( \sum_{s=0}^S\|w_s \|^2 - \beta_{t_k}^2 \| z_s \|^2_{V_{t_k}^{-1}} \Big) \leq 0,
\end{equation*}
where $\{ x_s\}_{s\geq 0}$ follows the extended dynamic in Eq.~\ref{eq:perturbed.lqr2} parametrized by $\wh\theta_{t_k}$. This equivalently translates in a constraint on $K^{u}_{t_k}$ and $K^{w}_{t_k}$, i.e.,
\begin{equation}
\label{eq:laglq.average.constraint.control.form}
\Tr\big( K^{w \transp}_{t_k} K^{w}_{t_k} \Sigma(\theta_{t_k},K^{u}_{t_k})\big) \leq \beta^2_{t_k} \Tr\big( L_{t_k}^\transp V_{t_k}^{-1} L_{t_k} \Sigma(\theta_{t_k},K^{u}_{t_k})\big),
\end{equation}
where $L_{t_k} = \begin{pmatrix} I \\ K^{u}_{t_k} \end{pmatrix}$ and $\Sigma(\theta_{t_k},K^{u}_{t_k})$ is the steady-state variance of the process $\{x_s\}_{s\geq 0}$. Further,\cref{prop.optimistic.control.properties.stability.laglq}  guarantees that
$I \preccurlyeq\Sigma(\theta_{t_k},K^{u}_{t_k})$ and $\Tr\Big(\Sigma(\theta_{t_k},K^{u}_{t_k})\big) \leq \kappa$ . As a result, 
\begin{equation*}
\Tr( K^{w \transp}_{t_k} K^{w}_{t_k}) \leq \kappa  \beta_{t_k}^2 \lambda_{\max}( L_{t_k}^\transp V_{t_k}^{-1} L_{t_k}) \quad \Rightarrow \quad \|K^{w}_{t_k}\|_2 \leq \sqrt{\kappa} \beta_{t_k} \|V_{t_k}^{-1/2} L_{t_k}\|_2
\end{equation*}
As a result, one gets:
\begin{equation*}
R^{pred}_w(T) \leq X \sum_{t=0}^T \| K^{w}_{t} \|_2 \leq X \sum_{k=0}^m \sum_{t=t_k}^{t_{k+1}-1} \| K^{w}_{t_k} \|_2 \leq \sqrt{\kappa} \beta_T X \sum_{k=0}^m \sum_{t=t_k}^{t_{k+1}-1} \|V_{t_k}^{-1/2} L_{t_k}\|_2.
\end{equation*} 
At a high-level, proving Lem.~\ref{lem:bound.cumulative.perturbation} turns into proving that the cumulative sum of weighted \textit{controllers} $\sum_{t=0}^T \|L_{t}\|_{V_t^{-1}}$ is bounded by $\sqrt{T}$, whereas Prop.~\ref{p:lq-self_normalized_determinant} only provides a bound for the cumulative sum of weighted \textit{controls} $\sum_{t=0}^T \|z_t \|_{V_t^{-1}} = \sum_{t=0}^T \|L_{t}x_t\|_{V_t^{-1}}$. However, as hinted in~\cite{abeille2018improved}, those latter quantity is related to the exploration performed by covariates $z_t = L_t x_t$, which are driven by $L_t$. Intuitively, since, on $F_{t-1}$, $\mathbb{V}(x_t | \mathcal{F}_{t-1}) \geq I$, $x_t$ covers on average all directions, turning into an accurate exploration on every directions of $L_t$.  Formally, for any $t \in (t_k, t_{k+1}-1]$ (i.e., for every time steps strictly within an episode), $V_{t_k}$, $L_{t_k}$ are $\mathcal{F}_{t-1}$ measurable and so are $\lambda_{\max}( L_{t_k}^\transp V_{t_k}^{-1} L_{t_k})$ and its associated eigenvector $v_{t_k}^{\max}$. As a result, we have
\begin{equation}
\label{eq:bounding.cumulative.perturb.max.eigenvalue}
\|L_{t_k} x_{t} \|_{V_{t_k}^{-1}} \geq \|V_{t_k}^{-1/2} L_{t_k} v_{t_k}^{\max} x_t^\transp v_{t_k}^{\max} \|  \geq \|V_{t_k}^{-1/2} L_{t_k}\|_2 |  x_t^\transp v_{t_k}^{\max}|.
\end{equation}
Without loss of generality, we can assume that $x^\transp_{t-1} \Ac(\theta_*, K^{u}_{t_k})^\transp v_{t_k} \geq 0$ and hence, 
\begin{equation*}
|  x_t^\transp v_{t_k}^{\max}| \geq |  x_t^\transp v_{t_k}^{\max}| \I_{\epsilon_{t}^\transp v_{t_k}^{\max}  \geq 0}  \geq \epsilon_{t} \I_{\epsilon_{t}^\transp v_{t_k}^{\max}  \geq 0}.
\end{equation*}
Applying~\cref{property.folded.subgaussian.expectation} to $\epsilon_{t}^\transp v_{t_k}^{\max}$ and taking the expectation in~\cref{eq:bounding.cumulative.perturb.max.eigenvalue} leads to 
\begin{equation*}
 \|L_{t_k}\|_{V_{t_k}^{-1}} \leq 64 \sigma^3 \mathbb{E} \big( \|L_{t_k} x_{t} \|_{V_{t_k}^{-1}}  | \mathcal{F}_{t-1}, F_{t-1} \big).
\end{equation*}

Further, we have to guarantee that the state $x_t$ remains bounded, and hence work under $F_t = \{ s \leq t \text{ s.t. } \|x_s\| \leq X \}$ (notice that the conditioning only ensures that we work under $F_{t-1}$). To do so, we use~\cref{property.bounded.tail.state} and obtain:
\begin{equation*}
\begin{aligned}
 \|L_{t_k}\|_{V_{t_k}^{-1}} &\leq 64 \sigma^3 \mathbb{E} \big( \|L_{t_k} x_{t} \|_{V_{t_k}^{-1}} \I_{\|x_t \leq X}  | \mathcal{F}_{t-1}, F_{t-1} \big) + 64 \sigma^3 \sqrt{\kappa / \lambda} \mathbb{E} \big( \| x_{t} \| \I_{\|x_t \geq X}  | \mathcal{F}_{t-1}, F_{t-1} \big)\\
 & \leq 64 \sigma^3 \mathbb{E} \big( \|L_{t_k} x_{t} \|_{V_{t_k}^{-1}} \I_{\|x_t \leq X}  | \mathcal{F}_{t-1}, F_{t-1} \big) +
\sigma^3 \delta  /T.
\end{aligned}
\end{equation*}
Notice that $\{ t = t_k \}$ is a condition on $V_t$ and hence is $\mathcal{F}_{t-1}$ measurable. Thus, 
\begin{equation*}
\begin{aligned}
\sum_{k=0}^m \sum_{t=t_k}^{t_{k+1}-1} \|L_{t_k}\|_{V_{t_k}^{-1}} &=\sum_{k=0}^m \sum_{t=t_k+1}^{t_{k+1}-1} \|L_{t_k}\|_{V_{t_k}^{-1}} + \sum_{k\geq0} \|L_{t_k}\|_{V_{t_k}^{-1}} \\
&=\sum_{k=0}^m \sum_{t=t_k+1}^{t_{k+1}-1} \|L_{t_k}\|_{V_{t_k}^{-1}} + \sqrt{\kappa/ \lambda} m\\
&\leq 64 \sigma^3\sum_{k=0}^m \sum_{t=t_k+1}^{t_{k+1}-1}  \mathbb{E} \big( \|L_{t_k} x_{t} \|_{V_{t_k}^{-1}} \I_{\|x_t \leq X}  | \mathcal{F}_{t-1}, F_{t-1} \big) + \sigma^3 \delta + \sqrt{\kappa/ \lambda} m\\
&\leq 128 \sigma^3 \sum_{t=0}^T  \mathbb{E} \big( \|z_t \|_{V_{t}^{-1}} \I_{\|x_t \leq X}  | \mathcal{F}_{t-1}, F_{t-1} \big) \I_{t \in [t_k+1, t_{k+1}-1]}  + \sigma^3 \delta + \sqrt{\kappa/ \lambda} m\\
\end{aligned}
\end{equation*}
Finally, using Azuma's inequality, with probability at least $1-\delta/4$
\begin{equation*}
\sum_{t=0}^T\mathbb{E} \big( \|z_t\|_{V_{t}^{-1}} \I_{\|x_t\| \leq X} | \mathcal{F}_{t-1}\big) \leq \sum_{t=0}^T \|z_t\|_{V_{t}^{-1}} + \sqrt{ \log(4/\delta) T},
\end{equation*}
We conclude using Cauchy-Schwarz and Prop.~\ref{p:lq-self_normalized_determinant},
\begin{equation*}
\sum_{t=0}^T\mathbb{E} \big( \|z_t\|_{V_{t}^{-1}} \I_{\|x_t\| \leq X} | \mathcal{F}_{t-1}\big) \leq  \gamma_u \sqrt{T} + \sqrt{ \log(4/\delta) T}.
\end{equation*}
Finally, we have
\begin{equation*}
R_w^{pred}(T) = \wt{O} \big(\sqrt{\kappa}X \beta_T \gamma_u \sqrt{T} \big).
\end{equation*}
\end{proof}

\paragraph{Bounding $R^{pred}(T)$.} Summarizing the bounds on $R^{pred}_{u}(T)$ and $R^{pred}_{w}(T)$ we obtain that with probability at least $1 - \delta/4$, 
\begin{equation*}
R^{pred}(T) \leq 3 X \|P(\theta_*)\|_2 \big( R^{pred}_{u}(T) + R^{pred}_{w}(T)\big) = \wt{O} \Big( (1+\sqrt{\kappa} X) X \|P(\theta_*)\|_2 \beta_T \gamma_u \sqrt{T} \Big).
\end{equation*}
From $\beta_T = \wt{O} \big(\sqrt{n(n+d)} \big)$, $X = \wt{O} \big(\sqrt{\kappa \|P(\theta_*)\|_2} \big)$ and $\gamma_u = \wt{O}\big( \sqrt{(n+d)}\big)$ we obtain,
\begin{equation*}
R^{pred}(T) = \wt{O} \big( (n+d) \sqrt{n} \kappa^{3/2} \|P(\theta_*)\|^{2}_2 \sqrt{T} \big).
\end{equation*}

\subsubsection{Putting everything together}

Since the regret decomposition and the analysis were derived on $\{E_t \cap F_t\}_{t\geq 1}$, we have that with probability at least $1 - \delta/2$, 
\begin{equation*}
R(T) \leq R^{opt}(T) + R^{mart}(T) + R^{lazy}(T) + R^{pred}(T).
\end{equation*}
Further, on $\{E_t \cap F_t\}_{t\geq 1}$, $R^{opt}(T) \leq 0$ and $R^{lazy}(T) = \wt{O}\big(\kappa \|P(\theta_*)\|^2_2 (n+d) \big)$ while $R^{mart}(T) = \wt{O} \big(\kappa \|P(\theta_*)\|^2_2 \sqrt{T}\big)$ with probability at least $1 - \delta/4$ and $R^{pred}(T) = \wt{O} \big( (n+d) \sqrt{n} \kappa^{3/2} \|P(\theta_*)\|_2^2 \sqrt{T} \big)$ with probability at least $1 - \delta/4$. As a result, a union bound argument ensures that we have, with probability at least $1- \delta$, 
\begin{equation*}
R(T) = \wt{O} \big( (n+d) \sqrt{n} \kappa^{3/2} \|P(\theta_*)\|_2^2 \sqrt{T} \big).
\end{equation*}

\subsubsection{Technical results}
\label{ssec:proof.technical.results}
\begin{property} 
\label{property.folded.subgaussian.expectation}
Let $\epsilon$ be a zero-mean, $\sigma-$subGaussian random variable such that $\mathbb{V}(\epsilon) = 1$. Then, 
\begin{equation*}
\mathbb{E} \big( \epsilon \I_{\epsilon \geq 0} \big) = - \mathbb{E}\big( -\epsilon \I_{\epsilon \leq 0}\big) \geq 1/ \big( 64 \sigma^3\big).
\end{equation*}
\end{property}
\begin{proof}[Proof of~\cref{property.folded.subgaussian.expectation}]
From $\mathbb{E}(\epsilon) = 0$, we obtain that 
\begin{equation*}
0 = \mathbb{E} \big( \epsilon \I_{\epsilon \geq 0} \big) + \mathbb{E}\big( -\epsilon \I_{\epsilon \leq 0}\big),
\end{equation*}
which provides the l.h.s equality. In particular, this implies that $ \mathbb{E} \big( |\epsilon| \big) = 2  \mathbb{E} \big( \epsilon \I_{\epsilon \geq 0} \big)$.
Further, using Holder's inequality one has:
\begin{equation*}
1 = \mathbb{E}\big( \epsilon^2 \big)^2 \leq \mathbb{E}\big( |\epsilon|^3 \big) \mathbb{E}( | \epsilon| \big) \leq \mathbb{E}\big(|\epsilon|^4 \big)^{3/4}  \mathbb{E}( | \epsilon| \big) \leq \big( 32 \sigma^4\big)^{3/4} \mathbb{E}( | \epsilon| \big),
\end{equation*}
where we used the properties of subGaussian r.v. to obtain the last inequality. As a result, 
\begin{equation*}
\mathbb{E} \big( \epsilon \I_{\epsilon \geq 0} \big)  \geq \frac{1}{64  \sigma^3}.
\end{equation*}
\end{proof}

\begin{property} 
\label{property.bounded.tail.state}
On $F_{t-1}$, one has that
\begin{equation*}
\mathbb{E} \big( \|x_t\| \I_{\|x_t\| \geq X} |\mathcal{F}_{t-1} \big)\leq \frac{\delta X } {2T}
\end{equation*}
\end{property}
\begin{proof}[Proof of~\cref{property.bounded.tail.state}]
The proof borrows many steps in the proof of~\cref{p:ues.perturbed}, that we omit for sake of readability.
Let $f_{\|x_t\|}$ and $\bar{F}_{\|x_t\|}$ be respectively the pdf and complementary cdf of $\|x_t\|$ conditionally to $\mathcal{F}_{t-1}$, $F_{t-1}$. Then, 
\begin{equation*}
\mathbb{E} \big( \|x_t\| \I_{\|x_t\| \geq X} |\mathcal{F}_{t-1} \big) = \int_{X}^{+\infty} u f_{\|x_t\|}(u) du = X \bar{F}_{\|x_t\|}(X) + \int_{X}^{+\infty} \bar{F}_{\|x_t\|}(u) du.
\end{equation*}
From~\cref{p:ues.perturbed}, we have that for all $u \geq X$, $\bar{F}_{\|x_t\|}(u) \leq e^{ - \big(u / 12 \sigma \kappa  \big)^2}$. Thus, $X \bar{F}_{\|x_t\|}(X) \leq \frac{\delta X } {4T}$ and 
\begin{equation*}
 \int_{X}^{+\infty} \bar{F}_{\|x_t\|}(u) du \leq \frac{ (12 \sigma \kappa)^2 }{2 X } \int_{X}^{+\infty} -\Big[e^{ - \big(u / 12 \sigma \kappa  \big)^2}\Big]^\prime du \leq \frac{ (12 \sigma \kappa)^2 }{2 X } e^{ - \big(X / 12 \sigma \kappa  \big)^2} \leq \frac{\delta X } {4T}.
\end{equation*}
\end{proof}

\newpage

\section{Lagrangian LQR and Strong Duality}\label{sec:app.proof.duality}

In this section, we study the constrained and Lagrangian LQR problem and in particular focus on proving strong duality. 

\textbf{Set of constrained LQR system.} We summarize the parametrization of constrained LQR systems in $\sys = (A,\wt{B}, C_\dagger, C_g)$ where $A$ and $\wt{B}$ are the parameters of the dynamic, $C_\dagger$ is the matrix parametrizing the original cost function and $C_g$ the matrix parametrizing the constraints. We focus on LQR system which exhibit a certain structure, and denote as $\Sys$ such set of system. Formally, 
\begin{definition}
\label{def:admissible.extended.constrained.system}
$\Sys$ is the set of system $\sys = (A,\wt{B}, C_\dagger, C_g)$ which has the following structure:
\begin{equation*}
\begin{aligned}
& A \in \mathbb{R}^{n \times n}, \wt{B} = (B; I_n), B \in \mathbb{R}^{n\times d}, \\
& C_\dagger = \begin{pmatrix} C_\dagger^1 & 0 \\ 0 & 0 \end{pmatrix} + C_\dagger^2,  \quad C_\dagger^{1} \in \mathbb{S}^{n+d}_{++}; \quad C_\dagger^2 \in \mathbb{S}^{2n+d}_{+}; \\
 &C_g = \begin{pmatrix} -V^{-1} & 0 \\ 0 & I_n \end{pmatrix}, V^{-1} \in \mathbb{S}^{n+d}_{++},
 \end{aligned}
\end{equation*}
and for which it exists at least a feasible stable linear policy $\tpi$ characterized by a linear controller $\wt{K}$, i.e., such that $\rho(\Ac(\wt{K})) < 1$ and $g_{\tpi}(\sys) \leq 0$, where $\Ac(\wt{K}) = A + \wt{B} \wt{K}$ and $g_{\tpi}(\sys)$ is given in~\cref{eq:constraint.generic}.
\end{definition}
The reason why we study strong-duality for a larger set of system $\Sys$ is motivated by the fact that we may have to slightly modify the original constrained LQR problem in~\cref{eq:mapping.original.constrained.lqr.into.Sys} to enforce numerical stability of the algorithmic resolution.\\
Notice that for all $t\geq 0$, the constrained LQR problem \laglq has to solve is indeed contained in $\Sys$ whenever $\theta_* \in \mathcal{C}_t$. At each time $t\geq0$, it corresponds to the parametrization
\begin{equation}
\label{eq:mapping.original.constrained.lqr.into.Sys}
C_\dagger^{1} = \begin{pmatrix} Q & 0 \\ 0 & R  \end{pmatrix}; \quad C_\dagger^{2} = 0; \quad C_g = \begin{pmatrix} - \beta_t^2 V_t^{-1} & 0 \\ 0 & I \end{pmatrix}; \quad A = \wh{A}_t; \quad \wt{B} = ( \wh{B}_t; I),
\end{equation}
which belongs to $\Sys$ since $\beta_t^2 V_t^{-1}  \in \mathbb{S}^{n+d}_{++}$ and $\begin{pmatrix} Q & 0 \\ 0 & R  \end{pmatrix} \in \mathbb{S}^{n+d}_{++}$ under~\cref{asm:good.lqr}. Further, it exists a feasible stable linear policy from~\cref{lem:optimism}\footnote{The feasibility condition actually holds for almost every parametrization. Indeed, a system $\sys$ with the appropriate structure is feasible as soon as $(A,B)$ is stabilizable. Since the set of controllable pairs (and hence stabilizable pairs) is open and dense, this is the case for almost every systems.}. As a consequence, the statements we prove in~\cref{sec:app.proof.lem:dual.function,sec:app.proof.duality.final} are slightly more general than the one displayed in the main text, which follow as corollaries. More in details,~\cref{thm:dual.gap.specific} is implied by~\cref{thm:dual.gap},~\cref{lem:dual.function.specific,prop:lyapunov.characterization.constraint} follow from~\cref{lem:dual.function},~\cref{le:smoothness.gradient,lem:mu.max} are respectively proven in~\cref{le:smoothness.D,prop:positive.gradient.bounded.domain}.\\

\textbf{From constrained to Lagrangian problem.} For any $\sys \in \Sys$, we consider the constrained problem
\begin{equation}\label{eq:optimal.avg.cost.extended.generic}
\begin{aligned}
 &  \mathcal{J}_{*}(\sys) &  = & \min_{\tpi} \mathcal{J}_{\tpi}(\sys):= \limsup_{S \rightarrow \infty} \frac{1}{S} \mathbb{E}\bigg[\sum_{s=0}^S \begin{pmatrix} x_s \\ \wt{u}_s \end{pmatrix}^\transp C_\dagger \begin{pmatrix} x_s \\ \wt{u}_s \end{pmatrix}  \bigg] \\
&  \text{subject to} & & x_{s+1} = A x_s + \wt{B} \wt u_s + \epsilon_{s+1} \\
& & &  g_{\tpi}(\sys) \leq 0
\end{aligned},
\end{equation}
where 
\begin{equation}\label{eq:constraint.generic}
g_{\tpi}(\sys) = \lim_{S \rightarrow \infty} \frac{1}{S}  \mathbb{E} \Big( \sum_{s=0}^S  \begin{pmatrix} x_s \\ \wt{u}_s \end{pmatrix}^\transp C_g \begin{pmatrix} x_s \\ \wt{u}_s \end{pmatrix}  \Big).
\end{equation}
Notice that~\cref{eq:optimal.avg.cost.extended.generic} and~\cref{eq:constraint.generic} exactly correspond to the constrained LQR problem \laglq has to solve at time step $t$ in~\cref{eq:optimal.avg.cost.extended} and~\cref{eq:constraint} whenever $\sys$ is defined by~\cref{eq:mapping.original.constrained.lqr.into.Sys}.\\

For any $\sys \in \Sys$, the Lagrangian formulation of the extended LQR share the same dynamics as in~\cref{eq:optimal.avg.cost.extended.generic}, while the cost is defined as the combination of the original average cost $\mathcal{J}$, and the constraint $g_{\wt\pi}$ as
\begin{align}\label{eq:lagrangian.lqr}
\mathcal L_{\tpi}&(\mu ; \sys) = \mathcal{J}_{\tpi}(\sys) + \mu g_{\tpi}(\sys).
\end{align}
Furthermore, it can be conveniently written as a quadratic cost function
\begin{equation}
\label{eq:lagrangian.lqr.quadratic}
\begin{aligned}
&\mathcal L_{\tpi}(\mu; \sys) =\lim_{S \rightarrow \infty} \frac{1}{S} \mathbb{E} \bigg[ \sum_{s=0}^{S-1} \begin{pmatrix} x_s^\transp & \tilde{u}_s^\transp \end{pmatrix} C_{\mu} \begin{pmatrix} x_s \\ \tilde{u}_s \end{pmatrix} \bigg] \\
&  \text{subject to} \quad x_{s+1} = A x_s + \wt{B} \wt u_s + \epsilon_{s+1},\quad\quad \wt{u}_s = \tpi(x_s) \quad \forall s\geq 0\\
\end{aligned}
\end{equation}
where $C_{\mu,\nu} = C_\dagger + \mu C_g$. Finally, we will denote as $\tpi_{\mu}(\sys)$ the optimal policy achieving $\min_{\tpi} \mathcal L_{\tpi}(\mu; \sys)$ for the system $\sys$ and Lagrangian parameter $\mu$ whenever it exists.\\ 

\textbf{Strong duality.} We are concerned in this section in proving that for any $\sys \in \Sys$, strong duality holds i.e. $ \mathcal{J}_{*}(\sys) = \sup_{\mu} \min_{\tpi} \mathcal L_{\tpi}(\mu ; \sys)$. Further, we derive an algorithm that can find efficiently a feasible linear policy $\tpi_*$ such that $\mathcal{J}_{\tpi_*}(\sys)$ is arbitrarily close to $ \mathcal{J}_{*}(\sys)$. The proofs are structured as follow: 
\begin{itemize}
\item In~\cref{sec:app.proof.lem:dual.function}, we study the inner minimization problem $\tpi_\mu(\sys) = \arg\min_{\tpi} \mathcal L_{\tpi}(\mu ; \sys)$ and characterize a set $\mathcal{M}$ such that for all $\mu \in \mathcal{M}$, the optimal policy is linear in the state, i.e., $\tpi_{\mu}(\sys)(x) = \wt{K}_\mu x$. Those results are summarized in~\cref{lem:dual.function}.
\end{itemize}
\begin{lemma}\label{lem:dual.function}
For any $\sys \in \Sys$, it exists $\wt \mu > 0$ such that on $\mathcal{M} =  [0, \wt \mu)$ the following properties hold
\begin{enumerate}
\item The extended Lagrangian LQ in~\eqref{eq:lagrangian.lqr.quadratic} admits a unique solution $\wt\pi_{\mu}(\sys) = \arg\min_{\tpi}\mathcal L_{\tpi}(\mu,\sys)$ obtained by solving a discrete algebraic Riccati equation. As a result, 
$\wt\pi_{\mu}(\sys)(x) = K_{\mu} x$ and $\mathcal{D}(\mu,\sys) = \Tr(P_\mu)$ where $P_\mu$ (resp. $K_\mu$) are solution of the Riccati equation (resp. the optimal control) associated with $(A,\wt{B}, C_\mu)$.
\item $\D(\cdot,\sys)$ is concave, $\D(\cdot,\sys) \in \C^{1}(\mathcal{M})$ and the derivative $\D'(\mu,\sys) = g_{\wt\pi_{\mu}(\sys)}(\sys)$, i.e., it coincides with the constraint $g$ evaluated at the optimal control $\wt\pi_{\mu}(\sys)$.
\end{enumerate}
\end{lemma}

\begin{itemize}
\item In~\cref{sec:app.proof.duality.final}, we leverage the structure of the inner minimization solution to show that strong duality holds (see~\cref{thm:dual.gap})
\end{itemize}

\begin{theorem}\label{thm:dual.gap}
For any $\sys \in \Sys$, let $\mathcal{M}$ be the admissible Riccati set associated with $\sys$ as defined in~\cref{eq:mutilde.definition}. Then, $$\mathcal{J}_{*}(\sys) = \sup_{\mu \in \mathcal{M}} \min_{\tpi} \mathcal{L}_{\tpi}(\mu;\sys).$$ 
Further, for any $\epsilon > 0$, it exists $\mu^\epsilon \in \mathcal{M}$ and a linear policy $\tpi^\epsilon$ such that 
\begin{center}
\begin{enumerate*}
\item $g_{\tpi^\epsilon}(\sys) \leq 0$,\hspace{1cm} \item $\mu^\epsilon g_{\tpi^\epsilon}(\sys) = 0$,\hspace{1cm} \item $J_*(\sys) \geq \mathcal{D}(\mu^\epsilon,\sys) \geq \mathcal{J}_*(\sys) - \epsilon$.
\end{enumerate*}
\end{center}
\end{theorem}

\textbf{Lyapunov structure of $\mathcal L_{\tpi}$, $\mathcal{J}_{\tpi}$, $g_{\tpi}$ under linear controller.}
As explained in the previous paragraph, we aim to show that the optimal policy  of the inner minimization problem is linear in the state, and hence we will ultimately focus on such structured policies. Before entering the proof detail, we recall here an important technical result that allows us to express all the Lagrangian quantities $\mathcal L_{\tpi}$, $\mathcal{J}_{\tpi}$, $g_{\tpi}$ as solutions of Lyapunov equation, as long as the policy $\tpi$ is linear and stable, i.e., $\tpi(x) = \wt{K} x$ and $\Ac(\wt{K}) = A + \wt{B} \wt{K}$ is stable.

\begin{proposition}\label{prop:lyapunov.characterization.under.linear.policy}
For all $\sys \in \Sys$, for all $\mu \geq 0$ and for all linear policy $\tpi(x) = \wt{K}x$ such that $\Ac(\wt{K}) = A + \wt{B} \wt{K}$ is stable,\footnote{Notice the crucial difference between $\Tr(P_\mu(\wt{K}))$ that coincides with the average expected cost of the \textit{Lagrangian} LQR, which is characterized by the cost matrix $C_{\mu}$, and $\Tr(P(\wt K))$ that corresponds to the average expected cost of the \textit{extended} LQR, which is characterized by the cost $C_\dagger$.}\footnote{For sake of readability, we omit the dependency in $\sys$ of the Lyapunov solutions $P(\wt K)$, $G(\wt K)$ and $P_\mu(\wt{K})$, although there are entirely characterized by $\sys$, $\mu$ and $\wt{K}$.}
\begin{equation}
\label{eq:average.cost.constraints.lyapunov1}
\mathcal{J}_{\tpi}(\sys) = \Tr\big( P(\wt K)\big); \quad\quad g_{\tpi}(\sys) = \Tr\big(G(\wt K)\big); \quad\quad \mathcal{L}_{\tpi}(\mu,\sys) = \Tr\big( P_\mu(\wt{K})\big);
\end{equation} 
where $P(\wt K)$, $G(\wt K)$ and $P_\mu(\wt{K})$ are the unique solutions of the following Lyapunov equations:
\begin{equation}
\label{eq:average.cost.constraints.lyapunov2}
\begin{aligned}
P(\wt K) &= \big( \Ac(\wt{K})\big)^\transp P(\wt K) \Ac(\wt{K}) + \begin{pmatrix} I \\ \wt K \end{pmatrix}^\transp C_{\dagger} \begin{pmatrix} I \\ \wt K \end{pmatrix},\\
G(\wt K) &= \big( \Ac(\wt{K})\big)^\transp G(\wt K)  \Ac(\wt{K}) + \begin{pmatrix} I \\ \wt K \end{pmatrix}^\transp C_{g} \begin{pmatrix} I \\ \wt K \end{pmatrix},\\
P_\mu(\wt K) &= \big( \Ac(\wt{K})\big)^\transp P_\mu(\wt K)  \Ac(\wt{K}) + \begin{pmatrix} I \\ \wt K \end{pmatrix}^\transp C_\mu \begin{pmatrix} I \\ \wt K \end{pmatrix}.
\end{aligned}
\end{equation}
\end{proposition}
\begin{proof}
The proof relies on the fact that for any \textit{linear} extended controller $\wt K$, let $\Ac(\wt K)$ be the induced closed-loop matrix and $\Sigma(\wt K) = \Ac(\wt K) \Sigma(\wt K)  \Ac(\wt K)^\transp + I$ the steady-state associated variance, for any cost matrix $C$, 
\begin{equation*}
\Tr \left( \begin{pmatrix} I \\ \wt K \end{pmatrix}^\transp C \begin{pmatrix} I \\ \wt K \end{pmatrix} \Sigma(\wt K) \right) = \Tr( X),
\end{equation*}
where $X$ is the solution of the Lyapunov equation $X = \Ac(\wt K)^\transp X  \Ac(\wt K) + \begin{pmatrix} I \\ \wt K \end{pmatrix}^\transp C \begin{pmatrix} I \\ \wt K \end{pmatrix}$.\\ 
\end{proof}
Notice that from $C_\mu = C_\dagger + \mu C_g$, standard Lyapunov algebraic manipulation ensures that 
\begin{equation}
\label{eq:average.cost.constraints.lyapunov3}
P_\mu(\wt{K}) = P(\wt{K}) + \mu G(\wt{K}).
\end{equation}
 As a result, for any $\tpi$ (resp. for any $\wt{K}$),
\begin{equation}\label{eq:average.cost.constraints.lyapunov4}
\mathcal{L}_{\tpi}(\mu,\sys) = \mathcal{J}_{\tpi}(\sys) + \mu g_{\tpi}(\sys)  = \Tr\big(P_{\mu}(\wt{K})\big) = \Tr\big( P(\wt{K}) \big) + \mu \Tr\big( G(\wt{K}) \big),
\end{equation}
which summarizes the three different views of the Lagrangian cost in~\cref{eq:lagrangian.lqr},~\cref{eq:lagrangian.lqr.quadratic}, and~\cref{eq:average.cost.constraints.lyapunov3}.


\section{Proof of~\cref{lem:dual.function}}\label{sec:app.proof.lem:dual.function}
In this section, we consider an arbitrary $\sys \in \Sys$ and focus on the minimization problem
\begin{equation}\label{eq:lagrangian.lqr.min}
\begin{aligned}
\wt\pi_{\mu}(\sys) &= \arg\min_{\tpi}\mathcal L_{\tpi}(\mu;\sys); \quad \D(\mu;\sys) = \min_{\tpi}\mathcal L_{\tpi}(\mu ; \sys).
\end{aligned}
\end{equation}
We study the conditions such that for a given Lagrangian parameter $\mu$, the optimal policy is linear and can be obtained by solving a Riccati equation. Further, we characterize the properties of $\mathcal{D}(\mu,\sys)$ for those \textit{admissible} Lagrangian parameters.

\subsection{Riccati characterization of the inner minimization}\label{subsec:riccati.equation.inner.minimization}

We first notice that for any $\sys \in \Sys$ the pair $(A,\wt{B})$ is controllable even if the original one $(A,B)$ is not.\footnote{This directly follows from the fact that $\wt B$ is full column rank even if $B$ may not.} Further, $C_\mu$ is a symmetric matrix, which can be decomposed into a more ``standard'' LQ form isolating the cost matrices related to the state, the extended control, and the cross terms as

\begin{equation}\label{eq:cost.matrices.def.mu.nu}
C_{\mu} = \begin{pmatrix} Q_\mu & N_\mu^\transp \\ N_\mu & R_\mu \end{pmatrix}, \quad \text{ where } Q_\mu \in \mathbb{R}^{n \times n}, N_\mu \in \mathbb{R}^{(n+d)\times n}, R_\mu \in \mathbb{R}^{(n+d)\times(n+d)}.
\end{equation}

%
As a result, it is clear from~\cref{eq:lagrangian.lqr.quadratic,eq:cost.matrices.def.mu.nu} that minimizing $\mathcal L_{\tpi}(\mu,\sys)$ resembles to solving an LQR problem. However, $C_\mu$ may not be p.s.d. which would violate the standard Riccati assumptions (see Eq.~\ref{eq:lqr.solution}). Nonetheless, Riccati theory extends to a more general setting, when $C_{\mu}$ is not p.s.d.\\

 Formally, for any $\sys \in \Sys$, we say that the associated Lagrangian inner minimization corresponds to an admissible Riccati solution at $\mu$ if it exists a symmetric real matrix $P_{\mu}$\footnote{We drop the dependency in $\sys$ for sake of readability.} that satisfies the set of conditions:
\begin{equation}
\label{eq:admissible.riccati.solution}
\begin{aligned}
& R_{\mu}+ \tilde{B}^\transp P_{\mu} \tilde{B} \succ 0 \\
& P_{\mu} = Q_{\mu} + A^\transp P_{\mu} A  - [ A^\transp P_{\mu} \tilde{B} + N_{\mu}^\transp ] [ R_{\mu} + \tilde{B}^\transp P_{\mu} \tilde{B} ]^{-1} [ \tilde{B}^\transp P_{\mu} A + N_{\mu} ] \\
& \rho(\Ac_{\mu}) < 1 \quad \text{ where } \quad \Ac_{\mu} := A - \tilde{B} [ R_{\mu} + \tilde{B}^\transp P_{\mu} \tilde{B} ]^{-1} [ \tilde{B}^\transp P_{\mu} A + N_{\mu} ].
\end{aligned}
\end{equation} 
The following lemma maps the existence and uniqueness of a solution satisfying Eq.~\eqref{eq:admissible.riccati.solution} and the solution of the Lagrangian LQR.

\begin{lemma}\label{le:riccati.lqr.solution}
For any $\sys \in \Sys$, if for some Lagrangian parameter $\mu \geq 0$ it exists a symmetric matrix $P_{\mu}$ satisfying the set of conditions~\cref{eq:admissible.riccati.solution} defined w.r.t. $\sys$ and $\mu$, then:
\begin{enumerate}
\item $P_\mu$ is unique,
\item $P_\mu$ the solution of the Lagrangian LQR~\cref{eq:lagrangian.lqr.min}, i.e., $\mathcal{D}(\mu,\sys) = \Tr(P_\mu)$,
\item The optimal policy is linear in the state, i.e., 
\begin{equation}
\label{eq:optimal.riccati.controller}
\wt\pi_{\mu}(\sys)(x) = \wt K_{\mu} x, \quad \text{ where } \wt K_{\mu} = - [ R_{\mu} + \tilde{B}^\transp P_{\mu} \tilde{B} ]^{-1} [ \tilde{B}^\transp P_{\mu} A + N_{\mu} ],
\end{equation}
\item $\mathcal{J}_{\tpi_{\mu}(\sys)}(\sys) = \Tr\big( P(\wt K_{\mu})\big)$, $g_{\tpi_{\mu}(\sys)}(\sys) = \Tr\big(G(\wt K_{\mu})\big)$, $\mathcal{L}_{\tpi_{\mu}(\sys)}(\mu,\sys) = \Tr\big( P_\mu\big)$.
\end{enumerate}
\end{lemma}
\begin{proof}
The first assertion directly follows from Thm.~1 in \cite{molinari1975} which ensures that if a solution to the Riccati equation exists, then it is unique. Furthermore, as it satisfies the Bellman equation, this ensures that $\mathcal{D}(\mu,\sys) = \Tr (P_{\mu})$ and that $\tpi_{\mu}(\sys)(x) = \wt{K}_{\mu} x$ (see [Bertsekas, Vol.2, Prop.5.6.1]). Finally, since the optimal policy is linear, we can use~\cref{prop:lyapunov.characterization.under.linear.policy} to obtain the Lyapunov characterization of $\mathcal{J}_{\tpi_{\mu}(\sys)}(\sys),g_{\tpi_{\mu}(\sys)}(\sys),\mathcal{L}_{\tpi_{\mu}(\sys)}(\mu,\sys)$.
\end{proof}

In essence,~\cref{le:riccati.lqr.solution} indicates that \textit{when}~\cref{eq:admissible.riccati.solution} admits a solution, then the inner minimization is \textit{nice} at $\mu$, since it corresponds to the solution of a Riccati equation, that is unique and leads to a linear optimal controller. As a result,~\cref{le:riccati.lqr.solution} provide us with a highly implicit characterization of \textit{good} Lagrangian parameters $\mu$.

\subsection{Domain of Admissible Riccati Solutions}\label{ssec:app.dual.domain}

In the previous section we identified the condition for which at a specific Lagrangian parameter $\mu$, the solution of~\cref{eq:lagrangian.lqr.min} is obtained by solving a Riccati equation. We now proceed with characterizing the set of $\mathcal{M}$ in which \cref{eq:admissible.riccati.solution} holds.

Clearly, for $\{ \mu \in \Re^+ \text{ s.t } C_\mu \succ 0\}$,~\cref{eq:admissible.riccati.solution} holds as it coincides with the "standard" LQR setting. However, this set can be extended to a larger one, that we denote as $\mathcal{M}$.
We first relate the existence of a Riccati solution to a certain Popov criterion. Following~\citet{molinari1975}, we introduce the Popov functions:
\begin{definition}
\label{def:popov.function}
The Popov function $\Psi_{\mu}(\cdot,\sys)$ associated with $\sys \in \Sys$ and Lagrangian parameter $\mu$ is defined as:
\begin{equation}
\label{eq:popov.function.definition}
\Psi_{\mu}(z; \sys) = \begin{pmatrix} (I z^{-1} - A)^{-1} \tilde{B} \\ I \end{pmatrix}^\transp C_\mu \begin{pmatrix} (I z - A)^{-1} \tilde{B} \\ I \end{pmatrix}, \quad \forall z \in \mathbb{C}, \hspace{1mm} |z| = 1.
\end{equation}
For any controller $\wt{K}$, the controlled Popov function $\Psi_{\mu}^{\wt{K}}(\cdot,\sys)$ associated with $\sys \in \Sys$ and Lagrangian parameter $\mu$ is defined as:
\begin{equation}
\label{eq:popov.K.definition}
\Psi_{\mu}^{\wt K} (z;\sys)  = \begin{pmatrix} (I z^{-1} - \Ac(\wt K))^{-1} \tilde{B} \\ I \end{pmatrix}^\transp \begin{pmatrix} I & \wt{K}^\transp \\ 0 & I \end{pmatrix} C_\mu \begin{pmatrix} I & 0 \\ \wt K & I \end{pmatrix} \begin{pmatrix} (I z - \Ac(\wt K))^{-1} \tilde{B} \\ I \end{pmatrix}
\end{equation}
\end{definition}
In essence, the controlled Popov function by $\wt{K}$ simply consists in the Popov function when $\sys$ is pre-controlled by $\wt{K}$. The link between those two is given by
\begin{equation}
\label{eq:from.popov.controlled.popov}
\Psi_{\mu}^{\wt K} (z;\sys)  = Y^{\wt{K}}(z^{-1};\sys)^\transp \Psi_{\mu}(z;\sys) Y^{\wt{K}}(z;\sys)  \quad \quad \text{where} \quad Y^{\wt{K}}(z;\sys) = I + \wt K [Iz - \Ac(\wt K)]^{-1} \tilde{B}.
\end{equation}

The interest of considering Popov functions rather than solution of~\cref{eq:admissible.riccati.solution} is that the latter are well defined for all $\mu \geq 0$ while the former may not. Further, one can map the existence of a solution to~\cref{eq:admissible.riccati.solution} to the positive definiteness of the Popov functions. Those results are provided in~\citep{molinari1975}.
\begin{proposition}[Thm.1\&2~\citet{molinari1975}]\label{th:molinari75.1}
For any controllable system $\sys$, the following conditions are equivalent:
\begin{enumerate}
	\item There exists a real symmetric solution satisfying Eq.~\eqref{eq:admissible.riccati.solution}, necessarily unique.
	\item For some (and hence all) $\wt K$ such that $|\lambda\big(\Ac_{\mu}(\wt{K})\big)| \neq 1$, $\Psi_{\mu}^{\wt K}(z;\sys) \succ 0$ on the unit circle $|z| = 1$.
\end{enumerate}
\end{proposition}
We also recall two important identities can be extracted from~\citet{molinari1975} that links the Popov function to the optimal quantities $\Ac_\mu$, $D_\mu$, $\wt{K}_\mu$ whenever they are defined:
\begin{equation}
\label{eq:popov.identities}
\begin{aligned}
\forall |z| =1, \quad &\Psi_{\mu}^{\wt K_\mu} (z;\sys) = D_\mu; \\
 \forall |z| =1, \quad &\forall \wt{K}_{1}, \wt{K}_2 \text{ s.t } | \lambda(\Ac(\wt{K}_{1}))| \neq 1,| \lambda(\Ac(\wt{K}_{2}))| \neq 1,    \Psi_{\mu}^{\wt K_1} (z;\sys)  =Y^{\wt{K}_1,\wt{K}_2}(z^{-1};\sys)^\transp \Psi_{\mu}^{\wt K_2} (z;\sys)Y^{\wt{K}_1,\wt{K}_2}(z;\sys), \\
& \text{where} \quad Y^{\wt{K}_1,\wt{K}_2}(z;\sys) = I + (\wt{K}_1 - \wt{K}_2) [Iz - \Ac(\wt K_1)]^{-1} \tilde{B}.
\end{aligned}
\end{equation}
Thanks to~\cref{th:molinari75.1}, we can now associate to any $\sys \in \Sys$ a set $\mathcal{M}$ such that for all $\mu \in \mathcal{M}$, a symmetric solution to~\cref{eq:admissible.riccati.solution} exists and is strictly stabilizing, which, by~\cref{le:riccati.lqr.solution} implies that the optimal solution to~\cref{eq:lagrangian.lqr.min} is a linear policy.
\begin{lemma}\label{prop:domain.mutilde.nu}
For any $\sys \in \Sys$, let $\Psi_{\mu}^{\wt K}(z,\sys)$ be defined by~\cref{def:popov.function}. Let
\begin{equation}\label{eq:mutilde.definition}
\tilde{\mu} := \sup \big\{ \mu \geq 0 \text{ s.t. }\; \Psi_{\mu}^{\wt K}(z,\sys) \succ 0,\; \text{ for all } |z| = 1 \text{ and } \wt K \text{ such that } |\lambda\big(\Ac_{\mu}(\wt{K})\big)| \neq 1\big\}.
\end{equation}
Then, the set $\mathcal{M} = [0, \wt{\mu})$ is non-empty and for all $\mu \in\M$, Eq.~\eqref{eq:admissible.riccati.solution} admits a solution such that $\rho(\Ac_{\mu}) < 1$.
\end{lemma}

\begin{proof} 
\begin{enumerate}
\item For $\mu = 0$, 
\begin{equation*}
C_0 = C_\dagger + 0  \times C_g =  \begin{pmatrix} C_\dagger^1 & 0 \\ 0 & 0 \end{pmatrix} + C_\dagger^2.
\end{equation*}
Let $\wb{K} = \begin{pmatrix} 0 \\ -A \end{pmatrix}$ which is such that $\Ac(\wb{K}) = 0$. From $C_\dagger^1 \succ 0$ and $C_\dagger^2 \succcurlyeq 0$, one has for all $|z|=1$,
\begin{equation*}
\begin{aligned}
\Psi_{0}^{\wb K}(z,\sys)  &= \begin{pmatrix} (I z^{-1} - \Ac(\wb K))^{-1} \tilde{B} \\ I \end{pmatrix}^\transp \begin{pmatrix} I & \wb{K}^\transp \\ 0 & I \end{pmatrix} C_\dagger \begin{pmatrix} I & 0 \\ \wb K & I \end{pmatrix} \begin{pmatrix} (I z - \Ac(\wb K))^{-1} \tilde{B} \\ I \end{pmatrix} \\
&=   \begin{pmatrix} z \tilde{B} \\ I \end{pmatrix}^\transp \begin{pmatrix} I & \wb{K}^\transp \\ 0 & I \end{pmatrix} C_\dagger \begin{pmatrix} I & 0 \\ \wb K & I \end{pmatrix} \begin{pmatrix} z^{-1} \tilde{B} \\ I \end{pmatrix} \\
&\succcurlyeq  \begin{pmatrix} z \tilde{B} \\ I \end{pmatrix}^\transp \begin{pmatrix} I & \wb{K}^\transp \\ 0 & I \end{pmatrix} \begin{pmatrix} C_\dagger^1 & 0 \\ 0 & 0 \end{pmatrix}  \begin{pmatrix} I & 0 \\ \wb K & I \end{pmatrix} \begin{pmatrix} z^{-1} \tilde{B} \\ I \end{pmatrix}\\
&\succcurlyeq \begin{pmatrix} z B & I_{n} \\ I_d & 0 \end{pmatrix}^\transp C_{\dagger}^1 \begin{pmatrix} z^{-1} B & I_{n} \\ I_d & 0 \end{pmatrix} \\
&\succ 0
\end{aligned}
\end{equation*}
where we used that $\left | \det \begin{pmatrix} z B & I_{n} \\ I_d & 0 \end{pmatrix} \right | = \left | \det \begin{pmatrix}  I_{n} & z \wt{B} \\ 0& I_d \end{pmatrix} \right | =1$ and that $C_\dagger^1 \succ 0$ to obtain the last inequality. Since we exhibit a $\wb{K}$ such that $\Psi_{0}^{\wb K}(z,\sys)  \succ 0$ for all $|z|=1$,~\cref{th:molinari75.1} ensures that it holds for all $\wt{K}$ such that $|\lambda(\Ac(\wt{K}))| \neq 1$ and thus that $0 \in \mathcal{M}$ and $\mathcal{M}$ is non-empty.

 \item To show that all $\mu \in \mathcal{M}$ are associated with a strictly stabilizing solution of~\cref{eq:admissible.riccati.solution}, we will invoque~\cref{th:molinari75.1}. As a result, we are left to prove that for all $\mu \in \mathcal{M}$, it exists $\wt{K}$ such that $|\lambda\big(\Ac_{\mu}(\wt{K})\big)| \neq 1$ and $\Psi^{\wt{K}}_{\mu}(z,\sys) \succ 0$ for all $|z|=1$.\\
 Exploiting the linearity of $\Psi_{\mu}$ (as $C_{\mu}$ is linear in $\mu$, for all $|z| =1$ and $\wt K$), so is $\Psi_{\mu}^{\wt K}(z,\sys)$. Hence, for all $|z|=1$, for all $v \in \mathbb{C}^{n+d} \setminus \{0\}$, the function 
 \begin{equation*}
 f : \mu \rightarrow v^* \Psi^{\wt{K}}_{\mu}(z,\sys) v,
 \end{equation*}
 is linear in $\mu$. Moreover, $\lim_{\mu \rightarrow \wt{\mu}} f(\mu) \geq 0$ and $f(0) > 0$, which implies that for all $\mu \in \mathcal{M}$, for all $|z|=1$,
 \begin{equation*}
 \forall v \in \mathbb{C}^{n+d}\setminus \{0\}, \quad  v^* \Psi^{\wt{K}}_{\mu}(z,\sys) v > 0 \quad \Rightarrow \quad \Psi^{\wt{K}}_{\mu}(z,\sys) \succ 0.
 \end{equation*}
\end{enumerate}
\end{proof}

\subsection{Characterization of $\mathcal{D}(\mu,\sys)$}\label{ssec:app.dual.characterization}

We are now ready to characterize the dual function $\mathcal{D}$. For any $\sys \in \Sys$, the use of the Popov functions allows us to define a set $\mathcal{M}$ (~\cref{prop:domain.mutilde.nu}) that depends explicitely (yet in a non-trivial fashion) in $\sys$, and on which the inner minimization problem~\cref{eq:lagrangian.lqr.min} corresponds to a well-defined (yet non-standard) Riccati equation. We summarize~\cref{le:riccati.lqr.solution,prop:domain.mutilde.nu} in~\cref{prop:dual.riccati} which proves the first statements of~\cref{lem:dual.function}.

\begin{proposition}
\label{prop:dual.riccati}
For any $\sys \in \Sys$, let $\mathcal{M}$ be the admissible Riccati set associated with $\sys$ as defined in~\cref{eq:mutilde.definition} and $\mathcal{D}(\mu,\sys)$ the dual function defined in~\cref{eq:lagrangian.lqr.min}. Then, \begin{enumerate}
\item for all $\mu \in \mathcal{M}$, it exists a unique symmetric matrix $P_\mu$ satisfying~\cref{eq:admissible.riccati.solution}, 
\item the optimal policy $\tpi_\mu(\sys) = \arg\min_{\tpi}\mathcal L_{\tpi}(\mu ; \sys)$ is linear in the state, i.e., $\tpi_\mu(\sys)(x)= \wt{K}_\mu x$, where 
\begin{equation*}
\wt K_{\mu} = - [ R_{\mu} + \tilde{B}^\transp P_{\mu} \tilde{B} ]^{-1} [ \tilde{B}^\transp P_{\mu} A + N_{\mu} ],
\end{equation*}
\item $\mathcal{J}_{\tpi_{\mu}(\sys)}(\sys) = \Tr\big( P(\wt K_{\mu})\big)$, $g_{\tpi_{\mu}(\sys)}(\sys) = \Tr\big(G(\wt K_{\mu})\big)$, $\mathcal{D}(\mu,\sys) = \mathcal{L}_{\tpi_{\mu}(\sys)}(\mu,\sys) = \Tr\big( P_\mu\big)$.
\end{enumerate}
\end{proposition}

\begin{proof}
The proof of~\cref{prop:dual.riccati} directly follows from merging~\cref{le:riccati.lqr.solution} and~\cref{prop:domain.mutilde.nu}.
\end{proof}

Now that~\cref{prop:dual.riccati} guarantees that for any $\sys\in\Sys$, $\mathcal{D}(\cdot,\sys)$ is well defined on some explicit convex set $\mathcal{M}$, we are able to prove the last statements of~\cref{lem:dual.function}.
\begin{proposition}
\label{prop:derivative.riccati.cost.matrix}
For any $\sys \in \Sys$, let $\mathcal{M}$ be the admissible Riccati set associated with $\sys$ as defined in~\cref{eq:mutilde.definition}. Let $\mathcal{D}^\prime$ denote the derivative of $\mathcal{D}$ w.r.t. $\mu$. Then,
\begin{enumerate}
\item $\mathcal{M}$ is a non-empty convex set,
\item $\mu  \rightarrow \mathcal{D}(\mu,\sys)$ is concave on $\mathcal{M}$ and $\mathcal{D}(\mu, \sys) > - \infty$ on $\mathcal{M}$.
\item $\mu  \rightarrow \mathcal{D}(\mu,\sys) \in \mathcal{C}^{1}(\mathcal{M})$.
\item $\D^\prime(\mu,\sys)) =g_{\tpi_\mu(\sys)}(\sys) = \Tr(G(\wt{K}_\mu))$, where $\wt{K}_\mu$ and $\wt\pi_{\mu}(\sys)(x) = \wt K_{\mu} x$ be respectively the optimal controller and policy as defined in~\cref{prop:dual.riccati} and $G(\wt{K}_\mu)$ is given in~\cref{eq:average.cost.constraints.lyapunov2}. As a result,
\begin{equation}
\label{eq:lagrangian.dual.derivative}
\mathcal{D}(\mu,\sys) = \mathcal{J}_{\tpi_{\mu}(\sys)}(\sys) + \mu g_{\tpi_{\mu}(\sys)}(\sys) = \mathcal{J}_{\tpi_{\mu}(\sys)}(\sys) + \mu \mathcal{D}^\prime(\mu,\sys).
\end{equation}
\end{enumerate}
\end{proposition}

\begin{proof}
\begin{enumerate}
\item $\mathcal{M}$ is non-empty from~\cref{prop:domain.mutilde.nu}. Further, $\mathcal{M} = [0,\wt{\mu})$ where $\wt{\mu}$ is defined in~\cref{eq:mutilde.definition}. Thus, $\mathcal{M}$ is convex.
\item The concavity of $\mathcal{D}(\cdot,\sys)$ is provided by construction, as $\mathcal{D}$ is the minimum of affine function. The lower boundedness property of $\mathcal{D}(\cdot,\sys)$ follows from $D_\mu  = R_\mu + \wt{B}^\transp P_\mu \wt{B} \succ 0$ for all $\mu \in \mathcal{M}$. Since $\R_\mu \prec +\infty$ and $\wt{B}$ is full column rank, 
\begin{equation*}
\D_\mu \succ 0 \quad \Rightarrow  \quad \wt{B}^\transp P_\mu \wt{B} \succ -\infty \quad \Rightarrow \quad P_\mu \succ - \infty \quad \Rightarrow \quad  \D(\mu,\sys) = \Tr(P_\mu) > - \infty.
\end{equation*}
\item The smoothness of $\mathcal{D}(\cdot,\sys)$ follows from Alexandrov's Theorem that guarantees that $\mathcal{D}(\cdot,\sys)$ is twice differentiable on $\mathcal{M}$ and hence in $\mathcal{C}^{1}(\mathcal{M})$.
\item $g_{\tpi_\mu(\sys)}(\sys) = \Tr(G(\wt{K}_\mu))$ by~\cref{prop:dual.riccati}. We are thus left to prove that $\D^\prime(\mu,\sys)) =g_{\tpi_\mu(\sys)}(\sys)$, which is a structural property inherited from the Lagrangian formulation and that is directly implied by the fact that $\mathcal{D}(\mu,\sys) >- \infty$ for all $\mu \in \mathcal{M}$ and that $\mathcal{D}(\cdot,\sys)$ is smooth and concave. Observe that by construction, for any $(\mu,\wb{\mu}) \in \mathcal{M}^2$, 
\begin{equation*}
\mathcal{D}(\wb{\mu},\sys) = \mathcal{J}_{\tpi_{\wb{\mu}}(\sys)}(\sys) + \wb{\mu} g_{\tpi_{\wb{\mu}}(\sys)}(\sys) = \min_{\tpi}\left(\mathcal{J}_{\tpi}(\sys) + \wb{\mu} g_{\tpi}(\sys) \right) \geq  \mathcal{J}_{\tpi_{\mu}(\sys)}(\sys) + \wb{\mu} g_{\tpi_{\mu}(\sys)}(\sys)  = \mathcal{D}(\mu,\sys) + (\wb{\mu} - \mu)g_{\tpi_{\mu}(\sys)}(\sys).
\end{equation*}
Thus, $\mathcal{D}(\wb{\mu},\sys)  \geq \mathcal{D}(\mu,\sys) + (\wb{\mu} - \mu)g_{\tpi_{\mu}(\sys)}(\sys)$ 
which means by definition that $g_{\tpi_{\mu}(\sys)}(\sys)$ belongs to the super-gradient set at $\mu$ (which is defined for concave function that are not $-\infty$, see~\citep{rockafellar1970convex} for a survey). However, since $\mathcal{D}(\cdot,\sys) \in \mathcal{C}^1(\mathcal{M})$, the super-gradient set is reduced to a singleton, which coincides with the derivative. As a result, $\mathcal{D}^\prime(\mu,\sys) = g_{\tpi_{\mu}(\sys)}(\sys)$.
\end{enumerate}
\end{proof}

\section{Proof of~\cref{thm:dual.gap}}\label{sec:app.proof.duality.final}

We are now ready to proceed with the proof of~\cref{thm:dual.gap} and aim at proving strong duality, i.e., that for any $\sys \in \Sys$ associated with admissible Riccati set $\mathcal{M}$,
\begin{equation*}
\mathcal{J}_{*}(\sys) = \sup_{\mu \in \mathcal{M}} \D(\mu;\sys) =  \sup_{\mu \in \mathcal{M}}  \min_{\tpi}\mathcal L_{\tpi}(\mu ; \sys).
\end{equation*}
Notice that the definition of $\Sys$ guarantees that it exists a feasible stable linear policy for $\sys$, and hence that 
\begin{equation}
\label{eq:bounded.objective.under.feasibility}
\mathcal{J}_{*}(\sys) < \infty.
\end{equation}
Further, from weak-duality, 
\begin{equation}
\label{eq:proof.duality.weak.duality}
\mathcal{J}_{*}(\sys) \geq \sup_{\mu \in \mathcal{M}} \D(\mu;\sys),
\end{equation}
thus the proof consists in showing the converse inequality.\\

The reason why we write the strong-duality in a $\sup-\min$ sense is that strong duality might hold on the frontier of the open set $\mathcal{M}$. More in details, we aim to show that for any $\sys \in \Sys$, 
for any arbitrary $\epsilon > 0$, 
\begin{equation*}
\mathcal{J}_{*}(\sys) \geq \sup_{\mu \in \mathcal{M}} \D(\mu;\sys) =  \sup_{\mu \in \mathcal{M}}  \min_{\tpi}\mathcal L_{\tpi}(\mu ; \sys) \geq \mathcal{J}_{*}(\sys) - \epsilon.
\end{equation*}
Since the l.h.s. inequality is always true by weak-duality, we focus on proving the r.h.s. inequality.\\

The proof distinguishes between two distinct cases, based on the monotonicity of the dual function $\mathcal{D}$. Notice that~\cref{lem:dual.function} ensures that $\mathcal{D}$ is a concave function over the real convex set $\mathcal{M}$. As a result, it can be either \textbf{1)} increasing then non-increasing or non-increasing everywhere \textbf{2)} (strictly) increasing everywhere on $\mathcal{M}$.\\

In case \textbf{1)}, we show that strong duality holds in a ``strict'' sense\footnote{We say that strong duality holds in a strict sense when the $\sup-\min$ is replaced by a $\max-\min$}, that is it exists $\mu_* \in \mathcal{M}$ such that
\begin{equation}
\label{eq:strong.duality.strict.sense}
\begin{aligned}
\text{strong duality } \quad &  \mathcal{J}_{*}(\sys) = \max_{\mu \in \mathcal{M}} \D(\mu;\sys) = \D(\mu_*;\sys) \\
\text{primal feasibility } \quad & g_{\tpi_{\mu_*}(\sys)} \leq 0 \\
\text{complementary slackness } \quad & \mu_* g_{\tpi_{\mu_*}(\sys)} = 0.
\end{aligned}
\end{equation}

In case \textbf{2)}, we show that strong duality holds in a ``weak'' sense\footnote{We say that strong duality holds in a weak sense when $\max$ is not attain on the domain, and hence has to be replaced by a $\sup$.}, that is that 
\begin{equation}
\label{eq:strong.duality.weak.sense}
\mathcal{J}_{*}(\sys) = \sup_{\mu \in \mathcal{M}} \D(\mu;\sys).
\end{equation}
More in detail, we show that for any arbitrary $\epsilon >0$, it exists $\mu^\epsilon \in \mathcal{M}$ such that $\mathcal{D}(\mu^{\epsilon},\sys) \geq \mathcal{J}_{*}(\sys) - \epsilon$. As a by-product, we provide an explicit linear policy $\tpi^{\epsilon}(x) = \wt{K}^\epsilon x$ which satisfies the primal feasibility i.e., $g_{\tpi^\epsilon}(\sys) \leq 0$ and is $\epsilon$-optimal i.e., $\mathcal{J}_{\tpi^\epsilon}(\sys) - \mathcal{J}_*(\sys) \leq \epsilon$, either in closed-form, either as the solution of a well-defined Riccati equation.\\

\subsection{Case 1). The dual function $\mathcal{D}(\cdot,\sys)$ is non-increasing on a subset of $\mathcal{M}$}
\label{app:subsec.strong.duality.strict}
By~\cref{lem:dual.function}, $\D(\cdot,\sys)$ is concave and $\mathcal{C}^{1}(\mathcal{M})$. As a result, $\mathcal{D}(\cdot,\sys)$ is non-increasing on a subset of $\mathcal{M}$ if and only if it exists $\mu_{-} \in \mathcal{M}$ such that $\mathcal{D}^\prime(\mu_{-},\sys) \leq 0$. In this case, we have the following result.

\begin{lemma}
\label{prop:strong.duality.strict}
For any $\sys \in \Sys$, let $\mathcal{M}$ be the admissible Riccati set associated with $\sys$ as defined in~\cref{eq:mutilde.definition}. If it exists $\mu_{-} \in \mathcal{M}$ such that $\mathcal{D}^\prime(\mu_{-},\sys) \leq 0$, then,
$\mathcal{D}(\cdot,\sys)$ attains its maximum on $\mathcal{M}$ and it exists $\mu_* \in \arg\max_{\mu \in \mathcal{M}} \mathcal{D}(\mu,\sys)$ such that
\begin{enumerate}
\item $\mathcal{J}_{*}(\sys) = \max_{\mu \in \mathcal{M}} \D(\mu;\sys) = \D(\mu_*;\sys)$,
\item $g_{\tpi_{\mu_*}(\sys)} \leq 0$
\item $\mu_* g_{\tpi_{\mu_*}(\sys)} = 0$.
\end{enumerate}
\end{lemma}

\begin{proof}
The proof of~\cref{prop:strong.duality.strict} follows from~\cref{lem:dual.function}. Since it exists $\mu_{-} \in \mathcal{M}$ such that $\mathcal{D}^\prime(\mu_{-},\sys) \leq 0$, by concavity, $\mathcal{D}(\cdot,\sys)$ attains its maximum on $\mathcal{M}$. Let $\mu_*$ be the smallest Lagrangian parameter in $\left\{ \arg\max_{\mu \in \mathcal{M}} \mathcal{D}(\mu,\sys) \right\}$. Then, 
\begin{equation}
\label{eq:proof.strict.strong.duality.1}
\text{ Either } \quad \mu_* = 0 \text{ and } \mathcal{D}^\prime(\mu_*,\sys) \leq 0; \quad \text{ Either } \quad \mu_* > 0 \text{ and } \mathcal{D}^\prime(\mu_*,\sys) = 0.
\end{equation}
Since $g_{\tpi_\mu(\sys)}(\sys) = \mathcal{D}^\prime(\mu,\sys)$ for all $\mu \in \mathcal{M}$,~\cref{eq:proof.strict.strong.duality.1} translates into
\begin{equation}
\label{eq:proof.strict.strong.duality.2}
\text{ Either } \quad \mu_* = 0 \text{ and } g_{\tpi_{\mu_*}(\sys)}(\sys) \leq 0; \quad \text{ Either } \quad \mu_* > 0 \text{ and } g_{\tpi_{\mu_*}(\sys)} (\sys) = 0.
\end{equation}
In both case, we obtain the primal feasibility and complementary slackness.

We conclude the proof noting that $\mathcal{J}_*(\sys) \geq \mathcal{D}(\mu_*,\sys)$ from weak duality, that $ \mathcal{D}(\mu_*,\sys) = \mathcal{L}_{\tpi_{\mu_*}(\sys)} (\mu_*,\sys) = \mathcal{J}_{\tpi_{\mu_*}(\sys)}(\sys) + \mu_* g_{\tpi_{\mu_*}(\sys)}(\sys)=  \mathcal{J}_{\tpi_{\mu_*}(\sys)}(\sys)$ from complementary slackness, and that $\mathcal{J}_{\tpi_{\mu_*}(\sys)}(\sys)  \geq \mathcal{J}_*(\sys)$ from primal feasibility. Hence, $\mathcal{J}_*(\sys) = \mathcal{D}(\mu_*,\sys)$.

\end{proof}

\subsection{Case 2). The dual function $\mathcal{D}(\cdot,\sys)$ is strictly increasing on $\mathcal{M}$}

When $\mathcal{D}(\cdot,\sys)$ is strictly increasing on $\mathcal{M}$, we have
\begin{equation}
\label{eq:strong.duality.weak.positive.gradient.condition}
\forall \mu \in \mathcal{M}, \quad \quad \mathcal{D}^\prime(\mu,\sys) > 0,
\end{equation}
which implies that the maximum is not attained on $\mathcal{M}$, and prevents us from using a similar path as in~\cref{app:subsec.strong.duality.strict} but indicates that we should study the behavior of $\mathcal{D}(\mu,\sys)$ when $\mu \rightarrow \wt{\mu}$ that is the closure of $\mathcal{M}$.\\
We first stress in~\cref{app:sssec.stability.boundedness} the implications of~\cref{eq:strong.duality.weak.positive.gradient.condition} in term of boundedness of $\mathcal{M}$ and in term of stability of the optimal policy $\tpi_{\mu}(\sys)$. In light of those properties, we discuss in~\cref{app:sssec.closure.M} the behavior of $\mathcal{D}(\mu,\sys)$ when $\mu$ approaches $\wt{\mu}$ and show that there exists suboptimal linear policies $\tpi \neq \tpi_{\mu}(\sys)$ that yield almost the same performance, i.e., such that $\mathcal{L}_{\tpi}(\mu,\sys) \simeq \mathcal{D}(\mu,\sys)$. Finally, in~\cref{app:sssec.linear.epsilon.optimal.feasible.policies} we provide an explicit feasible policy $\tpi^{\epsilon}$ whose performance can be made arbitrarily close to the optimal one, i.e., such that $\mathcal{L}_{\tpi^\epsilon}(\mu^\epsilon,\sys) - \mathcal{D}(\mu^\epsilon,\sys)\leq \epsilon$ for some $\mu^\epsilon \in \mathcal{M}$ which implies that $\mathcal{J}_{\tpi^\epsilon}(\sys) - \mathcal{J}_*(\sys) \leq  \epsilon$.

\subsubsection{Stability and boundedness properties}
\label{app:sssec.stability.boundedness}

The condition~\cref{eq:strong.duality.weak.positive.gradient.condition} provides us with information both on the value of $\wt{\mu}$ and on the stability of $\tpi_{\mu}(\sys)$ over $\mathcal{M}$. Those are formalized by in the following statements.

\begin{proposition}
\label{prop:positive.gradient.bounded.domain}
For any $\sys \in \Sys$, let $\mathcal{M} = [0,\wt{\mu})$ be the admissible Riccati set associated with $\sys$ as defined in~\cref{eq:mutilde.definition}. If for all $\mu \in \mathcal{M}$, $\mathcal{D}^\prime(\mu,\sys) > 0$, then,
\begin{equation}
\label{eq:strong.duality.weak.bound.mutilde}
\wt{\mu} \leq \lambda_{\max}\left( \begin{pmatrix} I_{n+d}  & 0 \\ 0 & 0 \end{pmatrix} C_\dagger \begin{pmatrix} I_{n+d}  & 0 \\ 0 & 0 \end{pmatrix} \right) / \lambda_{\min}(V^{-1}).
\end{equation}
\end{proposition}

\begin{proposition}
\label{prop:positive.gradient.stable.policy}
For any $\sys \in \Sys$, let $\mathcal{M}$ be the admissible Riccati set associated with $\sys$ as defined in~\cref{eq:mutilde.definition}. Let $\wt{K}_\mu$, $\Ac_\mu = A + \wt{B} \wt{K}_\mu$  and $\Sigma_\mu = \Ac_\mu \Sigma_\mu \big(\Ac_\mu)^\transp + I$ be the optimal controller, closed-loop matrix and steady-state covariance matrix at $\mu \in \mathcal{M}$ respectively. Then, for all $\mu \in \mathcal{M}$, 
\begin{equation}
\label{eq:strong.duality.weak.stable.policy}
\mathcal{D}^\prime(\mu,\sys) \geq 0 \quad \Rightarrow \quad \Tr(\Sigma_{\mu}) \leq \kappa(\mu,\sys)\leq \kappa(\sys) \quad \Rightarrow \rho(\Ac_\mu)^2 \leq 1 - 1 / \kappa(\mu,\sys) \leq 1 - 1 / \kappa(\sys),
\end{equation}
where $1 \leq \kappa(\mu,\sys) := \mathcal{D}(\mu,\sys) / \lambda_{\min}(C_\dagger^1)$ and $1 \leq \kappa(\sys) := \mathcal{J}_*(\sys) / \lambda_{\min}(C_\dagger^1)$. 
\end{proposition}
The next proposition follows directly as a corollary from~\cref{eq:strong.duality.weak.stable.policy}.
\begin{proposition}
\label{prop:positive.gradient.bounded.closed.loop.controller}
For any $\sys \in \Sys$, let $\mathcal{M}$ be the admissible Riccati set associated with $\sys$ as defined in~\cref{eq:mutilde.definition}. Let $\wt{K}_\mu$, $\Ac_\mu = A + \wt{B} \wt{K}_\mu$  and $\Sigma_\mu = \Ac_\mu \Sigma_\mu \big(\Ac_\mu)^\transp + I$ be the optimal controller, closed-loop matrix and steady-state covariance matrix at $\mu \in \mathcal{M}$ respectively. Then, for all $\mu \in \mathcal{M}$ such that $\mathcal{D}^\prime(\mu,\sys) \geq 0$, 
\begin{equation}
\label{eq:positive.gradient.bounded.closed.loop.controller}
\mathcal{D}^\prime(\mu,\sys) \geq 0 \quad \Rightarrow \quad 
\left\{ 
\begin{aligned}
& \forall k \in \mathbb{N}, \left\|\big( \Ac_\mu \big)^k \right\|_2 \leq \sqrt{\kappa(\mu,\sys)} \big(1-1/\kappa(\mu,\sys)\big)^{-k} \leq \sqrt{\kappa(\sys)} \big(1-1/\kappa(\sys)\big)^{-k} \\
& \left\| \begin{pmatrix} I \\ \wt{K}_{\mu} \end{pmatrix} \right\|_2 \leq \sqrt{\kappa(\mu,\sys)} (2 + \|A\|_2\|B\|_2)\leq \sqrt{\kappa(\sys)} (2 + \|A\|_2\|B\|_2)
\end{aligned}
\right.
\end{equation}
%
where $1 \leq \kappa(\mu,\sys) := \mathcal{D}(\mu,\sys) / \lambda_{\min}(C_\dagger^1)$ and $1 \leq \kappa(\sys) := \mathcal{J}_*(\sys) / \lambda_{\min}(C_\dagger^1)$. 
\end{proposition}

\begin{proof}[Proof of~\cref{prop:positive.gradient.bounded.domain}]
The proof is done by contradiction. Suppose that~\cref{eq:strong.duality.weak.bound.mutilde} do not hold. Then, it exists $\mu \in \mathcal{M}$ such that
\begin{equation*}
\mu > \lambda_{\max}\left( \begin{pmatrix} I_{n+d} & 0 \\ 0 & 0 \end{pmatrix} C_\dagger \begin{pmatrix} I_{n+d} & 0 \\ 0 & 0 \end{pmatrix} \right) / \lambda_{\min}(V^{-1}).
\end{equation*}
As a result, from $C_\mu = C_\dagger + \mu C_g$ we obtain
\begin{equation*}
\begin{pmatrix} I_{n+d}  & 0 \\ 0 & 0 \end{pmatrix}C_\mu \begin{pmatrix} I_{n+d}  & 0 \\ 0 & 0 \end{pmatrix} \prec \lambda_{\max}\left( \begin{pmatrix} I_{n+d}  & 0 \\ 0 & 0 \end{pmatrix} C_\dagger \begin{pmatrix} I_{n+d}  & 0 \\ 0 & 0 \end{pmatrix} \right) I  - \mu \lambda_{\min} (V^{-1}) I \prec 0.
\end{equation*}
Hence, for all linear policy $\tpi(x) = \wt{K} x$ of the form $\wt{K} = \begin{pmatrix} K_u \\ 0 \end{pmatrix}$
\begin{equation*}
 \begin{pmatrix} I \\ \wt{K} \end{pmatrix}^\transp C_\mu \begin{pmatrix} I \\ \wt{K} \end{pmatrix} = \begin{pmatrix} I \\ K_u\\ 0 \end{pmatrix}^\transp C_\mu \begin{pmatrix} I \\ K_u \\ 0 \end{pmatrix} 
= \begin{pmatrix} I \\ K_u  \end{pmatrix}^\transp \begin{pmatrix} I_{n+d}  & 0 \\ 0 & 0 \end{pmatrix}C_\mu \begin{pmatrix} I_{n+d}  & 0 \\ 0 & 0 \end{pmatrix} \begin{pmatrix} I \\ K_u  \end{pmatrix} \prec 0.
\end{equation*}
Thus, for such $\tpi$,
\begin{equation*}
\begin{aligned}
&\mathcal L_{\tpi}(\mu; \sys) =\lim_{S \rightarrow \infty} \frac{1}{S} \mathbb{E} \bigg[ \sum_{s=0}^{S-1} \begin{pmatrix} x_s^\transp & \tilde{u}_s^\transp \end{pmatrix} C_{\mu} \begin{pmatrix} x_s \\ \tilde{u}_s \end{pmatrix} \bigg] 
\leq  \lim_{S \rightarrow \infty} \frac{1}{S} \mathbb{E} \bigg[ \sum_{s=0}^{S-1} x_s^\transp \begin{pmatrix} I \\ K_u  \end{pmatrix}^\transp \begin{pmatrix} I_{n+d}  & 0 \\ 0 & 0 \end{pmatrix}C_\mu \begin{pmatrix} I_{n+d}  & 0 \\ 0 & 0 \end{pmatrix} \begin{pmatrix} I \\ K_u  \end{pmatrix}  x_s\bigg]   < 0\\
&  \text{subject to} \quad x_{s+1} = \Ac(\wt{K})  x_s  + \epsilon_{s+1}.
\end{aligned}
\end{equation*}

As a result,
\begin{itemize}
\item either $\Ac(\wt{K})$ is not stable and $\mathcal{D}(\mu,\sys) \leq \mathcal L_{\tpi}(\mu; \sys) = -\infty$, that contradicts the second statement in~\cref{prop:derivative.riccati.cost.matrix},
\item either $\Ac(\wt{K})$ is stable and $- \infty < \mathcal{D}(\mu,\sys) < 0$. However, from $C_0 \succcurlyeq 0$, we know that $\mathcal{D}(0,\sys) \geq 0$, which contradicts the fact that $\mathcal{D}(\cdot,\sys)$ is (strictly) increasing on $\mathcal{M}$.
\end{itemize}
\end{proof}

\begin{proof}[Proof of~\cref{prop:positive.gradient.stable.policy}]
Let $\mu \in \mathcal{M}$ such that $\mathcal{D}^\prime(\mu,\sys) \geq 0$. Then, we have
\begin{equation*}
\mathcal{J}_*(\sys) \geq \mathcal{D}(\mu,\sys) = \mathcal{J}_{\tpi_{\mu}(\sys)}(\sys) + \mu g_{{\tpi}_{\mu}(\sys)}(\sys) = \mathcal{J}_{\tpi_{\mu}(\sys)}(\sys) + \mu \mathcal{D}^\prime(\mu,\sys) \geq  \mathcal{J}_{\tpi_{\mu}(\sys)}(\sys),
\end{equation*}
where the l.h.s. inequality follows from weak duality while the equalities follows from~\cref{eq:lagrangian.dual.derivative}. Moreover, let $\Sigma_\mu$ be the steady-state covariance matrix defined by
\begin{equation*}
\Sigma_\mu = \Ac_\mu \Sigma_\mu \big(\Ac_\mu\big)^\transp + I,
\end{equation*}
we have by~\cref{eq:average.cost.constraints.lyapunov1,eq:average.cost.constraints.lyapunov2},
\begin{equation*}
\mathcal{J}_{\tpi_{\mu}(\sys)}(\sys) = \Tr\big(P(\wt{K}_\mu)\big) \geq \Tr (\Sigma_\mu ) \lambda_{\min}(C_\dagger^1). 
\end{equation*}
As a result,
\begin{equation*}
\frac{1}{1 - \rho(\Ac_\mu)^2} = \|\Sigma_{\mu}\|_2 \leq \Tr(\Sigma_\mu) \leq \mathcal{D}(\mu,\sys) / \lambda_{\min}(C_\dagger^1):= \kappa(\mu,\sys).
\end{equation*}
Finally, weak duality ensures that $\kappa(\mu,\sys) \leq \kappa(\sys) := \mathcal{J}_*(\sys) / \lambda_{\min}(C_\dagger^1)$.
\end{proof}

\begin{proof}[Proof of~\cref{prop:positive.gradient.bounded.closed.loop.controller}]
\begin{enumerate}
\item The first assertion follows from~\cref{thm:ues.rugh,prop:positive.gradient.stable.policy}. Since by~\cref{prop:positive.gradient.stable.policy}, $\kappa(\mu,\sys) \succcurlyeq\Sigma_\mu \succcurlyeq I$, and $\Ac_\mu \Sigma_\mu \big(\Ac_\mu\big)^\transp - \Sigma_\mu = -I$, ~\cref{thm:ues.rugh} ensures that for any $k \in \mathbb{N}$, 
\begin{equation*}
\left\|\big( \Ac_\mu \big)^k \right\|_2 \leq \sqrt{\kappa(\mu,\sys)} \big(1-1/\kappa(\mu,\sys)\big)^{-k}.
\end{equation*}
\item Decomposing $\wt{K}_{\mu} = \begin{pmatrix} K^{u}_{\mu} \\ K^{w}_{\mu} \end{pmatrix}$, where $K^{u}_{\mu} \in \mathbb{R}^{d,n}$, $K^{w}_{\mu} \in \mathbb{R}^{n,n}$ we have that
\begin{equation*}
\begin{aligned}
\begin{pmatrix} I \\ \wt{K}^{u}_\mu \end{pmatrix}^\transp \begin{pmatrix} I \\ \wt{K}^{u}_\mu \end{pmatrix} = \begin{pmatrix} I \\ \wt{K}_\mu \end{pmatrix}^\transp \begin{pmatrix} I_{n+d} & 0 \\ 0 & 0 \end{pmatrix} \begin{pmatrix} I \\ \wt{K}_\mu \end{pmatrix} &\preccurlyeq \frac{1}{\lambda_{\min}(C_{\dagger}^1)} \begin{pmatrix} I \\ \wt{K}_\mu \end{pmatrix}^\transp C_{\dagger}^1 \begin{pmatrix} I \\ \wt{K}_\mu \end{pmatrix} \preccurlyeq \frac{1}{\lambda_{\min}(C_{\dagger}^1)} \begin{pmatrix} I \\ \wt{K}_\mu \end{pmatrix}^\transp C_{\dagger} \begin{pmatrix} I \\ \wt{K}_\mu \end{pmatrix} \\
& \preccurlyeq \frac{1}{\lambda_{\min}(C_{\dagger}^1)} P(\wt{K}_\mu).
\end{aligned}
\end{equation*}
Hence, $\left\| \begin{pmatrix} I \\ \wt{K}^{u}_\mu \end{pmatrix}\right\|_2^2 \leq \frac{\Tr\big(P(\wt{K}_\mu)\big)}{\lambda_{\min}(C_{\dagger}^1)} \leq \kappa(\sys,\mu)$. Further, we have $K^{w}_\mu = \Ac_\mu - \begin{pmatrix} A & B \end{pmatrix} \begin{pmatrix} I \\ \wt{K}^{u}_\mu \end{pmatrix}$, we obtain 
\begin{equation*}
\|K^{w}_\mu\|_2 \leq \| \Ac_\mu\|_2 + \|A\|_2\|B\|_2  \left\| \begin{pmatrix} I \\ \wt{K}^{u}_\mu \end{pmatrix} \right\|_2\leq \sqrt{\kappa(\mu,\sys)}( 1 + \|A\|_2\|B\|_2).
\end{equation*}
As a result, since $\| \Ac_\mu\|_2 \leq \sqrt{\kappa(\mu,\sys)}$, we obtain $\left\| \begin{pmatrix} I \\ \wt{K}_\mu \end{pmatrix} \right\|_2 \leq \left\| \begin{pmatrix} I \\ \wt{K}^{u}_\mu \end{pmatrix} \right\|_2 + \|K^{w}_\mu\|_2 \leq \sqrt{\kappa(\mu,\sys)} (2 + \|A\|_2\|B\|_2)$.

\end{enumerate}

\end{proof}

\subsubsection{Closure of $\mathcal{M}$}
\label{app:sssec.closure.M}
From condition~\cref{eq:strong.duality.weak.positive.gradient.condition}, it is clear that all optimal policies $\{\tpi_{\mu}(\sys)\}_{\mu \in \mathcal{M}}$ will violate the constraint since it is equivalent to $g_{\tpi_\mu}(\sys) > 0$ for all $\mu \in \mathcal{M}$. However, we show here that when $\wt{\mu} - \mu$ is small, there exist other linear policies $\tpi$ whose constraint value $g_{\tpi}(\sys)$ may be very different than $g_{\tpi_{\mu}(\sys)}(\sys)$ but such that $\mathcal{L}_{\tpi}(\mu,\sys) - \mathcal{D}(\mu,\sys)$ is small.

\begin{proposition}
\label{prop:behavior.Dmu.closure.M}
For any $\sys \in \Sys$, let $\mathcal{M} = [0,\wt{\mu})$ be the admissible Riccati set associated with $\sys$ as defined in~\cref{eq:mutilde.definition}. For all $\mu \in \mathcal{M}$, let $\Ac_\mu$, $D_\mu$ be the associated Riccati quantities defined in~\cref{eq:admissible.riccati.solution}. Then, whenever $\wt{\mu} < + \infty$, it exists $|z|=1$ such that
\begin{equation}
\label{eq:riccati.closure.M}
	\lim_{\mu \rightarrow \wt{\mu}} | \det(Iz - \Ac_{\mu}) |^2 \det(D_{\mu}) = 0.
\end{equation}
\end{proposition}

\cref{prop:behavior.Dmu.closure.M} indicates that when $\mu$ approaches $\wt{\mu}$, the stability and positivity conditions in~\cref{eq:admissible.riccati.solution} saturate, that is $\rho(\Ac_\mu) \rightarrow 1$ and/or $\lambda_{\min}(D_\mu) \rightarrow 0$. However, according to~\cref{prop:positive.gradient.stable.policy}, the optimal policies $\{ \tpi_{\mu}(\sys)\}_{\mu \in \mathcal{M}}$ are uniformly stable when~\cref{eq:strong.duality.weak.positive.gradient.condition} holds. As a result, we have:
 \begin{corollary}
\label{cor:behavior.Dmu.closure.M}
For any $\sys \in \Sys$, let $\mathcal{M} = [0,\wt{\mu})$ be the admissible Riccati set associated with $\sys$ as defined in~\cref{eq:mutilde.definition}. For all $\mu \in \mathcal{M}$, let $\Ac_\mu$, $D_\mu$ be the associated Riccati quantities defined in~\cref{eq:admissible.riccati.solution}. Then, if for all $\mu \in \mathcal{M}$, $\mathcal{D}^\prime(\mu,\sys) > 0$, 
\begin{equation}
\label{eq:Dmu.closure.M}
	\lim_{\mu \rightarrow \wt{\mu}} \det(D_\mu) = 0.
\end{equation}
Further,
\begin{equation}
\label{eq:Dmu.closure.M.bound}
	\lambda_{\min}(D_\mu) \leq \alpha_D(\sys) |\wt{\mu} - \mu|, \text{ where } \quad \alpha_D(\sys) = 8 \|C_g \|_2 \kappa(\sys)^4 \left( (2 + \|A\|_2\|B\|_2) (1 + \|B\|_2) \right)^2.
\end{equation}
\end{corollary}

\paragraph{Comment on the role of $D_{\mu} = R_{\mu}+ \tilde{B}^\transp P_{\mu} \tilde{B}$.} As stressed in~\cref{cor:behavior.Dmu.closure.M}, we are facing here the case where $\lim_{\mu \rightarrow \wt{\mu}} \lambda_{\min}(D_\mu) = 0$. This paragraph aims at providing some insight on why this is equivalent to the fact that the optimal controller $\wt{K}_\mu$ become less and less "unique" as $\mu$ approaches $\wt{\mu}$. \\
First,  notice that~\cref{eq:admissible.riccati.solution} focuses on particular solution of the Riccati equation: \textbf{1)} the condition $\rho(\Ac_{\mu}) < 1$ imposes the solution to be stabilizing; \textbf{2)} the condition $D_\mu  \succ 0$ imposes some sort of global and unique optimality of the solution. Indeed, for any linear controller $\wt{\pi}(x) = \wt{K} x$, let $P_{\mu}(\wt{K})$ be the solution of the Lyapunov equation
 \begin{equation*}
 P_{\mu}(\wt K) = \Ac(\wt K)^\transp P_{\mu}(\wt K)\Ac(\wt K) + \begin{pmatrix}  I \\ \wt{K} \end{pmatrix}^\transp C_\mu \begin{pmatrix}  I \\ \wt{K} \end{pmatrix}
\end{equation*}
Then, $\mathcal L_{\tpi}(\mu) = \Tr\big( P_{\mu}(\wt K) \big)$ and the sub-optimality gap is given by $\mathcal L_{\tpi}(\mu) - \mathcal{D}(\mu) =  \Tr\big( \Delta)$ where $\Delta = P_{\mu}(\wt{K}) -  P_{\mu}$. Moreover, algebraic manipulations ensure that $\Delta$ is the solution of the Lyapunov equation:
\begin{equation}\label{eq:lyap.suboptimal.K}
\Delta = \Ac(\wt K)^\transp \Delta \Ac(\wt K) + (\wt K_{\mu} - \wt K)^\transp D_\mu (\wt K_{\mu} - \wt K).
\end{equation}
\cref{eq:lyap.suboptimal.K} links the sub-optimality gap to the quantity $(\wt K_{\mu} - \wt K)^\transp D_\mu (\wt K_{\mu} - \wt K)$. As a result, global optimality is asserted as soon as $D_\mu \succcurlyeq 0$ since this implies that $\Delta \succcurlyeq 0$. Further, $D_\mu \succ 0$ guarantees the uniqueness, while there exists a subspace of optimal solution if $D_\mu$ is rank deficient. Finally, when $D_\mu$ is arbitrarily close to $0$, it may exist controllers $\wt K \neq \wt K_{\mu}$ with arbitrarily close value $\mathcal L_{\tpi}(\mu) \approx \mathcal{D}(\mu)$.\\

\begin{proof}[Proof of~\cref{prop:behavior.Dmu.closure.M}]
The proof follows from~\cref{eq:popov.K.definition},~\cref{eq:popov.identities} and~\cref{eq:mutilde.definition}. Let $\wb{K} = \begin{pmatrix} 0 \\ -A \end{pmatrix}$, which is such that $\Ac(\wb{K}) = A + \wt{B} \wb{K} = A -A = 0$. The definition of $\wt{\mu}$ in~\cref{eq:mutilde.definition} and the fact that $\wt{\mu} < +\infty$ ensure that it exists at least a $|z|=1$ such that 
\begin{equation*}
\det\big( \Psi^{\wb{K}}_{\wt{\mu}}(z,\sys)\big) = 0,
\end{equation*}
hence the linearity of $\Psi$ w.r.t. $\mu$ and~\cref{eq:popov.K.definition} leads to $\lim_{\mu \rightarrow \wt{\mu}} \det\big( \Psi^{\wb{K}}_{\mu}(z,\sys)\big) = 0$.
Further,~\cref{eq:popov.identities} provides that for all $\mu \in \mathcal{M}$,
\begin{equation*}
\begin{aligned}
\det\big( \Psi^{\wb{K}}_{\mu}(z,\sys)\big) &= \det\left( \big(I + z (\wb{K} - \wt{K}_{\mu}) \wt{B} \big)^\transp D_{\mu} \big(I + z^{-1} (\wb{K} - \wt{K}_{\mu}) \wt{B} \big)\right) \\
&= \det(D_\mu) \left | \det \big(Iz +  (\wb{K} - \wt{K}_{\mu}) \wt{B} \big)\right |^2.
\end{aligned}
\end{equation*}
Using Sylvester's determinant identity ($\det(I + XY) = \det(I + Y X)$ for any matrices $X,Y$), and the fact that $\wt{B}(\wb{K} - \wt{K}_\mu) = -\Ac_\mu$, we obtain
\begin{equation}
\label{eq:PsiK.uniform.convergence.0}
\det\big( \Psi^{\wb{K}}_{\mu}(z,\sys)\big) = \det(D_\mu) \left | \det \big(Iz -  \Ac_\mu \big) \right |^2
\end{equation}
Thus,
\begin{equation*}
\lim_{\mu \rightarrow \wt{\mu}}\det(D_\mu) \left | \det \big(Iz -  \Ac_\mu \big) \right |^2 = 0.
\end{equation*}
\end{proof}

\begin{proof}[Proof of~\cref{cor:behavior.Dmu.closure.M}]
The proof of~\cref{cor:behavior.Dmu.closure.M} directly follows from~\cref{prop:positive.gradient.bounded.domain,prop:positive.gradient.stable.policy,prop:behavior.Dmu.closure.M}.
\begin{enumerate}
\item Since $\rho(\Ac_\mu)< 1$ for all $\mu \in \mathcal{M}$, $\left | \det \big(Iz -  \Ac_\mu \big) \right |^2 > 0$ for all $|z|=1$. Hence, we have necessarily that $\lim_{\mu \rightarrow \wt{\mu}} \det(D_\mu) = 0$.

\item $\lim_{\mu \rightarrow \wt{\mu}} \det(D_\mu) = 0$ implies that for any $\wt{K}$ such that $|\lambda(\Ac(\wt{K}))| \neq 1$, 
\begin{equation*}
\det\big( \Psi^{\wt{K}}_{\wt{\mu}}(z,\sys)\big) = 0 \quad \forall |z|=1 \quad \Rightarrow \quad \det\big( \Psi^{\wt{K}}_{\wt{\mu}}(1,\sys)\big) = 0\quad \Rightarrow \quad \lambda_{\min}\big( \Psi^{\wt{K}}_{\wt{\mu}}(1,\sys)\big)=0.
\end{equation*}
Fix $\mu \in \mathcal{M}$. By definition $|\lambda(\Ac(\wt{K}_\mu))| \neq 1$, and thus $\lambda_{\min}\big( \Psi^{\wt{K}_\mu}_{\wt{\mu}}(1,\sys)\big)=0$.
Let $v_{\min}$ be an eigenvector associated with the zero eigenvalue of $\Psi^{\wt{K}_{\mu}}_{\wt{\mu}}(1,\sys)$. The linearity of $\Psi$ leads to
\begin{equation*}
0 = v_{\min}^\transp \Psi^{\wt{K}_{\mu}}_{\wt{\mu}}(1,\sys) v_{\min} = v_{\min}^\transp \Psi^{\wt{K}_\mu}_{\mu}(1,\sys) v_{\min} + 
(\wt{\mu} - \mu) v_{\min}^\transp \begin{pmatrix} (I  - \Ac_\mu))^{-1} \tilde{B} \\ I \end{pmatrix}^\transp \begin{pmatrix} I & \wt{K}_{\mu}^\transp \\ 0 & I \end{pmatrix} C_g \begin{pmatrix} I & 0 \\ \wt K_{\mu} & I \end{pmatrix} \begin{pmatrix} (I  - \Ac_\mu)^{-1} \tilde{B} \\ I \end{pmatrix}v_{\min}.
\end{equation*}
Thus, 
\begin{equation*}
\begin{aligned}
 v_{\min}^\transp \Psi^{\wt{K}_\mu}_{\mu}(1,\sys) v_{\min} &\leq |\wt{\mu} - \mu|\|C_g \|_2 \left\| \begin{pmatrix} I & 0 \\ \wt K_{\mu} & I \end{pmatrix} \right\|_2^2 v_{\min}^\transp  \big(I + \wt{B}^\transp (I - \Ac_\mu)^{-\transp}(I - \Ac_\mu)^{-1} \wt{B} \big) v_{\min} \\
&\leq  2 |\wt{\mu} - \mu|\|C_g \|_2 \left(1 + \left\| \begin{pmatrix} I \\ \wt K_{\mu}\end{pmatrix} \right\|_2^2 \right) \big( 1 + \|\wt{B}\|_2^2 \|(I - \Ac_\mu)^{-1}\|_2^2 \big).
 \end{aligned}
\end{equation*}
From~\cref{prop:positive.gradient.bounded.closed.loop.controller}, we have $\left\| \begin{pmatrix} I \\ \wt K_{\mu}\end{pmatrix} \right\|_2^2 \leq \kappa(\sys) (2 + \|A\|_2\|B\|_2)^2$. Further, since $\Ac_\mu$ is stable, 
\begin{equation*}
\|(I - \Ac_\mu)^{-1}\|_2  = \left\| \sum_{k=0}^\infty \big(\Ac_\mu\big)^k \right\|_2 \leq  \kappa(\sys)^{3/2}.
\end{equation*}
As a result, 
\begin{equation*}
\begin{aligned}
 v_{\min}^\transp \Psi^{\wt{K}_\mu}_{\mu}(1,\sys) v_{\min} &\leq 2 |\wt{\mu} - \mu|\|C_g \|_2 \big( 1 + \|\wt{B}\|_2^2 \kappa(\sys)^3 \big) \big(1 +  \kappa(\sys) (2 + \|A\|_2\|B\|_2)^2 \big) \\
 &\leq 8 |\wt{\mu} - \mu|\|C_g \|_2 \kappa(\sys)^4 \left( (2 + \|A\|_2\|B\|_2) (1 + \|B\|_2) \right)^2.
 \end{aligned}
\end{equation*}
We conclude noting that by~\cref{eq:popov.identities}, we have
\begin{equation*}
\lambda_{\min}(D_\mu) \leq v_{\min}^\transp D_\mu v_{\min} =  v_{\min}^\transp \Psi^{\wt{K}_\mu}_{\mu}(1,\sys) v_{\min}.
\end{equation*}

\end{enumerate}
\end{proof}

\subsubsection{Feasible and $\epsilon-$optimal linear policies}
\label{app:sssec.linear.epsilon.optimal.feasible.policies}

When for all $\mu \in \mathcal{M}$, $\mathcal{D}^\prime(\mu,\sys) > 0$,~\cref{cor:behavior.Dmu.closure.M} shows that $\lim_{\mu \rightarrow \wt{\mu}} \lambda_{\min}(D_\mu) = 0$. Further, the previous comment suggests that it exists linear policies $\tpi$ that yield close-to-optimal performances but can differ significantly from $\tpi_\mu$.~\cref{lem:strong.duality.weak.epsilon.feasible.controller} shows that this is enough to guarantee the existence of $\epsilon-$ optimal feasible linear policy, while such policies are explicitely given in~\cref{prop:strong.duality.weak.epsilon.feasible.controller.null.space.kerBtilde,prop:strong.duality.weak.epsilon.feasible.controller.null.space.imBtilde}, either in closed-form, either as the optimal solution of a modified well-posed Riccati equation.
\begin{lemma}
\label{lem:strong.duality.weak.epsilon.feasible.controller}
For any $\sys \in \Sys$, let $\mathcal{M}$ be the admissible Riccati set associated with $\sys$ as defined in~\cref{eq:mutilde.definition}. If for all $\mu \in \mathcal{M}$, $\mathcal{D}^\prime(\mu,\sys) > 0$, then, for any $\epsilon > 0$, it exists $\tpi^\epsilon(x) = \wt{K}^\epsilon x$ and $\mu^\epsilon \in \mathcal{M}$ such that
\begin{enumerate} 
\item $g_{\tpi^\epsilon}(\sys) \leq 0$,
\item $\mu^\epsilon g_{\tpi^\epsilon}(\sys) = 0$,
\item $\mathcal{D}(\mu^\epsilon,\sys) \geq \mathcal{L}_{\tpi^\epsilon}(\mu^\epsilon,\sys) - \epsilon$.
\end{enumerate}
\end{lemma}
~\cref{lem:strong.duality.weak.epsilon.feasible.controller} proves strong-duality (in a weak "$\sup-\min$" sense) when $\mathcal{D}^\prime(\mu,\sys) > 0$ for all $\mu \in \mathcal{M}$. Indeed, it implies that
\begin{equation*}
\mathcal{J}_*(\sys) \overset{(1)}{\geq} \sup_{\mu \in \mathcal{M}} \mathcal{D}(\mu,\sys) \overset{(2)}{\geq} \mathcal{D}(\mu^\epsilon,\sys) \overset{(3)}{\geq} \mathcal{L}_{\tpi^\epsilon}(\mu^\epsilon,\sys) - \epsilon  \overset{(4)}{=} \mathcal{J}_{\tpi^\epsilon} \overset{(5)}{\geq} \mathcal{J}_*(\sys)- \epsilon,
\end{equation*}
where $(1)$ follows from weak duality, $(2)$ from $\mu^\epsilon \in \mathcal{M}$, $(3)$ from the $\epsilon-$optimality of $\tpi^\epsilon$, $(4)$ from complementary slackness and $(5)$ from primal feasibility.\\
As a result,~\cref{lem:strong.duality.weak.epsilon.feasible.controller} in conjunction with~\cref{prop:strong.duality.strict} proves~\cref{thm:dual.gap}.\\

We are left to prove~\cref{lem:strong.duality.weak.epsilon.feasible.controller}. The analysis is conducted differently depending on the limiting null-space of $D_\mu$. Formally, we distinguish between the following cases:

\begin{itemize}[leftmargin=2cm]
\item[\textbf{Case A).}] $\lambda_{\ker(\wt{B})}(D_{\mu^\epsilon}) < \sqrt{\lambda_{\min}(D_{\mu^{\epsilon}})}$, 
\item[\textbf{Case B).}] $\lambda_{\ker(\wt{B})}(D_{\mu^\epsilon}) \geq \sqrt{\lambda_{\min}(D_{\mu^\epsilon})}$,
\end{itemize}
where $\lambda_{\ker(\wt{B})}(D_{\mu^\epsilon}) = \min_{\|v\|=1, v \in \ker(\wt{B})} v^\transp D_{\mu^\epsilon} v$. Since $\mathcal{D}^\prime(\mu,\sys) > 0$ for all $\mu \in \mathcal{M}$,~\cref{prop:behavior.Dmu.closure.M,prop:positive.gradient.bounded.domain} ensure that for any arbitrary $1 > \epsilon > 0$, it exists $\mu^\epsilon$ such that
\begin{equation*}
\lambda_{\min}(D_{\mu^\epsilon}) = \nu(\epsilon)^2, \text{ where } \quad \nu(\epsilon) = \min\left( \frac{\lambda_{\min}(C_\dagger^1)}{ 2 \|\wt{B}\|_2^2 \max(\mathcal{J}_*(\sys),1)}, \frac{1}{ 8^{2n+1} \kappa(\sys)^{2n} } \min\left(1 , \frac{\lambda_{\min}(C_{\dagger}^1)  \sigma_{\wt{B}}^2}{ 2 \kappa(\sys)^2c({\wb{\mu}})} \right) \right)^2 \epsilon ^2.
\end{equation*}

In \textbf{Case A)}, i.e., when $\lambda_{\ker(\wt{B})}(D_{\mu^\epsilon}) < \nu(\epsilon)$, we explicitly provide an $\epsilon-$optimal feasible controller. Formally, we use~\cref{prop:strong.duality.weak.epsilon.feasible.controller.null.space.kerBtilde} (we postpone the proof in to the next subsection), which explicit a modified policy linear $\tpi^\epsilon$ given by the linear controller $\wt{K}^\epsilon$.
\begin{enumerate}
\item $g_{\tpi^\epsilon}(\sys) = 0$,
\item $\mathcal{D}(\mu^\epsilon,\sys) \geq \mathcal{L}_{\tpi^\epsilon}(\mu^\epsilon,\sys) - \epsilon$.
\end{enumerate}
\begin{proposition}
\label{prop:strong.duality.weak.epsilon.feasible.controller.null.space.kerBtilde}
For any $\sys \in \Sys$, let $\mathcal{M}$ be the admissible Riccati set associated with $\sys$ as defined in~\cref{eq:mutilde.definition}. For any $\mu \in \mathcal{M}$, let $\wt{K}_\mu$ and $\Ac_\mu$ be respectively the optimal controller and closed-loop matrix defined by~\cref{eq:admissible.riccati.solution}, and let $\Sigma_{\mu} = \Ac_\mu \Sigma \big(\Ac_\mu\big)^\transp + I$ be its associated steady-state covariance matrix. If it exists $\wb{\mu} \in \mathcal{M}$ such that 
\begin{enumerate}
\item  $\mathcal{D}(\mu,\sys) > 0$
\item  $\lambda_{\ker(\wt{B})}(D_{\wb{\mu}}) \leq \frac{\lambda_{\min}(C_\dagger^1)}{ 2 \|\wt{B}\|_2^2 \max(\mathcal{J}_*(\sys),1)}$ where $\lambda_{\ker(\wt{B})}(D_{\wb{\mu}}) = \min_{\|v\|=1, v \in \ker(\wt{B})} v^\transp D_{\wb{\mu}} v$,
\end{enumerate}
 then, the modified linear policy $\bpi$ given by the modified controller $\wb{K}$ defined as:
\begin{equation}
\label{eq:strong.duality.weak.nullspace.kerBtilde.modified.controller}
\begin{aligned}
& \wb{K} = \wt{K}_{\wb{\mu}} + \delta K; \quad \delta K = \alpha v x^\transp, \\
& v = \arg \min_{\|v\|=1, v \in \ker(\wt{B})} v^\transp D_{\wb{\mu}} v, \\
& x \in \mathbb{R}^{n} \text{ is such that } \|x\|=1 \text{ and} \quad v^\transp Y \Sigma_{\wb{\mu}} x = 0, \text{ where } Y =  \begin{pmatrix} 0 \\ I_{n+d} \end{pmatrix}^\transp C_g \begin{pmatrix} I \\ \wt{K}_{\wb{\mu}} \end{pmatrix}, \\
& \alpha ^2 = \frac{g_{\tpi_{\wb{\mu}}(\sys)}(\sys)}{ - v^\transp Z v x^\transp \Sigma_{\wb{\mu}} x } \quad \text{ where } \quad Z = \begin{pmatrix} 0 \\ I_{n+d} \end{pmatrix}^\transp C_g \begin{pmatrix} 0 \\ I_{n+d} \end{pmatrix}.
\end{aligned}
\end{equation}
is well defined. Further, $\mathcal{D}(\wb{\mu},\sys) \geq \mathcal{L}_{\bpi}(\wb{\mu},\sys) - \lambda_{\ker(\wt{B})}(D_{\wb{\mu}}) \frac{2 \|\wt{B}\|_2^2 \mathcal{J}_*(\sys)}{ \lambda_{\min}(C_\dagger^1)}$ and $g_{\bpi}(\sys) = 0$.
\end{proposition}

In \textbf{Case B)}, i.e., when $\lambda_{\ker(\wt{B})}(D_{\mu^\epsilon}) \geq \nu(\epsilon) $, we have that $\lambda_{\min}(D_{\mu^\epsilon}) /\lambda_{\ker(\wt{B})}(D_{\mu^\epsilon}) \leq \nu(\epsilon)$. We implicitely provide an $\epsilon-$optimal feasible controller as the solution of the dual problem of a modified system. Formally, we use~\cref{prop:strong.duality.weak.epsilon.feasible.controller.null.space.imBtilde} (we postpone the proof in to the next subsection), which guarantees the existence of an optimal linear policy $\tpi^\epsilon$ w.r.t. to a modified system for a Lagrangian parameter $\mu^\epsilon_* \in \mathcal{M}$ such that 
\begin{enumerate}
\item $g_{\tpi^\epsilon}(\sys) \leq 0$,
\item $\mu^\epsilon_* g_{\tpi^\epsilon}(\sys) = 0$,
\item $\mathcal{D}(\mu^\epsilon,\sys) \geq \mathcal{L}_{\tpi^\epsilon}(\mu^\epsilon,\sys) -  \epsilon$.
\end{enumerate}

\begin{proposition}
\label{prop:strong.duality.weak.epsilon.feasible.controller.null.space.imBtilde}
For any $\sys \in \Sys$, let $\mathcal{M}$ be the admissible Riccati set associated with $\sys$ as defined in~\cref{eq:mutilde.definition}. For any $\mu \in \mathcal{M}$, let $\Ac_\mu$, $D_\mu$ be defined by~\cref{eq:admissible.riccati.solution} and $\wt{K}_\mu$ be the optimal controller. We assume that it exists $\wb{\mu} \in \mathcal{M}$ such that 
\begin{enumerate}
\item  $\mathcal{D}(\wb{\mu},\sys) > 0$
\item  $\frac{\lambda_{\min}(D_{\wb{\mu}})}{ \lambda_{\ker(\wt{B})}(D_{\wb{\mu}})} \leq \frac{1}{ 8^{2n+1} \kappa(\sys)^{2n} } \min\left(1, \frac{\lambda_{\min}(C_{\dagger}^1) \sigma_{\wt{B}}^2}{ 2 c({\wb{\mu}}) \kappa(\sys)}  \right)$,
\end{enumerate}
where 
\begin{equation*}
\begin{aligned}
&\lambda_{\ker(\wt{B})}(D_{\wb{\mu}}) = \min_{\|v\|=1, v \in \ker(\wt{B})} v^\transp D_{\wb{\mu}} v,  \quad \quad \sigma_{\wt{B}}^2 = \min_{\|v\|=1, v \in Im(\wt{B})} v^\transp \wt{B}^\transp \wt{B} v, \\
&\kappa(\sys) = \mathcal{J}_*(\sys) / \lambda_{\min}(C_{\dagger}^1), \quad \quad c(\wb{\mu}) = \big(\lambda_{\max}(C_{\dagger}) +\wb{\mu} \big)\big( 1 + \|\wt{B}\|_2^2 (1 + \|A\|_2^2)\big).
\end{aligned}
\end{equation*}
Then, it exists a linear policy $\tpi(\sys^\eta)$ and a Lagrangian parameter $\mu^\eta \in \mathcal{M}$ such that 
\begin{enumerate}
\item $g_{\tpi(\sys^\eta)}(\sys) \leq 0$,
\item $\mu^\eta g_{\tpi(\sys^\eta)}(\sys) = 0$,
\item $\mathcal{D}(\mu^\eta,\sys) \geq \mathcal{L}_{\tpi(\sys^\eta)}(\mu^\eta,\sys) -  \frac{\lambda_{\min}(D_{\wb{\mu}})}{\lambda_{\ker(\wt{B})}(D_{\wb{\mu}})} 2 \kappa(\sys)^2 \frac{c({\wb{\mu}})}{ \sigma_{\wt{B}}^2} 8^{2n+1} \kappa(\sys)^{2n}$.
\end{enumerate}
\end{proposition}

For \textbf{Case A)} and \textbf{Case B)}, we constructed a feasible linear policy $\tpi^\epsilon$ that is $\epsilon-$ optimal, which concludes the proof of~\cref{lem:strong.duality.weak.epsilon.feasible.controller}
\subsubsection{Proofs}
\label{app:sssec.aux.proofs}

\paragraph{Proof of~\cref{prop:strong.duality.weak.epsilon.feasible.controller.null.space.kerBtilde}.} We first show that $\wb{K}$ is well-defined, then quantify its sub-optimality gap and finally show that it satisfies primal feasibility.

\begin{enumerate}
\item Clearly, the existence of $x$ is asserted (it is simply a unitary vector orthogonal to $y = Y \Sigma_{\wb{\mu}} Y^\transp v$ that is well defined). Further, $g_{\tpi_{\wb{\mu}}(\sys)}(\sys) = \mathcal{D}(\wb{\mu},\sys) > 0$ and $x^\transp \Sigma_{\wb{\mu}} x \geq 1 > 0$. Thus, we only have to show that $-v^\transp Z v > 0$ to assert the existence of $\wb{K}$.\\
From $v\in\ker(\wt{B})$, we have
\begin{equation*}
v^\transp R_{\wb{\mu}} v = v^\transp (R_{\wb{\mu}} + \wt{B}^\transp P_{\wb{\mu}} \wt{B} ) v = v^\transp D_{\wb{\mu}} v = \lambda_{\ker(\wt{B})}(D_{\wb{\mu}}) \leq\frac{\lambda_{\min}(C_\dagger^1)}{ 2 \|\wt{B}\|_2^2 \max(\mathcal{J}_*(\sys),1)}.
\end{equation*}
Further, 
\begin{equation*}
R_{\wb{\mu}} = \begin{pmatrix} 0 \\ I_{n+d} \end{pmatrix}^\transp C_{\wb{\mu}} \begin{pmatrix} 0 \\ I_{n+d} \end{pmatrix} = \begin{pmatrix} 0 \\ I_{n+d} \end{pmatrix}^\transp C_\dagger \begin{pmatrix} 0 \\ I_{n+d} \end{pmatrix} + \wb{\mu} \begin{pmatrix} 0 \\ I_{n+d} \end{pmatrix}^\transp C_g \begin{pmatrix} 0 \\ I_{n+d} \end{pmatrix} =  \begin{pmatrix} 0 \\ I_{n+d} \end{pmatrix}^\transp C_\dagger \begin{pmatrix} 0 \\ I_{n+d} \end{pmatrix} + \wb{\mu}  Z.
\end{equation*}
Since $\sys \in \Sys$, we have that $C_\dagger \succcurlyeq \begin{pmatrix} C_\dagger^1 & 0_{n,n+d} \\ 0_{n+d,n} & 0_{n} \end{pmatrix}$ and $C_\dagger^1 \in \mathbb{S}^{n+d}_{++}$. Thus,
\begin{equation*}
R_{\wb{\mu}} \succcurlyeq \begin{pmatrix} 0_{n,d} & 0_{n,n} \\ I_{d} & 0_{d,n} \\ 0_{n,d} & I_{n} \end{pmatrix}^\transp \begin{pmatrix} C_\dagger^1 & 0_{n,n+d} \\ 0_{n+d,n} & 0_{n} \end{pmatrix} \begin{pmatrix} 0_{n,d} & 0_{n,n} \\ I_{d} & 0_{d,n} \\ 0_{n,d} & I_{n} \end{pmatrix} + \wb{\mu}  Z 
= \begin{pmatrix} 0_{n,d} & 0_{n,n} \\ I_{d} & 0_{d,n} \end{pmatrix}^\transp C_\dagger^1 \begin{pmatrix} 0_{n,d} & 0_{n,n} \\ I_{d} & 0_{d,n} \end{pmatrix} + \wb{\mu}  Z.
\end{equation*}
Now, decomposing $v\in\mathbb{R}^{n+d}$ into control $u \in \mathbb{R}^d$ and perturbation $w \in \mathbb{R}^n$ as $v^\transp = \begin{pmatrix} u^\transp & w^\transp \end{pmatrix}$, we have 
\begin{equation*}
v^\transp R_{\wb{\mu}} v \geq \begin{pmatrix} 0 \\ u \end{pmatrix}^\transp C_\dagger^1  \begin{pmatrix} 0 \\ u \end{pmatrix} + \wb{\mu}  v^\transp Z v \geq \|u\|^2 \lambda_{\min}(C_\dagger^1) + \wb{\mu}  v^\transp Z v.
\end{equation*}
From $\wt{B} v = 0$  we have that $B u + w = 0$ and hence that $v = \begin{pmatrix} u \\ - B u \end{pmatrix} = \begin{pmatrix} I_d \\ - B\end{pmatrix} u$. Further, $\|v\|^2=1$ implies that 
\begin{equation*}
1 = u^\transp \begin{pmatrix} I_{d} & -B^\transp \end{pmatrix} \begin{pmatrix} I_{d} \\ - B \end{pmatrix} u = u^\transp (I + B^\transp B) u \leq \lambda_{\max}(I + B^\transp B) \|u\|^2 = \lambda_{\max}(I + B B^\transp) \|u\|^2 = \|\wt{B}\|_2^2 \|u\|^2.
\end{equation*}
Summarizing, we have that
\begin{equation*}
\begin{aligned}
&\frac{\lambda_{\min}(C_\dagger^1)}{ 2 \|\wt{B}\|_2^2 \max(\mathcal{J}_*(\sys),1)} \geq v^\transp R_{\wb{\mu}} v \geq  \lambda_{\min}(C_\dagger^1) / \|\wt{B}\|_2^2 + \wb{\mu}  v^\transp Z v \\
\Rightarrow  \quad\quad &\wb{\mu} v^\transp Z v \leq \frac{\lambda_{\min}(C_\dagger^1)}{ 2 \|\wt{B}\|_2^2} -  \lambda_{\min}(C_\dagger^1) / \|\wt{B}\|_2^2 \leq -\frac{\lambda_{\min}(C_\dagger^1)}{ 2 \|\wt{B}\|_2^2} < 0.\\
\end{aligned}
\end{equation*}
Hence, necessarily, $\wb{\mu} \neq 0$ and  $v^\transp Z v < 0$.
\item We now quantify the sub-optimality of $\wb{K}$. From~\cref{eq:lyap.suboptimal.K}
\begin{equation*}
P_{\wb{\mu}}(\wb{K}) - P_{\wb{\mu}} = \big(\Ac_{\wb{\mu}}\big)^\transp (P_{\wb{\mu}}(\wb{K}) - P_{\wb{\mu}} ) \Ac_{\wb{\mu}} + \delta K^\transp D_{\wb{\mu}} \delta K.
\end{equation*}
Further, $ \delta K^\transp D_{\wb{\mu}} \delta K = \alpha^2 v^\transp D_{\wb{\mu}} v \hspace{1mm} x x^\transp$, thus we obtain
\begin{equation*}
\mathcal{L}_{\bpi}(\wb{\mu},\sys) - \mathcal{D}(\wb{\mu},\sys) \leq   \alpha^2 v^\transp D_{\wb{\mu}} v \Tr \left( \Sigma_{\wb{\mu}}  x x^\transp \right) 
\leq \alpha^2 v^\transp D_{\wb{\mu}} v x^\transp \Sigma_{\wb{\mu}} x \leq v^\transp D_{\wb{\mu}} v \frac{g_{\tpi_{\wb{\mu}}(\sys)}(\sys)}{-v^\transp Z v}.
\end{equation*}

From $ v^\transp Z v \leq -\frac{\lambda_{\min}(C_\dagger^1)}{ 2 \wb{\mu} \|\wt{B}\|_2^2}$ and $\wb{\mu} g_{\tpi_{\wb{\mu}}(\sys)}(\sys) \leq \mathcal{J}_{\tpi_{\wb{\mu}}(\sys)}(\sys) + \wb{\mu} g_{\tpi_{\wb{\mu}}(\sys)}(\sys) = \mathcal{D}(\wb{\mu},\sys) \leq \mathcal{J}_*(\sys)$ we obtain
\begin{equation*}
\mathcal{L}_{\bpi}(\wb{\mu},\sys) - \mathcal{D}(\wb{\mu},\sys) \leq   v^\transp D_{\wb{\mu}} v \frac{2 \|\wt{B}\|_2^2 \mathcal{J}_*(\sys)}{ \lambda_{\min}(C_\dagger^1)} \leq \lambda_{\ker(\wt{B})}(D_{\wb{\mu}}) \frac{2 \|\wt{B}\|_2^2 \mathcal{J}_*(\sys)}{ \lambda_{\min}(C_\dagger^1)}.
\end{equation*}

\item We now show that $\wb{K}$ is feasible as $g_{\bpi}(\sys) = 0$. From~\cref{eq:average.cost.constraints.lyapunov1,eq:average.cost.constraints.lyapunov2},
\begin{equation*}
G(\wb{K}) =  \big(\Ac_{\wb{\mu}}\big)^\transp G(\wb{K})  \big(\Ac_{\wb{\mu}}\big) + \begin{pmatrix} I \\ \wb{K} \end{pmatrix}^\transp C_g  \begin{pmatrix} I \\ \wb{K} \end{pmatrix}.
\end{equation*}
Further, 
\begin{equation*}
 \begin{pmatrix} I \\ \wb{K} \end{pmatrix}^\transp C_g  \begin{pmatrix} I \\ \wb{K} \end{pmatrix} = \begin{pmatrix} I \\ \wt{K}_{\wb{\mu}} \end{pmatrix}^\transp C_g  \begin{pmatrix} I \\ \wt{K}_{\wb{\mu}} \end{pmatrix}  + \delta K^\transp Y + Y^\transp \delta K + \delta K^\transp Z \delta K.
 \end{equation*}
As a result, 
\begin{equation*}
\begin{aligned}
g_{\bpi}(\sys) = \Tr\big( G(\wb{K}) \big) &= \Tr\big(G(\wt{K}_{\wb{\mu}})\big) + \Tr\big( \Sigma_{\wb{\mu}} (\delta K^\transp Y + Y^\transp \delta K) \big) + \Tr\big( \Sigma_{\wb{\mu}} \delta K^\transp Z \delta K\big) \\
&= g_{\tpi_{\wb{\mu}}(\sys)}(\sys)  + 2 \Tr\big( Y \Sigma_{\wb{\mu}} \delta K) +  \Tr\big( \Sigma_{\wb{\mu}} \delta K^\transp Z \delta K\big)\\
& = g_{\tpi_{\wb{\mu}}(\sys)}(\sys) + 2 \alpha v^\transp Y \Sigma_{\wb{\mu}} x +  \Tr\big( \Sigma_{\wb{\mu}} \delta K^\transp Z \delta K\big) \\
&=  g_{\tpi_{\wb{\mu}}(\sys)}(\sys) + \Tr\big( \Sigma_{\wb{\mu}} \delta K^\transp Z \delta K\big)\\
&= g_{\tpi_{\wb{\mu}}(\sys)}(\sys)+ \alpha^2 v^\transp Z v \hspace{1mm} \Tr\big( \Sigma_{\wb{\mu}} x x^\transp\big)\\
&= g_{\tpi_{\wb{\mu}}(\sys)}(\sys) +  \alpha^2 v^\transp Z v  \hspace{1mm} x^\transp \Sigma_{\wb{\mu}} x = g_{\tpi_{\wb{\mu}}(\sys)}(\sys) - g_{\tpi_{\wb{\mu}}(\sys)}(\sys) = 0.
\end{aligned}
\end{equation*}

\end{enumerate}

\paragraph{Proof of~\cref{prop:strong.duality.weak.epsilon.feasible.controller.null.space.imBtilde}.} 
We prove~\cref{prop:strong.duality.weak.epsilon.feasible.controller.null.space.imBtilde} as a corollary of the following proposition.
\begin{proposition}
\label{prop:modified.system.property}
For any $\sys \in \Sys$, let $\mathcal{M}$ be the admissible Riccati set associated with $\sys$ as defined in~\cref{eq:mutilde.definition}. For any $\mu \in \mathcal{M}$, let $\Ac_\mu$, $D_\mu$ be defined by~\cref{eq:admissible.riccati.solution} and $\wt{K}_\mu$ be the optimal controller. We assume that it exists $\wb{\mu} \in \mathcal{M}$ such that 
\begin{enumerate}
\item  $\mathcal{D}(\wb{\mu},\sys) > 0$
\item  $\frac{\lambda_{\min}(D_{\wb{\mu}})}{ \lambda_{\ker(\wt{B})}(D_{\wb{\mu}})} \leq \frac{1}{ 8^{2n+1} \kappa(\sys)^{2n} } \min\left(1, \frac{\lambda_{\min}(C_{\dagger}^1) \sigma_{\wt{B}}^2}{ 2 c({\wb{\mu}}) \kappa(\sys)}  \right)$,
\end{enumerate}
where 
\begin{equation*}
\begin{aligned}
&\lambda_{\ker(\wt{B})}(D_{\wb{\mu}}) = \min_{\|v\|=1, v \in \ker(\wt{B})} v^\transp D_{\wb{\mu}} v,  \quad \quad \sigma_{\wt{B}}^2 = \min_{\|v\|=1, v \in Im(\wt{B})} v^\transp \wt{B}^\transp \wt{B} v, \\
&\kappa(\sys) = \mathcal{J}_*(\sys) / \lambda_{\min}(C_{\dagger}^1), \quad \quad c(\wb{\mu}) = \big(\lambda_{\max}(C_{\dagger}) +\wb{\mu} \big)\big( 1 + \|\wt{B}\|_2^2 (1 + \|A\|_2^2)\big).
\end{aligned}
\end{equation*}

Let $\sys^\eta$ be a modified system w.r.t. $\sys$  as $\sys^\eta = (A,\wt{B}, C_\dagger + \eta \Delta, C_g)$ where 
\begin{equation}
\label{eq:modified.sys}
\begin{aligned}
& \Delta =  \begin{pmatrix} I & - \wb{K}^\transp \\ 0 & I\end{pmatrix}  \begin{pmatrix} I \\ - \wt{B}^\transp \end{pmatrix} \begin{pmatrix} I & -  \wt{B} \end{pmatrix} \begin{pmatrix} I & 0 \\ - \wb{K} & I \end{pmatrix}, \\
& \wb{K} = \begin{pmatrix} 0 \\ - A \end{pmatrix} \\
& \frac{\lambda_{\min}(C_{\dagger}^{1})}{2 \kappa(\sys)} \geq \eta \geq \frac{c(\wb{\mu})}{\sigma_{\wt{B}}^2}\frac{8^{2n+1} \kappa(\sys)^{2n}\lambda_{\min}(D_{\wb{\mu}})}{\lambda_{\ker(\wt{B})}(D_{\wb{\mu}})}.
\end{aligned}
\end{equation}
Let $\mathcal{M}^\eta$ be the admissible Riccati set associated with $\sys^\eta$ as defined in~\cref{eq:mutilde.definition}. Then,
\begin{enumerate}
\item $\sys^\eta \in \Sys$,
\item $[0, \wb{\mu}] \subset \mathcal{M} \subset \mathcal{M}^\eta$,
\item for all $\mu \in [0, \wb{\mu}]$, $\mathcal{D}(\mu,\sys^\eta) \leq \mathcal{D}(\mu,\sys) + 2 \eta \kappa(\sys)^2$ and  $\mathcal{D}(\mu,\sys^\eta)  \geq \mathcal{L}_{\tpi_\mu(\sys^\eta)}(\mu,\sys)$ where $\tpi_{\mu}(\sys^\eta) = \arg\min_{\tpi} \mathcal{L}_{\tpi}(\mu,\sys^\eta)$,
\item $\mathcal{D}^\prime(\wb{\mu},\sys^\eta) < 0$,
\item for all $\mu \in [0, \wb{\mu}]$, $D_{\mu}^{\eta} \succ \min\left( \lambda_{\min}(D_0), \min \big(1, \eta \sigma_{\wt{B}}^2 / c(\wb{\mu}) \big) \frac{\lambda_{\ker(\wt{B})}(D_{\wb{\mu}})}{ 8}  \right) I$, where $D_{\mu}^{\eta}$ is given in~\cref{eq:admissible.riccati.solution} for the modified system $\sys^\eta$.
\end{enumerate}
\end{proposition}

\cref{prop:modified.system.property} explicit a set modified system $\{ \sys^\eta\}_\eta \in \Sys$ such that the dual function $\mathcal{D}(\cdot,\sys^\eta)$ is non-increasing on a subset of its domain. Hence, we can invoque~\cref{prop:strong.duality.strict} to obtain that it exists $\mu^\eta \in [0,\wb{\mu}] \subset \mathcal{M}$ and $\tpi(\sys)^\eta = \tpi_{\mu^\eta}(\sys^\eta)$ such that:
\begin{enumerate}
\item $g_{\tpi(\sys^\eta)}(\sys) \leq 0$,
\item $\mu^\eta g_{\tpi(\sys^\eta)}(\sys) = 0$.
\end{enumerate}
This proves the first two assertions in~\cref{prop:strong.duality.weak.epsilon.feasible.controller.null.space.imBtilde}. The third assertion in~\cref{prop:modified.system.property} leads to 
\begin{equation*}
\mathcal{D}(\mu^\eta, \sys) \geq \mathcal{D}(\mu^\eta,\sys^\eta) - 2 \kappa(\sys)^2 \eta = \mathcal{L}_{\tpi(\sys^\eta)}(\mu^\eta, \sys^\eta) - 2 \kappa(\sys)^2 \eta \geq \mathcal{L}_{\tpi(\sys^\eta)}(\mu^\eta,\sys) - 2 \kappa(\sys)^2 \eta.
\end{equation*}
Setting $\eta = \frac{c(\wb{\mu})}{\sigma_{\wt{B}}^2}\frac{8^{2n+1} \kappa(\sys)^{2n}\lambda_{\min}(D_{\wb{\mu}})}{\lambda_{\ker(\wt{B})}(D_{\wb{\mu}})}$ proves the last assertion in~\cref{prop:strong.duality.weak.epsilon.feasible.controller.null.space.imBtilde} and concludes the proof. We are thus left to prove~\cref{prop:modified.system.property}.

\begin{proof}[Proof of~\cref{prop:modified.system.property}]
\begin{enumerate}
\item First, notice that from $\mathcal{D}^\prime(\wb{\mu},\sys) > 0$,~\cref{prop:positive.gradient.stable.policy} ensures that $\rho(\Ac_{\wb{\mu}}) < 1$ and hence $(I - \Ac_{\wb{\mu}})$ is full rank. Thus, $\Delta$ is well defined.\\
The modified $\sys^\eta$  consists in adding a p.s.d perturbation $\eta \Delta$ to the cost matrix of the original system $C_\dagger$, while leaving the dynamics parametrization $(A,\wt{B})$ and the constraint cost $C_g$ unchanged.  Thus, the structure is still compatible with $\Sys$ while the existence of a stable linear feasible policy for $\sys^\eta$ is inherited from $\sys \in \Sys$. As a result, $\sys^\eta \in \Sys$.
\item Since $\eta \Delta \succcurlyeq 0$, $C_\dagger + \eta \Delta \succcurlyeq C_\dagger$, the modification enforces the positive definiteness of $\sys$, and as a result, extends its Riccati admissible domain. For all $\mu \geq 0$, for any $\wt K$ such that $|\lambda\big(\Ac_{\mu}(\wt{K})\big)| \neq 1$, for all $|z|=1$,
\begin{equation*}
\Psi_{\mu}^{\wt K}(z,\sys^\eta) = \Psi_{\mu}^{\wt{K}}(z,\sys) + \eta Y^{\wt{K}}(z^{-1})^\transp \Delta Y^{\wt{K}}(z) \succcurlyeq \Psi_{\mu}^{\wt{K}}(z,\sys).
\end{equation*}
Thus, $\mu \in \mathcal{M}$ implies that $\Psi_{\mu}^{\wt{K}}(z,\sys) \succ 0$ which implies that $\Psi_{\mu}^{\wt K}(z,\sys^\eta) \succ 0$ which leads to $\mu \in \mathcal{M}^\eta$. As a result, $\mathcal{M} \subset \mathcal{M}^\eta$.

\item For all $\mu \in [0,\wb{\mu}]$, let $\tpi_{\mu}(\sys^\eta)$ and $\tpi_\mu(\sys)$ the optimal policy at $\mu$ for the system $\sys^\eta$ and $\sys$. Since $\Delta$ is p.s.d., we immediately have that $\mathcal{L}_{\tpi_\mu(\sys^\eta)}(\mu,\sys) \leq \mathcal{D}(\mu,\sys^\eta)$. To prove the r.h.s. inequality, we rely on~\cref{prop:positive.gradient.stable.policy}.

Let $\wt{K}_\mu^\eta$, $\wt{K}_\mu$, $\Aceta_{\mu}$ and $\Ac_{\mu}$ be the linear controller and closed-loop matrix associated with $\tpi_{\mu}(\sys^\eta)$ and $\tpi_\mu(\sys)$. Then, 
\begin{equation}
\label{eq:modified.sys.proof.1}
\mathcal{D}(\mu,\sys^\eta)  = \mathcal{L}_{\tpi_\mu(\sys^\eta)}(\mu,\sys^\eta)  \leq  \mathcal{L}_{\tpi_\mu(\sys)}(\mu,\sys^\eta) = \mathcal{L}_{\tpi_\mu(\sys)}(\mu,\sys) + \eta \Tr(X) = \mathcal{D}(\mu,\sys) + \eta \Tr(X),
\end{equation}
where $X$ is the solution of the Lyapunov equation
\begin{equation*}
X =\big( \Ac_\mu \big)^\transp X \Ac_\mu + \eta \begin{pmatrix}I \\\wt{K}_{\mu} \end{pmatrix}^\transp \Delta \begin{pmatrix}I \\\wt{K}_{\mu}\end{pmatrix}.
\end{equation*}
From $\wt{B} (\wt{K}_{\mu} -\wb{K})  = \wt{B} \wt{K}_{\mu}  + \begin{pmatrix} B & I \end{pmatrix} \begin{pmatrix} 0 \\ - A \end{pmatrix} = \Ac_\mu$, we have 
\begin{equation*}
\begin{aligned}
\begin{pmatrix}I \\\wt{K}_{\mu} \end{pmatrix}^\transp \Delta \begin{pmatrix}I \\\wt{K}_{\mu} \end{pmatrix} &= \left( I - \wt{B}  (\wt{K}_{\mu} - \wb{K})\right)^\transp \left( I -  \wt{B} (\wt{K}_{\mu} - \wb{K})\right) \\
&= \big(I - \Ac_\mu\big)^\transp \big(I - \Ac_\mu\big) \\
&\preccurlyeq \|I - \Ac_\mu\|_2^2 I\\
&\preccurlyeq 2 (1 + \|\Ac_\mu\|_2^2) I
\end{aligned}
\end{equation*}
As a result, $\Tr(X) \leq  2 (1 + \|\Ac_\mu\|_2^2) \Tr(\Sigma_\mu)$ where $\Sigma_\mu = \Ac_\mu \Sigma_\mu \big(\Ac_\mu\big)^\transp +  I$ is the steady-state covariance of the state process driven by $\Ac_\mu$. Since $\mathcal{D}(\wb{\mu},\sys) > 0$, the concavity of $\mathcal{D}(\cdot,\sys)$ ensures that $\mathcal{D}(\mu,\sys) >0$ for all $\mu \in [0,\wb{\mu}]$. Hence,~\cref{prop:positive.gradient.stable.policy} guarantees that 
\begin{equation*}
\Tr(\Sigma_\mu) \leq \kappa(\sys) := \mathcal{J}_*(\sys)/\lambda_{\min}(C_\dagger^1).
\end{equation*}
Further, noticing that $\Sigma_\mu \succcurlyeq I$, we have 
\begin{equation*}
\kappa(\sys) \geq \Tr(\Sigma_\mu) \geq \|\Sigma_\mu\|_2 = \max_{\|x\|=1} x^\transp \Sigma_\mu x = \max_{\|x\|=1} x^\transp \Ac_\mu \Sigma_\mu \big(\Ac_\mu\big)^\transp x + 1 \geq 1 + \max_{\|x\|=1} x^\transp \Ac_\mu \big(\Ac_\mu\big)^\transp x = 1 + \|\Ac_\mu\|_2^2,
\end{equation*}
and we obtain that $\Tr(X) \leq 2 \kappa(\sys)^2$, which we use in~\cref{eq:modified.sys.proof.1} to prove the third statement.

\item Let $\Psi_{\mu}^{\wt K} (z;\sys^\eta)$ and $\Psi_{\mu}^{\wt K} (z;\sys)$ be the Popov functions associated with $\sys^\eta$ and $\sys$ as defined in~\cref{eq:popov.K.definition}. The structure of the modification $\Delta$ ensures that
\begin{equation}
\label{eq:modified.sys.proof.2}
\begin{aligned}
\Psi_{\mu}^{\wb K} (1;\sys^\eta) &= \Psi_{\mu}^{\wb K} (1;\sys) + \eta \begin{pmatrix} \wt{B} \\ I \end{pmatrix}^\transp  \begin{pmatrix} I & \wb{K}^\transp \\ 0 & I \end{pmatrix} \Delta \begin{pmatrix} I & 0 \\ \wb K & I \end{pmatrix} \begin{pmatrix} \wt{B} \\ I \end{pmatrix} \\
&= \Psi_{\mu}^{\wb K} (1;\sys) + \eta \begin{pmatrix}  \wt{B} \\ I \end{pmatrix}^\transp   \begin{pmatrix} I \\ - \wt{B}^\transp \end{pmatrix} \begin{pmatrix} I & -  \wt{B} \end{pmatrix} \begin{pmatrix} \wt{B} \\ I \end{pmatrix}  \\
&=  \Psi_{\mu}^{\wb K} (1;\sys).
\end{aligned}
\end{equation}
Further,~\cref{eq:popov.identities} provides that
\begin{equation}
\label{eq:modified.sys.proof.3}
\begin{aligned}
\Psi_{\mu}^{\wb K} (1;\sys^\eta)&= \big(I + (\wb{K} - \wt{K}^\eta_\mu) \tilde{B} \big)^\transp D_\mu^\eta \big(I + (\wb{K} - \wt{K}^\eta_\mu) \tilde{B} \big),\\
\Psi_{\mu}^{\wb K} (1;\sys)&=\big(I + (\wb{K} - \wt{K}_\mu) \tilde{B} \big)^\transp D_\mu \big(I + (\wb{K} - \wt{K}_\mu) \tilde{B} \big),
\end{aligned}
\end{equation}
where $D_\mu$ and $D_\mu^\eta$ are given in~\cref{eq:admissible.riccati.solution} for $\sys$ and $\sys^\eta$ respectively. Thus combining~\cref{eq:modified.sys.proof.2,eq:modified.sys.proof.3} at $\wb{\mu}$ leads to
\begin{equation*}
\big(I + (\wb{K} - \wt{K}^\eta_{\wb{\mu}}) \tilde{B} \big)^\transp D_{\wb{\mu}}^\eta \big(I + (\wb{K} - \wt{K}^\eta_{\wb{\mu}}) \tilde{B} \big) = \big(I + (\wb{K} - \wt{K}_{\wb{\mu}}) \tilde{B} \big)^\transp D_{\wb{\mu}} \big(I + (\wb{K} - \wt{K}_{\wb{\mu}}) \tilde{B} \big).
\end{equation*}
Taking the determinant and using Sylvester's determinant identity ($\det( I + XY) = \det(I + YX)$ for any rectangular matrices $X,Y$), we obtain
\begin{equation*}
\begin{aligned}
& \quad \det \big(I + (\wb{K} - \wt{K}^\eta_{\wb{\mu}}) \tilde{B} \big) ^2 \det( D^\eta_{\wb{\mu}}) = \det \big(I + (\wb{K} - \wt{K}_{\wb{\mu}}) \tilde{B} \big)^2 \det(D_{\wb{\mu}}), \\
\Leftrightarrow \quad & \quad \det ( I - \Aceta_{\wb{\mu}})^2 \det( D^\eta_{\wb{\mu}}) = \det( I - \Ac_{\wb{\mu}})^2 \det(D_{\wb{\mu}}).
\end{aligned} 
\end{equation*}
Let $\{ \lambda_i^\eta\}_{i=1,\dots,n}$ and $\{ \lambda_i\}_{i=1,\dots,n}$ be the eigenvalues of $\Aceta_{\wb{\mu}}$ and $\Ac_{\wb{\mu}}$ respectively. Since $\Ac_{\wb{\mu}}$ is stable, 
\begin{equation*}
\begin{aligned}
& \det( I - \Ac_{\wb{\mu}})^2 \leq \prod_{i=1}^n | 1 - \lambda_i |^2 \leq 2^{2n}, \\
&\det ( I - \Aceta_{\wb{\mu}}))^2 \geq \prod_{i=1}^n |1 - \lambda_i|^2 \geq \prod_{i=1}^n (1 - |\lambda_i|)^2 \geq (1 - \max_{i=1,\dots,n} |\lambda_i|)^{2n} = (1 - \rho(\Aceta_{\wb{\mu}}))^{2n}.
\end{aligned}
\end{equation*}
Hence,
\begin{equation*}
(1 - \rho(\Aceta_{\wb{\mu}}))^{2n} \leq 2^{2n} \det(D_{\wb{\mu}}) / \det( D^\eta_{\wb{\mu}}).
\end{equation*}
Let $P_{\wb{\mu}}^\eta$ and $P_{\wb{\mu}}$ be the solution of the Riccati equation at $\wb{\mu}$ associated with $\sys^\eta$ and $\sys$ respectively. Since $\Delta$ is p.s.d. we have that $P_{\wb{\mu}}^\eta \succcurlyeq P_{\wb{\mu}}$. Algebraic manipulations show that $D^\eta_{\wb{\mu}} = R_{\wb{\mu}} + \eta \wt{B}^\transp \wt{B} + \wt{B}^\transp P_{\wb{\mu}}^\eta \wt{B} \succcurlyeq D_{\wb{\mu}} + \eta \wt{B}^\transp \wt{B}$. Thus,
\begin{equation*}
\frac{\det(D_{\wb{\mu}})}{\det( D^\eta_{\wb{\mu}})} \leq \frac{\det(D_{\wb{\mu}})}{\det(D_{\wb{\mu}} + \eta \wt{B}^\transp \wt{B} )}.
\end{equation*}
Further, for any p.s.d matrix $X$ and $Y$ such that $X \succcurlyeq Y \succ 0$, we have that $\frac{\det(Y)}{\det(X)} \leq \inf_{x \neq 0} \frac{x^\transp Y x}{x^\transp X x}$. This identity can be found in~\citep[Lem.11]{abbasi2011regret} for instance. Clearly, $D_{\wb{\mu}} + \eta \wt{B}^\transp \wt{B}  \succcurlyeq D_{\wb{\mu}}$, and applying the given identity at $v_{\min}$ where $v_{\min}$ is a eigenvector associated with the smallest eigenvalue of $D_{\wb{\mu}}$, we obtain
\begin{equation*}
(1 - \rho(\Aceta_{\wb{\mu}}))^{2n} \leq 2^{2n} \frac{\lambda_{\min}(D_{\wb{\mu}}) }{v_{\min}^\transp(D_{\wb{\mu}} + \eta \wt{B}^\transp \wt{B} ) v_{\min} } \leq  2^{2n} \frac{\lambda_{\min}(D_{\wb{\mu}}) }{\lambda_{\min}(D_{\wb{\mu}} + \eta \wt{B} \wt{B}^\transp) }.
\end{equation*}
However, we have by~\cref{property:null.space.distinct.cases} that $\lambda_{\min}(D_{\wb{\mu}} + \eta \wt{B} \wt{B}^\transp)> \lambda_{\ker(\wt{B})}(D_{\wb{\mu}}) \min \left( 1/8,\eta \sigma_{\wt{B}}^2 / \big(8\lambda_{\max}(D_{\wb{\mu}}) \big)\right)$, which together with~\cref{prop:decreasing.gradient.lowner.pseudo.concave.Dmu} ensures that
\begin{equation*}
\lambda_{\min}(D_{\wb{\mu}} + \eta \wt{B} \wt{B}^\transp)> \lambda_{\ker(\wt{B})}(D_{\wb{\mu}}) \min \left( 1/8,\eta \sigma_{\wt{B}}^2 / \big(8c(\wb{\mu}) \big)\right).
\end{equation*}
 Thus,
\begin{equation*}
(1 -\rho(\Aceta_{\wb{\mu}}))^{2n} < 2^{2n}  \frac{\lambda_{\min}(D_{\wb{\mu}}) }{ \lambda_{\ker(\wt{B})}(D_{\wb{\mu}}) \min \left( 1/8,\eta \sigma_{\wt{B}}^2 / \big(8c(\wb{\mu}) \big)\right)}.
\end{equation*}
Since, $\eta \geq \frac{c(\wb{\mu})}{\sigma_{\wt{B}}^2}\frac{8^{2n+1} \kappa(\sys)^{2n}\lambda_{\min}(D_{\wb{\mu}})}{\lambda_{\ker(\wt{B})}(D_{\wb{\mu}})}$ and $\frac{\lambda_{\min}(D_{\wb{\mu}})}{ \lambda_{\ker(\wt{B})}(D_{\wb{\mu}})} \leq  \frac{1}{8^{2n+1} \kappa(\sys)^{2n}}$,
\begin{itemize}
\item when $\frac{\sigma_{\wt{B}}^2\eta}{c(\wb{\mu})} \leq 1$, $(1 -\rho(\Aceta_{\wb{\mu}}))^{2n}  < 2^{2n}  \frac{8 c(\wb{\mu}) \lambda_{\min}(D_{\wb{\mu}}) }{  \lambda_{\ker(\wt{B})}(D_{\wb{\mu}}) \eta \sigma_{\wt{B}}^2} \leq \frac{1}{  4^{2n} \kappa(\sys)^{2n}}$,
\item when $\frac{\sigma_{\wt{B}}^2\eta}{c(\wb{\mu})} \geq 1$, $(1 -\rho(\Aceta_{\wb{\mu}}))^{2n}  < 2^{2n}  8 \frac{\lambda_{\min}(D_{\wb{\mu}}) }{  \lambda_{\ker(\wt{B})}(D_{\wb{\mu}})} \leq \frac{1}{  4^{2n} \kappa(\sys)^{2n}}$.
\end{itemize}
Thus,
\begin{equation}
\label{eq:modified.sys.proof.4}
(1 - \rho(\Aceta_{\wb{\mu}})) < \frac{1}{4 \kappa(\sys)}.
\end{equation}
Now, suppose that $\mathcal{D}^\prime(\wb{\mu},\sys^\eta) \geq 0$, by~\cref{prop:positive.gradient.stable.policy} it implies that $ 1 - \rho(\Aceta_{\wb{\mu}})^2 \geq 1 / \kappa(\wb{\mu},\sys^\eta)$ where 
$\kappa(\wb{\mu},\sys^\eta) = \mathcal{D}(\wb{\mu},\sys^\eta) / \lambda_{\min}(C_\dagger^1)$. Further, from $\mathcal{D}(\wb{\mu},\sys^\eta)  \leq \mathcal{D}(\mu,\sys) + 2 \eta \kappa(\sys)^2$ and $\eta \leq \frac{\lambda_{\min}(C_{\dagger}^{1})}{2 \kappa(\sys)}$, we obtain
\begin{equation*}
\mathcal{D}(\wb{\mu},\sys^\eta) \leq \mathcal{J}_*(\sys) + \kappa(\sys) \lambda_{\min}(C_{\dagger}^1) \quad \Rightarrow \quad \kappa(\wb{\mu},\sys^\eta) \leq 2 \kappa(\sys).
\end{equation*}
Thus,
\begin{equation*}
\label{eq:modified.sys.proof.5}
2 (1 - \rho(\Aceta_{\wb{\mu}})) \geq 1 - \rho(\Aceta_{\wb{\mu}})^2 \geq \frac{1}{2 \kappa(\sys)} \quad \Rightarrow \quad 1 - \rho(\Aceta_{\wb{\mu}}) \geq \frac{1}{4 \kappa(\sys)},
\end{equation*}
which contradicts~\cref{eq:modified.sys.proof.4}. As a result, $\mathcal{D}^\prime(\wb{\mu},\sys^\eta) < 0$ which proves the statement.

\item To prove the final statement, we use~\cref{prop:decreasing.gradient.lowner.pseudo.concave.Dmu}, which states that for any $v\in\mathbb{R}^{n+d}$, the function $\mu \in \mathcal{M} \rightarrow v^\transp D^\eta_{\mu} v$ is quasi-concave. As a result, for all $\mu \in [0,\wb{\mu}]$, let $\lambda_{\min}(D_\mu^\eta)$ be the smallest eigenvalue of $D_\mu^\eta$ and let $v_{\min}$ be a corresponding eigenvector, we have that
\begin{equation*}
\lambda_{\min}(D_\mu^\eta) = v_{\min}^\transp D_\mu^\eta v_{\min} \geq \min\left( v_{\min}^\transp D_0^\eta v_{\min}, v_{\min}^\transp D_{\wb{\mu}}^\eta v_{\min} \right) \geq \min\left( \lambda_{\min}(D_0^\eta), \lambda_{\min}(D_{\wb{\mu}}^\eta) \right).
\end{equation*}
Thus, we conclude the proof using that
\begin{equation*}
\begin{aligned}
\lambda_{\min}(D^\eta_{\wb{\mu}}) &\geq \lambda_{\min}(D_{\wb{\mu}} + \eta \wt{B} \wt{B}^\transp)>  \min \big(1, \eta \sigma_{\wt{B}}^2 / c(\wb{\mu}) \big) \frac{\lambda_{\ker(\wt{B})}(D_{\wb{\mu}})}{ 8} ,\\
\lambda_{\min}(D^\eta_0) &\geq \lambda_{\min}(D_0 + \eta \wt{B} \wt{B}^\transp) \geq \lambda_{\min}(D_0).
\end{aligned}
\end{equation*}
Notice that $D_0 \succ 0$ since $0 \in \mathcal{M}$ and is fully characterized by $\sys$ i.e, it has no dependency on $\mu$. As a result, we obtain
\begin{equation*}
\lambda_{\min}(D_\mu^\eta)  \geq \min\left( \lambda_{\min}(D_0), \min \big(1, \eta \sigma_{\wt{B}}^2 / c(\wb{\mu}) \big) \frac{\lambda_{\ker(\wt{B})}(D_{\wb{\mu}})}{ 8}  \right). 
\end{equation*}

\end{enumerate}
\end{proof}

\begin{property}
\label{property:null.space.distinct.cases}
Let $D \in \mathbb{S}^{n+d}_{++}$ and $\wt{B} \in \mathbb{R}^{n,n+d}$. Let $\lambda_{\max}(D)$ denote the maximum eigenvalue of $D$ and 
\begin{equation*}
\lambda_{\ker(\wt{B})}(D) = \min_{\|v\|=1, v \in \ker(\wt{B})} v^\transp D v, \quad \quad \sigma_{\wt{B}}^2 = \min_{\|v\|=1, v \in Im(\wt{B})} v^\transp \wt{B}^\transp \wt{B} v,
\end{equation*}
be the singular value of $D$ and $\wt{B}$ on the null space and row space of $\wt{B}$ respectively\footnote{Notice that $\sigma^2_{\wt{B}} > 0$ as it is the smallest non-zero eigenvalue of $\wt{B}^\transp \wt{B}$.}. Then, for any $\eta \geq 0$, 
\begin{equation*}
\lambda_{\min}(D + \eta \wt{B}^\transp \wt{B}) >\lambda_{\ker(\wt{B})}(D) \min \left( 1/8,\eta \sigma_{\wt{B}}^2 / \big(8\lambda_{\max}(D) \big)\right).
\end{equation*}
\end{property}

\begin{proof}[Proof of~\cref{property:null.space.distinct.cases}]
First, we have that 
\begin{equation}
\label{eq:proof.null.space.distinct.cases.1}
\lambda_{\min}(D + \eta \wt{B}^\transp \wt{B}) \geq \min_{\|v\|=1} v^\transp (D + \eta \wt{B}^\transp \wt{B}) v \geq \min_{\|v\|=1} \max \Big(  v^\transp D v , \eta \|\wt{B} v \|^2 \Big).
\end{equation}
We thus provide a lower bound for $ v^\transp D v$ and $\eta \|\wt{B} v \|^2$ separately and then show a lower bound for the minimum of the maximum between those lower bound. Let $P_{\wt{B}}$ be defined the orthogonal projection matrix onto $\wt{B}$, which satisfies $ \wt{B} P_{\wt{B}} = 0$, $P_{\wt{B}} = P^\transp_{\wt{B}} = P^2_{\wt{B}}$. Then,
\begin{enumerate}
\item for any $v \in \mathbb{R}^{n+d}$, $\|\wt{B} v\| = \|\wt{B}(v - P_{\wt{B}} v)$. Since $v - P_{\wt{B}} v \in Im(\wt{B})$, $\|\wt{B}(v - P_{\wt{B}} v)\| \geq \sigma_{\wt{B}} \| v - P_{\wt{B}} v\|$. Thus,
\begin{equation}
\label{eq:proof.null.space.distinct.cases.2}
\eta \|\wt{B} v\|^2 \geq  \sigma_{\wt{B}}^2 \| v - P_{\wt{B}} v\|^2.
\end{equation}
\item the function $v \rightarrow v^\transp D v$ is convex since $D \in \mathbb{S}^{n+d}_{++}$. As a result, for any $v\in\mathbb{R}^{n+d}$,
\begin{equation}
\label{eq:proof.null.space.distinct.cases.3}
v^\transp D v \geq v^\transp P_{\wt{B}} D P_{\wt{B}} v + 2 v^\transp P_{\wt{B}} D (v - P_{\wt{B}} v) \geq \| P_{\wt{B}} v \|_D^2 - 2 \sqrt{\lambda_{\max}(D)} \| P_{\wt{B}} v \|_D \|v - P_{\wt{B}} v\|.
\end{equation}
\end{enumerate}
Plugging~\cref{eq:proof.null.space.distinct.cases.2,eq:proof.null.space.distinct.cases.3} in~\cref{eq:proof.null.space.distinct.cases.1} we obtain
\begin{equation}
\label{eq:proof.null.space.distinct.cases.4}
\lambda_{\min}(D + \eta \wt{B}^\transp \wt{B}) \geq \min_{\|v\|=1} \max \left(  \| P_{\wt{B}} v \|_D^2 - 2 \sqrt{\lambda_{\max}(D)} \| P_{\wt{B}} v \|_D \|v - P_{\wt{B}} v\|, \eta \sigma_{\wt{B}}^2 \| v - P_{\wt{B}} v\|^2\right).
\end{equation} 
\begin{itemize}
 \item When $\| v - P_{\wt{B}} v\|^2 > \frac{\lambda_{\ker(\wt{B})}(D)}{8 \lambda_{\max}(D)}$, we use the r.h.s lower bound and get 
\begin{equation*}
\max \left(  \| P_{\wt{B}} v \|_D^2 - 2 \sqrt{\lambda_{\max}(D)} \| P_{\wt{B}} v \|_D \|v - P_{\wt{B}} v\|, \eta \sigma_{\wt{B}}^2 \| v - P_{\wt{B}} v\|^2\right) > \eta \sigma_{\wt{B}}^2 \frac{\lambda_{\ker(\wt{B})}(D)}{8 \lambda_{\max}(D)}.
\end{equation*}
\item When $\| v - P_{\wt{B}} v\|^2 \leq \frac{\lambda_{\ker(\wt{B})}(D)}{8 \lambda_{\max}(D)}$, we use the l.h.s. lower bound and get 
\begin{equation*}
\| P_{\wt{B}} v \|_D^2 - 2 \sqrt{\lambda_{\max}(D)} \| P_{\wt{B}} v \|_D \|v - P_{\wt{B}} v\| \geq \| P_{\wt{B}} v \|_D^2 - \sqrt{\lambda_{\ker(\wt{B})}(D) / 2} \| P_{\wt{B}} v \|_D.
\end{equation*}
The r.h.s. is now lower bounded by a second order polynomial in $ \| P_{\wt{B}} v \|_D$, whose roots are given by $0$ and $\sqrt{\lambda_{\ker(\wt{B})}(D) / 2}$. However, from $\| v - P_{\wt{B}} v\|^2 \leq \frac{\lambda_{\ker(\wt{B})}(D)}{8 \lambda_{\max}(D)}$ and $\lambda_{\ker(\wt{B})}(D) \leq \lambda_{\max}(D)$, we have that 
\begin{equation*}
\| v - P_{\wt{B}} v\|^2 \leq 1/8 \quad \Rightarrow \quad \| P_{\wt{B}} \|^2 \geq 1 - 1/8  > 1/2 \quad \Rightarrow \quad \| P_{\wt{B}} v \|_D \geq \sqrt{\lambda_{\ker(\wt{B})} (D) 7/ 8} > \sqrt{\lambda_{\ker(\wt{B})}(D) / 2}.
\end{equation*}
As a result, the polynomial is lower bounded as
\begin{equation*}
\| P_{\wt{B}} v \|_D^2 - \sqrt{\lambda_{\ker(\wt{B})} (D) / 2} \| P_{\wt{B}} v \|_D \geq     \lambda_{\ker(\wt{B})}(D) \sqrt{\frac{7 }{16}} \left(\sqrt{7/4} - 1 \right) > \lambda_{\ker(\wt{B})}(D) /8.
\end{equation*}
\end{itemize}
Summarizing, we have that 
\begin{equation}
\label{eq:proof.null.space.distinct.cases.5}
\lambda_{\min}(D + \eta \wt{B}^\transp \wt{B})> \lambda_{\ker(\wt{B})}(D) \min \left( 1/8,\eta \sigma_{\wt{B}}^2 / 8\lambda_{\max}(D) \right).
\end{equation}

\end{proof}

\begin{proposition}
\label{prop:decreasing.gradient.lowner.pseudo.concave.Dmu}
For any $\sys \in \Sys$, let $\mathcal{M}$ be the admissible Riccati set associated with $\sys$ as defined in~\cref{eq:mutilde.definition}. For any $\mu \in \mathcal{M}$, let $\wt{K}_\mu$ be the optimal controller given in~\cref{eq:optimal.riccati.controller}, $P_\mu$, $\Ac_\mu$ and $D_\mu$ be given in~\cref{eq:admissible.riccati.solution} and for any $\wt{K}$, let $P(\wt{K})$, $G(\wt{K})$ and $P_{\mu}(\wt{K})$ be respectively the matrix representation of the objective, the constraint and the Lagrangian, defined in~\cref{eq:average.cost.constraints.lyapunov2}. Then,
\begin{enumerate}
\item for any $(\mu_1,\mu_2) \in \mathcal{M}^2$, $(\mu_1 - \mu_2) \big( G(\wt{K}_{\mu_2}) - G(\wt{K}_{\mu_1}) \big) \succcurlyeq 0$,
\item for any $v \in \mathbb{R}^{n+d}$, the function $\mu \in \mathcal{M} \rightarrow v^\transp D_\mu v$ is quasi-concave,
\item for any $\mu \in \mathcal{M}$, $\lambda_{\max}(D_\mu) \leq c(\mu)$, where $c(\mu) = \big(\lambda_{\max}(C_{\dagger}) + \mu \big)\big( 1 + \|\wt{B}\|_2^2 (1 + \|A\|_2^2)\big)$.
\end{enumerate}
\end{proposition}

\begin{proof}[Proof of~\cref{prop:decreasing.gradient.lowner.pseudo.concave.Dmu}]
\begin{enumerate}
\item From~\cref{eq:lyap.suboptimal.K}, we have that 
\begin{equation*}
\begin{aligned}
P_{\mu_1}(\wt{K}_{\mu_2}) - P_{\mu_1} &= \big(\Ac_{\mu_2}\big)^\transp (P_{\mu_1}(\wt{K}_{\mu_2}) - P_{\mu_1}) \Ac_{\mu_2} + (\wt K_{\mu_1} - \wt K_{\mu_2})^\transp D_{\mu_1} (\wt K_{\mu_1} - \wt K_{\mu_2})\\
P_{\mu_2}(\wt{K}_{\mu_1}) - P_{\mu_2} &= \big(\Ac_{\mu_1}\big)^\transp (P_{\mu_2}(\wt{K}_{\mu_1}) - P_{\mu_2}) \Ac_{\mu_1} + (\wt K_{\mu_1} - \wt K_{\mu_2})^\transp D_{\mu_2} (\wt K_{\mu_1} - \wt K_{\mu_2})
\end{aligned}
\end{equation*}
which since $P_{\mu_1}(\wt{K}_{\mu_2}) - P_{\mu_1} \succcurlyeq 0$, $P_{\mu_2}(\wt{K}_{\mu_1}) - P_{\mu_2} \succcurlyeq 0$, leads to
\begin{equation}
\label{eq:proof.decreasing.gradient.lowner}
\begin{aligned}
P_{\mu_1}(\wt{K}_{\mu_2}) - P_{\mu_1} &\succcurlyeq (\wt K_{\mu_1} - \wt K_{\mu_2})^\transp D_{\mu_1} (\wt K_{\mu_1} - \wt K_{\mu_2}),\\
P_{\mu_2}(\wt{K}_{\mu_1}) - P_{\mu_2} &\succcurlyeq (\wt K_{\mu_1} - \wt K_{\mu_2})^\transp D_{\mu_2} (\wt K_{\mu_1} - \wt K_{\mu_2}).
\end{aligned}
\end{equation}
From~\cref{eq:average.cost.constraints.lyapunov3}, we have that
\begin{equation*}
\begin{aligned}
P_{\mu_1}(\wt{K}_{\mu_2}) &= P(\wt{K}_{\mu_2}) + \mu_1 G(\wt{K}_{\mu_2}) = P_{\mu_2}(\wt{K}_{\mu_2}) + (\mu_1 - \mu_2) G(\wt{K}_{\mu_2}) = P_{\mu_2} +  (\mu_1 - \mu_2) G(\wt{K}_{\mu_2}), \\
P_{\mu_2}(\wt{K}_{\mu_1}) &= P(\wt{K}_{\mu_1})  + \mu_2 G(\wt{K}_{\mu_1})  = P_{\mu_1}(\wt{K}_{\mu_1})  + (\mu_2 - \mu_1)G(\wt{K}_{\mu_1})  = P_{\mu_1} + (\mu_2 - \mu_1) G(\wt{K}_{\mu_1}).
\end{aligned}
\end{equation*}
Combining these with~\cref{eq:proof.decreasing.gradient.lowner} and summing the two inequalities, one obtains:
\begin{equation*}
(\mu_1 - \mu_2) \big( G(\wt{K}_{\mu_2}) - G(\wt{K}_{\mu_1}) \big) \succcurlyeq (\wt K_{\mu_1} - \wt K_{\mu_2})^\transp (D_{\mu_1} + D_{\mu_2}) (\wt K_{\mu_1} - \wt K_{\mu_2}) \succcurlyeq 0.
\end{equation*}

\item For any $v\in\mathbb{R}^{n+d}$, let $f : \mu \rightarrow v^\transp D_\mu v$. Recalling that $D_\mu = R_{\mu} + \wt{B}^\transp P_\mu \wt{B}$ where $R_\mu = R_0 + \mu dR$ (with $R_0$ constructed from $C_{\dagger}$ and $dR$ constructed from $C_g$), we have that $f^\prime(\mu) = v^\transp (dR + \wt{B}^\transp dP_\mu \wt{B})v$ where $dP_\mu$ is the (matrix) derivative of $P_\mu$ w.r.t. $\mu$. Direct matrix differentiation of~\cref{eq:admissible.riccati.solution} shows, after tedious algebraic manipulations, that $dP_\mu = G(\wt{K}_\mu)$. As a result, 
\begin{equation*}
f^\prime(\mu) = v^\transp (dR + \wt{B}^\transp G(\wt{K}_\mu) \wt{B})v,
\end{equation*}
which, according to the first statement is a decreasing function. Thus, $f$ can be either increasing then decreasing, increasing or decreasing, which is enough to ensure quasi-concavity over $\mathcal{M} \subset \mathbb{R}_+$.

\item For any $\mu \in \mathcal{M}$, $P_{\mu} \preccurlyeq P_{\mu}(\wt{K})$ for any $\wt{K}$. Thus, for $\wb{K} = \begin{pmatrix} 0 \\ -A \end{pmatrix}$ that is such that $\Ac(\wb{K})=0$, one has:
\begin{equation*}
P_\mu \preccurlyeq P_\mu(\wb{K}) = \begin{pmatrix} I \\ \wb{K} \end{pmatrix}^\transp C_\mu \begin{pmatrix} I \\ \wb{K} \end{pmatrix} \preccurlyeq \lambda_{\max}(C_\mu) \|\wt{B}\|_2^2 (I + A^\transp A).
\end{equation*}
Notice that $\lambda_{\max}(C_\mu) \geq \mu \geq 0$. From $D_\mu = R_\mu + \wt{B}^\transp P_{\mu} \wt{B}$, and $\lambda_{\max}(R_\mu) \leq \lambda_{\max}(C_\mu)$ we obtain:
\begin{equation*}
D_\mu \preccurlyeq \lambda_{\max}(C_\mu) \left( I + \|\wt{B}\|_2^2 (I + A^\transp A) \right) \quad \Rightarrow \quad \lambda_{\max}(D_\mu) \leq \lambda_{\max}(C_\mu) \big( 1 + \|\wb{B}\|_2^2(1 + \|A\|_2^2) \big).
\end{equation*}
The proof concludes by noting that $\lambda_{\max}(C_\mu) \leq \lambda_{\max}(C_{\dagger}) + \mu$.
\end{enumerate}
\end{proof}

\subsection{Smoothness of $\mathcal{D}$}\label{sec:app.lipschitz}

We conclude the characterization of the dual function $\mathcal{D}$ discussing its smoothness, that is, its Lipschitz and gradient Lipschitz continuity. Indeed, since the objective is ultimately to maximize $\mathcal{D}$ over $\mathcal{M}$, one has to first guarantee that the function behaves 'nicely'. Unfortunately,~\cref{prop:behavior.Dmu.closure.M} shows that the closed-loop matrix may be arbitrarily close-to-unstable when $\mu \rightarrow \wt{\mu}$, that is $\rho(\Ac_\mu) \rightarrow 1$. This result in $\mathcal{D}^\prime(\mu,\sys) \rightarrow -\infty$ and indicates that $\mathcal{D}$ cannot be smooth \textit{everywhere} on $\mathcal{M}$. Nevertheless, we show that $\mathcal{D}$ is still smooth on a subset of $\mathcal{M}$, which will be sufficient to provide a practical algorithm to maximize the dual function.
\begin{lemma}\label{le:smoothness.D}
For any $\sys \in \Sys$, let $\mathcal{M} = [0,\wt{\mu})$ be the admissible Riccati set associated with $\sys$ as defined in~\cref{eq:mutilde.definition}. Let $\mathcal{M}_+$ be a subset of $\mathcal{M}$ such that
\begin{equation*}
\mathcal{M}_+ = \{ \mu \in \mathcal{M} \text{ s.t. } \mathcal{D}^\prime(\mu,\sys) \geq 0 \}.
\end{equation*}
Then, $\mathcal{D}(\cdot,\sys)$ has Lipschitz gradient, i.e., for all $(\mu_1,\mu_2) \in \mathcal{M}_+^2$, 
\begin{equation*}
|\mathcal{D}^\prime(\mu_1,\sys) - \mathcal{D}^\prime(\mu_2,\sys)| \leq  |\mu_1 - \mu_2| \frac{\|C_g\|_2 \alpha_{D}(\sys)}{2 \max\left( \lambda_{\min}(D_{\mu_1}), \lambda_{\min}(D_{\mu_2}) \right)},
\end{equation*}
where $\alpha_D(\sys) := 8 \|C_g \|_2 \kappa(\sys)^4 \left( (2 + \|A\|_2\|B\|_2) (1 + \|B\|_2) \right)^2$.
\end{lemma}

\begin{proof}[Proof of~\cref{le:smoothness.D}]
For any $\mu \in \mathcal{M}$, let $\wt{K}_\mu$, $\Ac_\mu$, $D_\mu$ and $P_\mu$ be the optimal quantities at $\mu$ associated with the Riccati equation for $\sys$. Without loss of generality, we assume that $\mu_2 \geq \mu_1$. To prove~\cref{le:smoothness.D}, we first show $\mu \rightarrow P_\mu$ is Lipschitz and that $\left\| \begin{pmatrix} I \\ \wt{K}_\mu \end{pmatrix} \right\|_2$ is bounded over $\mathcal{M}_+$. Then, we show that $\mathcal{D}$ has Lipschitz gradient on $\mathcal{M}_+$.
\begin{enumerate}
\item From~\cref{eq:average.cost.constraints.lyapunov3}, we have that
\begin{equation*}
\begin{aligned}
P_{\mu_2} - P_{\mu_1} &\preccurlyeq P_{\mu_2}(\wt{K}_{\mu_1}) - P_{\mu_1} = P(\wt{K}_{\mu_1}) + \mu_2 G(\wt{K}_{\mu_1}) - P_{\mu_1} = (\mu_2 - \mu_1) G(\wt{K}_{\mu_1}),\\
P_{\mu_1} - P_{\mu_2} &\preccurlyeq P_{\mu_1}(\wt{K}_{\mu_2}) - P_{\mu_2} = P(\wt{K}_{\mu_2}) + \mu_1 G(\wt{K}_{\mu_2}) - P_{\mu_2} = (\mu_1 - \mu_2) G(\wt{K}_{\mu_2}).
\end{aligned}
\end{equation*}
Thus, $  (\mu_2 - \mu_1) G(\wt{K}_{\mu_2}) \preccurlyeq P_{\mu_2} - P_{\mu_1}  \preccurlyeq (\mu_2 - \mu_1) G(\wt{K}_{\mu_1})$. Decomposing the positive and negative definite part of $C_g$, we obtain that for any $\mu \in \mathcal{M}_+$, 
\begin{equation*}
G(\wt{K}_{\mu}) \succcurlyeq - X, \text{ where } X = \big(\Ac_{\mu_1}\big)^\transp X \Ac_{\mu} + \begin{pmatrix} I \\ \wt{K}_{\mu} \end{pmatrix}^\transp \begin{pmatrix} V^{-1} & 0 \\ 0 & 0 \end{pmatrix}\begin{pmatrix} I \\ \wt{K}_{\mu} \end{pmatrix}.
\end{equation*}
Further, standard Lyapunov algebra ensures that $\lambda_{\max}(X) \leq \Tr(X) \leq \frac{\lambda_{\max}(V^{-1})}{\lambda_{\min}(C_{\dagger}^1)} \Tr(P(\wt{K}_\mu))$. Since $\mu \in \mathcal{M}_+$, $\Tr(P(\wt{K}_\mu)) = \mathcal{J}_{\tpi_\mu}(\sys) \leq \mathcal{D}(\mu,\sys) \leq \mathcal{J}_*(\sys)$ by weak duality and $g_{\tpi}(\sys) \geq 0$. Finally, $\lambda_{\max}(V^{-1}) \leq \|C_g\|_2$ leads to
\begin{equation}
\label{eq:lower.bound.matrix.gradient}
\lambda_{\min}\big( G(\wt{K}_{\mu}) \big) \geq - \|C_g\|_2 \kappa(\sys).
\end{equation}
To prove an upper bound on $\lambda_{\max}(G(\wt{K}_\mu))$, we use~\cref{prop:decreasing.gradient.lowner.pseudo.concave.Dmu} which implies that $G(\wt{K}_\mu) \preccurlyeq G(\wt{K}_0)$. $\wt{K}_0$ is the optimal controller at $\mu=0$, and one can check that it is such that $\Ac_0 = 0$ (by algebraic manipulations). Hence, decomposing the positive and negative definite part of $C_g$, we obtain that for any $\mu \in \mathcal{M}_+$, 
\begin{equation*}
G(\wt{K}_\mu) \preccurlyeq G(\wt{K}_0) \preccurlyeq \begin{pmatrix} I \\ \wt{K}_0 \end{pmatrix}^\transp \begin{pmatrix} 0_{n+d} & 0 \\ 0 & I_{n} \end{pmatrix} \begin{pmatrix} I \\ \wt{K}_0 \end{pmatrix}.
\end{equation*}
However, for all $\mu \in \mathcal{M}_+$, we have by~\cref{prop:positive.gradient.bounded.closed.loop.controller} that $\| \Ac_\mu\|_2 \leq \sqrt{\kappa(\sys)}$ and $\left\| \begin{pmatrix} I \\ \wt{K}_\mu \end{pmatrix} \right\|_2 \leq \sqrt{\kappa(\sys)} (2 + \|A\|_2\|B\|_2)$. Thus,
\begin{equation}
\label{eq:upper.bound.matrix.gradient.2}
\lambda_{\max}\big(G(\wt{K}_\mu)\big) \preccurlyeq \sqrt{\kappa(\sys)} (2 + \|A\|_2\|B\|_2).
\end{equation}
Summarizing,~\cref{eq:lower.bound.matrix.gradient,eq:upper.bound.matrix.gradient.2} leads to
\begin{equation}
\label{eq:bound.matrix.dual.norm2}
\|P_{\mu_1} - P_{\mu_2}\|_2 \leq |\mu_1 - \mu_2| \sqrt{\kappa(\sys)} \max \left(\|C_g\|_2 \sqrt{\kappa(\sys)} , (2 + \|A\|_2\|B\|_2) \right).
\end{equation}
\item We now show that $\mathcal{D}$ has Lipschitz gradient on $\mathcal{M}_+$. From~\cref{eq:lyap.suboptimal.K}, we have that
\begin{equation*}
\begin{aligned}
P_{\mu_1}(\wt{K}_{\mu_2}) - P_{\mu_1} &= \big(\Ac_{\mu_2}\big)^\transp (P_{\mu_1}(\wt{K}_{\mu_2}) - P_{\mu_1}) \Ac_{\mu_2} + (\wt K_{\mu_1} - \wt K_{\mu_2})^\transp D_{\mu_1} (\wt K_{\mu_1} - \wt K_{\mu_2}),\\
P_{\mu_2}(\wt{K}_{\mu_1}) - P_{\mu_2} &= \big(\Ac_{\mu_1}\big)^\transp (P_{\mu_2}(\wt{K}_{\mu_1}) - P_{\mu_2}) \Ac_{\mu_1} + (\wt K_{\mu_1} - \wt K_{\mu_2})^\transp D_{\mu_2} (\wt K_{\mu_1} - \wt K_{\mu_2}).
\end{aligned}
\end{equation*}
Thus, standard Lyapunov manipulation shows that 
\begin{equation*}
\begin{aligned}
0 \leq \Tr \left( P_{\mu_1}(\wt{K}_{\mu_2}) - P_{\mu_1} \right)  &\leq \|D^{1/2}_{\mu_1} (\wt K_{\mu_1} - \wt K_{\mu_2})\|_2^2 \Tr\big( \Sigma_{\mu_2}\big), \\
0 \leq \Tr \left( P_{\mu_2}(\wt{K}_{\mu_1}) - P_{\mu_2} \right)  &\leq  \|D^{1/2}_{\mu_2} (\wt K_{\mu_1} - \wt K_{\mu_2})\|_2^2 \Tr\big( \Sigma_{\mu_1}\big),
\end{aligned}
\end{equation*}
where $\Sigma_{\mu_1}$ and $\Sigma_{\mu_2}$ are the steady-state covariance matrix of the state processes driven by $\Ac_{\mu_1}$ and $\Ac_{\mu_2}$ respectively. Further, since $(\mu_1,\mu_2) \in \mathcal{M}_+^2$,~\cref{prop:positive.gradient.stable.policy} ensures that $\Tr\big( \Sigma_{\mu_1}\big) \leq \kappa(\sys)$ and $\Tr\big( \Sigma_{\mu_2}\big) \leq \kappa(\sys)$.

As a result,
\begin{equation}
\label{eq:smoothness.proof.1}
\begin{aligned}
0 \leq \Tr \left( P_{\mu_1}(\wt{K}_{\mu_2}) - P_{\mu_1} \right) &\leq \kappa(\sys) \|D^{1/2}_{\mu_1} (\wt K_{\mu_1} - \wt K_{\mu_2})\|_2^2, \\
0 \leq \Tr \left( P_{\mu_2}(\wt{K}_{\mu_1}) - P_{\mu_2} \right)  &\leq \kappa(\sys) \|D^{1/2}_{\mu_2} (\wt K_{\mu_1} - \wt K_{\mu_2})\|_2^2.
\end{aligned}
\end{equation}

From~\cref{eq:average.cost.constraints.lyapunov3}, we have that
\begin{equation*}
\begin{aligned}
P_{\mu_1}(\wt{K}_{\mu_2}) &= P(\wt{K}_{\mu_2}) + \mu_1 G(\wt{K}_{\mu_2}) = P_{\mu_2}(\wt{K}_{\mu_2}) + (\mu_1 - \mu_2) G(\wt{K}_{\mu_2}) = P_{\mu_2} +  (\mu_1 - \mu_2) G(\wt{K}_{\mu_2}), \\
P_{\mu_2}(\wt{K}_{\mu_1}) &= P(\wt{K}_{\mu_1})  + \mu_2 G(\wt{K}_{\mu_1})  = P_{\mu_1}(\wt{K}_{\mu_1})  + (\mu_2 - \mu_1)G(\wt{K}_{\mu_1})  = P_{\mu_1} + (\mu_2 - \mu_1) G(\wt{K}_{\mu_1}).
\end{aligned}
\end{equation*}
Combining the later identities with~\cref{eq:smoothness.proof.1} and summing the two inequalities we obtain:
\begin{equation*}
|\mu_1 - \mu_2| \left |  \Tr \left( G(\wt{K}_{\mu_2}) - G(\wt{K}_{\mu_1}) \right) \right | \leq \kappa(\sys) \|(D_{\mu_1} + D_{\mu_2})^{1/2} (\wt K_{\mu_1} - \wt K_{\mu_2})\|_2^2
\end{equation*}
Finally, from 
\begin{equation*}
\begin{aligned}
\|(D_{\mu_1} + D_{\mu_2})^{1/2} (\wt K_{\mu_1} - \wt K_{\mu_2})\|_2^2 &\leq \lambda_{\max}\left( (D_{\mu_1} + D_{\mu_2})^{-1} \right) \|(D_{\mu_1} + D_{\mu_2}) (\wt K_{\mu_1} - \wt K_{\mu_2})\|_2^2 \\
&\leq \frac{1}{\lambda_{\min}(D_{\mu_1} + D_{\mu_2})}  \|(D_{\mu_1} + D_{\mu_2}) (\wt K_{\mu_1} - \wt K_{\mu_2})\|_2^2 \\
&\leq \frac{1}{\max\left( \lambda_{\min}(D_{\mu_1}), \lambda_{\min}(D_{\mu_2}) \right)} \|(D_{\mu_1} + D_{\mu_2}) (\wt K_{\mu_1} - \wt K_{\mu_2})\|_2^2,
\end{aligned}
\end{equation*}
one has, noticing that $\left( \mathcal{D}^\prime(\mu_2,\sys) - \mathcal{D}^\prime(\mu_1,\sys) \right)  =  \Tr \left( G(\wt{K}_{\mu_2}) - G(\wt{K}_{\mu_1}) \right)$ (see Proof of~\cref{prop:decreasing.gradient.lowner.pseudo.concave.Dmu}),
\begin{equation}
\label{eq:smoothness.proof.2}
|\mu_1 - \mu_2|  \left| \mathcal{D}^\prime(\mu_2,\sys) - \mathcal{D}^\prime(\mu_1,\sys) \right| \leq \frac{\kappa(\sys)}{\max\left( \lambda_{\min}(D_{\mu_1}), \lambda_{\min}(D_{\mu_2}) \right)} \|(D_{\mu_1} + D_{\mu_2}) (\wt K_{\mu_1} - \wt K_{\mu_2})\|_2^2.
\end{equation}
\cref{eq:smoothness.proof.2} indicates that the smoothness of $\mathcal{D}^\prime$ is tied to the one of $\mu \rightarrow \wt{K}_{\mu}$. Using the expression of $\wt{K}_\mu$ in~\cref{prop:dual.riccati}, we have 
$D_{\mu_1} \wt K_{\mu_1} = - \wt{B}^\transp P_{\mu_1} A - N_{\mu_1}$ and $D_{\mu_2} \wt K_{\mu_2} = - \wt{B}^\transp P_{\mu_2} A - N_{\mu_2}$. Combining those two expressions leads to
\begin{equation*}
\begin{aligned}
D_{\mu_1}(\wt K_{\mu_1} - \wt K_{\mu_2} ) &=  \wt{B}^\transp (P_{\mu_2} - P_{\mu_1}) A + (N_{\mu_2} - N_{\mu_1}) + (D_{\mu_2} - D_{\mu_1}) \wt K_{\mu_2}, \\
D_{\mu_2}(\wt K_{\mu_1} - \wt K_{\mu_2} ) &=  \wt{B}^\transp (P_{\mu_2} - P_{\mu_1}) A + (N_{\mu_2} - N_{\mu_1}) + (D_{\mu_2} - D_{\mu_1}) \wt K_{\mu_1},\\
(D_{\mu_1} + D_{\mu_2}) (\wt K_{\mu_1} - \wt K_{\mu_2} ) &= 2  \wt{B}^\transp (P_{\mu_2} - P_{\mu_1}) A + 2 (N_{\mu_2} - N_{\mu_1}) + (D_{\mu_2} - D_{\mu_1}) (\wt K_{\mu_1} +\wt K_{\mu_2}).
\end{aligned}
\end{equation*}
Since $D_{\mu} = R_{\mu}  + \wt{B}^\transp P_\mu \wt{B}$ and $ 2 A + \wt{B}  (\wt K_{\mu_1} +\wt K_{\mu_2}) = \Ac_{\mu_1} + \Ac_{\mu_2}$ we obtain
\begin{equation*}
(D_{\mu_1} + D_{\mu_2}) (\wt K_{\mu_1} - \wt K_{\mu_2} ) =  \wt{B}^\transp (P_{\mu_2} - P_{\mu_1})  ( \Ac_{\mu_1} + \Ac_{\mu_2}) + 2 (N_{\mu_2} - N_{\mu_1}) + (R_{\mu_2} - R_{\mu_1}) (\wt K_{\mu_1} +\wt K_{\mu_2})\end{equation*}
From~\cref{prop:positive.gradient.bounded.closed.loop.controller} and~\cref{eq:bound.matrix.dual.norm2}, we have 
\begin{equation*}
\left\| \wt{B}^\transp (P_{\mu_2} - P_{\mu_1})  ( \Ac_{\mu_1} + \Ac_{\mu_2}) \right\|_2 \leq 2 |\mu_1 - \mu_2| (1 + \|B\|_2)\kappa(\sys) \max \left(\|C_g\|_2 \sqrt{\kappa(\sys)} , (2 + \|A\|_2\|B\|_2) \right).
\end{equation*}
Further, the linear structure of $C_\mu$ leads to 
\begin{equation*}
 (N_{\mu_2} - N_{\mu_1}) + (R_{\mu_2} - R_{\mu_1}) (\wt K_{\mu_1} +\wt K_{\mu_2}) = (\mu_2 - \mu_1) \begin{pmatrix} 0 \\ I_{n} \end{pmatrix} C_g \left( \begin{pmatrix} I \\ \wt K_{\mu_1} \end{pmatrix} + \begin{pmatrix} I \\ \wt K_{\mu_2} \end{pmatrix} \right).
\end{equation*}
Combined with~\cref{prop:positive.gradient.bounded.closed.loop.controller}, this leads to
\begin{equation*}
\left\| (N_{\mu_2} - N_{\mu_1}) + (R_{\mu_2} - R_{\mu_1}) (\wt K_{\mu_1} +\wt K_{\mu_2})\right\|_2 \leq  2 (\mu_2 - \mu_1) \|C_g\|_2 \sqrt{\kappa(\sys)} (2 + \|A\|_2\|B\|_2).
\end{equation*}
Putting everything together, we have 
\begin{equation*}
\|(D_{\mu_1} + D_{\mu_2}) (\wt K_{\mu_1} - \wt K_{\mu_2} )\|_2 \leq |\mu_1 - \mu_2| \kappa(\sys)^{3/2} c(\sys), \text{ where } c(\sys) := 2 \|C_g\|_2 \big(1 + \|B\|_2\big)(2 + \|A\|_2\|B\|_2).
\end{equation*}
Finally, using this upper bound in~\cref{eq:smoothness.proof.2} leads to
\begin{equation*}
\left| \mathcal{D}^\prime(\mu_2,\sys) - \mathcal{D}^\prime(\mu_1,\sys)  \right| \leq  |\mu_1 - \mu_2| \frac{\|C_g\|_2 \alpha_{D}(\sys)}{2 \max\left( \lambda_{\min}(D_{\mu_1}), \lambda_{\min}(D_{\mu_2}) \right)},
\end{equation*}
where $\alpha_D(\sys) := 8 \|C_g \|_2 \kappa(\sys)^4 \left( (2 + \|A\|_2\|B\|_2) (1 + \|B\|_2) \right)^2$ is the same problem dependent constant as in~\cref{cor:behavior.Dmu.closure.M}.
\end{enumerate}
\end{proof}

\newpage
\newpage
\section{Algorithms and complexity bound}\label{sec:app.algo}
In this section, we provide the pseudo-code of \ofulqplus and \laglq. We further detail the procedure of \dsllq which allows us to solve the Extended LQR with Relaxed Constraint efficiently by Dichotomy Search and then prove~\cref{thm:optimal.complexity.algo.main}.

\subsection{\ofulqplus and \laglq}

\begin{figure*}[!h]
	\begin{minipage}{0.48\textwidth}
	\begin{center}
	\bookboxx{
		\begin{small}
			\begin{algorithmic}[1]
				\renewcommand{\algorithmicrequire}{\textbf{Input:}}
				\renewcommand{\algorithmicensure}{\textbf{Output:}}
				\vspace{-0.02in}
				\REQUIRE $\Theta_0$, $T$, $\delta$, $Q$, $R$, $D$, $\sigma^2$.
				\STATE Set $\lambda$ by~\cref{eq:algos.lambda.def}, $V_0 = \lambda I$, $\wh{\theta}_0$ and $\mathcal{C}_0$ by~\cref{eq:algos.rls}.
				\STATE $\theta_0 = \arg\min_{\theta \in \mathcal{C}_0} J(\theta)$.
				\STATE Compute optimistic controller $K_0 = K(\wt{\theta}_0)$ by~\cref{eq:algo.riccati.equation}.
				\FOR{t=1,\dots,T}
					\IF{$\det(V_t) \geq 2 \det(V_0)$}
						\STATE Compute $\wh{\theta}_t$ and $\mathcal{C}_{t}$ by~\cref{eq:algos.rls}.
						\STATE Find $\theta_t = \arg\min_{\theta \in \mathcal{C}_t} J(\theta) + \frac{1}{\sqrt{t}}$.
						\STATE Compute optimistic controller $K_t = K(\theta_t)$ by~\cref{eq:algo.riccati.equation}.
						\STATE Let $V_0 = V_t$.
					\ELSE
						\STATE $K_{t} = K_{t-1}$.
					\ENDIF
					\STATE Compute control $u_t$ based on the current controller $K_t$ as $u_t = K_t x_t$.
					\STATE Execute control $u_t$ and observe next state $x_{t+1}$.
					\STATE Save $(z_t,x_{t+1})$ in the dataset, where $z_t^\transp = (x_t^\transp, u_t^\transp)$.
					\STATE Update $V_{t+1} = V_t + z_t z_t^\transp$.
				\ENDFOR
			\end{algorithmic}
			\vspace{-0.1in}
			\caption{\small \ofulqplus algorithm}
			\label{alg:ofulqplus}
		\end{small}
	}
	\end{center}		
	\end{minipage}
	\hfill
	\begin{minipage}{0.48\textwidth}
	\begin{center}
	\bookboxx{
		\begin{small}
			\begin{algorithmic}[1]
				\renewcommand{\algorithmicrequire}{\textbf{Input:}}
				\renewcommand{\algorithmicensure}{\textbf{Output:}}
				\vspace{-0.02in}
				\REQUIRE $\Theta_0$, $T$, $\delta$, $Q$, $R$, $D$, $\sigma^2$.
				\STATE Set $\lambda$ by~\cref{eq:algos.lambda.def}, $V_0 = \lambda I$, $\wh{\theta}_0$, $\beta_0$ and $\mathcal{C}_0$ by~\cref{eq:algos.rls}.
				\STATE Compute the extended controller $\wt{K}_0 = \dsllq()$.
				\STATE Extract  $K_0 =  \big[ \wt{K}_0 \big]_{1:d,1:n}$
				\FOR{t=1,\dots,T}
					\IF{$\det(V_t) \geq 2 \det(V_0)$}
						\STATE Compute $\wh{\theta}_t$, $\beta_t$ and $\mathcal{C}_{t}$ by~\cref{eq:algos.rls}.
						\STATE Compute the extended controller $\wt{K}_t = \dsllq()$.
						\STATE Extract  $K_t =  \big[ \wt{K}_t \big]_{1:d,1:n}$
						\STATE Let $V_0 = V_t$.
					\ELSE
						\STATE $K_{t} = K_{t-1}$.
					\ENDIF
					\STATE Compute control $u_t$ based on the current controller $K_t$ as $u_t = K_t x_t$.
					\STATE Execute control $u_t$ and observe next state $x_{t+1}$.
					\STATE Save $(z_t,x_{t+1})$ in the dataset, where $z_t^\transp = (x_t^\transp, u_t^\transp)$.
					\STATE Update $V_{t+1} = V_t + z_t z_t^\transp$.
				\ENDFOR
			\end{algorithmic}
			\vspace{-0.1in}
			\caption{\small \laglq algorithm}
			\label{alg:laglq}
		\end{small}
	}
	\end{center}
	\end{minipage}
\vspace{-0.1in}
\end{figure*}

\ofulqplus and \laglq builds on the original \ofulq algorithm of~\citet{abbasi2011regret} but differ in the initialization (to ensure that the state remains bounded over the trajectory) and in the way the optimistic controller is computed (for \laglq). \\

\textbf{Initialization.} \ofulq and \laglq are provided with the LQR costs matrices $Q,R$, the upper bound $D$ on $\Tr(P_*)$ (see~\cref{asm:stability.margin}), the proxy-variance upper bound $\sigma^2$ (see~\cref{asm:noise}), the time horizon $T$ and the confidence level $\delta$ as well as the stabilizing set $\Theta_0 = \{\theta: \|\theta-\theta_0\| \leq \epsilon_0 \}$, where $\epsilon_0$ satisfies the requirement of~\cref{lem:ofu.dynamic.stability2} for \ofulqplus and~\cref{eq:state.bound.ofulqplus.extended} for \laglq. Formally, we require that 
\begin{enumerate}[leftmargin=2cm]
\item[\ofulqplus: ] $\frac{2 \kappa \eta(2 \epsilon_0)}{1 - \eta(2 \epsilon_0)} \leq 1$ and $\eta(2 \epsilon_0) := 8 \kappa \epsilon_0 (1 + 2 \epsilon_0) < 1$ which is guaranteed as soon as $\epsilon_0 \leq 1/(64 \kappa^2)$, 
\item[\laglq: ] $\frac{2 \kappa \eta(2 \sqrt{\kappa} \epsilon_0)}{1 - \eta(2 \sqrt{\kappa} \epsilon_0)} \leq 1$ and $\eta(2 \sqrt{\kappa} \epsilon_0) := 8 \kappa \sqrt{\kappa} \epsilon_0 (1 +2 \sqrt{\kappa} \epsilon_0) < 1$ which is guaranteed as soon as $\epsilon_0 \leq 1/(64 \kappa^{5/2})$, 
\end{enumerate}
where $\kappa = D / \lambda_{\min}(C)$, $C= \begin{pmatrix} Q & 0 \\ 0 & R \end{pmatrix}$. The stabilizing set can be obtained in finite time following the procedure of~\citet{faradonbeh2018finite} or following the procedure of~\citet{simchowitz2020naive} if a stabilizing controller $K_0$ is known a priori. Finally, the stabilizing set $\Theta_0$ is used to regularize the RLS estimates as
\begin{equation}
\label{eq:algos.lambda.def}
\lambda = \frac{2 n^2 \sigma^2}{\epsilon_0^2}  \Big( \log  \Big( \!\frac{4T}{\delta}\Big) + (n+d) \log( 1 +\kappa X^2 T ) \Big).
\end{equation}

\textbf{Recursive Least-Square.} After the initial phase, \ofulqplus and \laglq proceed through episodes following the update rule of \ofulq in~\citep{abbasi2011regret} i.e., update whenever $\det(V_t) \geq 2\det(V_{t_k})$. At the beginning of each episode, they maintain an estimation of $\theta_*$ by RLS together with a confidence set given by:
 \begin{equation}
\label{eq:algos.rls}
\wh{\theta}_t =  V_{t}^{-1} \Big( \lambda \theta_0 + \sum_{s=0}^{t-1} z_s x_{s+1}^\transp \Big); \quad \beta_t = n\sigma \sqrt{2 \log \Big( \frac{\det(V_t)^{1/2} 4T}{\det(\lambda I)^{1/2} \delta} \Big)} + \lambda^{1/2} \epsilon_0; \quad \mathcal{C}_t =  \{\theta: \| \theta - \wh{\theta}_{t} \|_{V_{t}} \leq \beta_{t}\}.
\end{equation}

\textbf{Optimistic controllers.} \ofulqplus and \laglq differs in the way the compute the optimistic controller at each policy update.\\
\ofulqplus follows~\citep{abbasi2011regret} and has to find an optimistic parameter by a $1/\sqrt{t}$ margin as $\theta_t = \arg\min_{\theta \in \mathcal{C}_t} J(\theta) + \frac{1}{\sqrt{t}}$. Then, it computes the optimal controller associated with $\theta_t^\transp = (A_t, B_t)$ by solving a standard Riccati equation:
\begin{equation}
\label{eq:algo.riccati.equation}
\begin{aligned}
P_t &= Q + A_t^\transp P_t A_t - A_t^\transp P_t B_t \big( R + B_t^\transp P_t B_t) B_t^\transp P_t A_t, \\
K_t &= - \big( R + B_t^\transp P_t B_t) B_t^\transp P_t B_t.
\end{aligned}
\end{equation}
Unfortunately, finding the optimistic parameter $\wt{\theta}_t$ is intractable which makes \ofulqplus inefficient.\\
On the other hand, \laglq uses the Extended LQR with Relaxed Constraint approach to compute an extended optimistic controller $\wt{K}_t$ using the \dsllq routine (described in detail in the next subsection). Then, it simply extracts the control part of $\wt{K}^\transp_t = \begin{pmatrix}K^{u \transp}_t &K^{w \transp}_t \end{pmatrix}$ to obtain the controller used for the whole episode. The complexity of \dsllq is at most $O\big(n^3\big)$ and only requires solving Riccati equations associated with the extended LQR, which makes \laglq very efficient.

\subsection{\dsllq algorithm}

In this subsection, we detail the algorithm \dsllq and the back-up procedure \backproc used in \laglq to solve the Extended LQR with Relaxed Constraint problem~\cref{eq:optimal.avg.cost.extended}. We leverage the strong duality results of~\cref{thm:dual.gap.specific} to find a solution to~\cref{eq:optimal.avg.cost.extended} by solving the associated Lagrangian extended LQR problem~\cref{eq:lagrange.lqr}. The routines are summarized in~\cref{alg:dsllq.ext,alg:backupproc}.

\begin{figure*}[!h]
		\begin{minipage}{0.48\textwidth}
	\begin{center}
	\bookboxx{
		\begin{small}
			\begin{algorithmic}[1]
				\renewcommand{\algorithmicrequire}{\textbf{Input:}}
				\renewcommand{\algorithmicensure}{\textbf{Output:}}
				\vspace{-0.02in}
				\REQUIRE $Q$, $R$, $D$, $\wh\theta$, $\beta$, $V$, $\epsilon$, 
				\STATE Set $\mu_{\max}$ by~\cref{eq:algos.mu.max}, $\alpha$, $\lambda_0$ by~\cref{eq:algo.constants.phase1}.
				\IF{$\D(0) \leq 0$}
				\STATE Set $\wb \mu = 0$ and $\tpi_\epsilon = \tpi_{\wb\mu}$
				\ELSE
					\STATE Set $\mu_l = 0$, $\mu_r = \mu_{\max}$ (\cref{lem:mu.max})
					\WHILE{$\alpha \frac{\mu_r - \mu_l}{\lambda_{\min}(D_{\mu_l})} \geq \epsilon$ \textbf{ or } $\lambda_{\min}(D_{\mu_l}) \geq \lambda_0 \epsilon^2$}
						\STATE Set $\wb\mu = (\mu_l+\mu_r)/2$
						\IF{$\D'(\wb\mu) > 0$}
							\STATE $\mu_l = \wb\mu$
						\ELSE
							\STATE $\mu_r = \wb\mu$
						\ENDIF
					\ENDWHILE				
				\ENDIF
				\IF{$\alpha \frac{\mu_r - \mu_l}{\lambda_{\min}(D_{\mu_l})} \leq \epsilon$}
					\STATE Set $\wb\mu = \mu_l$ and $\tpi_\epsilon = \tpi_{\wb\mu}$
				\ELSE
					\STATE Compute $\tpi_\epsilon = \backproc(Q, R, D, \wh\theta, \beta, V, \epsilon, \bar{\mu})$.
				\ENDIF
				\RETURN Control policy $\tpi_\epsilon$
			\end{algorithmic}
			\vspace{-0.1in}
			\caption{\small The \dsllq algorithm to solve~\eqref{eq:lagrange.lqr}.}
			\label{alg:dsllq.ext}
		\end{small}
	}
	\end{center}

	\end{minipage}
	\hfill
	\begin{minipage}{0.48\textwidth}
		\begin{center}
	\bookboxx{
		\begin{small}
			\begin{algorithmic}[1]
				\renewcommand{\algorithmicrequire}{\textbf{Input:}}
				\renewcommand{\algorithmicensure}{\textbf{Output:}}
				\vspace{-0.02in}
				\REQUIRE $Q$, $R$, $D$, $\wh\theta$, $\beta$, $V$, $\epsilon$, $\bar{\mu}$.
				\STATE Compute $D_{\wb{\mu}}$ by~\cref{eq:algos.riccati.lyapunov.dual}
				\IF{$\min_{\|v\|=1, v \in \ker(\wt{B})}(D_\mu)  \leq \sqrt{\lambda_0} \epsilon$} 
				\STATE Compute $\wt{K}_\epsilon$ by~\cref{eq:algos.explicit.modified.policy}.
				\STATE Set $\tpi_{\epsilon}(x) = \wt{K}_\epsilon x$.
				\ELSE
					\STATE Modify the LQR system according to~\cref{eq:algos.explicit.modified.system}.
					\STATE Set $\mu_l = 0$, $\mu_r =\wb{\mu}$, $\alpha_{mod}$ by~\cref{eq:algos.alpha.modified}.
					\WHILE{$\alpha_{mod} (\mu_r - \mu_l ) \geq \epsilon^3$}
						\STATE Set $\mu = (\mu_l+\mu_r)/2$
						\IF{$\D^\prime_{mod}(\mu) > 0$}
							\STATE $\mu_l = \mu$
						\ELSE
							\STATE $\mu_r = \mu$
						\ENDIF
					\ENDWHILE
					\STATE Set $\mu = \mu_l$ and $\tpi_\epsilon = \tpi^{mod}_{\mu}$.				
				\ENDIF
				\RETURN Control policy $\tpi_\epsilon$
			\end{algorithmic}
			\vspace{-0.1in}
			\caption{\small Back up procedure \backproc}
			\label{alg:backupproc}
		\end{small}
	}
	\end{center}

	\end{minipage}
\vspace{-0.1in}
\end{figure*}

\textbf{Extended Lagrangian LQR.} \dsllq aims at finding a solution to~\cref{eq:lagrange.lqr} up to an $\epsilon$ accuracy. It takes as input an estimated LQR instance, parametrized by the cost matrices $Q,R$ and the RLS estimate $\wh{\theta}^\transp = ( \wh{A} , \wh{B})$ as well as the associated confidence set given by $V$ and $\beta$. Then, for each Lagrangian parameters $\mu$, it constructs the Lagrangian extended LQR system $(A,\wt{B},C_\mu)$ where $C_\mu = C_{\dagger} + \mu C_g$, 
\begin{equation*}
A = \wh{A}; \quad\quad \wt{B} = (\wh{B}, I ); \quad \quad C_\mu = \begin{pmatrix} Q_\mu & N_\mu\\N^\transp_\mu & R_\mu \end{pmatrix}; \quad\quad C_\dagger = \begin{pmatrix} Q & 0 & 0\\ 0 & R & 0 \\ 0 & 0 & 0 \end{pmatrix}; \quad \quad C_g = \begin{pmatrix} \beta^2 V^{-1} & 0 \\ 0 & I \end{pmatrix}.
\end{equation*}
\cref{lem:dual.function.specific} ensures that whenever $\mu$ is in some admissible set $\mathcal{M}$, the dual function $\mathcal{D}(\mu)$ and its derivative $\mathcal{D}^\prime(\mu)$ associated with the estimated extended LQR can be computed by Riccati and Lyapunov equations as $\mathcal{D}(\mu) = \Tr(P_\mu)$ and $\mathcal{D}^\prime(\mu) = \Tr\big(G(\wt{K}_\mu)\big)$ where 
\begin{equation}
\label{eq:algos.riccati.lyapunov.dual}
\begin{aligned}
P_\mu &=  Q_{\mu} + A^\transp P_{\mu} A  - \big( A^\transp P_{\mu} \tilde{B} + N_{\mu}^\transp \big)  D_\mu^{-1} \big( \tilde{B}^\transp P_{\mu} A + N_{\mu} \big) \\
D_\mu &= R_{\mu}+ \tilde{B}^\transp P_{\mu} \tilde{B};\quad\quad \wt{K}_\mu = -D_\mu^{-1} \big( \tilde{B}^\transp P_{\mu} A + N_{\mu} \big); \quad\quad \Ac_\mu = A + \wt{B} \wt{K}_\mu,\\
G(\wt{K}_\mu) &= \big(\Ac_\mu\big)^\transp G(\wt{K}_\mu) \Ac_\mu + \begin{pmatrix}I\\\wt{K}_\mu\end{pmatrix}^\transp C_g \begin{pmatrix}I\\\wt{K}_\mu\end{pmatrix}.
\end{aligned}
\end{equation}
Whenever $\mu \notin \mathcal{M}$ (or equivalently when~\cref{eq:algos.riccati.lyapunov.dual} does not admit a well-defined solution in the sense of~\cref{le:riccati.lqr.solution}), we simply set $\mathcal{D}^\prime(\mu) = - \infty$. Finally, since all the involved policies are linear in the state, we use either $\tpi$ or the linear controller $\wt{K}$ to characterize them.\\

\textbf{Initialization.} \dsllq performs a dichotomy search over some range $[0,\mu_{\max}]$ to find an optimistic feasible extended policy. Before starting the dichotomy, \dsllq computes the upper bound $\mu_{\max}$ by~\cref{eq:algos.mu.max} and~\cref{lem:mu.max} ensures that $\mathcal{M} \subset [0, \mu_{\max}]$
\begin{equation}
\label{eq:algos.mu.max}
\mu_{\max} = \beta^{-2}\lambda_{\max}\big(C\big) \lambda_{\max}(V).
\end{equation}
Finally, \dsllq uses the upper bound $D$ on $\Tr(P_*)$ to compute \textit{conservative} constants that will be use to set the termination rules.

\begin{equation}
\label{eq:algo.constants.phase1}
\begin{aligned}
 \alpha &= \max\big( 1, \|C_g\|_2 / 2 \big) 8 \|C_g \| \kappa^4 \left( (2 + \|\wh{A}\|_2\|\wh{B}\|_2) (1 + \|\wh{B}\|_2) \right)^2, \quad \kappa = D / \lambda_{\min}(C) \\
\lambda_0 &=\min\left( \frac{\lambda_{\min}(C)}{ 2 \|\wt{B}\|_2^2 \max(D,1)}, \frac{1}{ 8^{2n+1} \kappa^{2n} } \min\left(1 , \frac{\min\big(1,\lambda_{\min}(C) / 2\kappa\big)  \sigma_{\wt{B}}^2}{ 2 \kappa^2c(\mu_{\max})} \right) \right)^2, \\
\sigma_{\wt{B}}^2 &= \min_{\|v\|=1, v \in Im(\wt{B})} v^\transp \wt{B}^\transp \wt{B} v; \quad \quad c(\mu_{\max}) = \big(\lambda_{\max}(C) +\mu_{\max} \big)\big( 1 + \|\wt{B}\|_2^2 (1 + \|A\|_2^2)\big).
\end{aligned}
\end{equation}

\textbf{Dichotomy Search.} \dsllq proceeds with the dichotomy search until $\alpha \frac{\mu_r - \mu_l}{\lambda_{\min}(D_{\mu_l})} \leq \epsilon$ or $\lambda_{\min}(D_{\mu_l}) \leq \lambda_0 \epsilon^2$. When the first condition is met, a valid solution $\tpi_*$ is found by the dichotomy search. When the second condition is met, the dichotomy search failed to find a valid solution, and we use the \textit{backup procedure} \backproc.\\

\textbf{\backproc.} When the back-up procedure starts, it means that $\lambda_{\min}(D_{\wb{\mu}}) \leq \lambda_0 \epsilon^2$. \backproc then distinguishes between two cases, depending on the spectrum of $D_{\wb{\mu}}$ w.r.t. the null space of $\wt{B}$.
\begin{itemize}
\item When $\min_{\|v\|=1, v \in \ker(\wt{B})}(D_\mu)  \leq \sqrt{\lambda_0} \epsilon$, a feasible optimistic linear policy can be computed explicitly as:
\begin{equation}
\label{eq:algos.explicit.modified.policy}
\wt{K}_\epsilon = \wt{K}_{\wb{\mu}} + \eta v x^\transp,
\end{equation}
where 
\begin{equation*}
\begin{aligned}
\eta^2 &= \frac{\mathcal{D}^\prime(\wb{\mu})}{ - v^\transp Z v x^\transp \Sigma_{\wb{\mu}} x}; \quad Z = \begin{pmatrix} 0 \\ I_{n+d} \end{pmatrix}^\transp C_g \begin{pmatrix} 0 \\ I_{n+d} \end{pmatrix}; \quad  Y =  \begin{pmatrix} 0 \\ I_{n+d} \end{pmatrix}^\transp C_g \begin{pmatrix} I \\ \wt{K}_{\wb{\mu}} \end{pmatrix}; \quad  \Sigma_{\wb{\mu}}  = \big(\Ac_{\wb{\mu}}  \big)^\transp \Sigma_{\wb{\mu}}  \Ac_{\wb{\mu}} + I.
\end{aligned}
\end{equation*}
and $v$ is an eigenvector associated with $\lambda_{\min}(D_{\wb{\mu}})$, $x$ is a unitary vector such that $v^\transp Y \Sigma_{\wb{\mu}} x = 0$.

\item When $\min_{\|v\|=1, v \in \ker(\wt{B})}(D_\mu) >  \sqrt{\lambda_0} \epsilon$, we construct a modified extended Lagrangian LQR system as:
\begin{equation}
\label{eq:algos.explicit.modified.system}
A^{mod} = A; \quad \wt{B}^{mod} = \wt{B}; \quad \quad C^{mod}_g = C_g; \quad\quad C_{\dagger}^{mod} = C_{\dagger} + \eta \Delta; \quad\quad C_{\mu}^{mod} = C_{\dagger}^{mod} + \mu C_g^{mod},
\end{equation}
where 
\begin{equation*}
\begin{aligned}
\Delta &=  \begin{pmatrix} I & - \wb{K}^\transp \\ 0 & I\end{pmatrix}  \begin{pmatrix} I \\ - \wt{B}^\transp \end{pmatrix} \begin{pmatrix} I & -  \wt{B} \end{pmatrix} \begin{pmatrix} I & 0 \\ - \wb{K} & I \end{pmatrix}; \quad\quad \wb{K} = \begin{pmatrix} 0 \\ - A \end{pmatrix},\\
\eta &=\min\left(c(\mu_{\max})\big/ \sigma_{\wt{B}}^2,\min\big(1,\lambda_{\min}(C)/2 \kappa \big)\big /( 2 \kappa^2) \right) \epsilon.
\end{aligned}
\end{equation*}
We then consider $\mathcal{D}_{mod}$ and $\mathcal{D}^\prime_{mod}$, the dual function and its derivative associated with the modified system, and run a standard dichotomy search over $[0,\wb{\mu}]$ on the modified system. The system modification ensures $\mathcal{D}_{mod}$ to have a bounded curvature. As a result, the termination criterion of the dichotomy search for the modified system is simplified to $\alpha_{mod} (\mu_r - \mu_l ) \leq \epsilon^3$, where 
\begin{equation}
\label{eq:algos.alpha.modified}
\alpha_{mod} =  \frac{64 \|C_g\|^2_2 \kappa^4 \left( (2 + \|A\|_2\|B\|_2) (1 + \|B\|_2) \right)^2}{\min \left( (1 + \|B\|_2)^{-2} \lambda_{\min}(C), \sqrt{\lambda_0} /8 \right)}.
\end{equation}
\end{itemize}

\subsection{Guarantees}

\begin{customthm}{3}\label{thm:optimal.complexity.algo.main.formal}
	For any LQR parametrized by $\wh\theta$, $V$, $\beta$, and psd cost matrices $Q,R$, and any accuracy $\epsilon \in (0,1/2)$, whenever $\theta_* \in \mathcal{C}$,
	\begin{enumerate}
		\item \dsllq outputs an $\epsilon$-optimistic and $\epsilon$-feasible policy $\tpi_\epsilon$ given by the linear controller $\wt{K}_\epsilon$ such that
		\begin{equation*}
		\mathcal{J}_{\tpi_\epsilon} \leq  \mathcal{J}_*+  \epsilon \quad \text{ and }\quad  g_{\tpi_\epsilon} \leq \epsilon.
		\end{equation*}
		\item \dsllq terminates within at most $N = O\big(\log(\mu_{\max} /\epsilon) \big)$ iterations, each solving one Riccati and one Lyapunov equation for the extended Lagrangian LQR, both with complexity $O\big(n^3\big)$.
	\end{enumerate}
\end{customthm}


\begin{proof}[Proof of~\cref{thm:optimal.complexity.algo.main}]
We first prove that \dsllq actually finds an $\epsilon-$optimal feasible solution and then characterize its complexity.\\

\textbf{Performance guarantees for \dsllq.} The algorithm proceeds in two phases \textbf{1)} it performs a dichotomy search on the derivative of the dual function $\mathcal{D}$ associated with original extended LQR problem. \textbf{2)} if the solution of the first dichotomy search is not satisfactory, it enters the \textit{backup procedure} \backproc, and obtain modified valid solution, either explicitly, either performing a dichotomy search for a modified extended LQR problem. The proof is conducted in two steps, one for each phase.

\begin{enumerate}
\item The first dichotomy search is done w.r.t. $\mathcal{D}^\prime$. By~\cref{lem:dual.function.specific}, $\mathcal{D}$ is concave over $\mathcal{M}$, and the initialization procedure ensures by~\cref{lem:mu.max} that $\mathcal{M} \subset [0,\mu_{\max}]$, $\mathcal{D}^\prime(0)>0$ and $\mathcal{D}^\prime(\mu_{\max}) \leq 0$. As a result, the dichotomy maintains a non-empty interval $[\mu_l, \mu_r]$, which shrinks with the number of iterations.\\ Further, at the end of the dichotomy, we have that 
\begin{equation*}
\mathcal{D}^\prime(\mu_l) > 0; \quad\quad \mathcal{D}^\prime(\mu_r) \leq 0.
\end{equation*}
 
\begin{itemize}
\item If $\alpha \frac{\mu_r - \mu_l}{\lambda_{\min}(D_{\mu_l})} < \epsilon$, then, strong-duality holds in a 'strict' sense (see~\cref{prop:strong.duality.strict}) and it exists $\mu_* \in \mathcal{M}$ such that $\mathcal{D}^\prime(\mu_*) = 0$. Further, by construction, $\mu_* \in [\mu_l, \mu_r]$. We prove this assertion by contradiction. Suppose that strong-duality holds in a 'weak-sense', which corresponds to the case where $\mathcal{D}^\prime(\mu) > 0$ for all $\mu \in \mathcal{M}$. Then, $\wt{\mu} \in [\mu_l,\mu_r]$ and by~\cref{cor:behavior.Dmu.closure.M}, 
\begin{equation*}
\begin{aligned}
\lambda_{\min}(D_{\mu_l}) &\leq 8 \|C_g \| \kappa^4 \left( (2 + \|\wh{A}\|_2\|\wh{B}\|_2) (1 + \|\wh{B}\|_2) \right)^2 |\mu_l - \wt{\mu}| \\
&\leq 8 \|C_g \| \kappa^4 \left( (2 + \|\wh{A}\|_2\|\wh{B}\|_2) (1 + \|\wh{B}\|_2) \right)^2 |\mu_l - \mu_r| \\
&\leq \epsilon  \frac{\lambda_{\min}(D_{\mu_l}) }{\max\big( 1, \|C_g\|_2 / 2 \big) } \leq \epsilon \lambda_{\min}(D_{\mu_l}).
\end{aligned}
\end{equation*}
Since $\epsilon < 1$, we obtain the contradiction.\\

Now that the existence of $\mu_* \in [\mu_l,\mu_r]$ such that $\mathcal{D}^\prime(\mu_*)$ is asserted, we have from~\cref{le:smoothness.D},
\begin{equation*}
\begin{aligned}
 \mathcal{D}^\prime(\mu_l)  = | \mathcal{D}^\prime(\mu_l) - \mathcal{D}^\prime(\mu_*)| &\leq \frac{4 \|C_g \|^2 \kappa^4 \left( (2 + \|\wh{A}\|_2\|\wh{B}\|_2) (1 + \|\wh{B}\|_2) \right)^2}{\lambda_{\min}(D_{\mu_l})} |\mu_l - \mu_*| \\
 &\leq  \frac{4 \|C_g \|^2 \kappa^4 \left( (2 + \|\wh{A}\|_2\|\wh{B}\|_2) (1 + \|\wh{B}\|_2) \right)^2}{\lambda_{\min}(D_{\mu_l})} |\mu_l - \mu_r| \\
 &\leq \frac{4 \|C_g \|^2 \kappa^4 \left( (2 + \|\wh{A}\|_2\|\wh{B}\|_2) (1 + \|\wh{B}\|_2) \right)^2}{\alpha} \epsilon \leq \epsilon.
 \end{aligned}
\end{equation*}
As a result, from $g_{\tpi_{\mu_l}} = \mathcal{D}^\prime(\mu_l) \leq \epsilon$, we obtain the $\epsilon-$ feasibility. Further, from $g_{\tpi_{\mu_l}} = \mathcal{D}^\prime(\mu_l) > 0$, we have by weak-duality,
\begin{equation*}
\mathcal{J}_* \geq \mathcal{J}_{\tpi_{\mu_l}} + \mu_l g_{\tpi_{\mu_l}}  > \mathcal{J}_{\tpi_{\mu_l}}.
\end{equation*}

\item If $\alpha \frac{\mu_r - \mu_l}{\lambda_{\min}(D_{\mu_l})} \geq \epsilon$ we enter the \textit{back-up procedure}. Note that in this case, necessarily, $\lambda_{\min}(D_{\mu_l}) < \lambda_0 \epsilon^2$.
\end{itemize}

\item If \dsllq triggers \backproc, then $\lambda_{\min}(D_{\mu_l}) < \lambda_0 \epsilon^2$. 
\begin{itemize}
\item If $\min_{\|v\|=1, v \in \ker(\wt{B})}(D_\mu)  \leq \sqrt{\lambda_0} \epsilon$, the conditions of~\cref{prop:strong.duality.weak.epsilon.feasible.controller.null.space.kerBtilde} are satisfied. This guarantees that $\tpi_{\epsilon}(x) = \wt{K}_\epsilon x$, where $\wt{K}_\epsilon$ is obtained by~\cref{eq:algos.explicit.modified.policy}, is such that $g_{\tpi_\epsilon} = 0$ and
\begin{equation*}
\mathcal{J}_{\tpi_\epsilon}  = \mathcal{L}_{\tpi_\epsilon} (\wb{\mu}) \leq \mathcal{D}(\wb{\mu}) + \epsilon \leq \mathcal{J}_* + \epsilon.
\end{equation*}

\item If $\min_{\|v\|=1, v \in \ker(\wt{B})}(D_\mu) > \sqrt{\lambda_0} \epsilon$, the conditions of~\cref{prop:strong.duality.weak.epsilon.feasible.controller.null.space.imBtilde} are satisfied. This guarantees that strong duality holds in a 'strict' sense for the modified LQR system in~\cref{eq:algos.explicit.modified.system}. Formally, we have that
\begin{enumerate}
\item it exists $\mu_*^{mod} \in [0,\wb{\mu}]$ such that $\mathcal{D}^\prime_{mod}(\mu_*^{mod}) = 0$,
\item $\mathcal{D}^\prime_{mod}(0) > 0$, $\mathcal{D}^\prime_{mod}(\wb{\mu}) \leq 0$,
\item $\mathcal{J}_{\tpi^{mod}_{\mu_*^{mod}}} \leq \mathcal{J}_* + \epsilon$.
\end{enumerate}
As a result, when the dichotomy on the modified system terminates, we have 
\begin{equation*}
\begin{aligned}
 \mathcal{D}_{mod}^\prime(\mu_l)  = | \mathcal{D}_{mod}^\prime(\mu_l) - \mathcal{D}_{mod}^\prime(\mu^{mod}_*)| &\leq \frac{4 \|C_g \|^2 \kappa_{mod}^4 \left( (2 + \|\wh{A}\|_2\|\wh{B}\|_2) (1 + \|\wh{B}\|_2) \right)^2}{\lambda_{\min}(D^{mod}_{\mu_l})} |\mu_l - \mu_*| \\
 & \leq \frac{4 \|C_g \|^2 \kappa_{mod}^4 \left( (2 + \|\wh{A}\|_2\|\wh{B}\|_2) (1 + \|\wh{B}\|_2) \right)^2}{\lambda_{\min}(D^{mod}_{\mu_l})} |\mu_l - \mu_r| \\
  & \leq \frac{4 \|C_g \|^2 \kappa_{mod}^4 \left( (2 + \|\wh{A}\|_2\|\wh{B}\|_2) (1 + \|\wh{B}\|_2) \right)^2}{\lambda_{\min}(D^{mod}_{\mu_l}) \alpha_{mod}} \epsilon^3,
 \end{aligned}
\end{equation*}
where $\kappa_{mod}$ is the constant associated with the modified system an $D_{\mu}^{mod}$ is the matrix obtained by~\cref{eq:algos.riccati.lyapunov.dual} on the modified system. Futher, one can check that 
$\kappa_{mod} \leq 2 \kappa$ while~\cref{prop:modified.system.property} ensures that $\lambda_{\min}(D^{mod}_{\mu_l}) \geq \min\Big( \lambda_{\min}(D_0), \sqrt{\lambda_0} \epsilon^2  /8 \Big)$. Using that $\lambda_{\min}(D_0) \geq (1 + \|B\|_2)^{-2} \lambda_{\min}(C)$ and the value of $\alpha_{mod}$ in~\cref{eq:algos.alpha.modified}, we obtain by~\cref{le:smoothness.D},
\begin{equation*}
\begin{aligned}
 \mathcal{D}_{mod}^\prime(\mu_l)  &\leq \frac{4 \|C_g \|^2 \kappa_{mod}^4 \left( (2 + \|\wh{A}\|_2\|\wh{B}\|_2) (1 + \|\wh{B}\|_2) \right)^2}{\lambda_{\min}(D^{mod}_{\mu_l}) \alpha_{mod}} \epsilon^3 \\
 & \leq \frac{\min \left( (1 + \|B\|_2)^{-2} \lambda_{\min}(C), \sqrt{\lambda_0}/ 8\right)}{ \min\Big( (1 + \|B\|_2)^{-2} \lambda_{\min}(C), \sqrt{\lambda_0} /8 \Big)} \epsilon \leq \epsilon.
\end{aligned}
\end{equation*}
As a result, we have $g_{\tpi^{mod}_{\mu_l}} \leq \epsilon$ and $\mathcal{J}_{\tpi^{mod}_{\mu_l}} \leq \mathcal{J}_{\tpi^{mod}_{\mu_*^{mod}}} \leq \mathcal{J}_* + \epsilon$.
\end{itemize}
\end{enumerate}

\textbf{Complexity of \dsllq.} Since the dichotomy reduces the interval by a factor two at each iterations, we have:
\begin{enumerate}
\item The first dichotomy performed by \dsllq terminates within at most $N_1 = \log_2( \alpha / \lambda_0) + 3 \log_2(1 / \epsilon)$ iterations,
\item The back up procedure terminates within at most $N_2 = \log_2(\alpha_{mod}) + 2 \log_2(1/\epsilon)$ iterations.
\end{enumerate}
Overall, \dsllq performs at most $N = N_1 + N_2 = O\big( \log_2(1 / \epsilon ) \big)$ iterations. Finally, each iteration require solving a Lyapunov and a Riccati equation for the Extended Lagrangian LQR problem, that can be solved efficiently in $O\big(n^3\big)$, where $n$ is the dimension of the state variable (see~\cite{van1981generalized} for instance).

\end{proof}


\section{Experiments}\label{app:experiments}
We detail in this section the experiments that compare the performance of \dsllq to an $\epsilon-$greedy approach.

\textbf{LQR system.} 
The experiments are conducted for following LQR system, with $n=d=2$:
\begin{equation}
\label{eq:experiments.lq.system}
A_* = \begin{pmatrix} 1.01 & 0.01 \\ 0.01 & 0.5 \end{pmatrix}; \quad \quad B_* = I; \quad\quad Q = I; \quad \quad R = I.
\end{equation}
The LQR system is similar to the one used in~\cite{dean2018regret} but with smaller dimension ($n=2$ vs $n=3$) to simplify the experiment. As in~\cite{dean2018regret}, the system is marginally unstable as $\lambda(A_*) \approx (1.01,0.5)$ i.e., has one unstable and one stable mode.

\textbf{Baseline.}
Our experiment is not intended to provide an extensive comparison between adaptive LQR algorithms and their heuristic variations. We rather focused on developing an experiment that could complement the theoretical understanding we currently have. As a result, we considered comparing \dsllq with algorithms that \textbf{1)} offer a similar $\wt O(\sqrt{T})$ frequentist regret guarantee \textbf{2)} are provably efficient (i.e., tractable with theoretical guarantees on the fact that the chosen controller is indeed the one that the algorithm should use). We reviewed the following algorithms:
\begin{itemize}
\item We do not compare to TS algorithm because either the guarantees are provided for Bayesian regret~\citep{ouyang2017learning-based}, or for frequentist regret but limited to $1$ dimensional systems~\citep{abeille2018improved}. We also do not compare to~\cite{dean2018regret} as the regret guarantee is worse $\wt O(T^{2/3})$.
\item We do not compare to other OFULQ algorithms as they are intractable~\citep{abbasi2011regret,faradonbeh2017finite}\footnote{Heuristic gradient descent approaches to the non-convex problem could be used, but again, we wanted to focus on algorithms that provably solve their optimization problem}.
\item $\epsilon-$greedy approaches that has been studied in~\citep{faradonbeh2018input,mania2019certainty,simchowitz2020naive}.~\cite{mania2019certainty} do not provide explicit algorithm, but a proof that $\epsilon-$greedy (properly tuned) is optimal ($\sqrt{T}$ regret).~\cite{faradonbeh2018input,simchowitz2020naive} propose the same scheme, but the latter provide an algorithm that is fully explicit (in the way $\epsilon$ should be tuned) and for which they claim to have optimal dependencies in the dimension $n,d$. We thus focus on~\cite{simchowitz2020naive} (\cecce). 
\item  We do not compare to \oslo~\citep{cohen2019learning}, since it requires a long initial phase ($\sqrt{T}$). As a result, \oslo has $\sqrt{T}$ regret which is not due to the optimistic SPD resolution, but follows from~\cite{mania2019certainty}  (actually, \oslo is closer to an Explore-Then-Commit scheme).\footnote{Furthermore, if we implement the algorithm rigorously, the length of the initial phase implies that the policy is also updated very rarely. In particular, for the experiment conducted here, and given the time horizon, \oslo does not perform any update (the first update occurs for $T \approx 10^8$).}
\end{itemize} 
As a result, we only compare \dsllq and \cecce.

\textbf{Initialization phase.} 
Both \dsllq and \cecce require an initialization procedure before actually addressing the exploration-exploitation dilemma, to ensure that the state is bounded. \laglq takes a stabilizing set $\Theta_0$ as input, while \cecce provides a procedure to compute such set, given that a stabilizing controller $K_0$ is known. To allow for fair comparison, we use the stabilizing controller approach of \cecce both to perform the initialization procedure of \cecce and to construct a stabilizing set for \laglq that is given as input.

More in detail, we use a initial phase of length $T_0$ for all algorithms where we apply a stable controller $K_0$ and add random Gaussian noise $\mathcal{N}(0,I)$ to the control. This early phase is not counted towards the regret since it is common to the 2 algorithms. After this phase, we are given a stabilizing set and we set $t=0$ and $x_0 = 0$. Then, the safe set is provided as input to \cecce and \dsllq which follow respectively a $\epsilon-$greedy with decaying variance and optimistic strategies.

The length $T_0$ of the system identification phase we use in the experiments is not the one recommended by \cecce. The theoretical lower bound on $T_0$ is very worst case and would lead to a initialization phase of length $T_0 \approx 10^8$ in our case.
On the other hand, a shorter length $T_0 \approx 2 \times 10^4$ is enough to guarantee the stabilizability of the state for the problem at hand (easy to verify by looking at the value of $\|x_t\|$ over the trajectory). 

\textbf{Experiments parameters.} 
We run the \dsllq and \cecce  algorithms on trajectories of length $T=10^6$ for the LQR problem parametrized by~\cref{eq:experiments.lq.system}. We report in Fig.~\ref{fig:exp} the average and the 90th percentile of the regret over $100$ trajectories.  We observe that \laglq is better that \cecce when both are implemented as recommended by theory. The worse perf of \cecce is directly due to the amount of noise injected in the system, which seems to make \cecce over-explorative. This may be caused by two issues:
\begin{enumerate}
\item The injected noise is large because the exploration scheme does not adapt to the uncertainty (the confidence sets). As a result, it has to be conservative and set w.r.t. the most difficult direction (in terms of exploration).
\item The analysis is loose (the constants not the rate) which leads to bad design and hence bad perfs (notice that we may have the same effect for \laglq, coming from the fact that the confidence set is not tight). 
\end{enumerate}

In the attempt of addressing the second possible cause of ``bad'' performance, we also test a "tuned" instance of $\epsilon-$greedy, where the variance of the injected noise scales with $\|P_*\|^4$ (instead of $\|P_*\|^{5/2}$), and set all the others constants to $1$. We do not tune the rate as it is already optimal. In this case, $\epsilon-$ greedy seems to be comparable to \laglq, but the overall trend still seems worse than \laglq.

\textbf{Runtime of \laglq.} Overall the implementation of \laglq on this problem is very efficient. Each call to \dsllq takes around $50-80$m.s. on a standard computer and it takes around $35-40$ iterations to find an optimistic $\epsilon-$ feasible optimal policy for an accuracy $\epsilon = 10^{-12}$. Furthermore, we noticed that empirically the backup procedure \backproc is never really triggered. We conjecture that the corner cases considered by the backup procedure may actually correspond to a set of systems of zero Lesbegue measure.

\end{document}